\documentclass[12pt]{article}
\usepackage[vmargin=1.18in,hmargin=1.1in]{geometry}				
\usepackage{amssymb,amsfonts,amsthm,mathtools,bm,bbm,mathrsfs} 	
\usepackage{pifont} 											
\usepackage[dvipsnames]{xcolor} 								
\usepackage[colorlinks=true, citecolor=blue, urlcolor=blue]{hyperref}  
\usepackage{graphicx} 											
\usepackage{caption} 											
\usepackage{verbatim}
\usepackage[misc]{ifsym} 										
\usepackage[ruled, linesnumbered]{algorithm2e} 					
\usepackage{layouts} 
\usepackage[uniquename=false, 
            uniquelist=false, 
            url=false, 
            doi=false, 
            isbn=false, 
            natbib=true, 
            backend=bibtex, 
            style=authoryear, 
            maxcitenames=1,
            maxbibnames=5]{biblatex}                            
\usepackage{booktabs} 
\usepackage{multirow} 
\usepackage[noend]{algpseudocode}
\usepackage{subcaption}
\usepackage{multirow}
\usepackage{float}
\usepackage[misc]{ifsym} 										

\makeatother

\newenvironment{keywords}{
  \small
  \par\noindent\textbf{Keywords: }\ignorespaces
}{\par}

\newtheorem{assumption}{Assumption}
\newtheorem{assumptionstar}{Assumption}

\newtheorem{theorem}{Theorem}
\newtheorem{corollary}{Corollary}
\newtheorem{remark}{Remark}
\newtheorem{example}{Example}
\newtheorem{lemma}{Lemma}
\newtheorem{proposition}{Proposition}
\newtheorem{definition}{Definition}

\def\d{\mathrm{d}}

\def\KSD{\mathrm{KSD}}
\def\mH{\mathcal{H}}

\DeclareMathOperator{\tr}{Tr}

\def\bs{\bm{s}}

\def\A{\mathcal{A}}
\def\F{\mathcal{F}}
\def\X{\mathcal{X}}

\def\N{\mathcal{N}}

\def\p0{\P_{\theta_0}}
\def\ind{\bm 1}

\def\law{\mathcal{L}}

\def\mK{\mathbf{K}}
\def\mPi{\mathbf{\Pi}}

\newcommand{\E}{\mathbb E}								
\renewcommand{\P}{\mathbb{P}}							
\newcommand{\Q}{\mathbb{Q}}								
\newcommand{\R}{\mathbb{R}}								
\newcommand{\indicator}{\mathbbm 1}						
\newcommand{\convp}{\overset p \rightarrow}             

\let\oldnl\nl
\newcommand{\nonl}{\renewcommand{\nl}{\let\nl\oldnl}} 

\title{Computationally efficient goodness-of-fit tests through kernelized Stein discrepancy}

\begin{document}

\author{Zhihan Huang, Ziang Niu$\textsuperscript{\Letter}$\\
Department of Statistics and Data Science\\ University of Pennsylvania}

\renewcommand{\thefootnote}{\Letter} 
\footnotetext{Corresponding author, \texttt{ziangniu@wharton.upenn.edu}}

\date{}

\maketitle
\renewcommand{\thefootnote}{\arabic{footnote}}

\begin{abstract}
Models with intractable normalizing constants are widely used in statistics and machine learning. Testing whether these models fit the data is challenging: obtaining samples from a fitted model can require sophisticated sampling algorithms, and bootstrap procedures requiring repeated numerical refitting are expensive. We develop a general semiparametric goodness-of-fit test based on the kernelized Stein discrepancy. We establish consistency and the asymptotic null distribution of its statistic under general nuisance estimation. To obtain a level-$\alpha$ test, we propose a novel influence-adjusted wild bootstrap requiring neither model refitting nor model sampling. We prove consistency of this bootstrap test under the null and alternatives, and characterize its limiting power under contiguous local alternatives. Across simulations from classical normality testing to models with intractable likelihoods, the test has competitive or superior power at a computational cost orders of magnitude below existing approaches. We illustrate the method by assessing a protein signalling network model with reverse-phase protein array data from lung adenocarcinoma tumours. As a complementary insight, we show that our test can be regarded as a nonparametric score test under exponentially tilted models, connecting score-based and distance-based goodness-of-fit testing.
\end{abstract}

\begin{keywords}
	Goodness-of-fit test, Intractable model, Kernel method, Local power, Wild bootstrap.
\end{keywords}

\section{Introduction}

Consider $n$ observations $X_1, \ldots, X_n$ from an unknown distribution $\law$. Statistical analysis often begins by positing a parametric model $\{\law_{\theta}:\theta \in \Theta\}$ with parameter space $\Theta \subseteq \mathbb{R}^k$ to approximate the data-generating process. Inferential tasks such as confidence interval construction and hypothesis testing often rely on the model being correctly specified. Therefore, before proceeding with inference, it is crucial to assess model adequacy using a goodness-of-fit test. In particular, we are interested in the following problem:
\begin{align}\label{eq:hypothesis_testing}
	H_0: \law \in \{\law_\theta \in \mathcal{P}(\mathcal{X}) : \theta \in \Theta\} 
	\quad \text{versus} \quad 
	H_1: \law \notin \{\law_\theta \in \mathcal{P}(\mathcal{X}) : \theta \in \Theta\}.
\end{align}

Intractable models are probabilistic models whose likelihoods cannot be evaluated at a practical computational cost. In other words, the density of a distribution $\mathcal{L}_\theta$ can only be evaluated up to a normalizing constant. A common source of intractability is the normalizing constant: an unnormalized density can be evaluated, but evaluating the normalized density requires a computationally prohibitive integral or sum. These difficulties arise in widely used scientific model classes, including spatial interaction models~\citep{besag1975statistical}, undirected graphical models~\citep{wainwright2008graphical}, and energy-based models~\citep{grathwohl2020learning}.

\subsection{Pitfalls in intractable models and a kernel-based solution}

Assessing the goodness of fit for intractable models brings two limitations of existing approaches into focus, with one tailored to the hardness of learning a full model and the other to the difficulty of fitting and sampling from intractable null models.

\paragraph{Full model fitting.}
Classical tests against broad alternatives, such as likelihood ratio tests, often require fitting a flexible model under the alternative through density estimation, smoothing, or local fitting. Such fits typically require choosing the amount of smoothing: too much can hide localized departures from the null, while too little produces a noisy alternative fit~\citep{fan2007nonparametric}. Moreover, in high-dimensional settings, these estimators suffer from the curse of dimensionality, making level-$\alpha$ tests difficult to construct.

\paragraph{Null model fitting and sampling.} Intractable normalizing constants complicate direct likelihood-based model comparisons, as these constants can be difficult to estimate even in moderate dimensions. Generating samples from such models can also be computationally demanding. Consequently, diagnostic procedures that require repeated model fitting or sampling from the fitted model may incur substantial computational costs.

\paragraph{Solution with kernel-based tests.} Kernel methods offer a complementary approach that \emph{does not require fitting the full nonparametric model under the alternative}, which solves the first challenge above. They compare expectations over specific function classes in a reproducing kernel Hilbert space (RKHS), estimating expectations under the data distribution by empirical averages. This yields the maximum mean discrepancy (MMD)~\citep{gretton2012kernel} and, when combined with a Stein operator, the kernelized Stein discrepancy (KSD)~\citep{liu2016kernelized,chwialkowski2016kernel}. The kernel trick makes these discrepancies computable without estimating a full data density: KSD depends on the model only through its score function, i.e., the log-density gradient with respect to the data argument. Consequently, its statistic can be evaluated without the normalizing constant or model-generated samples when the score function is tractable. However, obtaining a level-$\alpha$ test requires calibrating the statistic's null distribution, which is challenging under a composite null hypothesis that we consider in this paper.

\paragraph{Challenges with a kernel-based test.} Under the composite null, two challenges emerge. With a general kernel choice, the KSD statistic has an infinite weighted sum of chi-squared variables as its null limit even for a point null. Theoretically, understanding the limiting behavior is substantially harder than under a point-null hypothesis. The additional step of estimating the model parameter can significantly affect the limiting distribution. Methodologically, the complex limiting behavior makes calibration harder. Critical values therefore need approximation, often by resampling, but a na\"ive wild bootstrap that treats the fitted parameter as fixed can miss this first-order effect and fail to control size (see Figure~\ref{fig:simulation-gaussian-mle} in Section~\ref{sec:simulation-gaussian-mle} and Figure~\ref{fig:kef-scorematching-nwb} in Appendix~\ref{sec:kef-scorematching-nwb} for examples). Parametric bootstrap~\citep{genest2008validity} can account for this effect, but it requires repeated model fitting and model-based resampling, which is computationally expensive for intractable models.

A linear kernel reduces the null limit to a finite weighted sum of chi-squared variables, but the resulting tests can miss alternatives detected by richer kernels~\citep{liu2016kernelized,chwialkowski2016kernel}. These formidable challenges motivate us to study the general behavior of kernel-based tests under composite null hypotheses and to develop new methodology to avoid repeated model fitting and model-based resampling under intractable models.

\subsection{Our contributions}

In this paper, we study KSD goodness-of-fit testing for the composite null that the data distribution belongs to a parametric family. We refer to the resulting procedure as the semiparametric test based on kernelized Stein discrepancy (\textit{SKSD test}). Our contributions are threefold:
\begin{enumerate}
	\item \textbf{Asymptotic theory under general nuisance estimation:} In practice, the choice of nuisance estimator depends on the null model, the desired statistical properties, and the downstream task. We thus establish the consistency and the asymptotic null distribution of the SKSD statistic when the nuisance parameter is fitted by generic asymptotically linear estimators. The limiting law is a weighted sum of chi-squared variables whose weights are given by a covariance operator built from influence-adjusted Stein features, making the effect of nuisance estimation on the null distribution explicit.
	\item \textbf{Two bootstrap calibrations:} We develop a novel influence-adjusted wild bootstrap, through which calibration requires \emph{only one model fit and no samples from the fitted model}. We also analyze a parametric bootstrap that re-estimates the parameter from model-generated samples when influence functions are difficult to estimate and fitted-model sampling is feasible. For each procedure, we characterize its conditional bootstrap limit under the null and alternatives and establish asymptotic level control.
	\item \textbf{Fixed-alternative consistency and local power:} Both calibrated tests are consistent against fixed alternatives under the stated regularity and characteristicity assumptions. We derive their limiting power under two families of $n^{-1/2}$ local alternatives, identifying both detectable and locally powerless directions.
\end{enumerate}

Across extensive simulation studies and a real-data application, the proposed procedures exhibit strong empirical performance. They are competitive with, and in several settings outperform, tests based on related discrepancies. These contributions together address the open problems raised by \citet{liu2016kernelized} on composite null testing with KSD.

Finally, as a complementary observation, we connect the SKSD construction to classical score testing~\citep{rao1948large}. Under exponentially tilted models indexed by a function class $\mathcal{F}$, aggregating the corresponding parametric scores over $\mathcal{F}$ recovers an integral probability metric (IPM) in the sense of \citet{muller1997integral}. This relation offers a score-based interpretation of familiar goodness-of-fit procedures based on the Kolmogorov--Smirnov distance~\citep{durbin1975kolmogorov}, Wasserstein-1 distance~\citep{hallin2021multivariate}, and maximum mean discrepancy~\citep{gretton2012kernel,key2025composite}, and provides context for our use of the kernelized Stein function class. See Appendix~\ref{sec:bridge_score_ipm} for details.

\subsection{Related work}

\paragraph{Kernel-based goodness-of-fit tests.} \citet{chwialkowski2016kernel,liu2016kernelized} developed KSD tests for point-null hypotheses. Subsequent developments include linear-time tests~\citep{jitkrittum2017linear}, aggregation across multiple kernels~\citep{schrab2022ksd}, and spectrally regularized KSD tests with minimax guarantees~\citep{hagrass2026minimax}. Kernel Stein tests have also been adapted to directional distributions~\citep{xu2020stein}, censored time-to-event data~\citep{fernandez2020kernelized}, conditional density models~\citep{jitkrittum2020testing}, variable-length sequences~\citep{baum2023kernel}, and exponential random graph models~\citep{xu2021stein}. Related work also addresses relative goodness of fit, including comparisons among multiple candidate models~\citep{lim2019kernel} and comparisons of latent-variable models~\citep{kanagawa2023kernel}. These studies primarily assess a specified model or compare candidate models. Our SKSD test instead evaluates the \emph{composite null} that the data distribution belongs to a parametric family, and its calibration accounts for the effect of estimating the null parameter from the same sample.

\paragraph{Goodness-of-fit tests for composite null hypotheses.} \citet{key2025composite} establish a composite goodness-of-fit test based on minimum-MMD estimation and parametric-bootstrap calibration. They also examine a minimum-KSD analogue empirically without a corresponding validity theorem. Moreover, they propose to use a parametric bootstrap to calibrate the test even in the presence of intractable normalizing constants, which is particularly computationally expensive when the model is difficult to sample from. In contrast, our influence-adjusted SKSD test uses a multiplier bootstrap inspired by \citet{kojadinovic2012goodness}. It requires neither repeated model fitting nor fitted-model sampling and avoids evaluating the normalizing constant. Our theory also permits a general class of asymptotically linear nuisance estimators.

\citet{brueck2025composite} also develop a composite KSD test with bootstrap calibration that accounts for parameter estimation. Their procedure refits the nuisance parameter for each empirical bootstrap sample, whereas our influence-adjusted bootstrap uses a single model fit and estimated influence functions. Both methods establish null calibration and consistency against fixed alternatives, but our analysis also characterizes power under contiguous local alternatives. Across our simulation studies, the proposed SKSD tests often attain power comparable to or greater than that of the test proposed by~\citet{brueck2025composite}, while our influence-adjusted test requires substantially less computation time (see Section~\ref{sec:applications}).

\subsection{Organization of the paper and notation}

We introduce the SKSD framework and establish consistency and the asymptotic null distribution of its test statistic in Section~\ref{sec:method}. We then develop two complementary calibration procedures: an influence-adjusted multiplier bootstrap in Section~\ref{sec:influence-adjusted-bootstrap} and a parametric bootstrap in Section~\ref{sec:asymptotic_bootstrap_distribution}. Section~\ref{sec:asymptotic_power} establishes consistency under fixed alternatives and characterizes local power for both tests. Section~\ref{sec:applications} evaluates the proposed procedures in four simulation settings and a lung adenocarcinoma application. We conclude in Section~\ref{sec:discussion} with a discussion of limitations and directions for future work. Additional technical details and all proofs are deferred to the Appendix. Before proceeding, we summarize the notation used throughout the paper.

\paragraph{Spaces and Measures.} 
Let \( \mathcal{X} \subseteq \R^d\) be a measurable set and let \( \mathscr{B}(\mathcal{X}) \) denote the Borel \( \sigma \)-algebra generated by the open sets of \( \mathcal{X} \). We use $X$ for a single observation in $\mathcal X$, $X_{1:n}=(X_1,\ldots,X_n)$ for the full sample, and $\tilde X_{1:n}$ for a bootstrap sample.
Consider two probability measures \( \mathcal{L}_p, \mathcal{L}_q \in \mathcal{P}(\mathcal{X}) \), the collection of all probability measures on \( (\mathcal{X}, \mathscr{B}(\mathcal{X})) \), with differentiable density functions \( p \) and \( q \). Let $p_\theta$ denote the density function of $\law_\theta$ with respect to the Lebesgue measure $\lambda$. Define the score function of $\mathcal{L}_p$ as \( s_p(x) = \nabla_x \log p(x) \) with the convention $s_\theta = s_{p_\theta}$.\footnote{Following the score-matching convention~\citep{hyvarinen2005estimation}, ``score function'' refers to differentiation of the log density with respect to the data argument. We call derivatives of log densities or log-likelihoods with respect to parameters ``parametric scores.''}

\paragraph{Vectors, matrices and functions.}
For vectors $u,v \in \R^d$, we use $\|v\|$ to denote the Euclidean norm and $\langle u, v \rangle$ to denote the inner product. For a matrix $A \in \R^{d_1 \times d_2}$, we use $\|A\|_{\mathrm{op}}$ to denote its operator norm, i.e., $\|A\|_{\mathrm{op}} = \sup_{\|x\| \le 1} \|Ax\|$. For higher-order tensors, the operator norm is defined similarly. For a matrix $A \in \R^{d \times d}$, we use $\tr(A)$ to denote its trace, i.e., $\tr(A) = \sum_{j=1}^d A_{jj}$. Let \( \mathcal{F}_{d,m} \) be the space of functions from \( \mathcal{X} \subseteq \mathbb{R}^d \) to \( \mathbb{R}^m \), for integers \( d, m \in \mathbb{N} \). For nonnegative integers $\alpha,\beta$, let $\mathcal{C}^{\alpha,\beta}(A,B)$ denote the space of functions $f:A\times B\to\R$ whose mixed partial derivatives of total order at most $\alpha$ in the first argument and at most $\beta$ in the second argument are continuous.

\paragraph{Kernels.}
The kernel in our setting is a measurable bivariate function $K(\cdot, \cdot): \X \times \X \to \R$. We say a kernel is symmetric if $K(x, y) = K(y, x)$ for all $x, y \in \X$. We say a kernel is positive definite if, for any $x_1, \ldots, x_n \in \X$ and $c_1, \ldots, c_n \in \R$, we have $\sum_{i=1}^n \sum_{j=1}^n c_i c_j K(x_i, x_j) \ge 0$.

\section{Semiparametric KSD test framework}\label{sec:method}

In the following sections, we will first present the preliminaries of KSD in Section~\ref{sec:prelim_KSD} and then propose the SKSD test in Section~\ref{sec:sksd_method} and investigate its asymptotic behavior in Section~\ref{sec:asymptotic_distribution}.

\subsection{Preliminaries: kernelized Stein class and KSD}\label{sec:prelim_KSD}

First introduced in the machine learning community, kernelized Stein discrepancy (KSD, \citet{liu2016kernelized}) was proposed as a measure of model discrepancy in settings where the model likelihood is intractable. Its popularity is largely due to its computational efficiency (see the computational complexity comparison in Table~\ref{tab:distance_computation} of Appendix~\ref{sec:characteristic_ipm_examples}), its desirable characteristic property that distinguishes between probability distributions (see Proposition~\ref{prop:characteristic_propety}), and, crucially, its ability to handle intractable models. We refer the reader to \citet{anastasiou2023stein} for a comprehensive review of these developments.

KSD is a \textit{Stein discrepancy} (SD), defined via the Stein operator \( \mathcal{A}_p: \mathcal{F}_{d,d} \to \mathcal{F}_{d,1} \)
\footnote{We use the Langevin Stein operator, which extends the classical scalar-input Stein operator and maps differentiable vector-valued functions to scalar-valued functions.}.
Recall that $s_p=\nabla_x\log p(x)$ is the score function and that the Stein operator is defined by $\mathcal{A}_p g(x) = \sum_{j=1}^d \partial_j g_j(x) + g^\top(x) s_p(x)$, where \( g \in \mathcal{F}_{d,d} \), and \( g_j(x) \) denotes the \( j \)-th coordinate of \( g(x) \). Under suitable regularity conditions, the following identity holds: $\mathbb{E}_{X \sim \mathcal{L}_p}[\mathcal{A}_p g(X)] = 0$, as established, for example, in Proposition 1 of \citet{gorham2015measuring}. This is \emph{Stein's identity}~\citep{stein1972bound}. Using this identity, the Stein discrepancy between \( \mathcal{L}_p \) and \( \mathcal{L}_q \) can be defined as $\mathrm{SD}(\mathcal{L}_p, \mathcal{L}_q)
= \sup_{g \in \overline{\mathcal{F}}_{d,d}} \left| \mathbb{E}_{p}[\mathcal{A}_p g(X)] - \mathbb{E}_{q}[\mathcal{A}_p g(X)] \right|$ for a non-empty set $\overline{\mathcal{F}}_{d,d} \subseteq \mathcal{F}_{d,d}$ with $\mathbb{E}_{p}[\mathcal{A}_p g(X)]=0$ by Stein's identity. If we restrict $g$ to \( \mathcal{H}^d\subseteq\mathcal{F}_{d,d} \), the product space of a reproducing kernel Hilbert space (RKHS) \( \mathcal{H} \), and take $\overline{\mathcal{F}}_{d,d}$ to be the unit ball in $\mH^d$, then the resulting discrepancy is called the kernelized Stein discrepancy (KSD):
\begin{align}\label{eq:kernel_stein_discrepancy}
\mathrm{KSD}(\mathcal{L}_p, \mathcal{L}_q) 
\equiv \sup_{\substack{\|g\|_{\mathcal{H}^d} \leq 1}} \left| \mathbb{E}_{p}[\mathcal{A}_p g(X)] - \mathbb{E}_{q}[\mathcal{A}_p g(X)] \right|
= \sup_{\substack{\|g\|_{\mathcal{H}^d} \leq 1}} \left| \mathbb{E}_{q}[\mathcal{A}_p g(X)] \right|.
\end{align}
We refer readers who are unfamiliar with RKHS to~\citet{kanagawa2018gaussian} for a detailed introduction and review. The definition of KSD in~\eqref{eq:kernel_stein_discrepancy} also admits a natural interpretation as an integral probability metric~\citep{muller1997integral}. We return to this connection in Appendix~\ref{sec:bridge_score_ipm}.
Furthermore, one can show that, under mild conditions on the kernel function, KSD is characteristic, i.e., $\mathrm{KSD}(\law_p,\law_q)=0$ if and only if $\law_p=\law_q$ (see Appendix~\ref{sec:characteristicity} for more details). Therefore, we consider the characteristic KSD throughout the manuscript.

\subsection{SKSD test with general nuisance estimators}\label{sec:sksd_method}

Let $\hat\law_n$ be the empirical measure $\frac{1}{n}\sum_{i=1}^n \delta_{X_i}$. We define the proposed \textit{semiparametric KSD} (SKSD) test statistic by
\begin{align}\label{eq:sksd_test_stat}
	T_n(X_{1:n},\hat\theta_n)\equiv \mathrm{KSD}^2(\law_{\hat\theta_n},\hat\law_n)=\sup_{\|g\|_{\mathcal{H}^d}\leq 1} \left(\mathbb{E}_{\hat\law_n}[\mathcal{A}_{p_{\hat\theta_n}} g(X)]\right)^2.\tag{SKSD}
\end{align}
For the optimization problem~\eqref{eq:sksd_test_stat}, a closed-form solution can be obtained by applying the reproducing property (the well-known ``kernel trick''), yielding:
\begin{align}\label{eq:ksd-kernel}
	T_n(X_{1:n},\hat\theta_n) = \mathbb{E}_{(x,y)\sim\hat\law_n\otimes\hat\law_n}[h_{p_{\hat\theta_n}}(x,y)] = \frac{1}{n^2}\sum_{i,j=1}^n h_{p_{\hat\theta_n}}(X_i,X_j),
\end{align}
where $\mathcal{L}_1\otimes\mathcal{L}_2$ is the product measure of $\mathcal{L}_1$ and $\mathcal{L}_2$, and the explicit form of \( h \) is provided in Appendix~\ref{sec:close-form}. The alternative formulation~\eqref{eq:ksd-kernel} enables a computationally efficient estimate of $\mathrm{KSD}$, as it only requires evaluating the expectation of a bivariate function \( h_{p_{\hat\theta_n}} \). Moreover, when the parameter is known and $\hat\theta_n=\theta_0$, the statistic in Eqn.~\eqref{eq:sksd_test_stat} reduces to the usual KSD test statistic for a point-null hypothesis~\citep{liu2016kernelized,chwialkowski2016kernel}.

\paragraph{Computational efficiency of the test statistic.} 
A nice property of the SKSD test statistic is that it does not rely on evaluating the likelihood function of the model $\law_{\hat\theta_n}$. 
In fact, the function $h$ only involves the score function $s_{p_{\hat\theta_n}}$ under the model $\law_{\hat\theta_n}$ and the kernel function $K$, which are normalizing-constant-free. Moreover, evaluating the function $h$ does not require any Monte Carlo simulation and the V-statistic~\eqref{eq:ksd-kernel} can be computed in quadratic time (see Table~\ref{tab:distance_computation} in Appendix~\ref{sec:characteristic_ipm_examples}). Moreover, there are certain choices of kernel function that can reduce the complexity to $O(n)$, such as a linear kernel (see Appendix~\ref{sec:computation_v_statistic} for more details).

\paragraph{Accommodation of general nuisance estimators.} We allow the estimation of $\hat\theta_n$ to be highly problem-specific and impose no restrictions on the estimation procedure, provided the resulting estimator is asymptotically linear. Formally, we make the following assumption on the nuisance estimator $\hat\theta_n$.
\begin{assumption}[Asymptotically linear estimator]\label{assu:ale}
	Consider the estimation procedure $\mathcal{E}:\mathcal{X}^n \rightarrow\Theta\subseteq\mathbb{R}^k$.
	Suppose there exist a function $I_{\law}(\cdot):\mathcal{X}\rightarrow\mathbb{R}^k$ and a parameter $\theta \in \Theta$ such that, for any $X_1,\ldots,X_n\overset{\mathrm{i.i.d.}}{\sim}\law$, we have $\E[I_{\law}(X_i)] = 0,\ \E[\left\|I_{\law}(X_i)\right\|^2] < \infty$ and
	\begin{align}\label{eq:ale}
		\left\|\sqrt{n}(\hat\theta_n- \theta)-\frac{1}{\sqrt{n}}\sum_{i=1}^n I_{\law}(X_i)\right\| \convp 0\quad\text{where}\quad \hat\theta_n=\mathcal{E}(X_{1:n}).
	\end{align}
	We then say the estimator $\hat\theta_n$ is an $\mathcal{E}$-asymptotically linear estimator (ALE) with influence function $I_{\law}$ under the distribution $\law$.
\end{assumption} 

Assumption~\ref{assu:ale} accommodates regular M-estimators~\citep[Theorem~5.23]{van2000asymptotic} and
minimum-distance estimators~\citep{wolfowitz1957minimum} under suitable
smoothness, identifiability, and moment conditions. Examples appear in Appendix~\ref{sec:general_estimators_UALE}.
\begin{remark}[Projection implication of Assumption~\ref{assu:ale}]\label{rem:projection}
	Assumption~\ref{assu:ale} can hold under both null and alternative hypotheses. In fact, as long as there is $\theta_0$ such that $\theta_0\equiv \arg\min_{\theta\in\Theta}\ell_{\law}(\theta)$, where $\ell_{\law}$ is the population loss under the distribution $\law$, we can choose $\theta=\theta_0$ in Assumption~\ref{assu:ale}. Under the null hypothesis, provided the estimator is consistent for a true model parameter, its limit $\theta_0$ satisfies $\law=\law_{\theta_0}$. Under the alternative hypothesis, the model distribution $\law_{\theta_0}$ is the projection of the true distribution $\law$ onto the model class $\{\law_\theta : \theta \in \Theta\}$ under the metric induced by the loss function.
\end{remark}

\subsection{Consistency and asymptotic null distribution}\label{sec:asymptotic_distribution}

The following theorem gives the probability limit of the SKSD statistic.
\begin{theorem}[Consistency of $\mathrm{SKSD}$ test statistic]\label{thm:sksd_consistency}
	Suppose $\hat\theta_n\convp\theta_0$ under the data-generating law $\law$ for some $\theta_0\in\Theta$. Under the regularity conditions in Appendix~\ref{sec:regularity_conditions}, we have $T_n(X_{1:n},\hat\theta_n)\convp \mathrm{KSD}^2(\law_{\theta_0},\law)$ as $n\rightarrow\infty$.
\end{theorem}
The proof is given in Appendix~\ref{sec:proof_sksd_consistency}.

\begin{remark}[On the positivity of the limiting value $\mathrm{KSD}(\law_{\theta_0},\law)$]
	The consistency result shows that the statistic converges to the squared KSD between the true distribution $\law$ and the model distribution $\law_{\theta_0}$, where $\theta_0$ is the probability limit of the nuisance estimator $\hat\theta_n$. Under the null hypothesis, if the estimator is consistent for a true model parameter, then $\law=\law_{\theta_0}$ and $\mathrm{KSD}(\law_{\theta_0},\law)=0$. Under the alternative hypothesis, $\law\neq\law_{\theta_0}$, so $\mathrm{KSD}(\law_{\theta_0},\law)>0$ by the characteristic property of KSD (Proposition~\ref{prop:characteristic_propety}). Thus, the limit separates the null from fixed alternatives.
\end{remark}

We now establish the asymptotic distribution of the observed test statistic $T_n(X_{1:n},\hat\theta_n)$, which we refer to as $T_n$ when there is no ambiguity.
To formally state the following results, let $\psi:\X\to\mH$ be the feature map associated with kernel $K$, so that $\psi(x)=K(x,\cdot)$. For $x\in\X$ and $\theta\in\Theta$, define \textit{the Stein feature} $\xi_{\theta}(x)\in\mH^d$ and \textit{the influence-adjusted Stein feature} $\Psi_{\theta, \law}(x)\in\mH^d$ by
\begin{subequations}\label{eq:stein-feature-definitions}
\begin{align}
	\xi_{\theta}(x)
	&=s_\theta(x)\psi(x)+\nabla_x\psi(x),
	\label{eq:stein-feature}\\
	\Psi_{\theta, \law}(x)
	&=\xi_{\theta}(x)
	+\E_{X'\sim\law}\left[
	\nabla_\theta s_\theta(X')I_{\law}(x)\psi(X')
	\right].
	\label{eq:influence-adjusted-stein-feature}
\end{align}
\end{subequations}
Here $\nabla_x\psi(x)$ is defined pointwise by $[\nabla_x\psi(x)](y)=\nabla_xK(x,y)$. Appendix~\ref{sec:close-form} gives the derivative reproducing identities and verifies that the features belong to $\mH^d$. The following theorem characterizes the asymptotic null distribution of the SKSD test statistic. For simplicity, we write $\Psi_\theta(x) \equiv \Psi_{\theta, \law_\theta}(x)$.
\begin{theorem}[Asymptotic null distribution]\label{thm:asymptotic_distribution}
	Suppose $X_1,\ldots,X_n$ are i.i.d. samples from $\law_{\theta_0}$ with $\theta_0\in\Theta$ and Assumption~\ref{assu:ale} holds with $\law=\law_{\theta_0}$. Define the covariance operator $\Sigma_{\theta_0}:\mH^d\to\mH^d$ by
	\begin{align}\label{eq:asymptotic_distribution}
		\Sigma_{\theta_0} f
		&=\E_{X\sim\law_{\theta_0}}\left[
		\langle f, \Psi_{\theta_0}(X)\rangle_{\mH^d}\Psi_{\theta_0}(X) 
		\right] \quad \text{for any } f\in\mH^d.
	\end{align}
	Let $\{\lambda_u\}_{u\ge1}$ enumerate the non-negative eigenvalues of $\Sigma_{\theta_0}$ according to multiplicity, and let $\{Z_u\}_{u\ge1}$ be independent standard normal variables. Under the regularity conditions in Appendix~\ref{sec:regularity_conditions}, as $n \to \infty$, we have $nT_n\overset{d}{\to}W\equiv\sum_{u=1}^{\infty}\lambda_uZ_u^2$.
\end{theorem}

The proof of Theorem~\ref{thm:asymptotic_distribution} is provided in Appendix~\ref{sec:proof_asymptotic_distribution}. 
The discussion in Appendix~\ref{sec:close-form} establishes that $\Sigma_{\theta_0}$ is positive, self-adjoint, and trace class, and explains the existence of its eigenvalues and of $W$ (see Eqns.~\eqref{eq:covariance-operator-trace} and~\eqref{eq:hilbert-gaussian-norm}). Theorem~\ref{thm:asymptotic_distribution} shows that the asymptotic null distribution of the SKSD test statistic is a weighted sum of independent chi-squared random variables. Although limits of this form are standard in the theory of degenerate U- and V-statistics~\citep{serfling1980approximation}, the distribution of $W$ is generally intractable: its weights $\{\lambda_u\}_{u\geq1}$ are the eigenvalues of the unknown operator $\Sigma_{\theta_0}$, so the quantiles required for a level-$\alpha$ test are not directly available. Therefore, we propose two bootstrap procedures to approximate the null distribution of $T_n$ in Sections~\ref{sec:influence-adjusted-bootstrap} and~\ref{sec:asymptotic_bootstrap_distribution}.

\section{Consistent bootstrap tests}\label{sec:bootstrap-tests}

Sections~\ref{sec:influence-adjusted-bootstrap} and~\ref{sec:asymptotic_bootstrap_distribution} propose two bootstrap calibration methods, respectively. Section~\ref{sec:bootstrap-computation} discusses their computational requirements.
 
\subsection{Influence-adjusted wild bootstrap approach}
\label{sec:influence-adjusted-bootstrap}

We propose an influence-adjusted bootstrap based on multiplier weights. For $i\in[n]$, let $\widehat I_i$ estimate $I_\law(X_i)$ and define the empirical influence-adjusted feature by
\begin{align}\label{eq:estimated-influence-adjusted-feature}
	\widehat\Psi_i
	&=\xi_{\hat\theta_n}(X_i)
	+\frac{1}{n}\sum_{\ell=1}^n
	\nabla_\theta s_{\hat\theta_n}(X_\ell)
	\widehat I_i\psi(X_\ell).
\end{align}
Let $\nu$ be a fixed probability law on $\R$. For each bootstrap
replication, generate wild bootstrap weights independently of the data law $\mathcal{L}$, such that
\begin{align}\label{eq:iid-multiplier-weights}
	w_1^*,\ldots,w_n^*\overset{\mathrm{i.i.d.}}{\sim}\nu,\qquad
	\E_\nu[w^*]=0,\qquad
	\E_\nu[(w^*)^2]=1.
\end{align}
Standard Gaussian and Rademacher weights are two examples satisfying Eqn.~\eqref{eq:iid-multiplier-weights}. The influence-adjusted wild bootstrap statistic is then defined in Algorithm~\ref{alg:influence-adjusted-SKSD}.

\begin{center}
	\begin{minipage}{\linewidth}
		\begin{algorithm}[H]
			\nonl \textbf{Input:} Distribution class $\{\law_\theta:\theta\in\Theta\}$, data $X_1,\ldots,X_n$, estimation algorithm $\mathcal{E}:\mathcal{X}^n\rightarrow\Theta\subseteq\mathbb{R}^k$, influence-function estimator, multiplier law $\nu$, bootstrap size $B$.\\
			Obtain $\hat\theta_n=\mathcal{E}(X_{1:n})$ and compute $T_n\equiv T_n(X_{1:n},\hat\theta_n)$ as in~\eqref{eq:sksd_test_stat}\;
			Compute the fitted influence functions $\{\widehat I_i:i\in[n]\}$\;
			Construct the influence-adjusted Gram matrix $\widehat G\in\R^{n\times n}$ as in Appendix~\ref{sec:ia-bootstrap-matrix}\;
			\For{$b=1,2,\ldots,B$}{
				Generate $w_1^{(b)},\ldots,w_n^{(b)}\overset{\mathrm{i.i.d.}}{\sim}\nu$ and set $\bm w^{(b)}=\left(w_i^{(b)}-\frac1n\sum_{\ell=1}^n w_\ell^{(b)}\right)_{i=1}^n$\;
				Compute
				$\displaystyle
				\widetilde T_{n,\mathrm{IA}}^{(b)}
				=\frac{1}{n^2}(\bm w^{(b)})^\top\widehat G\bm w^{(b)}$
				\refstepcounter{equation}\label{eq:ia-bootstrap-statistic}
				\hfill\textnormal{(\theequation)}\;
			}
			\nonl \textbf{Output:} $p$-value
			$\displaystyle
			\widehat p_{\mathrm{IA}}
			=\frac{1+\sum_{b=1}^B
			\indicator(\widetilde T_{n,\mathrm{IA}}^{(b)}\geq T_n)}{B+1}$.
			\caption{\bf SKSD test with influence-adjusted wild bootstrap}
			\label{alg:influence-adjusted-SKSD}
		\end{algorithm}
	\end{minipage}
\end{center}

For the theoretical analysis, write
$\mathcal F_n=\sigma(X_1,\ldots,X_n)$ and
$\mathcal F_\infty=\sigma(\bigcup_{n\ge1}\mathcal F_n)$
for the finite-sample and full-data sigma-fields, respectively. Omit the replication index $b$ and write
$\widetilde T_{n,\mathrm{IA}}$ for one generic multiplier statistic defined
in Eqn.~\eqref{eq:ia-bootstrap-statistic}. Consider the
influence-adjusted bootstrap test (letting $B \to \infty$ in Algorithm~\ref{alg:influence-adjusted-SKSD}) defined by:
\begin{align}\label{eq:influence-adjusted-bootstrap-test}
	\phi_{n,\alpha}^{\mathrm{IA}}
	\equiv\indicator\left(
	T_n\geq\Q_{1-\alpha}
	(\widetilde T_{n,\mathrm{IA}})
	\right).
	\tag{SKSD-IA}
\end{align}
For a bootstrap statistic $Y$ and $\tau\in(0,1)$, $\Q_\tau(Y)$ denotes the $\tau$-quantile of its distribution conditional on $\mathcal F_n$. The procedure involves estimating the influence function $I_\law$, and the following assumption states the required estimation accuracy.
\begin{assumption}[Estimation of the influence function]
\label{ass:estimated-influence-function}
	Suppose Assumption~\ref{assu:ale} holds under $\law$ with limit
	$\theta_0$. The fitted influence functions satisfy
	\begin{equation}\label{eq:estimated-influence-function-l2}
		\frac1n\sum_{i=1}^n
		\left\|
		\widehat I_i-I_\law(X_i)
		\right\|^2
		\overset{p}{\longrightarrow}0.
	\end{equation}
\end{assumption}

Appendix~\ref{sec:general_estimators_UALE} gives explicit examples of influence-function
estimators, such as M/Z-estimators and minimum-distance estimators. The following theorem establishes the weak convergence of the influence-adjusted bootstrap statistic to a law-specific limit under both the null and alternatives.
The proof of Theorem~\ref{thm:influence-adjusted-bootstrap} is provided in
Appendix~\ref{sec:proof_influence_adjusted_bootstrap}.
\begin{theorem}[General-law multiplier bootstrap limit]
\label{thm:influence-adjusted-bootstrap}
	Suppose $X_1,\ldots,X_n\overset{\mathrm{i.i.d.}}{\sim}\law$ and
	Assumption~\ref{assu:ale} holds under $\law$ with limit $\theta_0$.
	Let $\Psi_{\theta_0,\law}^{\circ}(x)
		=\Psi_{\theta_0,\law}(x)
		-\E_{X\sim\law}\left[\Psi_{\theta_0,\law}(X)\right]$ be the centered feature.
	Define the covariance operator $\Sigma_{\theta_0,\law}:\mH^d\to\mH^d$
	by:
	\begin{align*}
		\Sigma_{\theta_0,\law}f
		&=\E_{X\sim\law}\left[
		\langle f,\Psi_{\theta_0,\law}^{\circ}(X)\rangle_{\mH^d}
		\Psi_{\theta_0,\law}^{\circ}(X)
		\right] \quad \text{for any } f\in\mH^d.
	\end{align*}
	Let $\{\lambda_{u,\law}\}_{u\ge1}$ enumerate the non-negative eigenvalues of $\Sigma_{\theta_0,\law}$ according to multiplicity, and let $\{Z_u\}_{u\ge1}$ be independent standard normal variables. Define $W_\law\equiv\sum_{u=1}^\infty\lambda_{u,\law}Z_u^2$. Under Assumption~\ref{ass:estimated-influence-function} and the regularity
	conditions in Appendix~\ref{sec:regularity_conditions}, when $\Sigma_{\theta_0,\law}\ne0$, for every
	$\delta>0$ and $t \in \R$,
	\begin{align}\label{eq:general-law-bootstrap-convergence}
		\lim_{n\to\infty}
		\P\left[
		\left|
		\P\left[
		n\widetilde T_{n,\mathrm{IA}}^{(b)}\le t\mid\mathcal F_n
		\right]
		-\P[W_\law\le t]
		\right|>\delta
		\right]
		&=0.
	\end{align}
\end{theorem}

\begin{remark}[Agnosticity with respect to the data-generating law]
	Theorem~\ref{thm:influence-adjusted-bootstrap} gives a law-specific
	bootstrap limit for every data-generating law $\law$ satisfying its conditions,
	whether $\law=\law_{\theta_0}$ or $\law$ is a fixed alternative. Under
	the null, Corollary~\ref{cor:influence-adjusted-bootstrap-validity} below shows
	that this limit coincides with the form of the null limit of the observed statistic.
	Under a fixed alternative, Theorem~\ref{thm:influence-adjusted-bootstrap}
	instead characterizes the limiting distribution of
	the critical value used by the test~\eqref{eq:influence-adjusted-bootstrap-test}, providing
	the basis for fixed-alternative power analysis.
\end{remark}

\begin{remark}[Stable-convergence proof strategy]
	The proof of Theorem~\ref{thm:influence-adjusted-bootstrap} has three
	steps. First, we show that each one-dimensional projection of the
	multiplier feature sum has a Gaussian limit, using stable convergence
	\citep{hausler2015stable}. Next, we show that finitely many coordinates approximate the multiplier sum accurately with controllable remainder terms. 
	Together with the scalar stable convergence, this gives the Gaussian stable limit of the full sum of features and the bootstrap statistic. 
	Finally, we show that the stable convergence gives the tail probability convergence of the conditional bootstrap distribution. The proof of Theorem~\ref{thm:bootstrap_distribution} in Section~\ref{sec:asymptotic_bootstrap_distribution} follows similar steps.
\end{remark}

To establish validity of the test, it remains to compare the general-law
bootstrap limit in Theorem~\ref{thm:influence-adjusted-bootstrap} with the
null limit of the observed statistic in
Theorem~\ref{thm:asymptotic_distribution}. Under the null, the Stein
identity and the centering of the influence function make the two covariance
operators, and hence the two limiting distributions, coincide. The following
corollary states the resulting size control for the test~\eqref{eq:influence-adjusted-bootstrap-test}.

\begin{corollary}[Validity under the null]
\label{cor:influence-adjusted-bootstrap-validity}
	Under the assumptions of Theorem~\ref{thm:influence-adjusted-bootstrap} with $\law=\law_{\theta_0}$ and $\Sigma_{\theta_0}\ne0$, for
	every $\alpha\in(0,1)$, we have $\lim_{n\to\infty}\P_{\law_{\theta_0}}\left[\phi_{n,\alpha}^{\mathrm{IA}}=1\right]=\alpha$.
\end{corollary}

The proof of Corollary~\ref{cor:influence-adjusted-bootstrap-validity} is provided in Appendix~\ref{sec:proof_influence_adjusted_bootstrap_validity}. Corollary~\ref{cor:influence-adjusted-bootstrap-validity} justifies the validity of test~\eqref{eq:influence-adjusted-bootstrap-test}. 

\subsection{An alternative route: parametric bootstrap test}\label{sec:asymptotic_bootstrap_distribution}

The parametric bootstrap approximates the null distribution by generating
samples from the fitted model and applying the same estimation procedure
$\mathcal E$ to each sample \citep{genest2008validity}.
Conditionally on $\mathcal F_n$, draw
$\tilde X_{1:n}=(\tilde X_1,\ldots,\tilde X_n)\overset{\mathrm{i.i.d.}}{\sim}\law_{\hat\theta_n}$
and define $\tilde\theta_n=\mathcal E(\tilde X_{1:n})$, $\widetilde T_{n,\mathrm{PA}}=T_n(\tilde X_{1:n},\tilde\theta_n)$. Algorithm~\ref{alg:semiparametric-KSD} summarizes the procedure.

\begin{center}
	\begin{minipage}{\linewidth}
		\begin{algorithm}[H]\label{alg:main}
			\nonl  \textbf{Input:} Distribution class $\{\law_\theta:\theta\in\Theta\}$, data $X_1,\ldots,X_n$, estimation algorithm $\mathcal{E}:\mathcal{X}^n\rightarrow\Theta\subseteq\mathbb{R}^k$, bootstrap size $B$.\\
			Obtain the estimate $\hat\theta_n=\mathcal{E}(X_{1:n})$\;
			Compute $T_n\equiv T_n(X_{1:n},\hat\theta_n)$ as in \eqref{eq:sksd_test_stat}\;
			\For{$b = 1, 2, \dots, B$}{
				Generate bootstrap sample $\tilde X_{1:n}^{(b)}=(\tilde X_1^{(b)},\ldots, \tilde X_n^{(b)})\overset{\text{i.i.d.}}{\sim} \law_{\hat\theta_n}$\;
				Obtain bootstrap estimate $\tilde \theta_n^{(b)}=\mathcal{E}(\tilde X_{1:n}^{(b)})$\;
				Compute the resampled statistic $\widetilde T_{n,\mathrm{PA}}^{(b)}\equiv T_n(\tilde X_{1:n}^{(b)},\tilde{\theta}_n^{(b)})$.
			}
			\nonl \textbf{Output:} $p$-value $\sum_{b=1}^B \indicator(\widetilde T_{n,\mathrm{PA}}^{(b)}\geq T_n)/B$.
			\caption{\bf SKSD test with parametric bootstrap}
			\label{alg:semiparametric-KSD}
		\end{algorithm}
	\end{minipage}
\end{center}

We defer the detailed algorithm discussion to Section~\ref{sec:bootstrap-computation}. Using the conditional quantile notation from
Section~\ref{sec:influence-adjusted-bootstrap}, define the parametric
bootstrap test (letting $B\to\infty$ in Algorithm~\ref{alg:semiparametric-KSD}) by
\begin{align}\label{eq:bootstrap_test}
	\phi_{n,\alpha}^{\mathrm{PA}}\equiv
	\indicator\left(T_n\geq\Q_{1-\alpha}(\widetilde T_{n,\mathrm{PA}})\right).
	\tag{SKSD-PA}
\end{align}
The parametric bootstrap requires a slight extension of
Assumption~\ref{assu:ale}: the asymptotically linear expansion must extend continuously to a neighborhood of the limiting parameter $\theta_0$,
since parametric bootstrap samples are drawn from the fitted law $\law_{\hat\theta_n}$. We formalize this local uniform extension in Assumption~\ref{assu:uale} in Appendix~\ref{sec:parametric-bootstrap-conditions} and call an estimator satisfying it a uniformly asymptotically linear estimator (UALE). The following theorem characterizes the conditional limit distribution of
the parametric bootstrap statistic under both the null and alternatives.

\begin{theorem}[Bootstrap consistency]\label{thm:bootstrap_distribution}
	Suppose Assumption~\ref{assu:uale} holds, together with the regularity
	conditions in
	Appendix~\ref{sec:regularity_conditions}.
	Then, when $\Sigma_{\theta_0}\ne0$ in Eqn.~\eqref{eq:asymptotic_distribution}, for every $\delta>0$ and every $t\in\mathbb R$, 
	\begin{align}\label{eq:bootstrap_convergence}
		\lim_{n\rightarrow\infty}\P\left[\big|\P[n \widetilde T_{n,\mathrm{PA}}\leq t|\mathcal{F}_n]- \P[W\leq t]\big|>\delta\right]=0,
	\end{align}
	where the random variable $W$ is defined as in Theorem~\ref{thm:asymptotic_distribution}.
\end{theorem}
The proof of Theorem~\ref{thm:bootstrap_distribution} is provided in
Appendix~\ref{sec:proof_bootstrap_distribution}.
The limiting distribution is the null limit associated with
$\law_{\theta_0}$, even when the data-generating law does not belong to the
model. Under the null, Theorem~\ref{thm:asymptotic_distribution} shows
that this is also the limit of the observed statistic. Comparing these
two limits gives the following validity result for the test~\eqref{eq:bootstrap_test}.
\begin{corollary}\label{cor:asymptotic_validity}
	Under the assumptions in Theorem~\ref{thm:bootstrap_distribution} with 
	$\law=\law_{\theta_0}$ and $\Sigma_{\theta_0}\ne0$, for every $\alpha\in(0,1)$, we have $\lim_{n\rightarrow\infty}\P_{\law_{\theta_0}}\left[\phi_{n,\alpha}^{\mathrm{PA}}=1\right]=\alpha$.
\end{corollary}

The conclusion follows from the same quantile argument used for
Corollary~\ref{cor:influence-adjusted-bootstrap-validity} in
Appendix~\ref{sec:proof_influence_adjusted_bootstrap_validity}.
Corollary~\ref{cor:asymptotic_validity} justifies the conditional
calibration in test~\eqref{eq:bootstrap_test}.

\subsection{What makes the influence-adjusted bootstrap efficient?}
\label{sec:bootstrap-computation}

It is natural to ask which method to use in practice. We compare the computational costs of these two bootstrap tests and the nonparametric bootstrap proposed by~\citet{brueck2025composite}. We will show that the influence-adjusted bootstrap test is more efficient than the other two bootstrap tests in general. To make the comparison more concrete, we consider the computational cost of these bootstrap tests in the following paragraphs. Let $C_{\mathrm{fit}}(n,k)$ be the cost of fitting $k$ parameters with sample size $n$,
$C_{\text{inf-est}}(n,k)$ the additional cost of estimating all $n$ influence
vectors after fitting, and $C_{\mathrm{sample}}(n)$ the cost of generating
$n$ fitted-model observations.

\paragraph{Influence-adjusted wild bootstrap complexity.}
Algorithm~\ref{alg:influence-adjusted-SKSD} requires a single model fit and a single pass of influence estimation, after which every bootstrap replication only reweights precomputed quantities. Under the evaluation-cost conventions, the total cost of a dense implementation with $B$ replications is
\begin{align*}
	\mathrm{IA}:&\quad O\bigl(C_{\mathrm{fit}}(n,k)+C_{\text{inf-est}}(n,k)+n^2k+nk^2+Bn^2\bigr).
\end{align*}
For the nuisance fit with Z-estimators as in
Appendix~\ref{sec:general_estimators_UALE}, in the low-dimensional regime $k = O(1)$, the total cost reduces to $O\bigl(C_{\mathrm{fit}}(n,k)+n^2B\bigr)$. Appendix~\ref{sec:ia-bootstrap-matrix}
gives the matrix construction and detailed cost accounting.

\paragraph{Parametric bootstrap complexity.}
Each replication of Algorithm~\ref{alg:semiparametric-KSD} generates a fitted-model sample, refits the nuisance estimator, and evaluates the statistic, giving the total cost
\begin{align*}
	\mathrm{PA}:&\quad O\bigl(B(C_{\mathrm{fit}}(n,k)+C_{\mathrm{sample}}(n)+nk+n^2)\bigr),
\end{align*}
which reduces to $O\bigl(B(C_{\mathrm{fit}}(n,k)+C_{\mathrm{sample}}(n)) +n^2B\bigr)$ in the same regime. 

\paragraph{Br\"uck-type nonparametric bootstrap complexity.}
Each replication of the bootstrap in \citet{brueck2025composite} resamples the observed data with replacement, refits the nuisance estimator, and evaluates a gradient-corrected $U$-statistic over all pairs at the refitted parameter, giving the total cost
\begin{align*}
	\text{Br\"uck}:&\quad O\bigl(B(C_{\mathrm{fit}}(n,k)+n^2k)\bigr),
\end{align*}
which reduces to $O\bigl(BC_{\mathrm{fit}}(n,k)+n^2B\bigr)$ in the same regime. Data resampling avoids the sampling cost $C_{\mathrm{sample}}(n)$, but the fitting cost is still paid $B$ times.

\paragraph{Three practical regimes with significant cost savings.} Compared with the influence-adjusted bootstrap, the parametric bootstrap thus pays the nuisance-fitting and fitted-model-sampling costs $B$ times rather than once, and the nonparametric bootstrap proposed by~\citet{brueck2025composite}, while avoiding fitted-model sampling, still pays the fitting cost $B$ times. Both are significantly more expensive when $B$ is large. Assuming throughout this comparison that influence function estimates are available at a cost comparable to a single model fit, we identify three concrete scenarios where the influence-adjusted bootstrap is more efficient than the parametric bootstrap. \textbf{(1) When model fitting is expensive}, as for estimators without a closed form: noise-contrastive estimation~\citep{gutmann2012noise}, minimum MMD or Wasserstein estimation~\citep{briol2019statistical,bernton2019parameter}, and nonlinear score matching for latent variable models~\citep{hyvarinen2005estimation,swersky2011autoencoders}, as in the restricted Boltzmann machine of Section~\ref{sec:simulation-rbm}, all require iterative numerical optimization, sometimes with model simulation in each iteration. Refitting such estimators is then the dominant expense: the influence-adjusted bootstrap pays $C_{\mathrm{fit}}(n,k)$ once, while the parametric and Br\"uck-type bootstraps pay it $B$ times. \textbf{(2) When sampling is expensive}, as for models with intractable normalizing constants~\citep{wainwright2008graphical}, each parametric bootstrap replication needs its own Markov chain Monte Carlo run with a separate burn-in step, and the resulting draws are neither exact nor i.i.d. On the other hand, the influence-adjusted bootstrap draws only multiplier weights whose distribution is free of the model. \textbf{(3) When both aspects are expensive}, as for energy-based models~\citep{swersky2011autoencoders,grathwohl2020learning}, only the influence-adjusted test escapes both costs: the score function removes the normalizing constant, and score-matching or minimum-KSD procedures keep the single fit and influence estimates tractable.

In summary, the influence-adjusted bootstrap test is more computationally efficient than two other bootstrap methods, especially in the intractable model settings. We therefore recommend the influence-adjusted bootstrap test in practice, unless the influence function is difficult to estimate. Another refitting- and resampling-free alternative applies a wild bootstrap to a Neyman-orthogonalized kernel~\citep{chernozhukov2022locally,escanciano2024gaussian}. A straightforward implementation using dense $n\times n$ matrix products has total computational cost $O(n^3+n^2B)$, which can exceed that of the influence-adjusted bootstrap. See Appendix~\ref{sec:neyman_orthogonal_sksd_test} for the construction.

\section{Power analysis of SKSD test}\label{sec:asymptotic_power}

We now study the power of the tests~\eqref{eq:influence-adjusted-bootstrap-test} and~\eqref{eq:bootstrap_test}.
Theorem~\ref{thm:sksd_consistency} controls their common observed statistic,
while Theorems~\ref{thm:influence-adjusted-bootstrap} and~\ref{thm:bootstrap_distribution} control the critical
values in \eqref{eq:influence-adjusted-bootstrap-test} and \eqref{eq:bootstrap_test}, respectively. We first consider a fixed alternative.

\begin{theorem}[Consistency under fixed alternatives]
\label{thm:asymptotic_power}
	Suppose $X_1,\ldots,X_n\overset{\mathrm{i.i.d.}}{\sim}\law$ for a fixed
	alternative $\law\notin\{\law_\theta:\theta\in\Theta\}$.
	Suppose the following conditions hold under $\law$:
	\begin{enumerate}
		\item[(1)] For the test~\eqref{eq:influence-adjusted-bootstrap-test}, the assumptions of
		Theorem~\ref{thm:influence-adjusted-bootstrap} hold.
		\item[(2)] For the test~\eqref{eq:bootstrap_test}, the assumptions of
		Theorem~\ref{thm:bootstrap_distribution} hold.
	\end{enumerate}
	Then, $\lim_{n\rightarrow\infty}\P_{\law}\left[\phi_{n,\alpha}^{\star}=1\right]=1$ for $\enspace \star\in\{\mathrm{IA},\mathrm{PA}\}$ and every $\alpha\in(0,1)$.
\end{theorem}

The proof is provided in Appendix~\ref{sec:proof_fixed_alternative}.
Theorem~\ref{thm:asymptotic_power} shows that the proposed bootstrap tests are consistent against any fixed
alternative distribution satisfying the stated conditions. 
This result is quite general as it holds for any asymptotically linear estimator and does not require any
further specification of the alternative distribution. We emphasize that the agnostic
bootstrap consistency result under general hypotheses (Theorems~\ref{thm:influence-adjusted-bootstrap} and~\ref{thm:bootstrap_distribution}) is what makes such
power analyses possible.

In practice, what is more interesting is the power of the proposed test under local
alternative models. We therefore allow the data-generating law to depend on $n$ and write it as
$\law_n$ in the following triangular-array setup.

\begin{definition}[Local alternatives with regularity]
\label{ass:local}
	We consider local alternatives satisfying either of the following conditions given $\theta_0 \in \Theta$.
	\begin{enumerate}
		\item[(a)] \textbf{Multiplicative local alternative.}
		Following \citet{jankova2020goodness}, for a function $r:\X\to\R$, we assume $\frac{\mathrm d\law_n}{\mathrm d\law_{\theta_0}}(x)\propto1+\frac{r(x)}{\sqrt n}$ and $\E_{X\sim\law_{\theta_0}}[r(X)^2]<\infty$.
		\item[(b)] \textbf{Additive local alternative.}
		Following \citet{huber1965robust,niu2022distribution},
		for a law $\law_g$ that is absolutely continuous with respect to $\law_{\theta_0}$ and a constant $c>0$, we assume $\law_n=(1-c/\sqrt n)\law_{\theta_0}+ c\law_g/\sqrt n$ and $\E_{X\sim\law_{\theta_0}}\Big[\Big(\frac{\mathrm d\law_g}{\mathrm d\law_{\theta_0}}(X)\Big)^2\Big]<\infty$.
	\end{enumerate}
\end{definition}

The regularity condition in each case ensures that the corresponding local
departure has a square-integrable density ratio. The following theorem gives the limit
distribution and local power of both tests.

\begin{theorem}[Power under local alternatives]
\label{thm:power_local_alternative}
	Suppose the conditions in Definition~\ref{ass:local} and the following additional conditions hold:
	\begin{enumerate}
		\item[(1)] For the test~\eqref{eq:influence-adjusted-bootstrap-test}, the assumptions of
		Theorem~\ref{thm:influence-adjusted-bootstrap} hold.
		\item[(2)] For the test~\eqref{eq:bootstrap_test}, the assumptions of
		Theorem~\ref{thm:bootstrap_distribution} hold.
	\end{enumerate}
	Recall that $\Psi_{\theta}(x)$ denotes $\Psi_{\theta,\law_{\theta}}(x)$. Define the local feature mean shift $\mu\in\mH^d$ by
	\begin{align}\label{eq:local-alternative-shift}
		\mu
		&=
		\begin{cases}
		\E_{X\sim\law_{\theta_0}}\left[
		\Psi_{\theta_0}(X)r(X)
		\right],&\text{under Definition~\ref{ass:local}(a)},\\
		c\E_{X\sim\law_g}\left[
		\Psi_{\theta_0}(X)
		\right],&\text{under Definition~\ref{ass:local}(b)}.
		\end{cases}
	\end{align}
	Then, under $\law_n$, with $\Sigma_{\theta_0}$, $\lambda_u$ and $Z_u$ as in Theorem~\ref{thm:asymptotic_distribution},
	\begin{align}\label{eq:local-observed-limit}
		nT_n
		&\overset{d}{\longrightarrow}
		W_\mu\equiv\sum_{u:\lambda_u>0}\left(\sqrt{\lambda_u}Z_u+\langle e_u,\mu\rangle_{\mH^d}\right)^2,
	\end{align}
	where $e_u$ are orthonormal eigenfunctions of $\Sigma_{\theta_0}$ corresponding to its positive eigenvalues $\lambda_u$.
	For every $\alpha\in(0,1)$, define $\beta
		=\P\left[
		W_\mu\ge\Q_{1-\alpha}(W)
		\right]$.
	Then, if $\Sigma_{\theta_0}\ne0$,
	\begin{align*}
		\lim_{n\to\infty}\P_{\law_n}\left[
		\phi_{n,\alpha}^{\star}=1
		\right]
		&= \beta \quad \text{for}\enspace \star\in\{\mathrm{IA},\mathrm{PA}\},
	\end{align*}
	where $\beta\ge\alpha$, with $\beta>\alpha$ if and only if $\mu\ne0$.
\end{theorem}

The proof of Theorem~\ref{thm:power_local_alternative} is provided in
Appendix~\ref{sec:proof_asymptotic_power_local}. Theorem~\ref{thm:power_local_alternative} characterizes the asymptotic power function of our tests under both multiplicative and additive local alternatives.
The asymptotic power is fully characterized by the mean shift parameter $\mu$. 
Also, notice that both the influence-adjusted bootstrap and the parametric bootstrap tests have the same first-order local power.

\begin{remark}[Locally powerful and powerless directions]\label{rem:local-power-directions}
	Under Theorem~\ref{thm:power_local_alternative},
	a local direction is \emph{locally powerful} when $\mu\ne0$ ($\beta>\alpha$)
	and \emph{locally powerless} when $\mu=0$ ($\beta=\alpha$). 
	The estimator and the local direction together determine whether the local alternative is locally powerful or not.
	Take Definition~\ref{ass:local}(a) as an example.
	Then $r - \E_{X\sim\law_{\theta_0}}[r(X)]$ is the parametric score of the local distribution path,
	as shown in Appendix~\ref{sec:local-power-geometry}. Then
	\begin{align*}
		\mu\ne0\quad\Longleftrightarrow\quad
		\E_{X\sim\law_{\theta_0}}\left[
		\langle f,\Psi_{\theta_0}(X)\rangle_{\mH^d}r(X)
		\right]\ne0\quad\text{for some }f\in\mH^d.
	\end{align*}
	Since $\Psi_{\theta_0}(X)$ is the influence-adjusted Stein feature, which depends on the choice of the estimator, the SKSD tests are powerful if and only if the parametric score of the local distribution path is not $L^2$-orthogonal to the influence-adjusted Stein feature.
	Thus, for a given local direction, we can apply this criterion to check whether the SKSD tests are locally powerful for different estimators.
	We also relate this criterion to the tangent-space interpretation in semiparametric
	statistics~\citep[Chapter~25]{van2000asymptotic} in Appendix~\ref{sec:local-power-geometry}.
\end{remark}

\begin{remark}[Separation rates]\label{rem:local-alternative-rates}
	In Definition~\ref{ass:local}, we consider the standard parametric rate $n^{-1/2}$, which can be extended to a general rate $\rho_n \to 0$ as $n \to \infty$.
	\begin{itemize}
		\item \textbf{$n^{1/2}\rho_n \to 0$:} We use an information-theoretic argument. For two probability measures $\law_1,\law_2$ on $(\X,\mathscr B(\X))$ with $\law_1\ll\law_2$, define the total variation distance $\|\law_1-\law_2\|_{\mathrm{TV}}=\frac{1}{2}\E_{\law_2}[\left|{\mathrm{d}\law_1}/{\mathrm{d}\law_2} - 1\right|]$ and chi-squared divergence $\chi^2(\law_1,\law_2) = \E_{\law_2}[\left({\mathrm{d}\law_1}/{\mathrm{d}\law_2} - 1\right)^2]$. Then the Cauchy--Schwarz inequality gives
	\begin{align*}
		4\left\|\law_n^{\otimes n}-\law_{\theta_0}^{\otimes n}\right\|_{\mathrm{TV}}^2
		&\le\chi^2(\law_n^{\otimes n},\law_{\theta_0}^{\otimes n})=\left(1+\chi^2(\law_n,\law_{\theta_0})\right)^n-1
		=O(n\rho_n^2)\rightarrow0.
	\end{align*}
	Therefore, any test is powerless against such alternatives, and our tests have asymptotic rejection probability $\alpha$ in this case.
	
	\item \textbf{$n^{1/2}\rho_n \to a\in(0,\infty)$:} Theorem~\ref{thm:power_local_alternative} gives the local power with mean shift $a\mu$.
	
	\item \textbf{$n^{1/2}\rho_n \to \infty$:} In this regime, the power of either test generally tends to one when $\mu\ne0$. Indeed, under mild conditions, the Taylor expansion of the influence-adjusted Stein feature $\Psi_{\theta}(X)$ around the null model shows that
	\begin{align*}
		nT_n&=n\rho_n^2\left(\|\mu\|_{\mH^d}^2+o_p(1)\right)\convp\infty.
	\end{align*}
	In contrast, the bootstrap statistics remain of order
	$n^{-1}$, so the critical values of $n\widetilde T_{n,\star}$ are $O_p(1)$. Consequently,
	the rejection probability of either test tends to one.
	If $\mu=0$, higher-order terms determine whether such departures are detected, and we leave this analysis to future work.
\end{itemize}
\end{remark}

\section{Empirical evaluations}\label{sec:applications}

We evaluate our proposed methods in simulation and real-data studies.

\subsection{Numerical simulations}\label{sec:numerical-simulations}

We consider four settings of increasing computational difficulty: a Gaussian location-scale model fitted by maximum likelihood, a kernel exponential family fitted by minimum KSD, a conditional Gaussian model fitted by score matching, and a Gaussian--Bernoulli restricted Boltzmann machine requiring iterative score-matching estimation. An additional experiment on spherical support is reported in Appendix~\ref{sec:sphere-experiment}.

The first two experiments compare six procedures. (1) The influence-adjusted test (SKSD-IA) uses Algorithm~\ref{alg:influence-adjusted-SKSD}. (2) The parametric-bootstrap test (SKSD-PA) uses Algorithm~\ref{alg:semiparametric-KSD}. (3) The semiparametric maximum mean discrepancy test (SMMD) extends the composite test in~\citet{key2025composite} to allow a general estimation procedure. (4) The semiparametric Kolmogorov--Smirnov test (SKS) is calibrated by the parametric bootstrap. (5) The KSD test of Br\"uck et al. uses a distinct bootstrap calibration~\citep{brueck2025composite}. (6) The na\"ive wild-bootstrap test (SKSD-NWB) calibrates the fitted KSD statistic without adjusting for nuisance estimation. The influence-adjusted and na\"ive wild-bootstrap tests use standard Gaussian multipliers in these two experiments. The conditional Gaussian and restricted Boltzmann machine experiments focus on the influence-adjusted test and the test of Br\"uck et al. All figures report rejection probabilities at nominal level $0.05$. Kernel-based methods use $\exp\{-\lVert x-y\rVert^2/(2h^2)\}$, with bandwidths specified below.

\subsubsection{Testing Gaussianity with maximum-likelihood estimation}
\label{sec:simulation-gaussian-mle}

The first experiment considers the Gaussian location-scale family, for which nuisance estimation is available in closed form.

\paragraph{Experiment setup.}
We test the composite null $\mathcal{P}_0=\{\N(\mu,\sigma^2):\mu\in\R,\ \sigma>0\}$ against three departures from Gaussianity. The Student-$t$, symmetric-mixture, and generalized-$\chi$ alternatives are specified in Appendix~\ref{sec:simulation-gaussian-alternatives}. Each design cell uses $n=100$ independent observations, $B=300$ bootstrap draws, and $R=300$ Monte Carlo replications.

\paragraph{Methods, estimation, and bandwidth.}
We compare all six procedures listed above, using the maximum-likelihood estimates $\hat\mu=\bar X$ and $\hat\sigma^2=n^{-1}\sum_i(X_i-\bar X)^2$. Estimating both parameters makes this a test of Gaussian shape within the full location-scale family. In this setting, the Kolmogorov--Smirnov procedure coincides with a parametric-bootstrap implementation of the classical Lilliefors test~\citep{lilliefors1967kolmogorov}. The four KSD-based procedures use \mbox{$h=\operatorname{median}_{i<j}|X_i-X_j|$}, calculated from each observed dataset and held fixed throughout bootstrap calibration. The SMMD test uses the median pairwise distance of the pooled observed sample and an independent size-$n$ sample from the fitted null distribution, recomputed for each bootstrap pair.

\paragraph{Simulation results.}
Across the Student-$t$, mixture, and generalized-$\chi$ experiments, Figure~\ref{fig:simulation-gaussian-mle} shows null rejection rates of $0.070$, $0.047$, and $0.053$ for the influence-adjusted test, versus $0.010$, $0.010$, and $0.007$ for the na\"ive wild-bootstrap test. The latter rates illustrate conservative calibration when nuisance estimation is ignored. The influence-adjusted test and Br\"uck et al.'s test have similar power for the Student-$t$ alternatives. At mixture signal $\delta=1.25$, the influence-adjusted test has power $0.520$, compared with at most $0.433$ among the other procedures and $0.090$ for the na\"ive test. At generalized-$\chi$ signal $\alpha=1.5$, the corresponding rates are $0.557$, at most $0.507$, and $0.090$. In the Student-$t$ timing panel, the median cell-level mean times per 100 replications are $0.209$ seconds for the influence-adjusted test, $0.186$ for the na\"ive test, and $10.397$ for Br\"uck et al.'s test. The median paired runtime ratio (Br\"uck et al. divided by the influence-adjusted test) is $49.7$.

\begin{figure}[!ht]
  \centering
  \makebox[\textwidth][c]{\includegraphics[width=\textwidth]{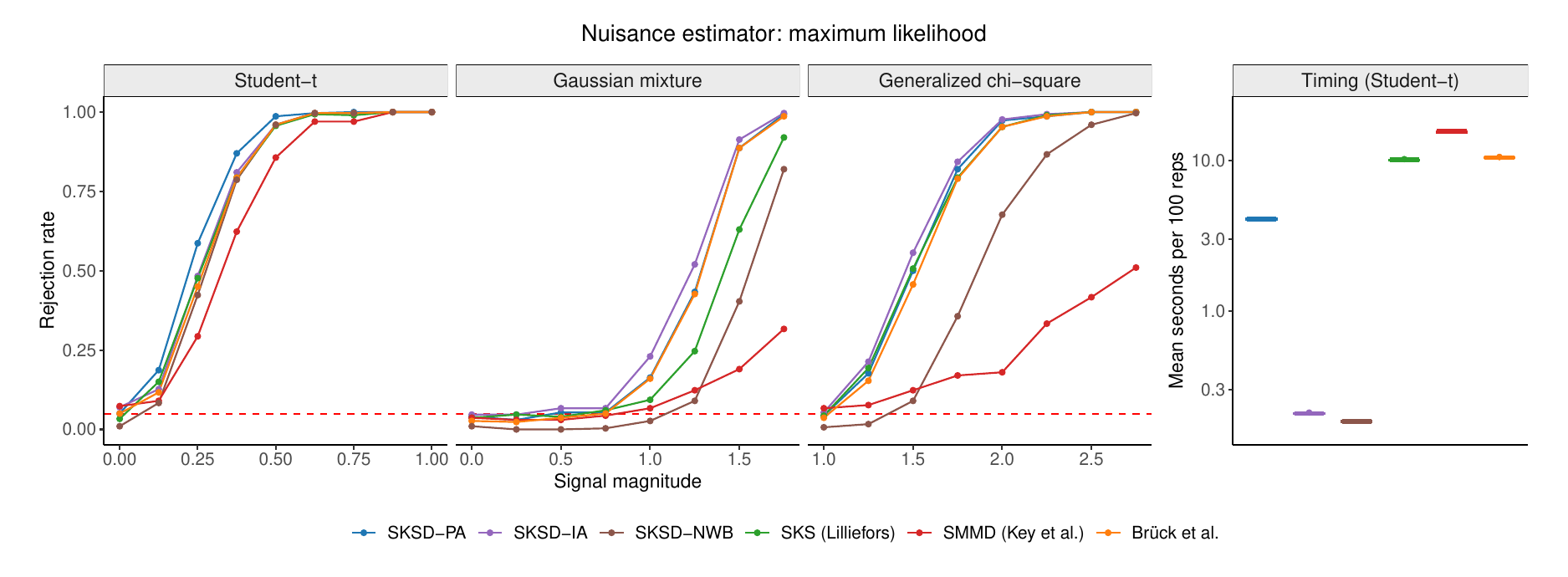}}
  \caption{Gaussianity tests with maximum-likelihood nuisance estimation.}
  \label{fig:simulation-gaussian-mle}
\end{figure}

\subsubsection{Testing a kernel exponential-family model with minimum KSD}\label{sec:simulation-kef-minksd}

The second experiment considers an intractable model whose normalizing constant is unavailable in closed form.

\paragraph{Experiment setup.}

Let $q$ be the $\N(0,9)$ density and define $\phi_r(z)=z^r\exp(-z^2/2)/\sqrt{r!}$. The composite null is the factorized kernel exponential family
\begin{align*}
	\prod_{j=1}^d\frac{q(x_j)\exp\{\theta_0\phi_0(x_j)\}}{Z(\theta_0)}.
\end{align*}
In other words, the coordinates are independent of each other. The data-generating alternatives are specified in Appendix~\ref{sec:simulation-kef-alternatives}. We take $d\in\{1,2,3\}$ and index the alternatives by the dimension-normalized signal $s\propto\sqrt d$. For each dataset, the null coefficient $\theta_0$ is estimated by minimizing the empirical KSD $U$-statistic~\citep{barp2019minimum}. Each design cell uses $n=100$ observations, $B=500$ bootstrap draws, and $R=500$ Monte Carlo replications. Rejection sampling is used both to generate the observed datasets and to draw fitted-model samples for the parametric bootstrap.

All four KSD-based procedures use the rule \mbox{$h=1.5\,\operatorname{median}_{i<j}\lVert X_i-X_j\rVert$}, calculated from each observed dataset and held fixed for minimum-KSD fitting and bootstrap calibration. The SMMD test-kernel bandwidth is the median pairwise distance of the pooled observed sample and an independent size-$n$ sample from the fitted null, recomputed for each bootstrap pair. Its minimum-KSD nuisance fits retain the observed-data KSD bandwidth.

\paragraph{Simulation results.}
Figure~\ref{fig:simulation-kef-minksd} compares the six procedures. At the null, rejection rates for the influence-adjusted test are $0.052$, $0.062$, and $0.050$ for $d=1,2,3$, compared with $0.032$, $0.064$, and $0.050$ for the na\"ive wild bootstrap test. The corresponding rates for Br\"uck et al.'s test are $0.074$, $0.118$, and $0.122$, showing substantial size inflation for $d=2,3$. At $|s|=1$, averaging the two signs, the influence-adjusted test has power $0.803$, $0.724$, and $0.682$; the na\"ive test has $0.734$, $0.723$, and $0.687$; and the parametric-bootstrap test has $0.796$, $0.749$, and $0.720$. Br\"uck et al.'s test has higher raw rejection rates, but these comparisons are not size-matched. The semiparametric Kolmogorov--Smirnov test gives $0.752$, $0.631$, and $0.606$, while the semiparametric maximum mean discrepancy test is near $0.52$ in each dimension. In the $d=1$ timing panel, the median cell-level mean times per 100 replications are $0.426$ seconds for the influence-adjusted test, $0.380$ for the na\"ive test, and $4.718$ for Br\"uck et al.'s test. The median paired runtime ratio (Br\"uck et al. divided by the influence-adjusted test) is $11.3$.

\begin{figure}[!ht]
	\centering
	\makebox[\textwidth][c]{\includegraphics[width=\textwidth]{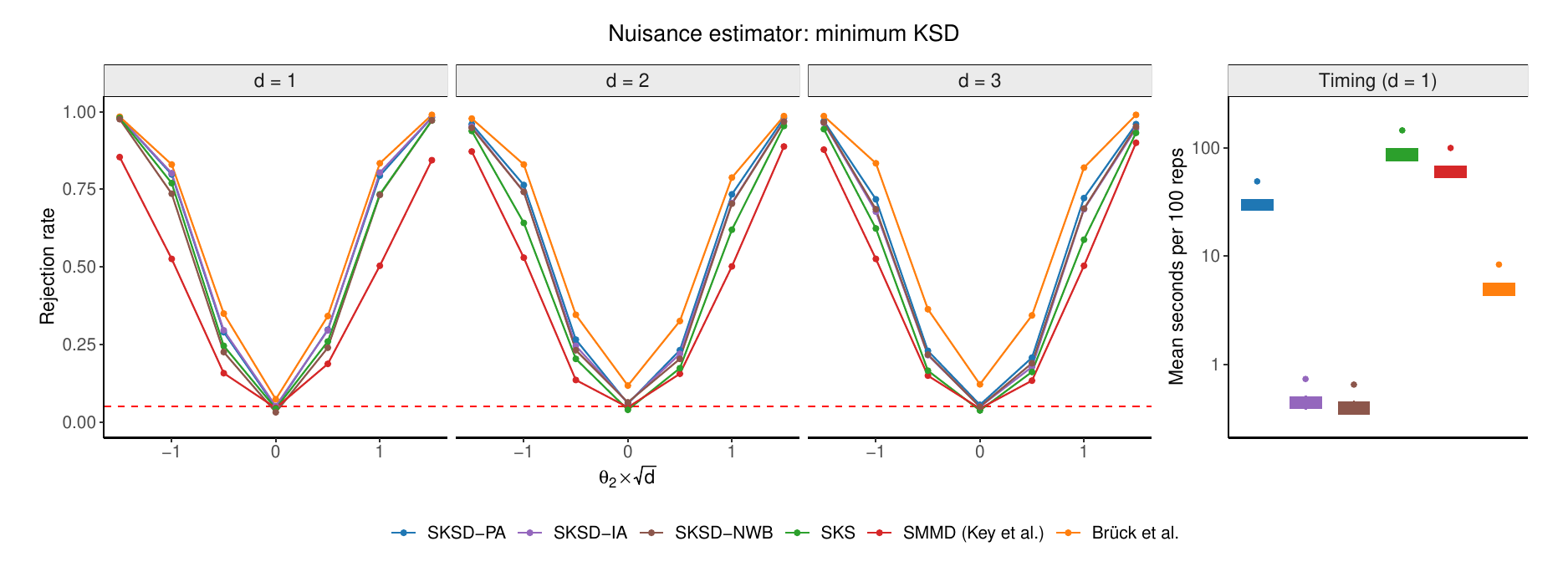}}
	\caption{Kernel exponential-family tests with minimum-KSD nuisance estimation. IA, PA, and NWB denote influence-adjusted, parametric-bootstrap, and na\"ive wild-bootstrap calibration. The horizontal axis is the dimension-normalized signal $s\propto\sqrt d$ for a fixed $\theta_2$ value.}
	\label{fig:simulation-kef-minksd}
\end{figure}

An additional one-dimensional kernel exponential-family comparison with score-matching nuisance estimation is reported in Appendix~\ref{sec:kef-scorematching-nwb}, where the influence-adjusted and na\"ive wild-bootstrap tests are compared. The influence-adjusted test has much higher power because the na\"ive bootstrap is conservatively calibrated.

\subsubsection{Structural detection in a conditional Gaussian model}
\label{sec:simulation-conditional-gaussian}

\paragraph{Experiment setup.}
The third experiment studies an eight-dimensional quadratic interaction model~\citep{arnold1999conditional},
\begin{align*}
  p_{\Theta}(x)
  \propto \exp\left\{
  \sum_{i\neq j}\Theta_{ij}x_i^2x_j^2
  -\frac12\sum_{i=1}^8x_i^2
  +2\sum_{i=1}^8x_i
  \right\},
\end{align*}
where $\Theta$ is symmetric, has zero diagonal, and has nonpositive off-diagonal entries. The composite null restricts $\Theta$ to a nearest-neighbor ring with eight unknown edge coefficients. We consider second-neighbor rewiring, within-block triadic closure, and diametric matched pairs, defined in Appendix~\ref{sec:conditional-gaussian-alternatives}. Each setup uses $n=100$, $B=300$, and $R=300$. Although the joint density has quartic interactions, every full conditional is Gaussian. Each observation is obtained from a separate Gibbs chain initialized from $\N(0,I_8)$ and run for $600$ sweeps.

\paragraph{Methods, estimation, and bandwidth.}
We compare the influence-adjusted test with the test of Br\"uck et al. on the same simulated datasets. The influence-adjusted test uses standard Gaussian multipliers for calibration. Both approaches use a score-matching estimate of the eight ring coefficients~\citep{hyvarinen2005estimation}. The quadratic objective has a closed-form interior solution (see Appendix~\ref{sec:conditional-gaussian-alternatives} for details). The bandwidth is fixed at $h=2.35$ for all datasets and bootstrap draws.

\paragraph{Simulation results.}
Figure~\ref{fig:simulation-conditional-gaussian} shows substantially higher power for the influence-adjusted test across all three alternatives. Br\"uck et al.'s test is strongly conservative at the common null, rejecting in $0.010$ of datasets compared with $0.047$ for the influence-adjusted test. This conservative calibration helps explain its lower nominal power. Across datasets, mean total time is $0.00369$ seconds for the influence-adjusted test and $0.238$ seconds for Br\"uck et al., a $64.7$-fold ratio driven by the latter's repeated score-matching fits.

\begin{figure}[!ht]
  \centering
  \makebox[\textwidth][c]{\includegraphics[width=\textwidth]{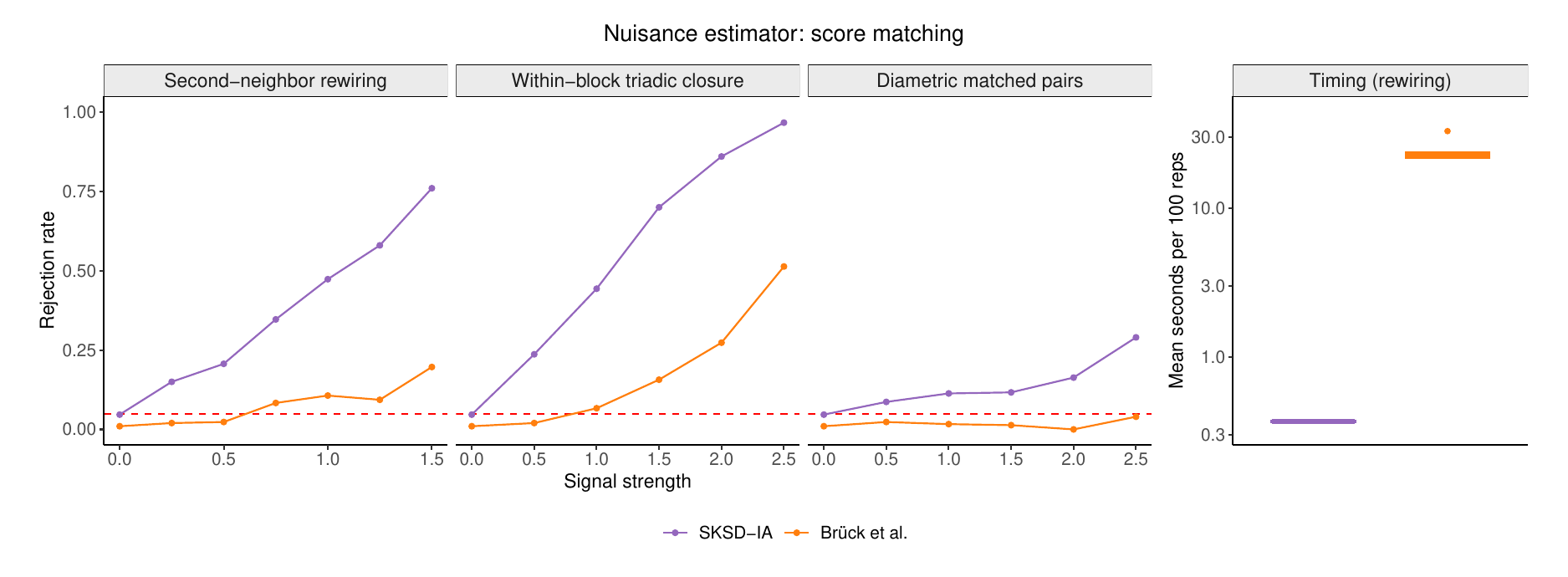}}
  \caption{Conditional Gaussian structural tests with score-matching nuisance estimation.}
  \label{fig:simulation-conditional-gaussian}
\end{figure}

\subsubsection{Testing a Gaussian--Bernoulli restricted Boltzmann machine}
\label{sec:simulation-rbm}

\paragraph{Experiment setup.}
The fourth experiment considers a latent model with $30$ visible and six hidden variables. For $x\in\R^{30}$ and $z\in\{-1,+1\}^6$, its joint density is
\begin{align*}
  p_{\rho}(x,z)
  \propto \exp\left\{-\frac12\lVert x\rVert^2
  +x^\top U\operatorname{diag}(e^{\rho_1},\ldots,e^{\rho_6})z\right\}.
\end{align*}
The loading directions $U\in\R^{30\times6}$ are fixed, and the six log-scales $\rho_j$ are unknown. We consider three alternatives: rotating the loading directions, adding an omitted seventh hidden unit, and replacing Gaussian visible noise with covariance-standardized Student noise. Appendix~\ref{sec:rbm-simulation-details} defines $U$ and the three paths. Each setup uses $n=300$ and $B=200$, with $R=300$ repetitions. Exact sampling is available from the joint density.

\paragraph{Methods, estimation, and bandwidth.}
We compare the influence-adjusted test with the test of Br\"uck et al. using a common nonlinear Hyv\"arinen score-matching fit~\citep{hyvarinen2005estimation}. Unlike the preceding experiment, no closed-form nuisance estimator is available. We optimize the six log-scales with L-BFGS-B, as detailed in Appendix~\ref{sec:rbm-simulation-details}. Both methods use the same datasets, but the influence-adjusted test uses standard Gaussian multipliers for calibration whereas Br\"uck et al. use Efron counts. The bandwidth $h=\sqrt{2d}=\sqrt{60}$ is fixed throughout the experiment.

\paragraph{Simulation results.}
Figure~\ref{fig:simulation-rbm} shows similar power for two alternatives. Br\"uck et al.'s test is more sensitive to heavy-tailed noise, with power $0.907$ at the largest signal compared with $0.780$ for the influence-adjusted test. Both methods are conservative at the shared null, with rejection rates $0.0167$ for the influence-adjusted test and $0.0100$ for Br\"uck et al. Including the nuisance fit and shared preprocessing, mean total time across unique datasets is $0.0288$ seconds for the influence-adjusted test and $1.2355$ seconds for Br\"uck et al., a $42.9$-fold ratio.

\begin{figure}[!ht]
  \centering
  \makebox[\textwidth][c]{\includegraphics[width=\textwidth]{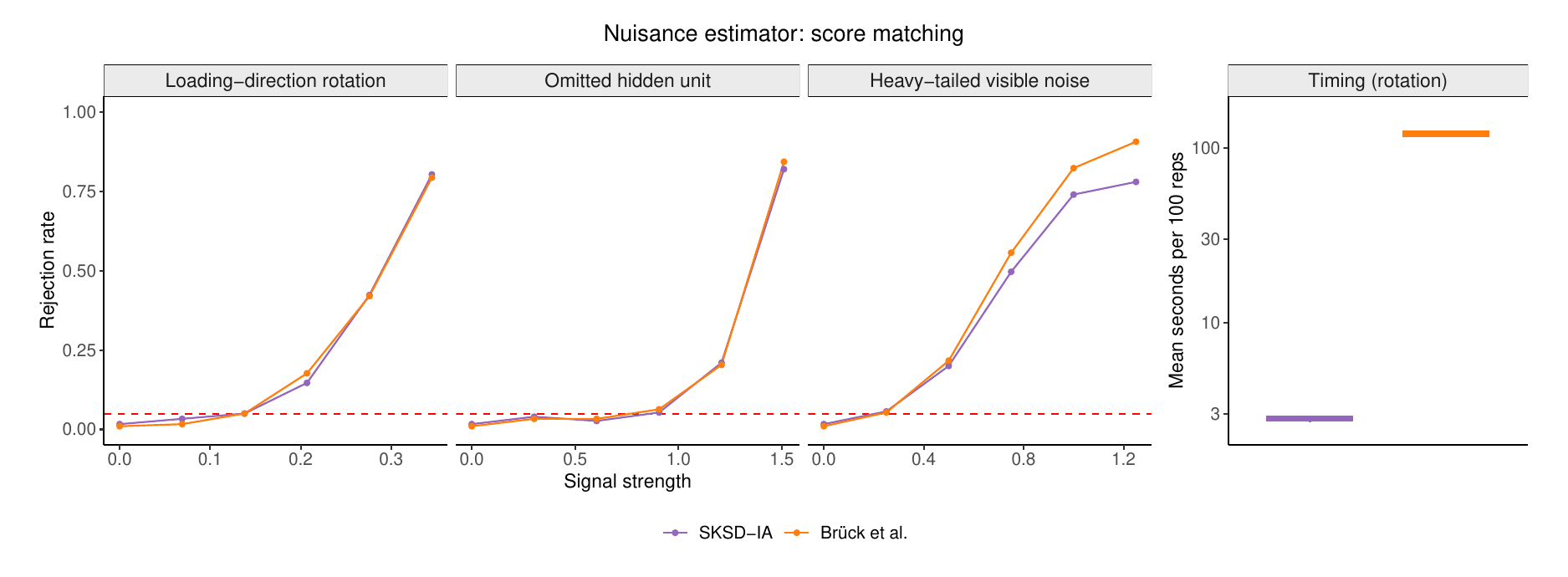}}
  \caption{Gaussian--Bernoulli restricted Boltzmann machine tests with nonlinear score-matching nuisance estimation.}
  \label{fig:simulation-rbm}
\end{figure}

\subsection{Real-data analysis: TCGA lung adenocarcinoma}
\label{sec:real-data}

Motivated by~\citet{yao2025robust}, we test whether a common pairwise protein interaction relationship adequately describes tumors with low and high signature values. We use lung adenocarcinoma data from The Cancer Genome Atlas (TCGA)~\citep{tcga2014luad} and allow the two strata (low and high signature) to have group-specific marginal protein distributions. Data sources, sample construction, and pilot-based preprocessing are detailed in Appendix~\ref{sec:tcga-data-preprocessing}. The resulting analysis includes measurements of ten proteins.

\paragraph{Graphical model and scientific hypothesis.}

Inspired by the pairwise exponential-family framework considered by~\citet{yu2020simultaneous}, we model protein activity using a Student-$t$ base measure to accommodate potentially heavy-tailed observations. For group $g\in\{L,H\}$, the density is specified up to a normalizing constant by
\begin{align*}
	p_g(z;a_g,J_g)
	&\propto \left\{\prod_{j=1}^{10}t_5(z_j;\mu_j,\sigma_j)\right\}
	\exp\left\{\sum_{j=1}^{10}\sum_{r=1}^{5}a_{gjr}\psi_r(z_j)
	+\sum_{(j,k)\in E}J_{g,jk}\phi(z_j)\phi(z_k)\right\},
\end{align*}
where $\psi_r(u)=\tanh\{(u-\kappa_r)/0.7\}$ for $\kappa_r\equiv r-3$ and $\phi(u)=\tanh(u)$. The normalizing constant is unavailable in closed form. We consider two edge-support specifications: (i) a sparse 18-edge support fixed a priori from high-confidence STRING functional associations~\citep{szklarczyk2023string}, and (ii) the complete 45-edge support comprising all pairs among the ten proteins, used to assess sensitivity to omitted edges. For each support we test the nested ladder
\begin{align*}
	M_0 &: J_L=J_H=0, &
	M_S &: J_L=J_H=J, &
	M_V &: J_L,J_H\ \text{unrestricted}.
\end{align*}
The marginal coefficients remain group-specific in every candidate.

\paragraph{Model fitting and testing procedure.}
We use ordinary Hyv\"arinen score matching, which yields a closed-form model fit and influence function without evaluating the normalizing constant~\citep{hyvarinen2005estimation}. We assess the absolute adequacy of each candidate using the proposed influence-adjusted bootstrap and, for comparison, the method of~\citet{brueck2025composite}. Holm adjustment is applied across the three candidates separately for each procedure and edge support.

\paragraph{Results.}
Figure~\ref{fig:tcga-real-data} compares model-adequacy $p$-values on the 18-edge and complete 45-edge supports. On the 18-edge support, absolute adequacy depends strongly on calibration. With $B=999$, the influence-adjusted procedure gives raw $p$-values $0.001$ for all three candidates and Holm $p$-values $0.003$, rejecting even the encompassing varying-interaction model. Br\"uck et al. give raw $p$-values $0.110$, $0.183$, and $0.237$ and Holm $p$-values $0.330$, $0.366$, and $0.366$, rejecting none.

\begin{figure}[!ht]
	\centering
	\includegraphics[width=0.96\textwidth]{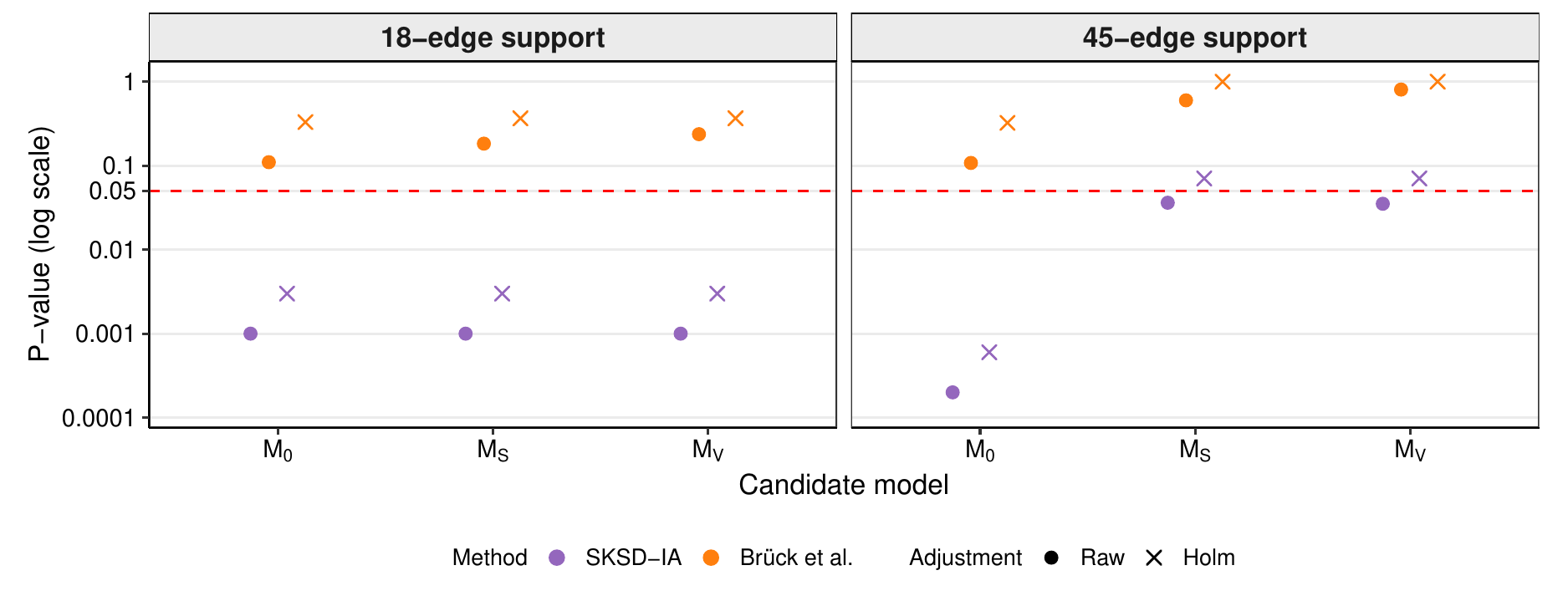}
	\caption{Model-adequacy $p$-values for The Cancer Genome Atlas lung adenocarcinoma cohort under the 18-edge support (left) and complete 45-edge support (right). Points show raw and Holm-adjusted $p$-values from the influence-adjusted procedure and Br\"uck et al. The dashed horizontal line marks $0.05$.}
	\label{fig:tcga-real-data}
\end{figure}

The 45-edge model gives a less discordant decision. With $B=4{,}999$, influence-adjusted raw $p$-values are $0.0002$, $0.0362$, and $0.0352$ for $M_0,M_S,M_V$, with Holm values $0.0006$, $0.0704$, and $0.0704$. The Holm-adjusted values from Br\"uck et al. are $0.3234$, $1$, and $1$. Thus neither procedure rejects the interacting complete models after adjustment.

\paragraph{Interpretation and discussion.}
The influence-adjusted procedure yields the more biologically coherent diagnostic pattern in this application: it rejects the no-interaction model under both supports, consistent with correlated total EGFR and mTOR (two proteins out of ten we added) expression reported in an independent non-small-cell lung cancer cohort~\citep{crees2020egfr}, and it detects lack of fit across all three candidates on the biologically informed 18-edge support. By contrast, Br\"uck et al. reject none of these candidates, indicating lower sensitivity in this setting.

\paragraph{Limitations.} Failure to reject the shared- and varying-interaction models on the complete 45-edge support should not be interpreted as evidence that either is adequate. The models contain 145 and 190 parameters, respectively, for only 240 tumors, so their additional flexibility may reduce finite-sample power to detect remaining misspecification. Accordingly, the analysis supports criticism of the independence and sparse-support (18-edge) specifications but does not establish that either a shared or signature-specific complete interaction matrix adequately describes the data.

\section{Discussion and extensions}\label{sec:discussion}

In this paper, we propose a new framework for composite goodness-of-fit testing based on the kernelized Stein discrepancy. We develop a novel influence-adjusted wild bootstrap to calibrate the test in the presence of nuisance parameter estimation, sidestepping the need for repeated model refitting and fitted-model sampling. We establish the asymptotic validity of the proposed procedure. We also interpret the SKSD test as a nonparametric score test, providing a complementary perspective on the proposed procedure. Simulations and a real-data analysis demonstrate that the proposed influence-adjusted SKSD test is both computationally efficient and statistically powerful, making it a valuable tool for practitioners.

\paragraph{Other sample spaces and conditional models.} There are many promising directions for future work. A first family of extensions concerns different data and model setups. We construct a spherical SKSD test with influence-adjusted calibration in Appendix~\ref{sec:sphere-test-construction}. Simulations with the Watson family in Appendix~\ref{sec:sphere-experiment} suggest that this approach retains useful empirical power on spherical support. Extending the asymptotic theory to general manifolds remains future work. For discrete distributions, the differential Stein operator can be replaced by an appropriate discrete Stein operator~\citep{yang2018goodness}, which would open applications to maximum-entropy and graphical models with intractable normalizing constants. For conditional models, a Stein operator that differentiates only in the response requires the conditional score function $\nabla_y\log p_\theta(y\mid x)$ and leaves the covariate distribution unspecified~\citep{jitkrittum2020testing}. Combining this construction with the influence-adjusted calibration offers a route to composite conditional goodness-of-fit tests for regression models with estimated parameters.

\paragraph{Bandwidth selection and high-dimensional models.} The choice of bandwidth in the kernel function of the SKSD test is a challenging problem and can have a substantial impact on the power of the resulting test. It would be valuable to draw inspiration from recent work on improving the power of kernel-based tests by aggregating multiple kernels~\citep{schrab2022ksd}. Moreover, the null model can be more complex, going beyond the fixed-dimensional setting to high- or even infinite-dimensional cases, and the composite nature of the goodness-of-fit problem further complicates the analysis. An important direction is to investigate how estimation error in high-dimensional settings manifests in the nonparametric goodness-of-fit testing problem and how this, in turn, suggests principled solutions.

\section*{Data availability}
The TCGA lung adenocarcinoma data used in the real-data analysis are described in Appendix~\ref{sec:tcga-data-preprocessing}. Code for reproducing the simulations and data analysis is available at \url{https://github.com/ZiangNiu6/SKSD-reproduction-code}.

\section*{Acknowledgements}

We thank Lester Mackey for reading the current manuscript and providing feedback. We also thank Fran\c{c}ois-Xavier Briol and Xing Liu for providing feedback on an early draft of the paper.

\printbibliography

\newpage
\appendix

\section{Additional simulation details}

\subsection{Comparison with a naive wild bootstrap}
\label{sec:kef-scorematching-nwb}

The na\"ive wild-bootstrap calibration used here is summarized in Algorithm~\ref{alg:naive-wild-SKSD}. It uses the fitted Stein Gram matrix without an influence adjustment.

\begin{center}
	\begin{minipage}{\linewidth}
		\begin{algorithm}[H]
			\nonl \textbf{Input:} Distribution class $\{\law_\theta:\theta\in\Theta\}$, data $X_{1:n}=(X_1,\ldots,X_n)$, estimation algorithm $\mathcal{E}:\mathcal{X}^n\rightarrow\Theta\subseteq\mathbb{R}^k$, multiplier law $\nu$, bootstrap size $B$.\\
			Obtain $\hat\theta_n=\mathcal{E}(X_{1:n})$ and compute $T_n\equiv T_n(X_{1:n},\hat\theta_n)$ as in~\eqref{eq:sksd_test_stat}\;
			Construct the Gram matrix $\widehat H\in\R^{n\times n}$ with $\widehat H_{ij}=h_{p_{\hat\theta_n}}(X_i,X_j)$ as in~\eqref{eq:ia-bootstrap-kernel-components}\;
			\For{$b=1,2,\ldots,B$}{
				Generate $w_1^{(b)},\ldots,w_n^{(b)}\overset{\mathrm{i.i.d.}}{\sim}\nu$ and set $\bm w^{(b)}=\left(w_i^{(b)}-\frac1n\sum_{\ell=1}^n w_\ell^{(b)}\right)_{i=1}^n$\;
				Compute
				$\displaystyle
				\widetilde T_{n,\mathrm{NWB}}^{(b)}
				=\frac{1}{n^2}(\bm w^{(b)})^\top\widehat H\bm w^{(b)}$\;
			}
			\nonl \textbf{Output:} $p$-value
			$\displaystyle
			\widehat p_{\mathrm{NWB}}
			=\frac{1+\sum_{b=1}^B
			\indicator(\widetilde T_{n,\mathrm{NWB}}^{(b)}\geq T_n)}{B+1}$.
			\caption{\bf SKSD test with naive wild bootstrap}
			\label{alg:naive-wild-SKSD}
		\end{algorithm}
	\end{minipage}
\end{center}

\paragraph{Data-generating procedure.} We complement the minimum-KSD kernel exponential-family experiment in Section~\ref{sec:simulation-kef-minksd} with a one-dimensional comparison in which the nuisance coefficients are fitted by ordinary score matching. Let $q$ be the $\N(0,9)$ density and $\phi_r(x)=x^r\exp(-x^2/2)/\sqrt{r!}$. For $k\in\{1,2,3\}$, the fitted composite null is
\begin{align*}
	\mathcal P_{0,k}=\left\{p_\vartheta(x)=\frac{q(x)\exp\{\sum_{r=0}^{k-1}\vartheta_r\phi_r(x)\}}{Z_k(\vartheta)}:\vartheta\in\R^k\right\}.
\end{align*}
The three fitted spans are therefore $\{\phi_0\}$, $\{\phi_0,\phi_1\}$, and $\{\phi_0,\phi_1,\phi_2\}$. In every case, the true null has coefficient $10$ on $\phi_0$ and zero on the other fitted terms. To generate a departure omitted from all three null families, we draw observations with density
\[
  p_t(x)=\frac{q(x)\exp\{10\phi_0(x)+\eta_t\phi_3(x)\}}{Z_t},
  \qquad \eta_t=\frac{t}{\sigma_3},\qquad
  \sigma_3^2=\operatorname{Var}_{P_0}\{\phi_3(X)\},
\]
where $P_0$ is the distribution with density $p_{t=0}$. Thus $t=0$ is the same data-generating null in all panels, and $t\in\{0,0.1,0.2,0.3,0.4,0.5\}$ indexes the standardized strength of the omitted $\phi_3$ term. We use rejection sampling to draw independent observations.

\paragraph{Kernel choice and simulation details.} The influence-adjusted and na\"ive wild-bootstrap tests use the Gaussian kernel $\exp\{-(x-y)^2/(2h^2)\}$ and $h=1.5\,\operatorname{median}_{i<j}|X_i-X_j|$. They also share standard Gaussian multiplier draws. Each signal and fitted span uses $n=100$ observations, $B=500$ multiplier draws, and $R=500$ independent datasets. We reject when the bootstrap p-value is at most $0.05$.

\begin{figure}[!ht]
  \centering
  \makebox[\textwidth][c]{\includegraphics[width=\textwidth]{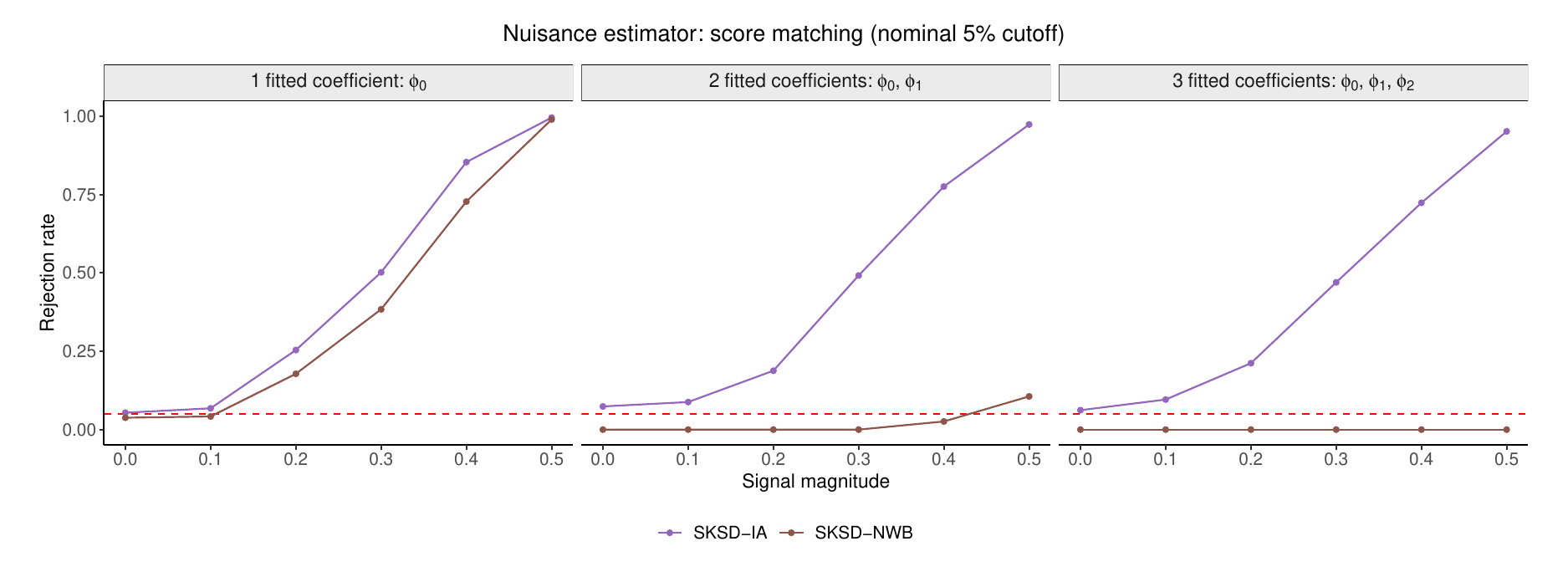}}
  \caption{One-dimensional kernel exponential-family comparison with score-matching nuisance estimation. IA and NWB denote influence-adjusted and na\"ive wild-bootstrap calibration. The panels vary the fitted basis span, while $\phi_3$ is omitted from each null family. The horizontal axis is $t=\eta_t\,\sigma_3$, where $\sigma_3^2=\operatorname{Var}_{P_0}\{\phi_3(X)\}$. Curves show rejection rates at the nominal $0.05$ p-value cutoff without size correction. The dashed line marks $0.05$.}
  \label{fig:kef-scorematching-nwb}
\end{figure}

\paragraph{Results and interpretation.} Figure~\ref{fig:kef-scorematching-nwb} shows that the na\"ive wild-bootstrap test has low nominal power when two or three nuisance coefficients are fitted, even as the omitted $\phi_3$ signal grows. In the same panels, the influence-adjusted test becomes more sensitive as the signal increases. The na\"ive wild-bootstrap test also rarely rejects at the shared null, consistent with conservative calibration contributing to the gap in nominal rejection rates.

\subsection{Gaussianity alternatives}
\label{sec:simulation-gaussian-alternatives}

Consider the Student-$t$ alternative, $X_{\mathrm t}(s)\sim t_{1/s}$ for $s>0$, with $X_{\mathrm t}(0)\sim\N(0,1)$. The second alternative is a symmetric mixture, $X_{\mathrm{mix}}(\delta)=\delta S+\varepsilon$ for $\delta\geq0$. The last alternative is the generalized-$\chi$ distribution, $X_{\chi}(\alpha)=S\sqrt{2Y_\alpha}$ for $\alpha\geq1$. Here $S$ is a symmetric sign, $\varepsilon\sim\N(0,1)$, and $Y_\alpha$ has a Gamma distribution with shape $\alpha/2$ and unit scale. These variables are mutually independent. The three null points are $s=0$, $\delta=0$, and $\alpha=1$, respectively, at each of which the data-generating distribution is exactly $\N(0,1)$.

\subsection{Kernel exponential-family alternatives}
\label{sec:simulation-kef-alternatives}

We use the $q$ and $\phi_r$ defined in Section~\ref{sec:simulation-kef-minksd}. We generate independent $d$-dimensional observations with mutually independent coordinates, each having unnormalized density
\begin{align*}
	q(z)\exp\{\theta_1\phi_0(z)+\theta_2\phi_1(z)\},
	\qquad \theta_1=10.
\end{align*}
When $\theta_2=0$, the data-generating law belongs to the composite null with $\theta_0=\theta_1=10$. When $\theta_2\neq0$, the term $\theta_2\phi_1$ introduces an odd-basis component omitted from the null model.

\subsection{Conditional Gaussian structural alternatives}
\label{sec:conditional-gaussian-alternatives}

This appendix defines the three structural alternatives used in Section~\ref{sec:simulation-conditional-gaussian}. Write $A=-\Theta$, so that the off-diagonal entries of $A$ are nonnegative. All matrices below have zero diagonal, and node indices are interpreted modulo $8$. The parenthesized superscript labels the interaction pattern: $(0)$ is the baseline null matrix $A^{(0)}$, while $(1)$, $(2)$, and $(3)$ label the alternatives below.

\subsubsection{Null distribution and estimation}

\paragraph{Null data-generating procedure.} The common data-generating null is the ring $a^{(0)}_{i,i+1}=w_i$ for $i=1,\ldots,8$. All non-ring edge weights are zero. The fixed baseline weights, rounded to three decimal places, are
\[
  w\approx(2.716,4.560,3.869,3.482,3.420,5.162,5.621,2.709)\in\mathbb{R}^{8}.
\]
The composite null allows the eight ring coefficients to vary and estimates them by score matching.

\paragraph{Score-matching estimator.}
Write $\beta_i=2\Theta_{i,i+1}$ and $t_i(x)=x_i^2x_{i+1}^2$ for $i=1,\ldots,8$, with node indices modulo $8$. Let $g_i(x)=\nabla_x t_i(x)$ and $s_0(x)=2\bm{1}-x$. Up to a term independent of $\beta$, the empirical Hyv\"arinen objective is
\begin{align*}
  L_n(\beta)=\tfrac12\beta^\top H_n\beta+b_n^\top\beta,
\end{align*}
where 
\begin{align*}
	(H_n)_{ij}=\frac1n\sum_{\ell=1}^n g_i(X_\ell)^\top g_j(X_\ell)\quad\text{and}\quad (b_n)_i=\frac1n\sum_{\ell=1}^n\bigl\{\Delta t_i(X_\ell)+g_i(X_\ell)^\top s_0(X_\ell)\bigr\}.
\end{align*}
When $H_n$ is nonsingular and the solution is interior, $\hat\beta=-H_n^{-1}b_n$ and $\hat\Theta_{i,i+1}=\hat\beta_i/2$. The implementation constrains $\beta\in[-50,0]^8$ and minimizes the same quadratic over this box if the unconstrained solution lies outside it.

\subsubsection{Alternative distributions}

For each path $r\in\{1,2,3\}$, the alternative distribution is $p_{\Theta^{(r)}(\delta)}$ with
\[
  \Theta^{(r)}(\delta)=-A^{(r)}(\delta),
  \qquad 0\leq\delta<\min_i w_i.
\]
The linear and quadratic coefficients in the density remain fixed at $2$ and $-1/2$, respectively. The restriction on $\delta$ keeps every weakened ring edge positive. Unspecified edge weights retain their baseline values, and each assignment applies symmetrically to both matrix entries.

Figure~\ref{fig:conditional-gaussian-edge-patterns} illustrates the three edge patterns.

\begin{figure}[!ht]
  \centering
  \includegraphics[width=\textwidth]{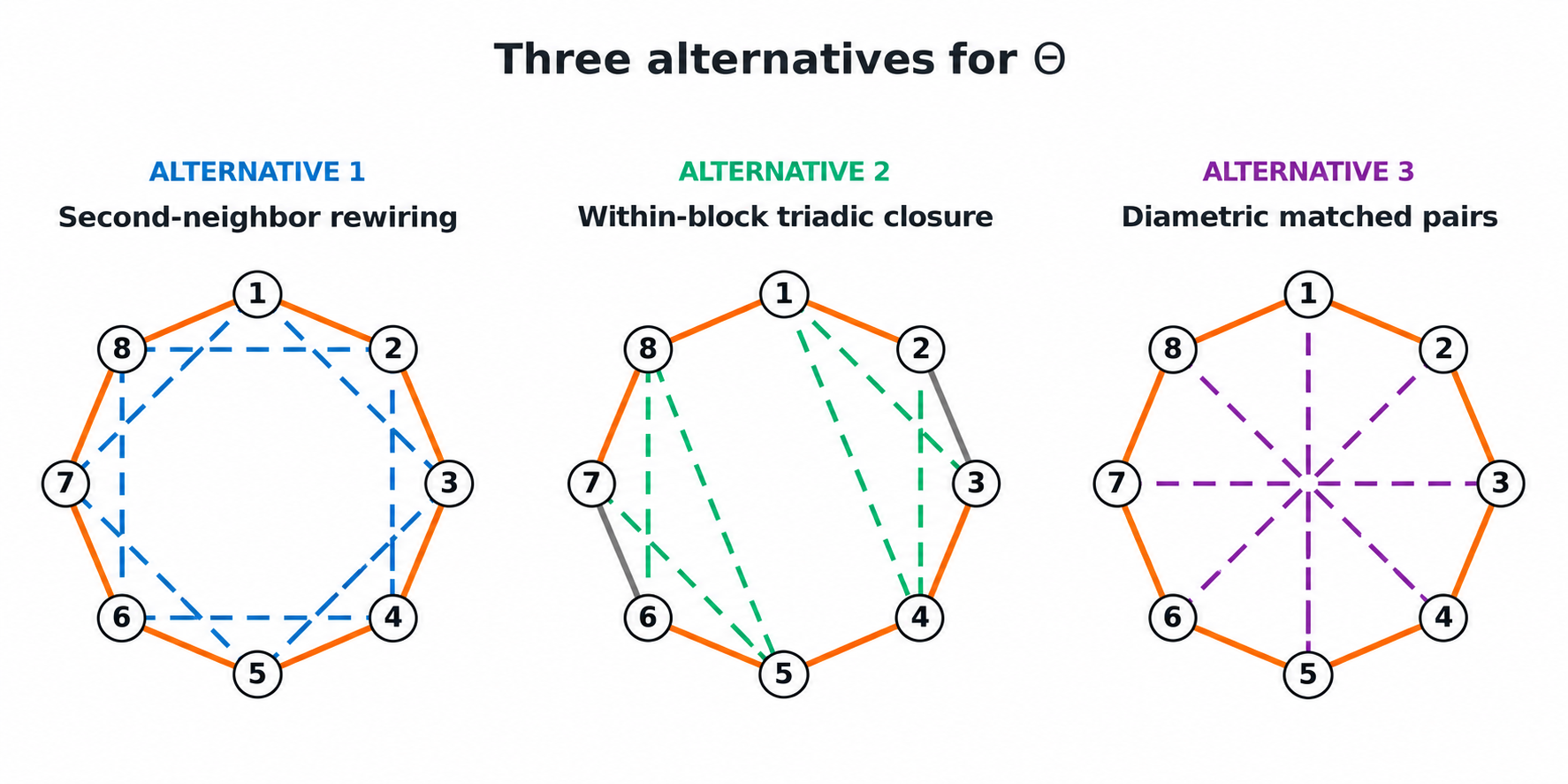}
  \caption{Structural alternatives to the conditional Gaussian null ring. Orange solid edges show weakened ring interactions. Gray solid edges are ring interactions that remain unchanged in the triadic-closure path. Colored dashed edges show newly added interactions. Their weights are $\delta$ in the first two paths and $2\delta$ in the third.}
  \label{fig:conditional-gaussian-edge-patterns}
\end{figure}

\paragraph{Second-neighbor rewiring.}
The first path transfers strength $\delta$ from every ring edge to every second-neighbor edge:
\begin{equation}
\label{eq:conditional-gaussian-rewiring}
  a^{(1)}_{i,i+1}(\delta)=w_i-\delta,
  \qquad
  a^{(1)}_{i,i+2}(\delta)=\delta,
  \qquad i=1,\ldots,8.
\end{equation}
Each node loses $2\delta$ across its two incident ring edges and gains $2\delta$ across its two incident second-neighbor edges.

\paragraph{Within-block triadic closure.}
The second path partitions the nodes into blocks $\{1,2,3,4\}$ and $\{5,6,7,8\}$. Define the donor ring edges and recipient within-block edges by
\begin{align*}
  D&=\{(1,2),(3,4),(4,5),(5,6),(7,8),(8,1)\},\\
  C&=\{(1,3),(2,4),(1,4),(5,7),(6,8),(5,8)\}.
\end{align*}
The edge weights are
\begin{equation}
\label{eq:conditional-gaussian-triadic-closure}
  a^{(2)}_{ij}(\delta)=
  \begin{cases}
    a^{(0)}_{ij}-\delta, & (i,j)\in D,\\
    \delta, & (i,j)\in C,\\
    a^{(0)}_{ij}, & \text{otherwise}.
  \end{cases}
\end{equation}
The added edges complete triangles within each block, while the two ring edges joining the blocks are weakened. The donor and recipient incidence counts both equal $(2,1,1,2,2,1,1,2)$, so every node loses and gains the same total weight.

\paragraph{Diametric matched pairs.}
The third path weakens all eight ring edges and adds four edges joining diametrically opposite nodes:
\begin{equation}
\label{eq:conditional-gaussian-matching}
  \begin{aligned}
    a^{(3)}_{i,i+1}(\delta)&=w_i-\delta,
    &&i=1,\ldots,8,\\
    a^{(3)}_{i,i+4}(\delta)&=2\delta,
    &&i=1,\ldots,4.
  \end{aligned}
\end{equation}
The factor $2$ compensates for each node losing weight $\delta$ on two ring edges while gaining only one matched edge.

\paragraph{Parameter setup.}
At $\delta=0$ all three paths give the same null distribution. For $\delta>0$ they introduce interactions outside the ring while keeping every node's strength and the total interaction mass fixed. The simulation uses $\delta\in\{0,0.25,0.5,0.75,1,1.25,1.5\}$ for second-neighbor rewiring and $\delta\in\{0,0.5,1,1.5,2,2.5\}$ for the other two paths. 

\subsection{Restricted Boltzmann machine design and implementation}
\label{sec:rbm-simulation-details}

\subsubsection{Null distribution and estimation}

\paragraph{Null data-generating procedure.}
For the experiment in Section~\ref{sec:simulation-rbm}, let $d=30$ and $m=6$. Set $U_{ij}=\xi_{ij}/\sqrt d$, where the $\xi_{ij}$ are independent symmetric signs. Let the $\alpha_{0j}$ be independent $1.5\,\operatorname{Unif}(0.75,1.10)$ variates and set $D_0=\operatorname{diag}(\alpha_{01},\ldots,\alpha_{06})$ and $W_0=UD_0$. Let $Z\in\{-1,+1\}^6$ have probabilities
\[
  \pi_0(z)=
  \frac{\exp\{\lVert W_0z\rVert^2/2\}}
       {\sum_{\tilde z\in\{-1,+1\}^6}\exp\{\lVert W_0\tilde z\rVert^2/2\}}.
\]
Independently, let $\varepsilon\sim\N(0,I_d)$. The common data-generating null is $W_0Z+\varepsilon$.

\paragraph{Score-matching estimator.}
Write $a_j=e^{\rho_j}$ and $U=(u_1,\ldots,u_6)$. Marginalizing the hidden variables gives the score function of the visible distribution
\[
  s_\rho(x)=-x+\sum_{j=1}^6a_j\tanh(a_ju_j^\top x)u_j.
\]
Since $\lVert u_j\rVert=1$, the Hyv\"arinen loss is
\[
  \ell_\rho(x)=\frac12\lVert s_\rho(x)\rVert^2-d
  +\sum_{j=1}^6a_j^2\{1-\tanh^2(a_ju_j^\top x)\}.
\]
We minimize $L_n(\rho)=n^{-1}\sum_i\ell_\rho(X_i)$ using L-BFGS-B with analytic gradients and bounds $\log(0.05)\leq\rho_j\leq\log4$. The initial value is $\rho_j=\log(0.9)$. 

\subsubsection{Alternative distributions}

Let $V\in\R^{d\times m}$ and a unit vector $v\in\R^d$ satisfy
\[
  U^\top V=0,\qquad V^\top V=U^\top U,\qquad U^\top v=0.
\]
We obtain them by orthogonalizing and scaling independent Gaussian draws, then keep them fixed across replications. 

\paragraph{Rotation.}
The rotation path is
\[
  X_{\mathrm{rot}}(s)=(\sqrt{1-s^2}\,U+sV)D_0Z+\varepsilon.
\]
Rotation preserves the loading Gram matrix and hence $\pi_0$ and the normalizing constant. The simulation uses $s_{\mathrm{rot}}\in\{0,.069,.138,.207,.276,.345\}$.

\paragraph{Omitted hidden unit.}
Let $H$ be a symmetric sign independent of $Z$ and $\varepsilon$. The omitted-unit path is
\[
  X_{\mathrm{omit}}(s)=W_0Z+s\,vH+\varepsilon.
\]
Orthogonality of $v$ to $U$ makes the added hidden sign independent of $Z$. The simulation uses $s_{\mathrm{omit}}\in\{0,.302,.604,.906,1.208,1.510\}$.

\paragraph{Heavy-tailed noise.}
Take $T\sim\chi^2_\nu$ independently and define
\[
  X_{\mathrm{tail}}(s)=W_0Z+\sqrt{\frac{\nu-2}{T}}\,\varepsilon,
  \qquad \nu=2+\frac1s,\quad s>0.
\]
The Student noise has covariance $I_d$, so the tail path preserves the null covariance. The simulation uses $s_{\mathrm{tail}}\in\{0,.25,.50,.75,1.00,1.25\}$.

\paragraph{Parameter setup.}
At $s=0$, all three paths use the exact null $W_0Z+\varepsilon$. Sampling enumerates all hidden states and their probabilities, followed by the indicated independent noise draw.

\section{Goodness-of-fit testing on the sphere}
\label{sec:simulation-sphere}

Spherical Stein discrepancies have been used for directional goodness-of-fit testing~\citep{xu2020stein}. Here we construct a spherical counterpart of the SKSD test and illustrate its influence-adjusted calibration on a curved support. Section~\ref{sec:sphere-test-construction} gives the spherical Stein operator and test statistic. Section~\ref{sec:sphere-experiment} presents the Watson distribution example, including estimation, influence-adjusted calibration, and simulation results. Proofs are collected in Section~\ref{sec:sphere-proofs}. We defer the extension of the asymptotic theory in the main text to general manifolds to future work.

\subsection{Construction of the test}
\label{sec:sphere-test-construction}

\paragraph{Spherical Stein operator.}
Let $\mathbb S^2=\{x\in\R^3:\lVert x\rVert=1\}$, and let $p_\theta$ be a strictly positive, continuously differentiable density relative to surface area measure $\sigma$. Write $P_x=I_3-xx^\top$ for the orthogonal projection onto the tangent plane $T_x\mathbb S^2=\{v\in\R^3:x^\top v=0\}$. 
For any ambient vector $v$, $P_xv=v-(x^\top v)x$ removes its radial component and retains its tangential component. Thus, $P_xf(x)$ gives a direction of infinitesimal motion along the sphere. The spherical score function is $s_\theta(x)=\nabla_{\mathbb S^2}\log p_\theta(x)$, where $\nabla_{\mathbb S^2}g(x)=P_x\nabla g(x)$ for any smooth ambient extension of $g$. This projects the ambient gradient onto directions along the surface. For a smooth ambient vector field $f$, let $Df(x)=[\partial f_i(x)/\partial x_j]_{i,j=1}^3$ denote its ordinary Jacobian. If $V:\mathbb S^2\to\R^3$ is continuously differentiable and $V(x)\in T_x\mathbb S^2$ for every $x$, define $\operatorname{div}_{\mathbb S^2}V(x)=\tr(P_xDV(x))$. Applying the spherical divergence to the tangent field $p_\theta P_xf$ gives
\begin{align}\label{eq:sphere-stein-operator}
	\mathcal A_\theta f(x)=\frac{1}{p_\theta(x)}\operatorname{div}_{\mathbb S^2}\left(p_\theta(x)P_xf(x)\right)=\tr(P_xDf(x))+\left(s_\theta(x)-2x\right)^\top f(x).
\end{align}
The surface divergence theorem below yields the Stein identity in Lemma~\ref{lemma:sphere-stein-identity}.

\paragraph{Surface divergence theorem.}
Let $\Omega\subseteq\mathbb S^2$ be a region with smooth boundary, and let $V:\overline{\Omega}\to\R^3$ be a continuously differentiable vector field up to the boundary such that $V(x)\in T_x\mathbb S^2=\{v\in\R^3:x^\top v=0\}$ for every $x\in\overline{\Omega}$. The surface divergence theorem states that
\begin{equation}\label{eq:sphere-surface-divergence-theorem}
  \int_\Omega\operatorname{div}_{\mathbb S^2}V\,\d\sigma
  =\int_{\partial\Omega}V^\top\nu_\Omega\,\d\ell,
\end{equation}
where $\d\ell$ is arc length and $\nu_\Omega$ is the outward unit co-normal: it lies in the tangent plane of the sphere, is perpendicular to the boundary curve, and points out of $\Omega$. Thus, total divergence over a surface region equals outward flow across its edge. Here $\partial\Omega$ denotes the boundary within the sphere. Taking $\Omega=\mathbb S^2$ gives $\partial\Omega=\varnothing$ and hence
\begin{equation}\label{eq:sphere-integrated-divergence}
  \int_{\mathbb S^2}\operatorname{div}_{\mathbb S^2}V\,\d\sigma=0.
\end{equation}
Now we present the spherical Stein identity, which is the basis for the spherical SKSD test.

Let $C^1(\mathbb S^2)$ denote real-valued functions on the sphere whose first derivatives are continuous in smooth local coordinates. The strictly positive $C^1$ densities with respect to $\sigma$ form the class
\[
  \mathcal P^1_+(\mathbb S^2,\sigma)
  =\left\{p:\mathbb S^2\to(0,\infty):p\in C^1(\mathbb S^2),\ \int_{\mathbb S^2}p(x)\,\d\sigma(x)=1\right\}.
\]
For an open set $U\subseteq\R^3$, $C^1(U;\R^3)$ denotes maps $f:U\to\R^3$ whose components have continuous first partial derivatives.

\begin{lemma}[Spherical Stein identity]\label{lemma:sphere-stein-identity}
Let $p_\theta\in\mathcal P^1_+(\mathbb S^2,\sigma)$. Let $U \subseteq \R^3$ be an open neighborhood of $\mathbb S^2$ and $f\in C^1(U;\R^3)$. Set $P_x=I_3-xx^\top$ and $s_\theta=\nabla_{\mathbb S^2}\log p_\theta$. Then the operator
\[
  \mathcal A_\theta f(x)
  =\tr(P_xDf(x))+(s_\theta(x)-2x)^\top f(x)
\]
admits the divergence representation in Eqn.~\eqref{eq:sphere-stein-operator} and satisfies
\[
  \E_\theta[\mathcal A_\theta f(X)]
  =\int_{\mathbb S^2}\mathcal A_\theta f(x)p_\theta(x)\,\d\sigma(x)=0.
\]
\end{lemma}

The proof, including the derivation of Eqn.~\eqref{eq:sphere-stein-operator}, is given in Section~\ref{sec:sphere-proofs}.

\paragraph{Stein feature and test statistic.}
Let $\mH$ be the RKHS induced by kernel $K_h(x,z)$ restricted to $\mathbb S^2$, with hyperparameter $h>0$. For $f\in\mH^3$, the reproducing property gives $\mathcal A_\theta f(x)=\langle f,\xi_\theta^{\mathrm{sphere}}(x)\rangle_{\mH^3}$, with spherical Stein feature
\begin{equation}\label{eq:sphere-stein-feature}
  \xi_\theta^{\mathrm{sphere}}(x)
  =K_h(x,\cdot)(s_\theta(x)-2x)
    +P_x\nabla_xK_h(x,\cdot).
\end{equation}
The spherical Stein kernel is $h_{p_\theta}^{\mathrm{sphere}}(x,z)=\langle\xi_\theta^{\mathrm{sphere}}(x),\xi_\theta^{\mathrm{sphere}}(z)\rangle_{\mH^3}$. Given a fitted parameter $\hat\theta_n$, the spherical SKSD statistic is
\begin{equation}\label{eq:sphere-test-statistic}
  T_n=\left\|\frac1n\sum_{i=1}^n\xi_{\hat\theta_n}^{\mathrm{sphere}}(X_i)\right\|_{\mH^3}^2
      =\frac1{n^2}\sum_{i,j=1}^n h_{p_{\hat\theta_n}}^{\mathrm{sphere}}(X_i,X_j).
\end{equation}
Only the spherical score function and kernel derivatives are needed, so the model's normalizing constant need not be evaluated.

\subsection{Watson family and simulation experiment}
\label{sec:sphere-experiment}

\paragraph{Model, estimation, and influence function.} As a simple one-parameter example, consider the Watson family~\citep{watson1965equatorial}
\begin{align*}
	p_\kappa(x)\propto\exp(\kappa x_3^2),\qquad x\in\mathbb S^2,\quad \kappa\in\R.
\end{align*}
Only the concentration parameter $\kappa$ is estimated. Put $v(x)=2x_3P_xe_3$, where $e_3=(0,0,1)^\top$. Therefore, the spherical score function is $s_\kappa(x)=\kappa v(x)$. Applying Lemma~\ref{lemma:sphere-stein-identity} with $f(x)=x_3e_3$ gives the moment equation
\[
  \E\left[1-3X_3^2+2\kappa X_3^2\left(1-X_3^2\right)\right]=0.
\]
Replacing the expectation by a sample average yields the explicit estimator
\[
  \hat\kappa
  =\frac{3\sum_{i=1}^n X_{i3}^2-n}
         {2\sum_{i=1}^n X_{i3}^2(1-X_{i3}^2)}.
\]
This is also the spherical score-matching estimator~\citep{mardia2016directional}. Fitting requires only the two sums in this ratio. Then we have $\hat\kappa\convp\kappa_0$. Under the null hypothesis, the moment equation gives $\kappa_0=\kappa$. The influence function of this moment ratio is the scalar
\begin{equation}\label{eq:sphere-estimator-influence}
  I_\law(x)=\frac{3x_3^2-1-2\kappa_0 x_3^2(1-x_3^2)}{D}
\end{equation}
where $D=2\E[X_3^2(1-X_3^2)]>0$ and $\kappa_0=(3\E[X_3^2]-1)/D$. It satisfies $\E_\law[I_\law(X)]=0$. For the empirical distribution, define
\[
  \hat D=\frac{2}{n}\sum_{i=1}^n X_{i3}^2(1-X_{i3}^2).
\]
On $\hat D>0$, replacing $(\kappa_0,D)$ by $(\hat\kappa,\hat D)$ gives
the plug-in influence estimates
\[
  \widehat I_i=I_{\hat\law_n}(X_i)
    =\frac{3X_{i3}^2-1-2\hat\kappa X_{i3}^2(1-X_{i3}^2)}{\hat D}.
\]
The moment equation and influence function are derived in Section~\ref{sec:sphere-proofs}.

\paragraph{Influence-adjusted calibration.} We will use the Gaussian kernel $K_h(x,z)=\exp(-\lVert x-z\rVert^2/(2h^2))$. In this setting, the influence-adjusted
Stein feature in Eqn.~\eqref{eq:influence-adjusted-stein-feature} reduces to
\begin{equation}\label{eq:sphere-adjusted-feature}
  \begin{aligned}
    [\Psi_{\kappa_0,\law}(x)](z)
      &=K_h(x,z)(\kappa_0 v(x)-2x+h^{-2}P_xz)+I_\law(x)\E_{Y\sim\law}[K_h(Y,z)v(Y)].
  \end{aligned}
\end{equation}
The second term adjusts for estimating the single nuisance parameter. 
Its derivation is given in Section~\ref{sec:sphere-proofs}.
The empirical feature $\widehat\Psi_i$ is obtained by replacing $\law$ by $\hat\law_n$ and $\kappa_0$ by $\hat\kappa$, setting $x=X_i$, and taking the expectation over $\hat{\law}_n$. 
Form $\widehat G_{ij}=\langle\widehat\Psi_i,\widehat\Psi_j\rangle_{\mH^3}$ as in Eqn.~\eqref{eq:ia-bootstrap-gram-matrix}. 
Then, for each bootstrap draw, generate the weight vector $\bm w^{(b)}$ as in Section~\ref{sec:influence-adjusted-bootstrap} and compute
\[
  \widetilde T_{n,\mathrm{IA}}^{(b)}
    =\frac{1}{n^2}(\bm w^{(b)})^\top\widehat G\bm w^{(b)}.
\]
The nominal level-$\alpha$ test rejects when the $p$-value defined in Algorithm~\ref{alg:influence-adjusted-SKSD} is at most $\alpha$.

\paragraph{Experiment setup.}
We test this Watson family against the Bingham alternatives~\citep{bingham1974antipodally}:
\[
  q_{\kappa_0,r}(x)\propto\exp\left(\kappa_0x_3^2+\kappa_0r x_1^2\right),
  \qquad \kappa_0\in\{4,16\}.
\]
We take $r\in\{0,0.15,0.30,0.45,0.60,0.75\}$. This contrast is motivated by Watson--Bingham models of neurite orientation dispersion~\citep{tariq2016bingham}. Our observations are independent spherical orientations. We consider $n=100$, $B=499$ with $1000$ replications. Independent rejection sampling is used for data generation and fitted-model draws.

\paragraph{Methods and bandwidth.}
We compare three KSD-based tests: (1) SKSD--influence, which uses influence-adjusted multiplier-bootstrap calibration; (2) SKSD--parametric, which refits $\kappa$ in a parametric bootstrap; and (3) the empirical-refitting procedure of \citet{brueck2025composite}. We extend the last procedure to the sphere using the spherical Stein operator in Lemma~\ref{lemma:sphere-stein-identity}. All three methods use the scalar moment estimator above and the Gaussian kernel in Eqn.~\eqref{eq:sphere-stein-feature}, with $h=1$ fixed throughout.

\paragraph{Simulation results.}
Figure~\ref{fig:simulation-sphere} reports rejection rates at level $0.05$. For $\kappa_0=4$ and $16$, respectively, null rejection rates are $0.044$ and $0.051$ for SKSD--influence, $0.041$ and $0.042$ for SKSD--parametric, and $0.036$ and $0.041$ for Br\"uck (midpoint). The two SKSD methods have similar power and exceed Br\"uck's observed rejection rates across the alternatives considered. At $(\kappa_0,r)=(4,0.60)$, their rates are $0.433$, $0.416$, and $0.385$. At $(16,0.75)$, they are $0.461$, $0.446$, and $0.381$. Including the common fit and statistic construction, pooled mean runtimes are $0.0098$, $0.4386$, and $0.0285$ seconds per dataset, respectively. 

\begin{figure}[H]
  \centering
  \makebox[\textwidth][c]{\includegraphics[width=\textwidth]{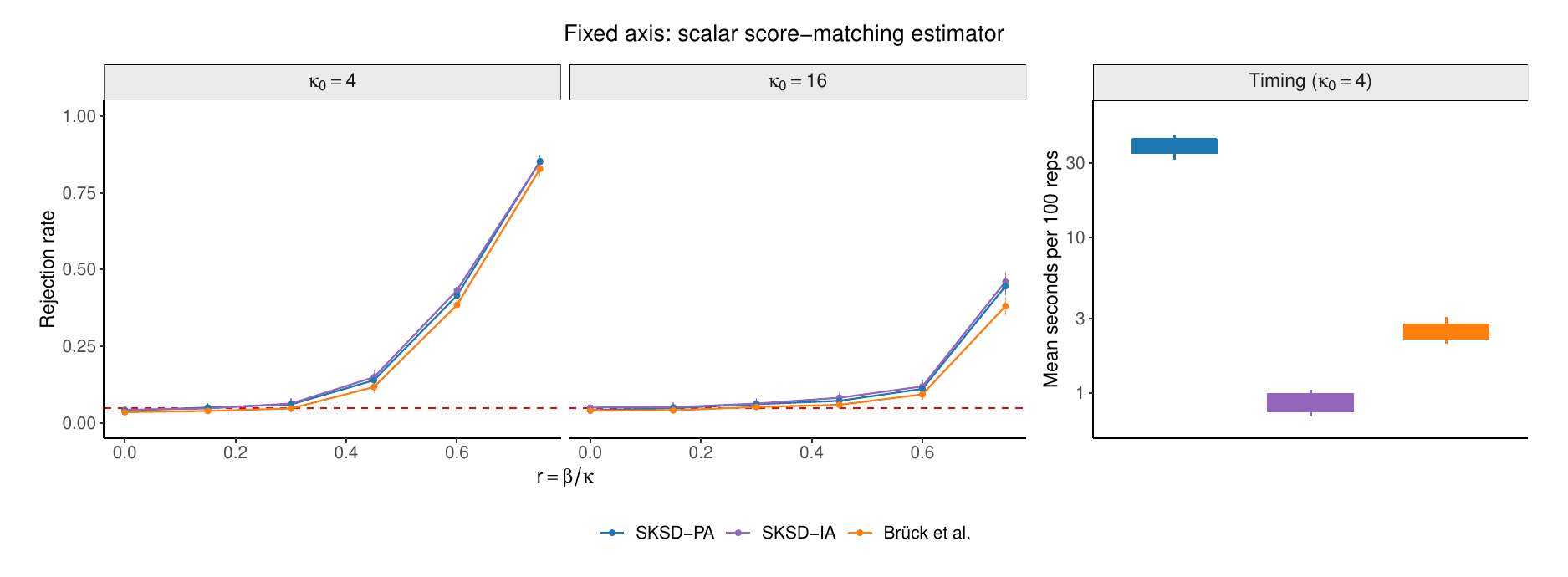}}
  \caption{Goodness-of-fit tests for the Watson family against Bingham dispersion alternatives on $\mathbb{S}^2$. IA and PA denote influence-adjusted and parametric-bootstrap calibration. Each panel uses $n=100$, $B=499$, and $1{,}000$ Monte Carlo replications per signal. The dashed line marks the nominal level $0.05$. The rightmost panel shows cell mean runtimes across the six signals for $\kappa_0=4$, scaled to $100$ replications and displayed on a logarithmic axis.}
  \label{fig:simulation-sphere}
\end{figure}

In summary, this example illustrates influence-adjusted calibration on a simple manifold support using a single nuisance parameter with an explicit estimator and influence function.

\subsection{Proofs for results in Appendices~\ref{sec:sphere-test-construction} and~\ref{sec:sphere-experiment}}
\label{sec:sphere-proofs}

\begin{proof}[Proof of Lemma~\ref{lemma:sphere-stein-identity}]
For a continuously differentiable tangent vector field $g$, let $e_1,e_2$ be an orthonormal basis of $T_x\mathbb S^2$. Then
\[
  \operatorname{div}_{\mathbb S^2}g(x)
  =\sum_{a=1}^2 e_a^\top Dg(x)e_a
  =\tr(P_xDg(x)).
\]
Put $g(x)=P_xf(x)=f(x)-xc(x)$, where $c(x)=x^\top f(x)$. Since $\nabla c(x)=f(x)+Df(x)^\top x$, the product rule gives
\[
  \begin{aligned}
    Dg(x)=Df(x)-c(x)I_3-x\nabla c(x)^\top=P_xDf(x)-\bigl(x^\top f(x)\bigr)I_3-xf(x)^\top.
  \end{aligned}
\]
Using $P_x^2=P_x$, $P_xx=0$, and $\tr(P_x)=2$, we obtain
\[
  \begin{aligned}
    \operatorname{div}_{\mathbb S^2}(P_xf)
      &=\tr(P_x^2Df(x))
        -\bigl(x^\top f(x)\bigr)\tr(P_x)-\tr(P_xxf(x)^\top)\\
      &=\tr(P_xDf(x))-2x^\top f(x).
  \end{aligned}
\]
Finally, applying the surface product rule once more yields
\[
  \begin{aligned}
    \frac{1}{p_\theta}\operatorname{div}_{\mathbb S^2}(p_\theta P_xf)
      &=\operatorname{div}_{\mathbb S^2}(P_xf)
        +(\nabla_{\mathbb S^2}\log p_\theta)^\top P_xf=\tr(P_xDf)+(s_\theta-2x)^\top f,
  \end{aligned}
\]
where $s_\theta^\top P_xf=s_\theta^\top f$ because $s_\theta$ is tangent. This proves Eqn.~\eqref{eq:sphere-stein-operator}.

The field $V(x)=p_\theta(x)P_xf(x)$ is continuously differentiable and tangent. Applying Eqn.~\eqref{eq:sphere-integrated-divergence} to this field gives
\[
  \begin{aligned}
    \E_\theta[\mathcal A_\theta f(X)]
      &=\int_{\mathbb S^2}
        \frac{\operatorname{div}_{\mathbb S^2}\bigl(p_\theta(x)P_xf(x)\bigr)}
             {p_\theta(x)}\,p_\theta(x)\,\d\sigma(x)\\
      &=\int_{\mathbb S^2}
        \operatorname{div}_{\mathbb S^2}\bigl(p_\theta(x)P_xf(x)\bigr)\,\d\sigma(x)
        =0.
  \end{aligned}
\]
The density $p_\theta$ cancels its reciprocal, and the boundary term vanishes because $\partial\mathbb S^2=\varnothing$. All integrals are finite under the stated regularity assumptions, since the sphere is compact.
\end{proof}

\paragraph{Derivation of the moment estimator and its influence function.}
For $f(x)=x_3e_3$, we have $Df(x)=e_3e_3^\top$ and $e_3^\top P_xe_3=1-x_3^2$. Substituting into Eqn.~\eqref{eq:sphere-stein-operator} gives
\[
  \mathcal A_\kappa f(x)
    =1-x_3^2+2\kappa x_3^2\bigl(1-x_3^2\bigr)-2x_3^2
    =1-3x_3^2+2\kappa x_3^2\bigl(1-x_3^2\bigr).
\]
Its expectation is zero by Lemma~\ref{lemma:sphere-stein-identity}. Solving its empirical version gives $\hat\kappa$. Equivalently, since $\lVert\nabla_{\mathbb S^2}(x_3^2)\rVert^2=4x_3^2(1-x_3^2)$ and $\Delta_{\mathbb S^2}(x_3^2)=2-6x_3^2$, the empirical spherical score-matching objective is
\begin{align*}
	L_n(\kappa)=\frac{2\kappa^2}{n}\sum_{i=1}^n X_{i3}^2(1-X_{i3}^2)+\left(2-\frac{6}{n}\sum_{i=1}^n X_{i3}^2\right)\kappa.
\end{align*}
Here $\Delta_{\mathbb S^2}=\operatorname{div}_{\mathbb S^2}\nabla_{\mathbb S^2}$ is the surface Laplacian. The minimizer of $L_n$ is the same moment ratio.

For the perturbation model $\law_\varepsilon=(1-\varepsilon)\law+\varepsilon\delta_x$, differentiation of the moment ratio yields
\begin{align*}
	\left.\frac{\d}{\d\varepsilon}\frac{3\E_{\law_\varepsilon}[X_3^2]-1}{2\E_{\law_\varepsilon}[X_3^2(1-X_3^2)]}\right|_{\varepsilon=0}=\frac{3\left(x_3^2-\E_\law[X_3^2]\right)-\kappa_0\left(2x_3^2(1-x_3^2)-D\right)}{D}=I_\law(x),
\end{align*}
where the last equality uses $\kappa_0 D=3\E_\law[X_3^2]-1$. This also shows $\E_\law[I_\law(X)]=0$.

\paragraph{Derivation of the adjusted feature.} We have $P_x\nabla_xK_h(x,z)=h^{-2}K_h(x,z)P_xz$ for the Gaussian kernel. Since $s_\kappa(y)=\kappa v(y)$, differentiation of Eqn.~\eqref{eq:sphere-stein-feature} gives $[\partial_\kappa\xi_\kappa^{\mathrm{sphere}}(y)](z)=K_h(y,z)v(y)$. Evaluating this derivative at $\kappa=\kappa_0$ and substituting Eqn.~\eqref{eq:sphere-estimator-influence} into the influence-adjusted feature construction gives Eqn.~\eqref{eq:sphere-adjusted-feature}.

\section{Data preprocessing for real-data analysis}
\label{sec:tcga-data-preprocessing}

The observational unit is one primary bulk tumor. The original data include 360 unique participants in The Cancer Genome Atlas lung adenocarcinoma cohort with matched RNA-derived immune scores and reverse-phase protein-array measurements~\citep{akbani2014pancancer,ha2018personalized}. We retain ten total-protein measurements: EGFR, HER2, HER3, Src, Akt, PTEN, Tuberin/TSC2, mTOR, S6, and 4E-BP1. Tumors are ranked by log T-cell signature~\citep{nirmal2018immune}. The groups with the lowest and highest signature values ($n_L=n_H=120$) are then used for fitting and testing.

\section{Additional details on the methodology}\label{sec:additional_details}

\subsection{Hilbert-space foundations and closed-form KSD}\label{sec:close-form}

\paragraph{Pointwise kernel derivatives and Stein features.}
Let $\psi(x)=K(x,\cdot)$. We write $s_\theta=s_{p_\theta}$ and
$\A_\theta=\mathcal A_{p_\theta}$ for the score function and Stein
operator of the model density $p_\theta$. For fixed $x\in\X$, define
$\partial_{x_r}K(x,\cdot)$ by the ordinary partial derivative of the
scalar kernel in its first argument. The smoothness
$K\in\mathcal C^{1,1}(\X,\X)$ in Assumption~\ref{ass:k} ensures that these
functions belong to $\mH$ and satisfy, for $f\in\mH$ and $r,s\in[d]$,
the derivative reproducing identities
\citep[Lemma~4.34 and Corollary~4.36]{steinwart2008support}:
\begin{align}\label{eq:derivative-reproducing-identities}
	\left\langle f,\partial_{x_r}K(x,\cdot)\right\rangle_{\mH}
	&=\partial_{x_r}f(x),\notag\\
	\left\langle\partial_{x_r}K(x,\cdot),
	\partial_{y_s}K(y,\cdot)\right\rangle_{\mH}
	&=\partial_{x_r}\partial_{y_s}K(x,y).
\end{align}
Thus $\nabla_x\psi(x)$ in Eqn.~\eqref{eq:stein-feature} is the vector of
pointwise kernel derivatives
\begin{align*}
	[\nabla_x\psi(x)](y)
	&=\left(\partial_{x_1}K(x,y),\ldots,\partial_{x_d}K(x,y)\right)^\top,\\
	\|\nabla_x\psi(x)\|_{\mH^d}^2
	&=\left.\sum_{r=1}^d
	\partial_{x_r}\partial_{y_r}K(x,y)\right|_{y=x}\le C_K.
\end{align*}
Consequently, for $f=(f_1,\ldots,f_d)^\top\in\mH^d$,
\begin{align}\label{eq:stein-feature-reproducing}
	\left\langle f,\xi_\theta(x)\right\rangle_{\mH^d}
	&=s_\theta(x)^\top f(x)+\sum_{r=1}^d\partial_{x_r}f_r(x)
	=\A_\theta f(x).
\end{align}
In particular, $\xi_\theta(x)$ is a well-defined element of $\mH^d$.

The parameter derivatives can also be computed pointwise. Since $K$
does not depend on $\theta$, for $x,y\in\X$ and $a\in\R^k$,
\begin{align}\label{eq:stein-feature-parameter-derivative}
	\left.\frac{\d}{\d t}[\xi_{\theta+ta}(x)](y)\right|_{t=0}
	&=\left(\nabla_\theta s_\theta(x)a\right)K(x,y).
\end{align}
The adjustment in Eqn.~\eqref{eq:influence-adjusted-stein-feature}
is obtained by evaluating the right-hand side at $X'\sim\law$,
holding $a=I_\law(x)$ fixed, and taking expectations over $X'$.

\paragraph{Covariance operators and their spectra.}\label{par:covariance-spectrum}
Expectations of features are Bochner expectations in the separable
Hilbert space $\mH^d$. Under the null, their integrability follows
from the bound, for $a\in\R^k$,
\begin{align*}
	\E_{X\sim\law_{\theta_0}}\left[
	\|\left(\nabla_\theta s_{\theta_0}(X)a\right)\psi(X)\|_{\mH^d}
	\right]
	&\le C_K^{1/2}\|a\|
	\E_{X\sim\law_{\theta_0}}\left[
	\|\nabla_\theta s_{\theta_0}(X)\|_{\mathrm{op}}
	\right]<\infty.
\end{align*}
Moreover,
\begin{align*}
	\|\xi_{\theta_0}(x)\|_{\mH^d}^2
	&\le2C_K\|s_{\theta_0}(x)\|^2
	+2\|\nabla_x\psi(x)\|_{\mH^d}^2,\\
	\left\|\E_{X'\sim\law_{\theta_0}}\left[
	\nabla_\theta s_{\theta_0}(X')I_{\law_{\theta_0}}(x)\psi(X')
	\right]\right\|_{\mH^d}
	&\le C_K^{1/2}
	\E_{X'\sim\law_{\theta_0}}\left[
	\|\nabla_\theta s_{\theta_0}(X')\|_{\mathrm{op}}
	\right]
	\|I_{\law_{\theta_0}}(x)\|.
\end{align*}
Therefore, Assumption~\ref{assu:ale} and Eqns.~\eqref{eq:model-density-moments}--\eqref{eq:kernel-regularity} imply
\begin{align}\label{eq:adjusted-feature-second-moment}
	\E_{X\sim\law_{\theta_0}}\left[
	\|\Psi_{\theta_0}(X)\|_{\mH^d}^2
	\right]<\infty.
\end{align}
Under the null, the Stein identity and $\E_{X\sim\law_{\theta_0}}[I_{\law_{\theta_0}}(X)]=0$ further give
\begin{align}\label{eq:adjusted-feature-centering}
	\E_{X\sim\law_{\theta_0}}\left[\Psi_{\theta_0}(X)\right]
	&=\E_{X\sim\law_{\theta_0}}\left[\xi_{\theta_0}(X)\right]\notag\\
	&\quad+\E_{X'\sim\law_{\theta_0}}\left[
	\nabla_\theta s_{\theta_0}(X')
	\E_{X\sim\law_{\theta_0}}\left[I_{\law_{\theta_0}}(X)\right]
	\psi(X')
	\right]=0.
\end{align}
Hence $\Sigma_{\theta_0}$ in Theorem~\ref{thm:asymptotic_distribution} is the covariance operator of $\Psi_{\theta_0}(X)$.

The expectation defining $\Sigma_{\theta_0}f$ exists in $\mH^d$ because
Eqn.~\eqref{eq:adjusted-feature-second-moment} gives
\begin{align*}
	\E\left[\left\|\langle f,\Psi_{\theta_0}(X)\rangle_{\mH^d}
	\Psi_{\theta_0}(X)\right\|_{\mH^d}\right]
	&\le\|f\|_{\mH^d}\E[\|\Psi_{\theta_0}(X)\|_{\mH^d}^2]<\infty,
\end{align*}
where $X\sim\law_{\theta_0}$. This bound also shows that
$\Sigma_{\theta_0}$ is a bounded linear operator. For $f,g\in\mH^d$,
\begin{align*}
	\left\langle g,\Sigma_{\theta_0}f\right\rangle_{\mH^d}
	&=\E\left[\langle f,\Psi_{\theta_0}(X)\rangle_{\mH^d}
	\langle g,\Psi_{\theta_0}(X)\rangle_{\mH^d}\right]
	=\left\langle\Sigma_{\theta_0}g,f\right\rangle_{\mH^d},\\
	\left\langle f,\Sigma_{\theta_0}f\right\rangle_{\mH^d}
	&=\E_{X\sim\law_{\theta_0}}\left[
	\left\langle f,\Psi_{\theta_0}(X)\right\rangle_{\mH^d}^2
	\right]\ge0,
\end{align*}
so $\Sigma_{\theta_0}$ is positive and self-adjoint. Throughout the Appendix, sums over an orthonormal basis terminate at the dimension of the space when it is finite, and tails beyond that dimension are empty. If $(e_j)_{j\ge1}$ is any orthonormal basis of $\mH^d$, then Tonelli's theorem and Parseval's identity yield
\begin{align}\label{eq:covariance-operator-trace}
	\tr(\Sigma_{\theta_0})
	&=\sum_{j=1}^\infty
	\left\langle e_j,\Sigma_{\theta_0}e_j\right\rangle_{\mH^d}\notag\\
	&=\E_{X\sim\law_{\theta_0}}\left[
	\sum_{j=1}^\infty
	\left\langle e_j,\Psi_{\theta_0}(X)\right\rangle_{\mH^d}^2
	\right]\notag\\
	&=\E_{X\sim\law_{\theta_0}}\left[
	\|\Psi_{\theta_0}(X)\|_{\mH^d}^2
	\right]<\infty.
\end{align}

By Eqn.~\eqref{eq:covariance-operator-trace}, $\Sigma_{\theta_0}$ is trace class and therefore compact. Let $r\in\{0,1,\ldots\}\cup\{\infty\}$ be the number of its positive eigenvalues, counted according to multiplicity. The spectral theorem for compact self-adjoint operators provides positive eigenvalues $(\lambda_u)_{u=1}^r$ and orthonormal eigenvectors $(e_u)_{u=1}^r$ satisfying
\begin{align*}
	\Sigma_{\theta_0}e_u=\lambda_ue_u,
	\qquad
	\sum_{u=1}^r\lambda_u
	=\tr(\Sigma_{\theta_0})<\infty.
\end{align*}
When $r<\infty$, set $\lambda_u=0$ for $u>r$ when writing an infinite weighted sum. For independent $Z_u\sim N(0,1)$, the series
\begin{align}\label{eq:hilbert-gaussian-expansion}
	\mathbb G_{\theta_0}
	&=\sum_{u=1}^r\sqrt{\lambda_u}Z_ue_u
\end{align}
converges in $L^2(\Omega;\mH^d)$ and almost surely. Indeed, its partial
sums form an $L^2$-bounded $\mH^d$-valued martingale because
\begin{align*}
	\E\left[
	\left\|
	\sum_{u=1}^{\min\{m,r\}}\sqrt{\lambda_u}Z_ue_u
	\right\|_{\mH^d}^2
	\right]
	&=\sum_{u=1}^{\min\{m,r\}}\lambda_u
	\le\tr(\Sigma_{\theta_0})<\infty.
\end{align*}
Moreover, for every $f\in\mH^d$,
\begin{align*}
	\left\langle f,\mathbb G_{\theta_0}\right\rangle_{\mH^d}
	&=\sum_{u=1}^r\sqrt{\lambda_u}Z_u
	\left\langle f,e_u\right\rangle_{\mH^d}\sim N\left(
	0,
	\sum_{u=1}^r\lambda_u
	\left\langle f,e_u\right\rangle_{\mH^d}^2
	\right)
	=N\left(
	0,\left\langle f,\Sigma_{\theta_0}f\right\rangle_{\mH^d}
	\right).
\end{align*}
Thus, $\mathbb G_{\theta_0}$ is a centered Gaussian random element in
$\mH^d$ with covariance operator $\Sigma_{\theta_0}$, which makes the
notation $N_{\mH^d}(0,\Sigma_{\theta_0})$ precise. Since point evaluations
are continuous linear maps on $\mH^d$, it can also be viewed as a centered
$\R^d$-valued Gaussian process indexed by $x\in\X$, with sample paths in
$\mH^d$. Finally,
\begin{align}\label{eq:hilbert-gaussian-norm}
	\|\mathbb G_{\theta_0}\|_{\mH^d}^2
	&=\sum_{u=1}^r\lambda_uZ_u^2
	=\sum_{u=1}^\infty\lambda_uZ_u^2
	\qquad\text{almost surely}.
\end{align}
This proves that the eigenvalue representation of $W$ in Theorem~\ref{thm:asymptotic_distribution} is well-defined.

Under the assumptions of Theorem~\ref{thm:influence-adjusted-bootstrap},
the same construction applies to the covariance operator
$\Sigma_{\theta_0,\law}$ of the centered feature
$\Psi_{\theta_0,\law}^{\circ}(X)$. For use in the proofs, write
\begin{align*}
	\mathbb G_{\theta_0,\law}
	&\sim N_{\mH^d}(0,\Sigma_{\theta_0,\law}),
	&
	\|\mathbb G_{\theta_0,\law}\|_{\mH^d}^2
	&\overset{d}{=}W_\law.
\end{align*}
Thus the scalar limits in the main text can also be represented as
squared norms of Gaussian random elements in the proofs.

\paragraph{Well-definedness and closed-form KSD.}

We now justify the finiteness of KSD and derive Eqn.~\eqref{eq:ksd-kernel}. Let $p$ be a positive, continuously differentiable density on an open set $\X\subseteq\R^d$, and let $K\in\mathcal C^{1,1}(\X,\X)$ be a symmetric positive definite kernel with separable RKHS $\mH$. Assume that every $g\in\mH^d$ satisfies Stein's identity with an integrable $\mathcal A_p g$ under $\law_p$. For a comparison probability law $\law_q$ on $\X$, a sufficient moment condition is
\begin{equation}\label{eq:ksd-integrability}
	\E_{X\sim\law_q}\left[
	\|s_p(X)\|\sqrt{K(X,X)}
	+\sqrt{\left.\tr\left(\nabla_x\nabla_y K(x,y)\right)\right|_{x=y=X}}
	\right]<\infty.
\end{equation}
Recall $\psi(x)=K(x,\cdot)$. For a general density $p$, write $\xi_p(x)=s_p(x)\psi(x)+\nabla_x\psi(x)$. This agrees with $\xi_\theta(x)$ in Eqn.~\eqref{eq:stein-feature} when $p=p_\theta$. Under these smoothness conditions, this feature is continuous and hence strongly measurable in the separable space $\mH^d$. The reproducing property and Eqn.~\eqref{eq:derivative-reproducing-identities} give
\begin{align}\label{eq:stein-feature-integrability-bound}
	\|\xi_p(x)\|_{\mH^d}
	&\le\|s_p(x)\psi(x)\|_{\mH^d}+\|\nabla_x\psi(x)\|_{\mH^d}\notag\\
	&=\|s_p(x)\|\sqrt{K(x,x)}
	+\sqrt{\left.\tr\left(\nabla_x\nabla_y K(x,y)\right)\right|_{y=x}}.
\end{align}
Combining Eqns.~\eqref{eq:ksd-integrability} and~\eqref{eq:stein-feature-integrability-bound} yields $\E_{X\sim\law_q}[\|\xi_p(X)\|_{\mH^d}]<\infty$, so $\E_{X\sim\law_q}[\xi_p(X)]$ exists as a Bochner integral. For every $g\in\mH^d$, Eqn.~\eqref{eq:stein-feature-reproducing} and Cauchy--Schwarz imply
\begin{align*}
	\E_{X\sim\law_q}[|\mathcal A_p g(X)|]
	&=\E_{X\sim\law_q}\left[|\langle g,\xi_p(X)\rangle_{\mH^d}|\right]
	\le\|g\|_{\mH^d}\E_{X\sim\law_q}[\|\xi_p(X)\|_{\mH^d}]<\infty.
\end{align*}
Thus $g\mapsto\E_{X\sim\law_q}[\mathcal A_p g(X)]$ is a bounded linear functional. By Eqns.~\eqref{eq:kernel_stein_discrepancy} and~\eqref{eq:stein-feature-reproducing},
\begin{align}\label{eq:sup-to-A}
	\KSD(\law_p,\law_q)
	&=\sup_{\|g\|_{\mH^d}\le1}
	\left|\left\langle g,\E_{X\sim\law_q}[\xi_p(X)]\right\rangle_{\mH^d}\right|
	=\left\|\E_{X\sim\law_q}[\xi_p(X)]\right\|_{\mH^d}<\infty.
\end{align}

To obtain the scalar kernel representation, apply the reproducing identities in Eqn.~\eqref{eq:derivative-reproducing-identities} to each coordinate:
\begin{align}\label{eq:A-to-h-kernel}
	h_p(x,y)
	&\equiv\langle\xi_p(x),\xi_p(y)\rangle_{\mH^d}\notag\\
	&=\sum_{j=1}^d\left\langle
	s_{p,j}(x)\psi(x)+\partial_{x_j}\psi(x),
	s_{p,j}(y)\psi(y)+\partial_{y_j}\psi(y)
	\right\rangle_{\mH}\notag\\
	&=s_p(x)^\top s_p(y)K(x,y)+s_p(x)^\top\nabla_y K(x,y)\notag\\
	&\quad+s_p(y)^\top\nabla_x K(x,y)+\tr\left(\nabla_x\nabla_y K(x,y)\right).
\end{align}
For independent $X,Y\sim\law_q$, this kernel is integrable because
\begin{align*}
	\E[|h_p(X,Y)|]
	&\le\E[\|\xi_p(X)\|_{\mH^d}\|\xi_p(Y)\|_{\mH^d}]
	=\left(\E_{X\sim\law_q}[\|\xi_p(X)\|_{\mH^d}]\right)^2<\infty.
\end{align*}
Consequently, Fubini's theorem and Eqns.~\eqref{eq:sup-to-A}--\eqref{eq:A-to-h-kernel} yield
\begin{align*}
	\KSD^2(\law_p,\law_q)
	&=\left\langle\E[\xi_p(X)],\E[\xi_p(Y)]\right\rangle_{\mH^d}
	=\E[\langle\xi_p(X),\xi_p(Y)\rangle_{\mH^d}]
	=\E[h_p(X,Y)],
\end{align*}
which proves the closed-form representation.

Finally, the model and kernel moment conditions used in our results imply Eqn.~\eqref{eq:ksd-integrability} at $p=p_{\theta_0}$ and $\law_q=\law$. Indeed, Eqns.~\eqref{eq:model-density-moments} and~\eqref{eq:kernel-regularity} give
\begin{align*}
	&\E_{X\sim\law}\left[
	\|s_{\theta_0}(X)\|\sqrt{K(X,X)}
	+\sqrt{\left.\tr\left(\nabla_x\nabla_y K(x,y)\right)\right|_{x=y=X}}
	\right]\\
	&\qquad\le\sqrt{C_K}\left(\E_{X\sim\law}[\|s_{\theta_0}(X)\|]+1\right)
	\le\sqrt{C_K}\left(\sqrt{\E_{X\sim\law}[\|s_{\theta_0}(X)\|^2]}+1\right)<\infty.
\end{align*}
Together with the Stein identity in Eqn.~\eqref{eq:stein-identity-assumption}, this establishes that $\KSD(\law_{\theta_0},\law)$ is finite under the stated regularity conditions.

\subsection{Computation of the V-statistic for KSD}\label{sec:computation_v_statistic}
Recall that the formula for the KSD V-statistic is given by
\begin{align*}
	T_n(X_{1:n},\theta) = \mathrm{KSD}^2(\law_{\theta},\hat\law_n) = \frac{1}{n^2}\sum_{i,j=1}^n h_{p_{\theta}}(X_i,X_j).
\end{align*}
For a general kernel, the closed-form expression above gives an $O(n^2)$ cost for summing the pairs for fixed $d$, assuming constant-cost kernel evaluations and derivatives. Including $O(k)$-cost evaluations of the score function gives $O(nk+n^2)$, as used in Appendix~\ref{sec:ia-bootstrap-matrix}.
When the linear kernel $K(x, y) = x^\top y$ is applied in the computation, we have
\begin{align*}
	h_{p_\theta}(x, y) &= s_{p_\theta}(x)^{\top}s_{p_\theta}(y)(x^\top y) + s_{p_\theta}(x)^{\top} \nabla_y (x^\top y)+ s_{p_\theta}(y)^{\top}\nabla_x (x^\top y) + \tr(\nabla_x \nabla_y x^\top y)\\
	&= s_{p_\theta}(x)^{\top}s_{p_\theta}(y)(x^\top y) + s_{p_\theta}(x)^{\top} x + s_{p_\theta}(y)^{\top} y + d\\
	&= \tr\left(\left(x s_{p_\theta}(x)^{\top} + I_d\right)^{\top}\left(y s_{p_\theta}(y)^{\top} + I_d\right)\right).
\end{align*}
By plugging the above equation into the V-statistic, we have
\begin{align*}
	T_n(X_{1:n},\theta) &= \frac{1}{n^2} \sum_{i,j=1}^n h_{p_\theta}(X_i,X_j) \\
	&= \frac{1}{n^2}\sum_{i,j=1}^n \tr\left(\left(X_i s_{p_\theta}(X_i)^{\top} + I_d\right)^{\top}\left(X_j s_{p_\theta}(X_j)^{\top} + I_d\right)\right)\\
	&= \left\|\frac{1}{n}\sum_{i=1}^n X_i s_{p_\theta}(X_i)^{\top} + I_d\right\|_F^2,
\end{align*}
which requires only $O(n)$ time after the score function has been evaluated at all observations, for fixed $d$. Including these evaluations gives $O(nk+n)$ under the same convention.

\subsection{Computational cost of the bootstrap procedures}
\label{sec:ia-bootstrap-matrix}

We give the construction of the Gram matrix $\widehat G\in\R^{n\times n}$
used in Algorithm~\ref{alg:influence-adjusted-SKSD}. Here $\nabla_rK$
denotes the gradient with respect to the $r$th argument of $K$, and
$\nabla_1\nabla_2K$ denotes the corresponding mixed derivative. For
$i,j\in[n]$, compute
\begin{align}\label{eq:ia-bootstrap-kernel-components}
	\widehat H_{ij}
	&=s_{\hat\theta_n}(X_i)^\top s_{\hat\theta_n}(X_j)K(X_i,X_j)
	+s_{\hat\theta_n}(X_i)^\top\nabla_2K(X_i,X_j)\notag\\
	&\quad+s_{\hat\theta_n}(X_j)^\top\nabla_1K(X_i,X_j)
	+\tr\left(\nabla_1\nabla_2K(X_i,X_j)\right),\notag\\
	\widehat C_i
	&=\frac1n\sum_{\ell=1}^n
	\nabla_\theta s_{\hat\theta_n}(X_\ell)^\top
	\left(s_{\hat\theta_n}(X_i)K(X_i,X_\ell)
	+\nabla_1K(X_i,X_\ell)\right),\notag\\
	\widehat Q
	&=\frac1{n^2}\sum_{\ell,m=1}^n K(X_\ell,X_m)
	\nabla_\theta s_{\hat\theta_n}(X_\ell)^\top
	\nabla_\theta s_{\hat\theta_n}(X_m).
\end{align}
Here $\widehat H_{ij}\in\R$, $\widehat C_i\in\R^k$, and
$\widehat Q\in\R^{k\times k}$. The Gram matrix has entries
\begin{align}\label{eq:ia-bootstrap-gram-matrix}
	\widehat G_{ij}
	&=\widehat H_{ij}+\widehat C_i^\top\widehat I_j
	+\widehat I_i^\top\widehat C_j
	+\widehat I_i^\top\widehat Q\widehat I_j.
\end{align}
To verify the construction, expand the inner product of the estimated
features in Eqn.~\eqref{eq:estimated-influence-adjusted-feature} using
the derivative reproducing identities in Eqn.~\eqref{eq:derivative-reproducing-identities}.
The unadjusted term gives $\widehat H_{ij}$, the two cross terms give
$\widehat C_i^\top\widehat I_j+\widehat I_i^\top\widehat C_j$, and the
product of the two adjustments gives $\widehat I_i^\top\widehat Q\widehat I_j$.
Thus Eqn.~\eqref{eq:ia-bootstrap-gram-matrix} gives
$\widehat G_{ij}=\langle\widehat\Psi_i,\widehat\Psi_j\rangle_{\mH^d}$,
so the construction only requires the score function, its parameter derivative,
the kernel and its derivatives, and the estimated influence functions.

Centering the multipliers is equivalent to centering the features, since
\begin{align*}
	\sum_{i=1}^n\left(w_i^{(b)}-\frac1n\sum_{\ell=1}^n w_\ell^{(b)}\right)\widehat\Psi_i
	&=\sum_{i=1}^n w_i^{(b)}\left(\widehat\Psi_i-\frac1n\sum_{\ell=1}^n\widehat\Psi_\ell\right).
\end{align*}
Consequently, for the centered vector $\bm w^{(b)}$ defined in
Algorithm~\ref{alg:influence-adjusted-SKSD},
\begin{align*}
	\frac1{n^2}(\bm w^{(b)})^\top\widehat G\bm w^{(b)}
	&=\left\|\frac1n\sum_{i=1}^n
	w_i^{(b)}\left(\widehat\Psi_i-\frac1n\sum_{\ell=1}^n\widehat\Psi_\ell\right)
	\right\|_{\mH^d}^2.
\end{align*}
This is the empirically centered bootstrap statistic in
Algorithm~\ref{alg:influence-adjusted-SKSD}.

\paragraph{Computational costs.}
We keep $d$ fixed and assume constant-cost evaluations of $K$ and its
derivatives and $O(k)$-cost evaluations of the score function and its parameter
Jacobian. These evaluation costs are model-dependent, not consequences
of the regularity conditions. We compare dense implementations using
the same sample size $n$, parameter dimension $k$, bootstrap size $B$,
and fitting accuracy. As in Section~\ref{sec:bootstrap-computation},
$C_{\mathrm{fit}}(n,k)$ bounds the cost of each fit, including bootstrap
refits. $C_{\mathrm{sample}}(n)$ is the cost of generating one fitted-model
sample of size $n$, with model and parameter dependence suppressed.
Any sampler setup must also be accounted for in this sampling cost.
The cost $C_{\text{inf-est}}(n,k)$ is additional to the original fit,
so quantities reused from fitting are not charged twice.

For Algorithm~\ref{alg:influence-adjusted-SKSD}, after fitting the model
and estimating the influence functions, evaluating the score function and
its parameter Jacobian at all $n$ observations once costs $O(nk)$.
In Eqn.~\eqref{eq:ia-bootstrap-kernel-components}, constructing
$\widehat H$ and all $\widehat C_i$ costs $O(n^2)$ and $O(n^2k)$,
respectively. The observed statistic is
$T_n=n^{-2}\sum_{i,j=1}^n\widehat H_{ij}$, so its evaluation reuses
$\widehat H$ and adds only $O(n^2)$. To compute $\widehat Q$, first multiply the $n\times n$
kernel matrix by each $n\times k$ matrix of parameter-Jacobian coordinates of the score function,
then multiply on the left by its transpose. Summing over the fixed $d$
coordinates costs $O(n^2k+nk^2)$. The products in
Eqn.~\eqref{eq:ia-bootstrap-gram-matrix} have the same order, and the
$B$ quadratic forms cost $O(Bn^2)$. Generating and centering the
multipliers costs $O(Bn)$, and counting exceedances for the reported
$p$-value costs $O(B)$. Both are absorbed into $O(Bn^2)$.
Including the one-time $O(nk)$
evaluation cost gives $O(n^2k+nk^2+Bn^2)$. Adding the one-time fit and
influence-function estimation gives the total cost
\begin{align*}
	O\bigl(C_{\mathrm{fit}}(n,k)+C_{\text{inf-est}}(n,k)+n^2k+nk^2+Bn^2\bigr).
\end{align*}
Here $C_{\text{inf-est}}(n,k)$ includes forming the estimating-equation
Jacobian and right-hand sides unless already available from fitting.
For the M/Z-estimators in Appendix~\ref{sec:general_estimators_UALE},
if each estimating function and its Jacobian cost $O(k)$ and $O(k^2)$,
respectively, forming them costs $O(nk^2)$. A dense factorization costs
$O(k^3)$, and solving for all $n$ influence vectors costs $O(nk^2)$, giving
$C_{\text{inf-est}}(n,k)=O(nk^2+k^3)$. Other estimators can have different
costs. In particular, minimum-KSD estimation involves pairwise derivatives.
For the direct implementation in Appendix~\ref{sec:general_estimators_UALE},
if evaluating a pairwise gradient and Hessian costs $O(k)$ and $O(k^2)$,
respectively, their construction costs $O(n^2k^2)$, giving
$C_{\text{inf-est}}(n,k)=O(n^2k^2+k^3)$ without reuse or further structure.

For Algorithm~\ref{alg:semiparametric-KSD}, evaluating the score function at all $n$ observations before computing
the $n^2$ pairs in Eqn.~\eqref{eq:ksd-kernel} gives a cost of
$O(nk+n^2)$ per statistic. The original data require one fit and one
statistic. Each of the $B$ replications requires sampling, refitting,
and a new statistic. Thus the total cost is
\begin{align*}
	& O\bigl((B+1)C_{\mathrm{fit}}(n,k)+BC_{\mathrm{sample}}(n)+(B+1)(nk+n^2)\bigr)\\
	&\qquad=O\bigl(B(C_{\mathrm{fit}}(n,k)+C_{\mathrm{sample}}(n)+nk+n^2)\bigr),\qquad B\geq1.
\end{align*}
Counting bootstrap exceedances again adds only $O(B)$. Hence both totals
include the original fit and statistic as well as calibration. They are
upper bounds for the stated dense implementations, not claims of optimal
complexity or a uniform computational advantage for either method.

Another no-refit approach uses a Neyman-orthogonalized kernel
\citep{chernozhukov2022locally,escanciano2024gaussian}. Its construction
and wild-bootstrap calibration are described in
Appendix~\ref{sec:neyman_orthogonal_sksd_test}.
Its cost also depends on how the model expectations and matrix products
are evaluated, so it need not have the same cost as influence adjustment.

\paragraph{Approximate sampling.}
For the parametric bootstrap, sampling from the fitted model may require a Markov chain Monte Carlo sampler.
The following bound quantifies how approximate sampling affects the asymptotic validity of the test when the approximate draws are conditionally independent.
Condition on $\mathcal F_n$ and let $\widehat{\law}_{\hat \theta_n}$ be the law of one
observation produced by an approximate fitted-model sampler. Define
\begin{align}\label{eq:approximate-sampling-tv}
	\varepsilon_n&=\left\|\widehat{\law}_{\hat\theta_n}^{\otimes n}-\law_{\hat\theta_n}^{\otimes n}\right\|_{\mathrm{TV}},
\end{align}
where $\|\law_1-\law_2\|_{\mathrm{TV}}=\sup_{A}|\law_1(A)-\law_2(A)|$ is the total variation distance, as in Remark~\ref{rem:local-alternative-rates}.
Let $\widetilde T_{n,\mathrm{PA}}^{\mathrm{app}}$ use a sample from
$\widehat{\law}_{\hat\theta_n}$ and the same refitting and statistic calculation as
$\widetilde T_{n,\mathrm{PA}}$. Applying
Eqn.~\eqref{eq:approximate-sampling-tv} and the definition of total variation norm to the events
$\left\{x:nT_n(x,\mathcal E(x))\leq z\right\}$ gives
\begin{align}\label{eq:approximate-sampling-cdf}
	\sup_{z\in\R}\left|
	\P\left[n\widetilde T_{n,\mathrm{PA}}^{\mathrm{app}}\leq z\mid\mathcal F_n\right]
	-\P\left[n\widetilde T_{n,\mathrm{PA}}\leq z\mid\mathcal F_n\right]
	\right|&\leq\varepsilon_n.
\end{align}
Thus the error in the conditional distribution is bounded by $\varepsilon_n$. As $\varepsilon_n\convp0$, this preserves the conditional limit in
Theorem~\ref{thm:bootstrap_distribution} by the triangle inequality,
and hence its null-validity conclusion under the same nondegeneracy
condition. 

To relate the bound to specific sampling algorithms, suppose each observation is the
terminal state of a Markov chain initialized and run independently for $m$
steps, with common terminal law $\widehat{\law}_{\hat\theta_n}^{(m)}$ conditional on
$\mathcal F_n$. Then the product-law coupling bound gives
\begin{align*}
	\left\|(\widehat{\law}_{\hat\theta_n}^{(m)})^{\otimes n}
	-\law_{\hat\theta_n}^{\otimes n}\right\|_{\mathrm{TV}}
	&\leq\min\left(1,n\left\|\widehat{\law}_{\hat\theta_n}^{(m)}
	-\law_{\hat\theta_n}\right\|_{\mathrm{TV}}\right).
\end{align*}
For example, if linear convergence
$\|\widehat{\law}_{\hat\theta_n}^{(m)}-\law_{\hat\theta_n}\|_{\mathrm{TV}}\leq C\rho^m$
holds uniformly over the fitted parameters, with $C<\infty$ and
$0<\rho<1$, the conditional distribution-function error is at most of order $\min(1,n\rho^m)$.
Thus $m=\lceil C_0\log n\rceil$ with $C_0>1/|\log\rho|$ makes $n\rho^m\to0$, preserving the conditional limit and validity of the parametric bootstrap test.

However, such bounds require quantitative mixing conditions
\citep{roberts2004general}, and they are not implied by our model regularity
assumptions. Common samplers, including the Metropolis-adjusted Langevin algorithm, can converge arbitrarily slowly for some target distributions, so the above bound may not hold uniformly over the fitted parameters.
Also, consecutive states from one chain are not
independent, even in stationarity and after thinning, so the marginal bound alone does not
justify Eqn.~\eqref{eq:approximate-sampling-cdf} with vanishing error.
One may instead control their joint law. For these reasons, we recommend the influence-adjusted bootstrap for such hard-to-sample models.

\subsection{Two general classes of estimators}\label{sec:general_estimators_UALE}

We describe two common estimation frameworks and give influence-function
estimates for use in Algorithm~\ref{alg:influence-adjusted-SKSD}. Throughout,
$\theta_0$ denotes the probability limit under the data-generating law
$\law$, which need not belong to the model.

\begin{itemize}
	\item \textbf{M-estimator:} For a loss function $\ell$, the population target and empirical estimator are
	\begin{align}\label{eq:m_estimator}
		\theta_0&=\arg\min_{\theta\in\Theta}\E_{X\sim\law}[\ell(\theta,X)],
		&\hat\theta_n&=\arg\min_{\theta\in\Theta}\frac1n\sum_{i=1}^n\ell(\theta,X_i).
	\end{align}
	Special instances in this class include the \textit{maximum likelihood estimator} and the \textit{maximum pseudo-likelihood estimator} for models with intractable likelihoods \citep{besag1975statistical}.
	\item \textbf{Minimum distance estimator:} For some statistical distance $D(\cdot,\cdot)$, we define the estimand and estimator, respectively, as
	\begin{align}\label{eq:mde}
		\theta_0&=\arg\min_{\theta\in\Theta}D(\law_\theta,\law),
		&\hat\theta_n&=\arg\min_{\theta\in\Theta}D(\law_\theta,\hat\law_n).
	\end{align}
	Special instances in this class include the \textit{minimum Wasserstein estimator} \citep{bernton2019parameter}, the \textit{minimum MMD estimator}~\citep{briol2019statistical} and the \textit{minimum-KSD estimator} \citep{barp2019minimum}.
\end{itemize}

\paragraph{M- and Z-estimators.}
Suppose $\hat\theta_n$ solves
$n^{-1}\sum_{i=1}^n\varphi(X_i,\hat\theta_n)=0$ for
$\varphi:\mathcal X\times\Theta\to\R^k$, with
$\E_\law[\varphi(X,\theta_0)]=0$. Under the usual local differentiability,
moment, and stochastic equicontinuity conditions for the estimating-equation
expansion \citep[Chapter~5]{van2000asymptotic}, a nonsingular Jacobian gives
\begin{align*}
	A_\law&=\E_\law[\nabla_\theta\varphi(X,\theta_0)],
	&I_\law(x)&=-A_\law^{-1}\varphi(x,\theta_0),\\
	\widehat A_n&=\frac1n\sum_{j=1}^n\nabla_\theta\varphi(X_j,\hat\theta_n),
	&\widehat I_i&=-\widehat A_n^{-1}\varphi(X_i,\hat\theta_n).
\end{align*}
For an interior M-estimator in Eqn.~\eqref{eq:m_estimator}, take
$\varphi(x,\theta)=\nabla_\theta\ell(\theta,x)$, so $A_\law$ is the
expected loss Hessian. For example, under its integration-by-parts
conditions, ordinary score matching uses \citep{hyvarinen2005estimation}
\begin{align*}
	\ell(\theta,x)&=\frac12\|s_\theta(x)\|^2+\operatorname{div}_x s_\theta(x).
\end{align*}
The empirical Hessian and loss gradients then give $\widehat I_i$ without
evaluating the model's normalizing constant.

Suppose $\E_\law[\|\varphi(X,\theta_0)\|^2]<\infty$ and $A_\law$ is
nonsingular. For this plug-in estimate,
Assumption~\ref{ass:estimated-influence-function} follows if
\begin{align*}
	\widehat A_n&\convp A_\law,
	&\frac1n\sum_{i=1}^n
	\|\varphi(X_i,\hat\theta_n)-\varphi(X_i,\theta_0)\|^2&\convp0.
\end{align*}
These conditions follow, for example, from consistency,
continuity of the gradients and Hessians, and integrable local envelopes
for the squared gradient norm and Hessian norm.

\paragraph{Minimum-KSD estimation.}
The influence function of a minimum-distance estimator depends on its
criterion. It is not generally obtained by differentiating a one-observation
loss. For the U-statistic minimum-KSD estimator
\begin{align*}
	\hat\theta_n&=\arg\min_{\theta\in\Theta}
	\frac1{n(n-1)}\sum_{i\ne j}h_{p_\theta}(X_i,X_j),
\end{align*}
let $X,X'\overset{\mathrm{i.i.d.}}{\sim}\law$ and suppose $\theta_0$ is
an interior minimizer of $\E[h_{p_\theta}(X,X')]$. Under the smoothness,
moment, and nonsingularity conditions for the corresponding U-statistic
estimating-equation expansion \citep{barp2019minimum}, write
\begin{align*}
	A_\law&=\E[\nabla_\theta^2h_{p_{\theta_0}}(X,X')],\\
	I_\law(x)&=-2A_\law^{-1}
	\E_{X'\sim\law}[\nabla_\theta h_{p_{\theta_0}}(x,X')].
\end{align*}
The factor $2$ is the first-order projection coefficient of a symmetric
U-statistic of order two. Its population gradient has mean zero at
$\theta_0$. Estimate the Hessian and the projected gradients by
\begin{align*}
	\widehat A_n&=\frac1{n(n-1)}
	\sum_{i\ne j}\nabla_\theta^2h_{p_{\hat\theta_n}}(X_i,X_j),\\
	\widehat\varphi_i&=\frac2{n-1}\sum_{j\ne i}
	\nabla_\theta h_{p_{\hat\theta_n}}(X_i,X_j),
	&\widehat I_i&=-\widehat A_n^{-1}
	\left(\widehat\varphi_i-\frac1n\sum_{j=1}^n\widehat\varphi_j\right).
\end{align*}
Here the derivatives are taken with respect to $\theta$ and then evaluated
at the indicated parameter. The empirical centering removes the fitted
criterion's gradient, which is zero at an exact interior optimum.
Consistency of $\widehat A_n$ and empirical $L^2$ consistency of these
projected gradients give Assumption~\ref{ass:estimated-influence-function}.
All expectations in both constructions are under $\law$, not the fitted
model law, so the formulas also apply under misspecification when the
stated regularity conditions hold. UALE additionally requires the
estimating-equation expansions to hold uniformly over nearby model laws.

\subsection{Nonparametric characteristicity}\label{sec:characteristicity}
The following definition and examples appear in \citet[Sections~2 and~3]{sriperumbudur2011universality}.
\begin{definition}[$C_0$-universality]\label{def:c0-universality}
Let $\X$ be a locally compact Hausdorff space and let $C_0(\X)$ denote the space of continuous real-valued functions on $\X$ that vanish at infinity, meaning that $\{x\in\X:|f(x)|\ge\varepsilon\}$ is compact for every $\varepsilon>0$.
A positive definite kernel $K$ is called \emph{$C_0$-universal} if its RKHS $\mH\subseteq C_0(\X)$ is dense in $C_0(\X)$ under the uniform norm $\|f\|_\infty=\sup_{x\in\X}|f(x)|$.
\end{definition}

\begin{example}[$C_0$-universal kernels]\label{ex:c0-universal-kernels}
On $\X=\R^d$, the following kernels are $C_0$-universal:
\begin{enumerate}
    \item \textbf{Gaussian kernel:} for $h>0$,
    \[
        K(x,y)=\exp\left(-\frac{\|x-y\|^2}{2h^2}\right).
    \]
    \item \textbf{Inverse multiquadric kernel:} for $c>0$ and $\beta>0$,
    \[
        K(x,y)=\left(c^2+\|x-y\|^2\right)^{-\beta}.
    \]
    \item \textbf{Mat\'ern kernel with $\nu=3/2$:} for $h>0$,
    \[
        K(x,y)=\left(1+\frac{\sqrt{3}\|x-y\|}{h}\right)
        \exp\left(-\frac{\sqrt{3}\|x-y\|}{h}\right).
    \]
\end{enumerate}
\end{example}

We can see from Example~\ref{ex:c0-universal-kernels} that $C_0$-universality is a weak condition, and many commonly used kernels are automatically $C_0$-universal.
The following proposition gives conditions under which KSD is \emph{characteristic} in the sense of Eqn.~\eqref{eq:characteristic_IPM}.
\begin{proposition}[Adapted from Theorem 2.2 in~\citet{chwialkowski2016kernel}]\label{prop:characteristic_propety}
	Let $p$ and $q$ be positive $C^1$ densities on a common connected open set $\X\subseteq\R^d$. Suppose $K$ is $C_0$-universal and satisfies Assumption~\ref{ass:k}, and the Stein identity in Assumption~\ref{ass:stein} holds for both $p$ and $q$ with their respective Stein operators.
	Here $X\sim\law_q$ denotes a single random observation in $\X$.
	If $\E_{X\sim\law_q}[h_p(X,X)]<\infty$ and $\E_{X\sim\law_q}[\|s_p(X)-s_q(X)\|^2]<\infty$, then $\mathrm{KSD}(\law_p,\law_q)=0$ if and only if $\law_p=\law_q$.
\end{proposition}
\begin{proof}
	Since $\xi_q(x)=\xi_p(x)+(s_q(x)-s_p(x))\psi(x)$, Eqns.~\eqref{eq:A-to-h-kernel} and~\eqref{eq:kernel-regularity} give
	\begin{align*}
		\E_{X\sim\law_q}[h_q(X,X)]
		&\leq2\E_{X\sim\law_q}[h_p(X,X)]
		+2C_K\E_{X\sim\law_q}[\|s_p(X)-s_q(X)\|^2]<\infty.
	\end{align*}
	The conclusion follows from \citet[Theorem~2.2]{chwialkowski2016kernel} by checking the conditions of their proof.
\end{proof}
For laws satisfying these conditions, Proposition~\ref{prop:characteristic_propety} ensures that the SKSD test statistic $T_n(X_{1:n},\hat\theta_n)$ can distinguish the data-generating law $\law$ from the fitted model $\law_{\hat\theta_n}$ when $\hat\theta_n\convp\theta_0$ for some $\theta_0\in\Theta$. Characteristicity is important for nonparametric goodness-of-fit testing because it supports power against any alternative law $\law\notin\{\law_\theta:\theta\in\Theta\}$ under the assumptions of Theorem~\ref{thm:asymptotic_power}.

\subsection{Connection to moment tests}

The SKSD statistic can be interpreted through the moment equation guaranteed by Stein's identity, $\E_{\law_{\hat\theta_n}}[\mathcal{A}_{p_{\hat\theta_n}}g(X)]=0$, for any test function $g$ in the RKHS $\mathcal{H}^d$. In other words, Eqn.~\eqref{eq:sksd_test_stat} defines a nonparametric \textit{moment test}. Such moment tests have been widely studied in the goodness-of-fit testing literature, especially in econometrics~\citep{white1996estimation}.

\subsection{Comparison to the approach by~\texorpdfstring{\citet{brueck2025composite}}{Brueck et al. (2025)}}\label{sec:comparison-with-brueck}

Concurrent work by~\citet{brueck2025composite} studies composite KSD testing and develops a bootstrap framework for parameter-dependent U-statistics. Both approaches account for the effect of nuisance estimation on the null distribution. Their analysis requires joint convergence of the estimator, the KSD statistic, and its parameter gradient. Our analysis starts from an asymptotically linear estimator and derives the limiting distribution using influence-adjusted Stein features. This representation makes the dependence between fitting and testing explicit and shows how the nuisance estimator affects the null distribution. We verify the required convergence through conditions on the estimator, model score, and kernel. These provide a complementary route to establishing validity, rather than a uniformly weaker set of assumptions.

The main algorithmic difference concerns repeated fitting. Their empirical bootstrap refits the nuisance estimator for each resampled dataset, whereas our influence-adjusted bootstrap uses one fit and estimated influence functions. Both avoid sampling from the fitted model, but our procedure also avoids repeated fitting, providing substantial savings when fitting is expensive and influence functions can be estimated efficiently. Both approaches establish null calibration and fixed-alternative consistency. Our analysis additionally characterizes the bootstrap limits under alternatives and derives local power, identifying which small departures remain detectable after nuisance estimation. The empirical study complements these results with direct comparisons of calibration, power, and runtime. It shows substantial computational savings and competitive power, while also illustrating that relative power depends on the alternative and finite-sample calibration.

\section{SKSD test as a nonparametric score test}\label{sec:score_ipm_examples}\label{sec:bridge_score_ipm}

We show that the SKSD test can be interpreted as a nonparametric score test, which is a generalization of classical parametric score tests. To this end, we first introduce the parametric score test under exponentially tilted models in Section~\ref{sec:tilted_null_model}, and then show in Section~\ref{sec:distance_to_score} that the SKSD test and several other IPM-based tests can be interpreted as nonparametric score tests.

\subsection{Parametric scores under exponentially tilted models}\label{sec:tilted_null_model}

For each $f$ in a class $\mathcal{F}$ of measurable functions, suppose the normalizing constant below is finite for $\gamma$ in an open interval containing zero, and consider the one-dimensional exponentially tilted submodel $\law_{(\gamma,\theta)}(\cdot;f)$ with density
\begin{align}\label{eq:nonparametric_model}
	p_{(\gamma,\theta)}(x; f)
	= \frac{\exp(\gamma f(x)) p_\theta(x)}{\int \exp(\gamma f(x)) p_\theta(x) \, \mathrm{d}x},
	\quad \theta\in\Theta.
\end{align}
At $\gamma=0$, the law is $\law_{(0,\theta)}=\law_\theta$ and does not depend on $f$. For any nuisance-parameter estimate $\hat\theta_n$, the parametric score with respect to $\gamma$ at the null model is
\begin{align}\label{eq:score_def}
	\bs_n(f; \hat\theta_n)
	\equiv \frac{\partial}{\partial \gamma} \frac{1}{n} \sum_{i=1}^n \log p_{(\gamma, \hat\theta_n)}(X_i; f) \Big|_{\gamma=0}
	= \frac{1}{n} \sum_{i=1}^n f(X_i) - \mathbb{E}_{\law_{(0, \hat\theta_n)}}[f(X)].
\end{align}
Here the bold notation $\bs_n$ distinguishes the sample-average parametric score from the score function $s_p$.
Thus a fixed $f$ probes one departure from the null, whereas aggregation over $\mathcal{F}$ gives
\begin{align}\label{eq:aggregate_score_test}
	\sup_{f \in \mathcal{F}} |\bs_n(f; \hat\theta_n)|
	= \sup_{f \in \mathcal{F}} \left| \mathbb{E}_{\hat\law_n}[f(X)]-\mathbb{E}_{\law_{(0,\hat\theta_n)}}[f(X)] \right|
	\equiv \mathrm{IPM}(\hat\law_n,\law_{(0,\hat\theta_n)};\mathcal{F}).
\end{align}
The supremum on the right-hand side of Eqn.~\eqref{eq:aggregate_score_test} is precisely the \emph{integral probability metric} (IPM) generated by the function class $\mathcal{F}$~\citep{muller1997integral}, here evaluated between the empirical measure and the fitted null model.
In other words, any IPM whose function class admits these local exponential tilts is an aggregate parametric score test for Eqn.~\eqref{eq:nonparametric_model}. Classical score tests correspond, up to parameterization-dependent scaling, to finite-dimensional, problem-specific choices of $\mathcal{F}$~\citep{rao1948large}.

\begin{example}[Testing the mean in a Gaussian location model, \citep{rao1948large}]\label{ex:gaussian_mean_shift}
	Let $\law_{(\gamma,\theta)}=\mathcal{N}(\gamma,\theta^2)$ with $\theta>0$, and write $\bar X_n=n^{-1}\sum_{i=1}^n X_i$. For testing $\gamma=0$, the parametric score is $\bs_n(\hat\theta_n)=\bar X_n/\hat\theta_n^2$. With $\mathcal{F}_{\mathrm{lin}}=\{x\mapsto ax+b:|a|\leq1,\ b\in\mathbb{R}\}$, Eqn.~\eqref{eq:aggregate_score_test} gives
	\begin{align*}
		\mathrm{IPM}(\hat\law_n,\law_{(0,\hat\theta_n)};\mathcal{F}_{\mathrm{lin}})
		=|\bar X_n|=\hat\theta_n^2|\bs_n(\hat\theta_n)|.
	\end{align*}
	Thus the usual Gaussian location score test is, up to scale, the IPM measuring departure of the first moment from its null value.
\end{example}

Appendix~\ref{sec:classical_score_examples} gives the analogous logarithmic-moment representation for testing exponentiality within a Gamma family.

\subsection{Nonparametric score tests via characteristic IPMs}\label{sec:distance_to_score}

The scope of alternatives detected by Eqn.~\eqref{eq:aggregate_score_test} is determined by $\mathcal{F}$. We call the induced IPM \emph{characteristic} if
\begin{align}\label{eq:characteristic_IPM}
	\mathrm{IPM}(\law_\theta, \law; \mathcal{F})
	\equiv \sup_{f \in \mathcal{F}} \left| \mathbb{E}_\law[f(X)] - \mathbb{E}_{\law_\theta}[f(X)] \right| = 0
	\quad \text{if and only if} \quad \law = \law_\theta.
\end{align}
We therefore call~$\sup_{f\in\mathcal{F}}|\bs_n(f;\hat\theta_n)|^2$ a \emph{nonparametric score-based test} whenever $\mathcal{F}$ induces a characteristic IPM. When the local-tilt condition above also holds, this gives score-based interpretations of the Kolmogorov--Smirnov distance, Wasserstein-1 distance, and maximum mean discrepancy under their standard function classes~\citep{smirnov1948table,villani2009optimal,gretton2012kernel}. The standard classes, characteristicity requirements, and computational costs are summarized in Appendix~\ref{sec:characteristic_ipm_examples}.

\paragraph{KSD as an IPM.} The KSD introduced in Section~\ref{sec:prelim_KSD} fits this framework directly. Taking the function class in Eqn.~\eqref{eq:aggregate_score_test} to be the \emph{kernelized Stein class}, \( \mathcal{F}_{\mathrm{KSD}} = \{ \mathcal{A}_p g : \| g \|_{\mathcal{H}^d} \leq 1 \} \), recovers the definition in~\eqref{eq:kernel_stein_discrepancy}:
\begin{align*}
	\mathrm{KSD}(\mathcal{L}_p, \mathcal{L}_q) = \sup_{f \in \mathcal{F}_{\mathrm{KSD}}} \left| \mathbb{E}_{p}[f(X)] - \mathbb{E}_{q}[f(X)] \right|,
\end{align*}
with $\mathbb{E}_{p}[f(X)]=0$ for any $f \in \mathcal{F}_{\mathrm{KSD}}$ by Stein's identity. For our purposes, KSD is especially attractive: this model-dependent Stein class induces a characteristic discrepancy under suitable conditions, while Stein's identity removes the model expectation from the statistic. Consequently, empirical KSD requires no model simulation and costs $O(n^2)$ for a general kernel, with faster special cases available. See Section~\ref{sec:prelim_KSD}, Appendix~\ref{sec:computation_v_statistic}, and Table~\ref{tab:distance_computation}.

Taking $p=p_{\hat\theta_n}$ in the kernelized Stein class, Stein's identity gives $\mathbb{E}_{\law_{(0,\hat\theta_n)}}[f(X)]=0$ for every $f\in\mathcal{F}_{\mathrm{KSD}}$, so the SKSD statistic in Eqn.~\eqref{eq:sksd_test_stat} satisfies $T_n(X_{1:n},\hat\theta_n)=\sup_{f\in\mathcal{F}_{\mathrm{KSD}}}|\bs_n(f;\hat\theta_n)|^2$. The SKSD test is therefore exactly a nonparametric score test under the exponentially tilted model~\eqref{eq:nonparametric_model} with the plugged-in nuisance estimate $\hat\theta_n$ and the model-adapted class $\mathcal{F}=\mathcal{F}_{\mathrm{KSD}}$.
This identity is also the sense in which the test is \textit{semiparametric}: the null family is \textit{parametric}, while the tilt directions $f$ range over a \textit{nonparametric} function class.

\subsection{Exponentiality score-test example}\label{sec:classical_score_examples}

This section supplements the Gaussian example in Section~\ref{sec:tilted_null_model}. We first give the corresponding exponentiality score-test example, then record the standard function classes and computational costs for several characteristic IPMs.

\begin{example}[Testing exponentiality within a Gamma family, \citep{moran1951random,haywood2008distribution}]\label{ex:exponential_gamma_model}
	Let $\law_{(\gamma,\theta)}$ be Gamma with shape $\gamma+1$ and rate $\theta$, where $\gamma>-1$ and $\theta>0$. At $\gamma=0$, the null family is exponential with density $p_\theta(x)=\theta\exp(-\theta x)$. The parametric score for testing $\gamma=0$ is
	\begin{align*}
		\bs_n(\hat\theta_n)
		=\frac{1}{n}\sum_{i=1}^n\log X_i
		-\E_{\law_{(0,\hat\theta_n)}}[\log X].
	\end{align*}
	For $\mathcal{F}_{\log}=\{x\mapsto a\log x:|a|\leq1\}$,
	\begin{align*}
		\sup_{f\in\mathcal{F}_{\log}}
		\left|\frac{1}{n}\sum_{i=1}^n f(X_i)
		-\E_{\law_{(0,\hat\theta_n)}}[f(X)]\right|^2
		=|\bs_n(\hat\theta_n)|^2.
	\end{align*}
	Hence the score test measures departure of the logarithmic moment from its null value.
\end{example}

\subsection{Characteristic integral probability metrics and computational complexity}\label{sec:characteristic_ipm_examples}

The richness of the function class controls both characteristicity and computational cost~\citep{gretton2012kernel,liu2016kernelized}. We detail four choices used in the main text.

\paragraph{Kolmogorov--Smirnov (K--S) distance.}
For $d=1$, take the lower-ray indicators $\mathcal{F}_{\mathrm{KS}}=\{x\mapsto\indicator(x\leq t):t\in\mathbb{R}\}$. Their multivariate analogues are indicators of lower orthants. The resulting IPM compares distribution functions and is therefore characteristic. In one dimension the empirical statistic can be evaluated in $O(n\log n)$ time by sorting. Its worst-case cost is $O(n^2)$ for common bivariate formulations~\citep{peacock1983two,fasano1987multidimensional}, while exact computation in general dimension is NP-hard when the dimension is part of the input~\citep{gnewuch2009finding}.

\paragraph{Wasserstein-1 (W-1) distance.}
Here $\mathcal{F}$ is the unit ball of 1-Lipschitz functions, $\{f:|f(x)-f(y)|\leq\|x-y\|\text{ for all }x,y\in\X\}$. The Kantorovich--Rubinstein duality identifies the induced IPM with W-1, which is characteristic under the usual finite-first-moment condition~\citep{villani2009optimal}. It costs $O(n\log n)$ for empirical measures on the line. In higher dimensions, exact transport between equally sized empirical measures can be posed as an assignment problem with a classical $O(n^3)$ solution~\citep{kuhn1955hungarian}. Faster approximations are available~\citep{altschuler2017near}.

\paragraph{Maximum mean discrepancy (MMD).}
Taking $\mathcal{F}$ to be the unit ball of an RKHS gives MMD, which is characteristic for suitable kernels~\citep{gretton2012kernel}. In one-sample goodness-of-fit testing, the expectations under the fitted model may be unavailable analytically. Approximating them with $m$ Monte Carlo draws gives the usual quadratic cost $O(n^2+nm+m^2)$. The $m^2$ term dominates when the reference sample is much larger than the data sample.

\paragraph{Kernelized Stein discrepancy (KSD).}
KSD uses the model-dependent Stein class $\mathcal{F}_{\mathrm{KSD}}=\{\mathcal{A}_p g:\|g\|_{\mathcal{H}^d}\leq1\}$. It can be shown to be characteristic under mild conditions~\citep{barp2024targeted}. See Appendix~\ref{sec:characteristicity}. Unlike ordinary MMD, Stein's identity eliminates expectations under the model, so no Monte Carlo reference sample is needed. The V-statistic costs $O(n^2)$ for a general kernel, and some choices permit linear-time computation~\citep{jitkrittum2017linear}. See Appendix~\ref{sec:computation_v_statistic}.

\begin{table}[!ht]
	\centering
	\begin{tabular}{lccc}
		\toprule
		\multirow{2}{*}{\shortstack[l]{Discrepancy/\\Distance}} & \multicolumn{3}{c}{Computational complexity} \\
		\cmidrule(lr){2-4}
		& $d=1$ & $d=2$ & $d>2$ \\
		\midrule
		K--S distance & $n\log n$ & $n^2$ & NP-hard \\
		W-1 distance & $n\log n$ & $n^3$ & $n^3$ \\
		MMD & $(n+m)^2$ & $(n+m)^2$ & $(n+m)^2$ \\
		KSD & $n^2$ & $n^2$ & $n^2$ \\
		\bottomrule
	\end{tabular}
	\caption{Representative costs for the statistics discussed above. Here $n$ is the data sample size and $m$ is the Monte Carlo reference-sample size used for MMD. The W-1 entries assume equally sized empirical measures.}
	\label{tab:distance_computation}
\end{table}

\section{Additional consequences of the limit theory}\label{sec:additional_theory}

The bootstrap limits $W_\law$ and $W$ in
Theorems~\ref{thm:influence-adjusted-bootstrap} and
\ref{thm:bootstrap_distribution} are squared norms of centered Gaussian
elements. The trace-class calculation in Appendix~\ref{sec:close-form}
and the moment bound in Eqn.~\eqref{eq:proof-adjusted-feature-second-moment}
give
\begin{align}
    \E[W_\law]
    &=\sum_{u=1}^\infty\lambda_{u,\law}
    =\tr(\Sigma_{\theta_0,\law})<\infty,
    \label{eq:finite-general-law-limit}\\
    \E[W]
    &=\sum_{u=1}^\infty\lambda_u
    =\tr(\Sigma_{\theta_0})<\infty.
    \label{eq:finite-null-limit}
\end{align}
Thus $\P(W_\law<\infty)=\P(W<\infty)=1$. Under the null,
$W_\law\overset{d}{=}W$. Under alternatives, the multiplier limit
remains law-specific, whereas the parametric bootstrap targets $W$.

We next strengthen the pointwise conditional convergence in
Theorems~\ref{thm:influence-adjusted-bootstrap} and
\ref{thm:bootstrap_distribution} to uniform convergence over the cutoff.
Although pointwise convergence suffices for validity at a fixed level,
uniform convergence is useful when the bootstrap distribution is used at
multiple significance levels. The proof is provided in
Appendix~\ref{sec:proof_bootstrap_as_convergence}.
\begin{corollary}[Uniform conditional bootstrap convergence]
\label{cor:bootstrap_convergence}
    \leavevmode
    \begin{enumerate}
        \item[(a)] Under the assumptions of
        Theorem~\ref{thm:influence-adjusted-bootstrap},
        \begin{equation}\label{eq:main_ia_strong_convergence}
            \sup_{t\in\R}\left|
            \P(n\widetilde T_{n,\mathrm{IA}}\le t\mid\mathcal F_n)
            -\P(W_\law\le t)\right|\overset{p}{\longrightarrow}0.
        \end{equation}
        If, additionally, the influence-function expansion in
        Eqn.~\eqref{eq:ale} and the estimation condition in
        Eqn.~\eqref{eq:estimated-influence-function-l2} both hold with
        almost-sure convergence in place of convergence in probability,
        then
        \begin{equation}\label{eq:main_ia_as_convergence}
            \sup_{t\in\R}\left|
            \P(n\widetilde T_{n,\mathrm{IA}}\le t\mid\mathcal F_n)
            -\P(W_\law\le t)\right|\longrightarrow0
            \qquad\text{almost surely}.
        \end{equation}
        \item[(b)] Under the assumptions of
        Theorem~\ref{thm:bootstrap_distribution},
        \begin{equation}\label{eq:main_strong_convergence}
            \sup_{t\in\R}\left|
            \P(n\widetilde T_{n,\mathrm{PA}}\le t\mid\mathcal F_n)-\P(W\le t)
            \right|\overset{p}{\longrightarrow}0.
        \end{equation}
        If, additionally,
        $\hat\theta_n\overset{\mathrm{a.s.}}{\to}\theta_0$, then
        \begin{equation}\label{eq:main_as_convergence}
            \sup_{t\in\R}\left|
            \P(n\widetilde T_{n,\mathrm{PA}}\le t\mid\mathcal F_n)-\P(W\le t)
            \right|\longrightarrow0\qquad\text{almost surely}.
        \end{equation}
    \end{enumerate}
\end{corollary}
Both parts give uniform approximation in probability over the data.
Under the respective additional conditions, this approximation holds
along almost every data sequence. In part~(a), these conditions concern
only the influence-function expansion and its estimation. Consistency
of the fitted features is derived in the proof.

\section{Neyman-orthogonal SKSD test}\label{sec:neyman_orthogonal_sksd_test}
This section describes a Neyman-orthogonal SKSD test and its computational cost.
For $f\in\mH^d$, define
\begin{align}\label{eq:sn}
	S_n(f)
	&\equiv\left\langle
	f,\frac1{\sqrt n}\sum_{i=1}^n\Psi_{\theta_0}(X_i)
	\right\rangle_{\mH^d}\notag\\
	&=\frac1{\sqrt n}\sum_{i=1}^n\left[
	\A_{\theta_0}f(X_i)
	+\left\langle
	\E_{X'\sim\law_{\theta_0}}\left[
	[\nabla_\theta s_{\theta_0}(X')]^\top f(X')
	\right],
	I_{\law_{\theta_0}}(X_i)
	\right\rangle
	\right].
\end{align}
Eqn.~\eqref{eq:observed-feature-expansion} and the Riesz representation
theorem give
\begin{align*}
	nT_n(X_{1:n},\hat\theta_n)
	&=\sup_{f\in\mathcal B_{\mH^d}(1)}S_n(f)^2+o_p(1),
\end{align*}
where $\mathcal B_{\mH^d}(1)$ is the unit ball in $\mH^d$.
The second term in Eqn.~\eqref{eq:sn} is due to parameter estimation and
explains the failure of a wild bootstrap that ignores this contribution, as
discussed in Section~\ref{sec:asymptotic_bootstrap_distribution}.
To account for this term, write $r_\theta(x)=\nabla_\theta s_\theta(x)\in\R^{d\times k}$ and assume that $M_0=\E_{\theta_0}[r_{\theta_0}(X)^\top r_{\theta_0}(X)]$ is nonsingular. Consider the following pointwise transformation of $f$:
\begin{align*}
    \widetilde{f}(x) \equiv f(x) - r_{\theta_0}(x)M_0^{-1}\E_{\theta_0}[r_{\theta_0}(X)^\top f(X)].
\end{align*}
It satisfies the Neyman orthogonality condition
\begin{align*}
    \E_{\theta_0}[[\nabla_\theta s_{\theta_0}(X)]^\top\widetilde{f}(X)] = 0
    \qquad\text{for every }f\in\mH^d.
\end{align*}
In expectations involving both $X$ and $X'$, they are independent with law $\law_{\theta_0}$.
The function $\widetilde{f}$ corresponds to the feature of the Neyman-orthogonal kernel~\citep{escanciano2024gaussian}, defined by
\begin{align*}
    \widetilde{\mK}(x_i,x_j)
    &=K(x_i,x_j)I_d
    -r_{\theta_0}(x_i)M_0^{-1}
    \E_{\theta_0}[r_{\theta_0}(X)^\top K(X,x_j)]\\
    &\quad-\E_{\theta_0}[K(x_i,X)r_{\theta_0}(X)]
    M_0^{-1}r_{\theta_0}(x_j)^\top\\
    &\quad+r_{\theta_0}(x_i)M_0^{-1}
    \E_{\theta_0}[K(X,X')r_{\theta_0}(X)^\top r_{\theta_0}(X')]
    M_0^{-1}r_{\theta_0}(x_j)^\top.
\end{align*}
Let $\mPi(x_i,x_j)=r_{\theta_0}(x_i)M_0^{-1}r_{\theta_0}(x_j)^\top$.
For simplicity, we omit the dependence on $\theta_0$ in the expectations below.
Then the Neyman-orthogonal kernel can be written as
\begin{align*}
    \widetilde{\mK}(x_i,x_j)
    &=K(x_i,x_j)I_d-\E[\mPi(x_i,X)K(X,x_j)]\\
    &\quad-\E[K(x_i,X)\mPi(X,x_j)]
    +\E[\mPi(x_i,X)K(X,X')\mPi(X',x_j)].
\end{align*}
The projected kernel $\widetilde{\mK}$ is generally matrix-valued. The same construction also applies to matrix-valued base kernels.
With the derived Neyman-orthogonal kernel, we define the corresponding SKSD test statistic using the same normalization as in Eqn.~\eqref{eq:sksd_test_stat}. Writing $\hat\law_n=n^{-1}\sum_{i=1}^n\delta_{X_i}$ for the empirical measure, a standard calculation as in Appendix~\ref{sec:additional_details} gives
\begin{align}\label{eq:ksd-neyman}
    T_{n,\widetilde{\mK}}(X_{1:n},\hat\theta_n)
    &\equiv\KSD^2(\law_{\hat\theta_n},\hat\law_n;\widetilde{\mK})\notag\\
    &=\E_{(X,X')\sim\hat\law_n\otimes\hat\law_n}
    [h_{p_{\hat\theta_n}}(X,X';\widetilde{\mK})]\notag\\
    &=\frac1{n^2}\sum_{i=1}^n\sum_{j=1}^n
    h_{p_{\hat\theta_n}}(X_i,X_j;\widetilde{\mK}),
\end{align}
where the function $h_{p_\theta}(x,y;\widetilde{\mK})$ is defined as
\begin{align}\label{eq:ksd-neyman-h}
    h_{p_\theta}(x,y;\widetilde{\mK})
    &=s_\theta(x)^\top\widetilde{\mK}(x,y)s_\theta(y)
    +s_\theta(x)^\top\nabla_2\cdot\widetilde{\mK}(x,y)\notag\\
    &\quad+\nabla_1\cdot\widetilde{\mK}(x,y)s_\theta(y)
    +\nabla_1\cdot\nabla_2\cdot\widetilde{\mK}(x,y).
\end{align}
Here $\nabla\cdot$ denotes divergence, with
$\nabla_1\cdot\widetilde{\mK}\in\R^{1\times d}$ and
$\nabla_2\cdot\widetilde{\mK}\in\R^{d\times1}$.
For the derivatives of the kernel, we have
\begin{align*}
    \nabla_1\cdot\widetilde{\mK}(x_i,x_j)
    &=(\nabla_1 K(x_i,x_j))^\top
    -\E[(\nabla_1\cdot\mPi(x_i,X))K(X,x_j)]\\
    &\quad-\E[(\nabla_1 K(x_i,X))^\top\mPi(X,x_j)]\\
    &\quad+\E[(\nabla_1\cdot\mPi(x_i,X))K(X,X')\mPi(X',x_j)],\\
    \nabla_2\cdot\widetilde{\mK}(x_i,x_j)
    &=\nabla_2 K(x_i,x_j)
    -\E[\mPi(x_i,X)\nabla_2 K(X,x_j)]\\
    &\quad-\E[K(x_i,X)(\nabla_2\cdot\mPi(X,x_j))]\\
    &\quad+\E[\mPi(x_i,X)K(X,X')(\nabla_2\cdot\mPi(X',x_j))],\\
    \nabla_1\cdot\nabla_2\cdot\widetilde{\mK}(x_i,x_j)
    &=\tr(\nabla_1\nabla_2 K(x_i,x_j))\\
    &\quad-\E[(\nabla_1\cdot\mPi(x_i,X))\nabla_2 K(X,x_j)]\\
    &\quad-\E[(\nabla_1 K(x_i,X))^\top(\nabla_2\cdot\mPi(X,x_j))]\\
    &\quad+\E[(\nabla_1\cdot\mPi(x_i,X))K(X,X')(\nabla_2\cdot\mPi(X',x_j))].
\end{align*}
Under suitable regularity conditions, the columns of $r_{\theta_0}$ and their parameter derivatives belong to the Stein class of $\law_{\theta_0}$, so that
\begin{align*}
    \E_{\theta_0}[\A_{\theta_0}r_{\theta_0}(X)]&=0,\qquad
    \E_{\theta_0}[\A_{\theta_0}\partial_{\theta_\ell}r_{\theta_0}(X)]=0,
    \quad \ell=1,\ldots,k,
\end{align*}
where the Stein operator acts columnwise.
Consequently, $\widetilde f$ also belongs to the Stein class, and under the null hypothesis,
\begin{align*}
    \E_{\theta_0}[\A_{\theta_0}\widetilde f(X)]&=0,\qquad
    T_{n,\widetilde{\mK}}(X_{1:n},\hat\theta_n)\convp0.
\end{align*}
Following these conditions, the Stein identities and Neyman orthogonality remove the first-order effect of estimating the nuisance parameter and the projection, so a wild bootstrap validly approximates the null distribution of the test statistic.
Compared with the influence-adjusted bootstrap, this approach avoids influence-function estimation but changes the test statistic itself.
However, the projection can remove directions that distinguish a fixed alternative from the null, so the resulting test can be powerless against some fixed alternatives.
The test procedure is summarized in Algorithm~\ref{alg:neyman-SKSD}.
Both the observed statistic and its bootstrap replicates use the $n^{-2}$ normalization. Multiplying both by $n$ leaves the bootstrap $p$-value unchanged.
\begin{center}
	\begin{minipage}{\linewidth}
		\begin{algorithm}[H]
			\nonl  \textbf{Input:} Distribution class $\{\law_\theta:\theta\in\Theta\}$, data $X_{1:n}=(X_1,\ldots,X_n)$, Neyman-orthogonal kernel $\widetilde{\mK}$, estimation algorithm $\mathcal{E}:\mathcal{X}^n\rightarrow\Theta\subseteq\mathbb{R}^k$, bootstrap size $B$.\\
			Obtain the estimate $\hat\theta_n=\mathcal{E}(X_{1:n})$\;
			Compute $T_{n, \widetilde{\mK}} \equiv T_{n, \widetilde{\mK}}(X_{1:n},\hat\theta_n)$ as in Eqns.~\eqref{eq:ksd-neyman} and \eqref{eq:ksd-neyman-h}\;
			\For{$b = 1, 2, \dots, B$}{
				Generate wild bootstrap weights $\{w_i^{(b)}\}_{i \in [n]} \overset{\mathrm{i.i.d.}}{\sim} \text{Unif}\{-1, +1\}$\;
				Compute $\widetilde{T}_{n, \widetilde{\mK}}^{(b)}\equiv \frac{1}{n^2} \sum_{i=1}^n \sum_{j=1}^n w_i^{(b)} w_j^{(b)} h_{p_{\hat\theta_n}}(X_i, X_j; \widetilde{\mK})$.
			}
			\nonl \textbf{Output:} $p$-value $\sum_{b=1}^B \ind(\widetilde{T}_{n, \widetilde{\mK}}^{(b)}\geq T_{n, \widetilde{\mK}})/B$.
			\caption{\bf Neyman-orthogonalized SKSD goodness-of-fit testing procedure}
			\label{alg:neyman-SKSD}
		\end{algorithm}
	\end{minipage}
\end{center}
A straightforward implementation using dense $n\times n$ matrix products has total computational cost $O(n^3+n^2B)$: constructing the test statistic costs $O(n^3)$, and each wild-bootstrap replication costs $O(n^2)$.
When $n \gg B$, the cost of this implementation is dominated by $O(n^3)$, which can be more expensive than the standard SKSD test.
Additionally, notice that the construction of the Neyman-orthogonal kernel $\widetilde{\mK}$ involves several expectation terms under the null model, which may not have closed-form expressions.
Approximating these expectations by Monte Carlo can add sampling cost and error to the kernel approximation.

\section{Regularity conditions}\label{sec:regularity_conditions}

We first state the common regularity conditions. The additional conditions
in Appendix~\ref{sec:parametric-bootstrap-conditions} are imposed only
for results that explicitly invoke them.

Throughout this section, suppose $\Theta\subseteq\R^k$ is open and
$\theta_0\in\Theta$, where the definition of $\theta_0$ is given by $\hat{\theta}_n\overset{p}{\to}\theta_0$ and the limit is well-defined in the statement of the results.
Let $U=B_{\theta_0}(\varepsilon)\subseteq\Theta$ for some $\varepsilon>0$.
Suppose the laws $\{\law_\theta:\theta\in U\}$ admit positive densities
$p_\theta$ with respect to a common $\sigma$-finite measure $\mu$ on a
common support $\X\subseteq\R^d$, and the data-generating distribution $\law$ is supported on $\X$.

\begin{assumption}[Regularity of the model densities]\label{ass:p}
	Assume
	\begin{equation}\label{eq:model-density-regularity}
		\big((x,\theta)\mapsto p_\theta(x)\big)
		\in\mathcal C^{1,2}(\X,U),
	\end{equation}
	where the first index denotes differentiation in $x$ and the second
	denotes differentiation in $\theta$. Write
	$s_\theta(x)=\nabla_x\log p_\theta(x)$. Assume that the supremum below
	is measurable and
	\begin{align}\label{eq:model-density-moments}
		\E_{X\sim\law}\left[
		\|s_{\theta_0}(X)\|^2
		+\|\nabla_\theta s_{\theta_0}(X)\|_{\mathrm{op}}^2
		\right]&<\infty,\notag\\
		\E_{X\sim\law}\left[
		\sup_{\theta\in U}
		\|\nabla_\theta^2s_\theta(X)\|_{\mathrm{op}}
		\right]&<\infty.
	\end{align}
\end{assumption}

\begin{assumption}[Regularity of the kernel]\label{ass:k}
	The scalar kernel $K$ is measurable, symmetric, positive definite, and does
	not depend on $\theta$, and its RKHS $\mH$ is separable. Suppose
	$K\in\mathcal C^{1,1}(\X,\X)$. Assume that
	there exists $C_K<\infty$ such that for any $x \in \X$,
	\begin{equation}\label{eq:kernel-regularity}
		K(x,x),\quad
		\left.\tr\left(\nabla_x\nabla_{x'}K(x,x')\right)\right|_{x'=x}
		\le C_K.
	\end{equation}
\end{assumption}

\begin{assumption}[Stein identity]\label{ass:stein}
	For every $\theta\in U$ and $f\in\mH^d$, the expectation below is
	well-defined and
	\begin{equation}\label{eq:stein-identity-assumption}
		\E_{X\sim\law_\theta}\left[
		\A_\theta f(X)
		\right]=0.
	\end{equation}
\end{assumption}

\subsection{Additional conditions for Theorem~\ref{thm:bootstrap_distribution}}
\label{sec:parametric-bootstrap-conditions}
For Theorem~\ref{thm:bootstrap_distribution}, we first strengthen the estimator condition: since the parametric bootstrap refits the estimator on samples drawn from fitted laws near $\theta_0$, the ALE expansion in Assumption~\ref{assu:ale} must hold locally uniformly. We call an estimator satisfying this local uniform extension of ALE a uniformly asymptotically linear estimator (UALE).
\begin{assumption}[UALE]\label{assu:uale}
	Consider the estimation procedure $\mathcal E$. Suppose $\hat\theta_n\convp\theta_0\in\Theta$ and that
	there exists an open neighborhood $U\subseteq\Theta$ of $\theta_0$,
	such that Assumption~\ref{assu:ale} holds uniformly under $X_{1:n}=(X_1,\ldots,X_n)\overset{\mathrm{i.i.d.}}{\sim}\law_\theta$ with 
	$\theta \in U$, i.e., there exists a measurable function $I_{\law_\theta}:\X\to\R^k$ continuous in $\theta$ for a.s. $x \in \X$, with $\E[I_{\law_\theta}(X_i)] = 0,\ \E[\left\|I_{\law_\theta}(X_i)\right\|^2] < \infty$,
	such that
	\begin{equation}\label{eq:uale}
		\left\|\sqrt n(\hat\theta_n-\theta)
		-\frac1{\sqrt n}\sum_{i=1}^n I_{\law_\theta}(X_i)\right\|
		\convp0\qquad\text{uniformly in }\theta\in U.
	\end{equation}
\end{assumption}
The equivalent probability formulation of Eqn.~\eqref{eq:uale} is given in Eqn.~\eqref{eq:uale-uniform-probability} in Appendix~\ref{sec:proof_bootstrap_distribution}.

In addition, retain the setup of
Appendix~\ref{sec:regularity_conditions} and
Assumptions~\ref{ass:k} and~\ref{ass:stein}, and replace
Assumption~\ref{ass:p} by a slightly different version, Assumption~\ref{ass:p-bootstrap} below.
The moment conditions in Eqn.~\eqref{eq:model-density-moments} are
not required under the data-generating law $\law$. The parametric
bootstrap instead uses moment bounds under the nearby model laws
$\law_\theta$.

\begin{assumptionstar}[Regularity of the model densities for the parametric bootstrap]
\label{ass:p-bootstrap}
	Assume Eqn.~\eqref{eq:model-density-regularity} holds. The supremum
	inside the expectation below is measurable, and the following
	moment bounds hold uniformly under the model laws:
	\begin{align}\label{eq:parametric-bootstrap-derivative-moments}
		\sup_{\theta\in U}\E_{X\sim\law_\theta}\left[
		\|s_\theta(X)\|^2
		+\|\nabla_\theta s_\theta(X)\|_{\mathrm{op}}^2
		\right]&<\infty,\notag\\
		\sup_{\theta\in U}\E_{X\sim\law_\theta}\left[
		\sup_{\theta'\in U}
		\|\nabla_{\theta'}^2s_{\theta'}(X)\|_{\mathrm{op}}
		\right]&<\infty.
	\end{align}
\end{assumptionstar}

Beyond the moment bounds in Assumption~\ref{ass:p-bootstrap}, we also require continuity of the score function and influence function in the $L_2$ sense, as stated in Assumption~\ref{ass:parametric-bootstrap} below.

\begin{assumption}[Continuity of the score function and influence function]
\label{ass:parametric-bootstrap}
	The moment functions $\E_{X\sim\law_\theta}[\|s_\theta(X)\|^2]$ and $\E_{X\sim\law_\theta}[\|I_{\law_\theta}(X)\|^2]$ are continuous at $\theta_0$.
\end{assumption}

Setting $\theta=\theta_0$ in
Eqn.~\eqref{eq:parametric-bootstrap-derivative-moments} gives
Eqn.~\eqref{eq:model-density-moments} under $\law_{\theta_0}$, even
when $\law\ne\law_{\theta_0}$. This supplies the existence of the moments under the null model
needed to define the null limit in Theorem~\ref{thm:bootstrap_distribution}.
The additional moment and continuity properties are stated when first
used in the proofs. Their proofs are provided in
Appendices~\ref{sec:proof_parametric_bootstrap_feature_continuity}--\ref{sec:proof_parametric_bootstrap_derivative_continuity}.

\section{Preliminaries on stable convergence}\label{sec:intro_stable_convergence}
For the parametric bootstrap in Algorithm~\ref{alg:main}, calibration requires the conditional distribution function $\P[n\widetilde T_{n,\mathrm{PA}}\le t\mid\mathcal F_n]$ to converge in probability to the distribution function of the null limit $W$ at its continuity points. Marginal weak convergence of the bootstrap statistic alone does not establish this conditional convergence. We use stable convergence to characterize joint limits with the observed data, and then analyze two conditionally independent bootstrap replications to obtain the required conditional distribution convergence.

Stable convergence \citep{hausler2015stable} strengthens ordinary weak convergence by retaining joint convergence with bounded random variables measurable with respect to a specified sub-$\sigma$-algebra.
\begin{definition}[Stable convergence]\label{defn:stable_convergence}
    Suppose $(Y_n)_{n\ge1}$ are random variables on $(\Omega,\mathscr A,\P)$ with values in a Polish space $\mathcal Y$ equipped with its Borel $\sigma$-algebra $\mathcal B$. Let $\mathcal G\subseteq\mathscr A$ be a sub-$\sigma$-algebra and let $Q$ be a $\mathcal G$-measurable probability kernel from $\Omega$ to $\mathcal Y$: $Q(\omega,\cdot)$ is a probability measure for each $\omega\in\Omega$, and $\omega\mapsto Q(\omega,B)$ is $\mathcal G$-measurable for each $B\in\mathcal B$.
    We say that $Y_n$ converges $\mathcal G$-stably to $Q$, and write
    \begin{align*}
        Y_n\overset{d}{\longrightarrow}Q\qquad\mathcal G\text{-stably},
    \end{align*}
    if, for every $\mathcal G$-measurable random variable $Z\in L_1(\P)$ and every bounded continuous function $f:\mathcal Y\to\R$,
    \begin{align}\label{eq:defn-stable-convergence}
        \lim_{n\to\infty}\E[Zf(Y_n)]
        =\int_\Omega\int_{\mathcal Y}Z(\omega)f(y)Q(\omega,\d y)\P(\d\omega).
    \end{align}
\end{definition}
Here $Q$ denotes the entire probability kernel. The defining identity is integrated over $\omega$ and does not assert convergence at a fixed sample point.
If $Q(\omega,\cdot)=\mu$ for $\P$-almost every $\omega$, we identify $Q$ with the fixed probability measure $\mu$ and write $Y_n\overset{d}{\to}\mu$ $\mathcal G$-stably.
We also use the shorthand
\begin{align*}
    Y_n\overset{d}{\to}Y\qquad\mathcal G\text{-stably}
\end{align*}
when $Q$ is a version of the conditional law of a random variable $Y$ given $\mathcal G$, where $Y$ may be defined on an extension of the original probability space. In particular, if $Y$ is independent of $\mathcal G$, this kernel is the fixed law of $Y$.

The following lemma, adapted from Theorem~3.2 in \citet{hausler2015stable}, characterizes stable convergence through weak convergence under conditioning on events in $\mathcal G$.
\begin{lemma}\label{lemma:stable_equivalence}
    Suppose the setting in Definition~\ref{defn:stable_convergence} holds. Then $Y_n$ converges $\mathcal G$-stably to $Q$ if and only if, for every $F\in\mathcal G$ with $\P(F)>0$,
    \begin{align*}
        Y_n\overset{d}{\longrightarrow}Q_F\qquad\text{under }\P_F,
    \end{align*}
    where $\P_F(\cdot)=\P(\cdot\mid F)$ and $Q_F$ is the probability measure given by
    \begin{align*}
        Q_F(B)=\frac{1}{\P(F)}\int_F Q(\omega,B)\P(\d\omega),\qquad B\in\mathcal B.
    \end{align*}
\end{lemma}
If $Q$ is the conditional law of $Y$ given $\mathcal G$, then $Q_F$ is the conditional law of $Y$ given $F$. If $Q$ is a fixed probability measure, $Q_F=Q$ for every such $F$. In particular, for $\mathcal G=\{\emptyset,\Omega\}$, stable convergence reduces to ordinary weak convergence.
The bootstrap proofs use this characterization to obtain joint stable limits and then establish conditional distribution convergence from the joint limit of two bootstrap replications.

The stable central limit theorem and its functional extension used below are stated when they are first invoked in the proof of Theorem~\ref{thm:influence-adjusted-bootstrap}. Their proofs are deferred to the later section on proofs of the main lemmas.

\section{Proofs of results in Sections~\ref{sec:method} and~\ref{sec:bootstrap-tests}}
\subsection{Proof of Theorem~\ref{thm:sksd_consistency}}\label{sec:proof_sksd_consistency}
\begin{proof}[Proof of Theorem~\ref{thm:sksd_consistency}]
	For every $g=(g_1,\ldots,g_d)\in\mH^d$, the reproducing and derivative
	reproducing properties give
	\begin{align*}
		\mathcal A_{p_\theta}g(x)
		&=\sum_{r=1}^d
		\left(s_{\theta,r}(x)g_r(x)+\partial_{x_r}g_r(x)\right)\\
		&=\sum_{r=1}^d\left\langle
		g_r,s_{\theta,r}(x)\psi(x)+\partial_{x_r}\psi(x)
		\right\rangle_{\mH}
		=\langle g,\xi_{\theta}(x)\rangle_{\mH^d},
	\end{align*}
	where $\xi_{\theta}(x)$ is defined in
	Eqn.~\eqref{eq:stein-feature}. Therefore, Eqn.~\eqref{eq:sksd_test_stat}
	and the Riesz representation theorem imply
	\begin{align*}
		T_n(X_{1:n},\hat\theta_n)
		&=\sup_{\|g\|_{\mH^d}\le1}
		\left\langle g,\frac1n\sum_{i=1}^n
		\xi_{\hat\theta_n}(X_i)\right\rangle_{\mH^d}^2
		=\left\|\frac1n\sum_{i=1}^n
		\xi_{\hat\theta_n}(X_i)\right\|_{\mH^d}^2\\
		&=\frac1{n^2}\sum_{i,j=1}^n
		\left\langle\xi_{\hat\theta_n}(X_i),
		\xi_{\hat\theta_n}(X_j)\right\rangle_{\mH^d}
		=\frac1{n^2}\sum_{i,j=1}^n
		h_{p_{\hat\theta_n}}(X_i,X_j).
	\end{align*}
	We decompose the statistic as
	\begin{align}\label{eq:skds_decomposition}
		T_n(X_{1:n},\hat\theta_n)
		&=\frac1{n^2}\sum_{i,j=1}^n h_{p_{\theta_0}}(X_i,X_j)\notag\\
		&\quad+\frac1{n^2}\sum_{i,j=1}^n
		\left(h_{p_{\hat\theta_n}}(X_i,X_j)
		-h_{p_{\theta_0}}(X_i,X_j)\right).
	\end{align}
	By Eqn.~\eqref{eq:stein-feature},
	$h_{p_{\theta_0}}(x,x')=
	\langle\xi_{\theta_0}(x),\xi_{\theta_0}(x')\rangle_{\mH^d}$.
	Moreover, Eqn.~\eqref{eq:kernel-regularity} and the derivative
	reproducing property give
	\begin{align*}
		\|\nabla_x\psi(x)\|_{\mH^d}^2
		&=\left.\tr(\nabla_x\nabla_{x'}K(x,x'))\right|_{x'=x}
		\le C_K,\\
		\|\xi_{\theta_0}(x)\|_{\mH^d}^2
		&\le2\|s_{\theta_0}(x)\psi(x)\|_{\mH^d}^2
		+2\|\nabla_x\psi(x)\|_{\mH^d}^2
		\le2C_K(\|s_{\theta_0}(x)\|^2+1).
	\end{align*}
	Hence the first line of Eqn.~\eqref{eq:model-density-moments} implies,
	for independent $X,X'\sim\law$,
	\begin{align*}
		\E_{X,X'\sim\law}[|h_{p_{\theta_0}}(X,X')|]
		&\le\E[\|\xi_{\theta_0}(X)\|_{\mH^d}
		\|\xi_{\theta_0}(X')\|_{\mH^d}]
		\le\E[\|\xi_{\theta_0}(X)\|_{\mH^d}^2]<\infty,\\
		\E_{X\sim\law}[|h_{p_{\theta_0}}(X,X)|]
		&=\E_{X\sim\law}[\|\xi_{\theta_0}(X)\|_{\mH^d}^2]<\infty.
	\end{align*}
	Therefore, the law of large numbers for
	U-statistics \citep[Theorem 12.3]{van2000asymptotic} and the ordinary
	law of large numbers give
	\begin{align*}
		\frac{1}{n^2}\sum_{i,j=1}^n h_{p_{\theta_0}}(X_i,X_j)
		&=\frac{n-1}{n}\frac{1}{n(n-1)}
		  \sum_{i\ne j}h_{p_{\theta_0}}(X_i,X_j)
		  +\frac{1}{n^2}\sum_{i=1}^n h_{p_{\theta_0}}(X_i,X_i)\\
		&\overset{p}{\to}
		\E_{X,X'\sim\law}[h_{p_{\theta_0}}(X,X')],
	\end{align*}
	where
	\begin{align*}
		\frac1{n^2}\sum_{i=1}^nh_{p_{\theta_0}}(X_i,X_i)
		&=\frac1n\left(\frac1n\sum_{i=1}^n
		h_{p_{\theta_0}}(X_i,X_i)\right)
		=\frac1n \cdot O_p(1)=o_p(1).
	\end{align*}
	By the closed-form representation of KSD in Appendix~\ref{sec:close-form},
	\begin{align}\label{eq:skds_decomposition_1}
		\frac{1}{n^2}\sum_{i,j=1}^n h_{p_{\theta_0}}(X_i,X_j)
		\overset{p}{\to}
		\mathrm{KSD}^2(\law_{\theta_0},\law).
	\end{align}

	For the second term in Eqn.~\eqref{eq:skds_decomposition}, first use the
	definition in Eqn.~\eqref{eq:stein-feature}. Since
	$\nabla_x\psi(X_i)$ does not depend on $\theta$,
	\begin{align}\label{eq:theorem-one-feature-difference}
		\xi_{\hat\theta_n}(X_i)-\xi_{\theta_0}(X_i)
		&=(s_{\hat\theta_n}(X_i)-s_{\theta_0}(X_i))\psi(X_i).
	\end{align}
	On the event $\hat\theta_n\in U$, Taylor's theorem gives
	\begin{align}\label{eq:theorem-one-score-taylor-bound}
		&\|s_{\hat\theta_n}(X_i)-s_{\theta_0}(X_i)\|\le
		\|\nabla_\theta s_{\theta_0}(X_i)\|_{\mathrm{op}}
		\|\hat\theta_n-\theta_0\|
		+\frac12\|\hat\theta_n-\theta_0\|^2
		\sup_{\theta\in U}
		\|\nabla_\theta^2s_\theta(X_i)\|_{\mathrm{op}}.
	\end{align}
	Substituting Eqn.~\eqref{eq:theorem-one-score-taylor-bound} into
	Eqn.~\eqref{eq:theorem-one-feature-difference} and using
	$\|\psi(X_i)\|_{\mH}\le C_K^{1/2}$ from
	Eqn.~\eqref{eq:kernel-regularity} give
	\begin{align}\label{eq:theorem-one-feature-average}
		&\frac1n\sum_{i=1}^n
		\|\xi_{\hat\theta_n}(X_i)-\xi_{\theta_0}(X_i)\|_{\mH^d}\notag\\
		&\quad\le C_K^{1/2}\|\hat\theta_n-\theta_0\|
		\left(\frac1n\sum_{i=1}^n
		\|\nabla_\theta s_{\theta_0}(X_i)\|_{\mathrm{op}}
		+\frac{\|\hat\theta_n-\theta_0\|}{2n}\sum_{i=1}^n
		\sup_{\theta\in U}\|\nabla_\theta^2s_\theta(X_i)\|_{\mathrm{op}}\right)\notag\\
		&=o_p(1),
	\end{align}
	where Eqn.~\eqref{eq:model-density-moments}, the law of large numbers,
	and $\hat\theta_n\overset{p}{\to}\theta_0$ imply
	\begin{align*}
		\frac1n\sum_{i=1}^n\|\nabla_\theta s_{\theta_0}(X_i)\|_{\mathrm{op}}
		&=O_p(1),\\
		\frac1n\sum_{i=1}^n\sup_{\theta\in U}
		\|\nabla_\theta^2s_\theta(X_i)\|_{\mathrm{op}}
		&=O_p(1),
		\qquad
		\P(\hat\theta_n\in U)\to1.
	\end{align*}
	Eqns.~\eqref{eq:model-density-moments}--\eqref{eq:kernel-regularity}
	also give
	\begin{align}\label{eq:theorem-one-unadjusted-feature-average}
		\left\|\frac1n\sum_{i=1}^n\xi_{\theta_0}(X_i)\right\|_{\mH^d}
		\le\frac1n\sum_{i=1}^n\|\xi_{\theta_0}(X_i)\|_{\mH^d}
		\le C_K^{1/2}\left(1+\frac1n\sum_{i=1}^n
		\|s_{\theta_0}(X_i)\|\right)=O_p(1).
	\end{align}
	Finally, the feature representation established at the beginning of the
	proof, together with Eqns.~\eqref{eq:theorem-one-feature-average}
	and~\eqref{eq:theorem-one-unadjusted-feature-average}, gives
	\begin{align}\label{eq:skds_decomposition_2}
		&\left|T_n(X_{1:n},\hat\theta_n)-T_n(X_{1:n},\theta_0)\right|\notag\\
		&=\left|
		2\left\langle\frac1n\sum_{i=1}^n\xi_{\theta_0}(X_i),
		\frac1n\sum_{i=1}^n
		(\xi_{\hat\theta_n}(X_i)-\xi_{\theta_0}(X_i))
		\right\rangle_{\mH^d}
		+\left\|\frac1n\sum_{i=1}^n
		(\xi_{\hat\theta_n}(X_i)-\xi_{\theta_0}(X_i))
		\right\|_{\mH^d}^2\right|\notag\\
		&\le
		2\left\|\frac1n\sum_{i=1}^n\xi_{\theta_0}(X_i)\right\|_{\mH^d}
		\left\|\frac1n\sum_{i=1}^n
		(\xi_{\hat\theta_n}(X_i)-\xi_{\theta_0}(X_i))
		\right\|_{\mH^d}
		+\left\|\frac1n\sum_{i=1}^n
		(\xi_{\hat\theta_n}(X_i)-\xi_{\theta_0}(X_i))
		\right\|_{\mH^d}^2\notag\\
		&=o_p(1).
	\end{align}
	Substituting Eqns.~\eqref{eq:skds_decomposition_1}--\eqref{eq:skds_decomposition_2} into Eqn.~\eqref{eq:skds_decomposition} gives
	\begin{align*}
		T_n(X_{1:n},\hat\theta_n)
		&\overset{p}{\to}
		\mathrm{KSD}^2(\law_{\theta_0},\law).
	\end{align*}
\end{proof}

\subsection{Proof of Theorem~\ref{thm:asymptotic_distribution}}\label{sec:proof_asymptotic_distribution}
\begin{proof}[Proof of Theorem~\ref{thm:asymptotic_distribution}]
	As in the proof of Theorem~\ref{thm:sksd_consistency}, the reproducing
	and derivative reproducing properties give
	$\mathcal A_{p_\theta}f(x)=
	\langle f,\xi_{\theta}(x)\rangle_{\mH^d}$ for every $f\in\mH^d$.
	Therefore, Eqns.~\eqref{eq:sksd_test_stat}
	and~\eqref{eq:stein-feature} give
	\begin{align}\label{eq:theorem-two-statistic-representation}
		nT_n
		&=\sup_{\|f\|_{\mH^d}\le1}
		\left\langle f,
		\frac1{\sqrt n}\sum_{i=1}^n
		\xi_{\hat\theta_n}(X_i)
		\right\rangle_{\mH^d}^2
		=\left\|
		\frac1{\sqrt n}\sum_{i=1}^n
		\xi_{\hat\theta_n}(X_i)
		\right\|_{\mH^d}^2.
	\end{align}

	The centering and second-moment conditions in
	Assumption~\ref{assu:ale} give
	\begin{align}\label{eq:theorem-two-influence-sum-bound}
		\E\left[
		\left\|
		\frac1{\sqrt n}\sum_{i=1}^n
		I_{\law_{\theta_0}}(X_i)
		\right\|^2
		\right]
		&=\frac1n\sum_{i=1}^n
		\E[\|I_{\law_{\theta_0}}(X_i)\|^2]
		+\frac1n\sum_{i\ne j}
			\E[\langle I_{\law_{\theta_0}}(X_i),
			I_{\law_{\theta_0}}(X_j)\rangle_{\R^k}]\notag\\
		&=\E_{X\sim\law_{\theta_0}}\left[
		\|I_{\law_{\theta_0}}(X)\|^2
		\right]<\infty.
	\end{align}
	The second equality uses independence and
	$\E[I_{\law_{\theta_0}}(X_i)]=0$, which make every cross term vanish.
	Therefore, Eqns.~\eqref{eq:ale}
	and~\eqref{eq:theorem-two-influence-sum-bound} imply
	\begin{align}
		\sqrt n(\hat\theta_n-\theta_0)
		&=\frac1{\sqrt n}\sum_{i=1}^n
		I_{\law_{\theta_0}}(X_i)+o_p(1)=O_p(1),\notag\\
		\hat\theta_n-\theta_0&=O_p(n^{-1/2}),
		\qquad
		\P(\hat\theta_n\in U)\to1.
		\label{eq:theorem-two-estimator-rate}
	\end{align}

	Only $s_\theta$ in the definition of $\xi_{\theta}(x)$ depends on
	$\theta$. On the event $\hat\theta_n\in U$, Taylor's theorem and
	Eqn.~\eqref{eq:stein-feature} give the pointwise bound
	\begin{align}\label{eq:theorem-two-pointwise-taylor-bound}
	&\left\|\xi_{\hat\theta_n}(X_i)-\xi_{\theta_0}(X_i)
	-\left(\nabla_\theta s_{\theta_0}(X_i)
	(\hat\theta_n-\theta_0)\right)\psi(X_i)\right\|_{\mH^d}\notag\\
	&\quad\le\frac12\|\psi(X_i)\|_{\mH}
	\|\hat\theta_n-\theta_0\|^2
	\sup_{\theta\in U}
	\|\nabla_\theta^2s_\theta(X_i)\|_{\mathrm{op}}.
	\end{align}
	Summing Eqn.~\eqref{eq:theorem-two-pointwise-taylor-bound} over $i$,
	applying the triangle inequality, and using
	Eqn.~\eqref{eq:model-density-regularity} and
	$\|\psi(x)\|_{\mH}=K(x,x)^{1/2}\le C_K^{1/2}$ from
	Eqn.~\eqref{eq:kernel-regularity} give
	\begin{align}
		&\left\|
		\frac1{\sqrt n}\sum_{i=1}^n
		\left(
		\xi_{\hat\theta_n}(X_i)-\xi_{\theta_0}(X_i)
		-\left(\nabla_\theta s_{\theta_0}(X_i)
		(\hat\theta_n-\theta_0)\right)\psi(X_i)
		\right)
		\right\|_{\mH^d}\notag\\
		&\qquad\le
		\frac{C_K^{1/2}\sqrt n}{2}
		\|\hat\theta_n-\theta_0\|^2
		\frac1n\sum_{i=1}^n
		\sup_{\theta\in U}
		\|\nabla_\theta^2s_\theta(X_i)\|_{\mathrm{op}}
		=o_p(1).
		\label{eq:theorem-two-taylor-remainder}
	\end{align}
	Indeed, Eqn.~\eqref{eq:theorem-two-estimator-rate} gives
	\begin{align*}
		\sqrt n\|\hat\theta_n-\theta_0\|^2
		&=\frac1{\sqrt n}
		\|\sqrt n(\hat\theta_n-\theta_0)\|^2
		=O_p(n^{-1/2})=o_p(1),
	\end{align*}
	whereas the scalar law
	of large numbers and Eqn.~\eqref{eq:model-density-moments} give
	\begin{align}\label{eq:theorem-two-hessian-average}
		\frac1n\sum_{i=1}^n\sup_{\theta\in U}
		\|\nabla_\theta^2s_\theta(X_i)\|_{\mathrm{op}}
		&\overset{p}{\to}
		\E_{X\sim\law_{\theta_0}}\left[
		\sup_{\theta\in U}
		\|\nabla_\theta^2s_\theta(X)\|_{\mathrm{op}}
		\right]<\infty.
	\end{align}
	Combining Eqns.~\eqref{eq:theorem-two-estimator-rate}
	and~\eqref{eq:theorem-two-hessian-average} proves the last equality in
	Eqn.~\eqref{eq:theorem-two-taylor-remainder}.

	For the first-order term, Eqns.~\eqref{eq:model-density-moments}
	and~\eqref{eq:kernel-regularity} imply
	\begin{align}
		&\E_{X\sim\law_{\theta_0}}\left[
		\|\left(\nabla_\theta s_{\theta_0}(X)\,\cdot\right)
		\psi(X)\|_{\mathrm{op}}
		\right]\le C_K^{1/2}
		\left(
		\E_{X\sim\law_{\theta_0}}\left[
		\|\nabla_\theta s_{\theta_0}(X)\|_{\mathrm{op}}^2
		\right]
		\right)^{1/2}<\infty.
		\label{eq:theorem-two-derivative-integrability}
	\end{align}
	Let $u_1,\ldots,u_k$ be the standard basis of $\R^k$. Because $\mH^d$
	is separable, Eqn.~\eqref{eq:theorem-two-derivative-integrability} permits
	the Hilbert-space law of large numbers in each coordinate direction.
	Thus
	\begin{align}\label{eq:theorem-two-derivative-lln}
		&\left\|
		\frac1n\sum_{i=1}^n
		\left(\nabla_\theta s_{\theta_0}(X_i)\,\cdot\right)\psi(X_i)
		-\E_{X\sim\law_{\theta_0}}\left[
		\left(\nabla_\theta s_{\theta_0}(X)\,\cdot\right)\psi(X)
		\right]
		\right\|_{\mathrm{op}}^2\notag\\
		&\quad\le\sum_{j=1}^k\left\|
		\frac1n\sum_{i=1}^n
		\left(\nabla_\theta s_{\theta_0}(X_i)u_j\right)\psi(X_i)
		-\E_{X\sim\law_{\theta_0}}\left[
		\left(\nabla_\theta s_{\theta_0}(X)u_j\right)\psi(X)
		\right]
		\right\|_{\mH^d}^2
		\overset{p}{\to}0.
	\end{align}
	The first-order term in the pointwise Taylor expansion satisfies the
	exact factorization
	\begin{align}\label{eq:theorem-two-first-order-factorization}
	&\frac1{\sqrt n}\sum_{i=1}^n
	\left(\nabla_\theta s_{\theta_0}(X_i)
	(\hat\theta_n-\theta_0)\right)\psi(X_i)=\left(
	\frac1n\sum_{i=1}^n
	\left(\nabla_\theta s_{\theta_0}(X_i)\,\cdot\right)\psi(X_i)
	\right)\sqrt n(\hat\theta_n-\theta_0).
	\end{align}
	Combining Eqns.~\eqref{eq:theorem-two-taylor-remainder}
	and~\eqref{eq:theorem-two-first-order-factorization} gives
	\begin{align}
		\frac1{\sqrt n}\sum_{i=1}^n
		\xi_{\hat\theta_n}(X_i)
		&=\frac1{\sqrt n}\sum_{i=1}^n
		\xi_{\theta_0}(X_i)\notag\\
		&\quad+
		\left(
		\frac1n\sum_{i=1}^n
		\left(\nabla_\theta s_{\theta_0}(X_i)\,\cdot\right)
		\psi(X_i)
		\right)
		\sqrt n(\hat\theta_n-\theta_0)+o_p(1).
		\label{eq:theorem-two-linearized-feature}
	\end{align}
	By Eqns.~\eqref{eq:theorem-two-estimator-rate}
	and~\eqref{eq:theorem-two-derivative-lln},
	\begin{align}\label{eq:theorem-two-operator-replacement}
		&\left\|
		\bigg(\frac1n\sum_{i=1}^n
		\left(\nabla_\theta s_{\theta_0}(X_i)\,\cdot\right)\psi(X_i)-\E_{X\sim\law_{\theta_0}}\left[
		\left(\nabla_\theta s_{\theta_0}(X)\,\cdot\right)\psi(X)
		\right]\bigg)\sqrt n(\hat\theta_n-\theta_0)
		\right\|_{\mH^d}\notag\\
		&\quad\le o_p(1)
		\|\sqrt n(\hat\theta_n-\theta_0)\|
		=o_p(1)O_p(1)=o_p(1).
	\end{align}
	Moreover, Eqn.~\eqref{eq:ale} gives
	\begin{align}
		&\E_{X\sim\law_{\theta_0}}\left[
		\left(\nabla_\theta s_{\theta_0}(X)\,\cdot\right)\psi(X)
		\right]
		\sqrt n(\hat\theta_n-\theta_0)\notag\\
		&\quad=\E_{X\sim\law_{\theta_0}}\left[
		\left(\nabla_\theta s_{\theta_0}(X)\,\cdot\right)\psi(X)
		\right]
		\left(\frac1{\sqrt n}\sum_{i=1}^n
		I_{\law_{\theta_0}}(X_i)\right)+o_p(1)\notag\\
		&\quad=\frac1{\sqrt n}\sum_{i=1}^n
		\E_{X'\sim\law_{\theta_0}}\left[
		\nabla_\theta s_{\theta_0}(X')
		I_{\law_{\theta_0}}(X_i)\psi(X')
		\right]+o_p(1).
		\label{eq:theorem-two-influence-substitution}
	\end{align}
	Combining Eqns.~\eqref{eq:theorem-two-linearized-feature}--
	\eqref{eq:theorem-two-influence-substitution} with the definition of
	$\Psi_{\theta_0}$ in Eqn.~\eqref{eq:influence-adjusted-stein-feature} yields
	\begin{align}\label{eq:observed-feature-expansion}
		\frac1{\sqrt n}\sum_{i=1}^n
		\xi_{\hat\theta_n}(X_i)
		&=\frac1{\sqrt n}\sum_{i=1}^n\left(
		\xi_{\theta_0}(X_i)
		+\E_{X'\sim\law_{\theta_0}}\left[
		\nabla_\theta s_{\theta_0}(X')
		I_{\law_{\theta_0}}(X_i)\psi(X')
		\right]\right)+o_p(1)\notag\\
		&=\frac1{\sqrt n}\sum_{i=1}^n
		\Psi_{\theta_0}(X_i)+o_p(1).
	\end{align}

	Eqns.~\eqref{eq:adjusted-feature-second-moment}
	and~\eqref{eq:adjusted-feature-centering} show that
	$\Psi_{\theta_0}(X_1),\ldots,\Psi_{\theta_0}(X_n)$ are i.i.d.,
	square-integrable, and centered in the separable Hilbert space $\mH^d$. The
	Hilbert-space central limit theorem therefore gives
	\begin{align}\label{eq:theorem-two-hilbert-clt}
		\frac1{\sqrt n}\sum_{i=1}^n
		\Psi_{\theta_0}(X_i)
		\overset{d}{\to}
		\mathbb G_{\theta_0}.
	\end{align}
	Finally, Eqns.~\eqref{eq:theorem-two-statistic-representation},
	\eqref{eq:observed-feature-expansion}, and
	\eqref{eq:theorem-two-hilbert-clt}, followed by Slutsky's theorem and the
	continuous mapping theorem for $z\mapsto\|z\|_{\mH^d}^2$, imply
	\begin{align*}
		nT_n
		&=\left\|\frac1{\sqrt n}\sum_{i=1}^n
		\Psi_{\theta_0}(X_i)+o_p(1)\right\|_{\mH^d}^2
		\overset{d}{\to}
		\|\mathbb G_{\theta_0}\|_{\mH^d}^2
		\overset{d}{=}
		\sum_{u=1}^\infty\lambda_uZ_u^2.
	\end{align*}
	The final equality is Eqn.~\eqref{eq:hilbert-gaussian-norm}.
\end{proof}

\subsection{Proof of Theorem~\ref{thm:influence-adjusted-bootstrap}}
\label{sec:proof_influence_adjusted_bootstrap}
\begin{proof}
    We prove the result in three steps. First, we establish stable convergence
    of every scalar projection of the centered multiplier process. Second, we
    lift these scalar limits to functional stable convergence and obtain the
    stable limit of the bootstrap statistic. Third, we translate this result
    into the conditional distribution convergence stated in the theorem.

    \paragraph{Step 1: Scalar stable convergence.}
	Realize the multipliers as a triangular array
	$\{w_{ni}^*:1\le i\le n,n\ge1\}$ whose rows are independent of one
	another and of $\mathcal F_\infty$, and suppress the first index below. Set
	$\mathcal F_{n,0}=\mathcal F_\infty$ and, for $1\le i\le n$,
	\begin{align*}
		\mathcal F_{n,i}=\sigma(\mathcal F_\infty,w_{n1}^*,\ldots,w_{ni}^*).
	\end{align*}
	Define
	\begin{align*}
		\mathbb G_n^*=\frac1{\sqrt n}\sum_{i=1}^n w_i^*(\widehat\Psi_i-\overline{\widehat\Psi}).
	\end{align*}
	Then $n\widetilde T_{n,\mathrm{IA}}=\|\mathbb G_n^*\|_{\mH^d}^2$.

	We use the following scalar stable central limit theorem, adapted from
	\citet[Theorem 6.1]{hausler2015stable}. Its proof is deferred to
	Appendix~\ref{sec:proof_stable_clt}.
	\begin{lemma}[Stable Central Limit Theorem]\label{lemma:stable-clt}
		Let $(Y_{nk})_{1\le k\le k_n,n\in\mathbb N}$ be a square-integrable
		martingale difference array adapted to
		$(\mathcal F_{nk})_{1\le k\le k_n,n\in\mathbb N}$. Let
		$\mathcal G_{nk}=\cap_{m\ge n}\mathcal F_{mk}$ and
		$\mathcal G=\sigma(\cup_{n\in\mathbb N}\mathcal G_{n k_n})$.
		Suppose that, for a nonnegative $\mathcal G$-measurable random variable
		$\eta$,
		\begin{align}\label{eq:stable-clt-1}
			\sum_{k=1}^{k_n}\E[Y_{nk}^2\mid\mathcal F_{n,k-1}]\overset{p}{\to}\eta^2,\tag{N}
		\end{align}
		and, for every $\varepsilon>0$,
		\begin{align}\label{eq:stable-clt-2}
			\sum_{k=1}^{k_n}\E[Y_{nk}^2\ind(|Y_{nk}|\ge\varepsilon)\mid\mathcal F_{n,k-1}]\overset{p}{\to}0.\tag{CLB}
		\end{align}
		Then
		\begin{align*}
			\sum_{k=1}^{k_n}Y_{nk}\overset{d}{\to}Q_\eta\qquad\mathcal G\text{-stably},
		\end{align*}
		where $Q_\eta$ is the Gaussian probability kernel defined by
		$Q_\eta(\omega,B)=\N(0,\eta(\omega)^2)(B)$ for $B\in\mathscr B(\R)$.
	\end{lemma}

	The independence of the rows implies that the limiting sigma-field
	$\mathcal G$ in Lemma~\ref{lemma:stable-clt} is $\mathcal F_\infty$.
	Fix $f\in\mH^d$ and define
	\begin{align*}
		D_{ni}(f)=\frac{w_i^*}{\sqrt n}\langle\widehat\Psi_i-\overline{\widehat\Psi},f\rangle_{\mH^d}.
	\end{align*}
	By Eqn.~\eqref{eq:iid-multiplier-weights},
	\begin{align*}
		\E[D_{ni}(f)\mid\mathcal F_{n,i-1}]&=0,
		&
		\sum_{i=1}^nD_{ni}(f)&=\langle\mathbb G_n^*,f\rangle_{\mH^d}.
	\end{align*}
	Thus $(D_{ni}(f))_{1\le i\le n}$ is a square-integrable martingale
	difference array. It is adapted to
	$(\mathcal F_{n,i})_{1\le i\le n}$.

	We first verify the conditional variance condition
	Eqn.~\eqref{eq:stable-clt-1}. This requires trace-norm consistency of the
	empirical covariance of the adjusted features. The feature comparison used
	to prove that consistency relies on a Taylor expansion near $\theta_0$, so
	we first obtain the required rate for $\hat\theta_n$. The independence and
	centering in Assumption~\ref{assu:ale} give
	\begin{align*}
		\E\left[
		\left\|\frac1{\sqrt n}\sum_{i=1}^nI_\law(X_i)\right\|^2
		\right]
		&=\frac1n\sum_{i=1}^n\E[\|I_\law(X_i)\|^2]
		=\E_{X\sim\law}[\|I_\law(X)\|^2]<\infty.
	\end{align*}
	Hence $n^{-1/2}\sum_{i=1}^nI_\law(X_i)=O_p(1)$, and the asymptotic
	linear expansion in Eqn.~\eqref{eq:ale} implies
	\begin{align}\label{eq:proof-ale-rate}
		\sqrt n\|\hat\theta_n-\theta_0\|
		&\le
		\left\|\frac1{\sqrt n}\sum_{i=1}^nI_\law(X_i)\right\|
		+o_p(1)=O_p(1).
	\end{align}
	Thus $\|\hat\theta_n-\theta_0\|=O_p(n^{-1/2})$ and, because
	$U$ is a neighborhood of $\theta_0$,
	$\P(\hat\theta_n\in U)\to1$.

	We next establish mean-square consistency of the estimated
	influence-adjusted features:
	\begin{align*}
		\frac1n\sum_{i=1}^n
		\|\widehat\Psi_i-\Psi_{\theta_0,\law}(X_i)\|_{\mH^d}^2
		&\overset{p}{\to}0.
	\end{align*}
	This requires controlling both the fitted Stein feature and the fitted
	adjustment operator.
	For $a\in\R^k$, define the linear operators
	\begin{align}
		\widehat{\mathsf J}_n a
		&=\frac1n\sum_{\ell=1}^n
		\left(\nabla_\theta s_{\hat\theta_n}(X_\ell)a\right)
		\psi(X_\ell),\label{eq:def-empirical-adjustment-operator}\\
		\mathsf J_{\theta_0,\law}a
		&=\E_{X\sim\law}\left[
		\left(\nabla_\theta s_{\theta_0}(X)a\right)\psi(X)
		\right].\label{eq:def-population-adjustment-operator}
	\end{align}
	We begin by verifying that the population adjustment operator defined in Eqn.~\eqref{eq:def-population-adjustment-operator} is bounded.
	Eqns.~\eqref{eq:model-density-moments}
	and~\eqref{eq:kernel-regularity} imply that, for every $a\in\R^k$ with
	$\|a\|\le1$,
	\begin{align}\label{eq:proof-population-adjustment-operator-bounded}
		\|\mathsf J_{\theta_0,\law}a\|_{\mH^d}
		&\le C_K^{1/2}\E_{X\sim\law}
		[\|\nabla_\theta s_{\theta_0}(X)\|_{\mathrm{op}}]
		\|a\| < \infty.
	\end{align}
	Thus $\mathsf J_{\theta_0,\law}:\R^k\to\mH^d$ is a bounded linear
	operator.

	We now show that the empirical adjustment operator defined in Eqn.~\eqref{eq:def-empirical-adjustment-operator} converges to its
	population counterpart in operator norm. We first compare the empirical
	operator evaluated at $\theta_0$ with $\mathsf J_{\theta_0,\law}$.
	Let $u_1,\ldots,u_k$ be the standard basis of $\R^k$. For every $j$,
	Eqns.~\eqref{eq:model-density-moments} and
	\eqref{eq:kernel-regularity} give
	\begin{align*}
		\E_{X\sim\law}\left[
		\|\left(\nabla_\theta s_{\theta_0}(X)u_j\right)
		\psi(X)\|_{\mH^d}
		\right]
		&\le C_K^{1/2}
		\E_{X\sim\law}\left[
		\|\nabla_\theta s_{\theta_0}(X)\|_{\mathrm{op}}
		\right]<\infty.
	\end{align*}
	The Hilbert-space law of large numbers in each of these $k$ directions
	therefore yields
	\begin{align*}
	&\left\|
	\frac1n\sum_{i=1}^n
	\left(\nabla_\theta s_{\theta_0}(X_i)\,\cdot\right)\psi(X_i)
	-\mathsf J_{\theta_0,\law}
	\right\|_{\mathrm{op}}^2\\
	&\quad\le
	\sum_{j=1}^k\left\|
	\frac1n\sum_{i=1}^n
	\left(\nabla_\theta s_{\theta_0}(X_i)u_j\right)\psi(X_i)
	-\mathsf J_{\theta_0,\law}u_j
	\right\|_{\mH^d}^2
	\overset{p}{\to}0.
	\end{align*}
	Thus the empirical operator evaluated at $\theta_0$ converges to
	$\mathsf J_{\theta_0,\law}$ in operator norm.

	We next control the change caused by replacing $\theta_0$ with
	$\hat\theta_n$. On the event $\hat\theta_n\in U$, Taylor's theorem and
	Eqn.~\eqref{eq:kernel-regularity} give, for every $\|a\|\le1$,
	\begin{align*}
	&\left\|
	\frac1n\sum_{i=1}^n
	\left(
	\nabla_\theta s_{\hat\theta_n}(X_i)
	-\nabla_\theta s_{\theta_0}(X_i)
	\right)a\psi(X_i)
	\right\|_{\mH^d}\hspace{-0.5em}
	\le C_K^{1/2}\|\hat\theta_n-\theta_0\|
	\frac1n\sum_{i=1}^n
	\sup_{\theta\in U}\|\nabla_\theta^2s_\theta(X_i)\|_{\mathrm{op}}.
	\end{align*}
	Taking the supremum over $\|a\|\le1$ and then adding and subtracting
	the empirical operator evaluated at $\theta_0$ gives
	\begin{align}\label{eq:proof-adjustment-operator-convergence}
	&\|\widehat{\mathsf J}_n-\mathsf J_{\theta_0,\law}\|_{\mathrm{op}}\notag\\
	&\quad\le
	\left\|\widehat{\mathsf J}_n
	-\frac1n\sum_{i=1}^n
	\left(\nabla_\theta s_{\theta_0}(X_i)\,\cdot\right)\psi(X_i)
	\right\|_{\mathrm{op}}\!+
	\left\|\frac1n\sum_{i=1}^n
	\left(\nabla_\theta s_{\theta_0}(X_i)\,\cdot\right)\psi(X_i)
	-\mathsf J_{\theta_0,\law}\right\|_{\mathrm{op}}\notag\\
	&\quad\le C_K^{1/2}\left\|\hat\theta_n-\theta_0\right\|
	\frac1n\sum_{i=1}^n
	\sup_{\theta\in U}\left\|\nabla_\theta^2s_\theta(X_i)\right\|_{\mathrm{op}}+
	\left\|\frac1n\sum_{i=1}^n
	\left(\nabla_\theta s_{\theta_0}(X_i)\,\cdot\right)\psi(X_i)
	-\mathsf J_{\theta_0,\law}\right\|_{\mathrm{op}}\notag\\
	&\quad=O_p(n^{-1/2})O_p(1)+o_p(1)=o_p(1).
	\end{align}
	In the last line, the first term is $O_p(n^{-1/2})O_p(1)$ by
	Eqn.~\eqref{eq:proof-ale-rate}, the second line of
	Eqn.~\eqref{eq:model-density-moments}, and the weak law of large numbers.
	The second term is $o_p(1)$ by the standard-basis calculation above.
	Hence
	$\widehat{\mathsf J}_n\overset{p}{\to}\mathsf J_{\theta_0,\law}$ in
	operator norm, as stated in
	Eqn.~\eqref{eq:proof-adjustment-operator-convergence}.
	Since $\mathsf J_{\theta_0,\law}$ is bounded as shown in
	Eqn.~\eqref{eq:proof-population-adjustment-operator-bounded}, the triangle inequality and
	Eqn.~\eqref{eq:proof-adjustment-operator-convergence} then give
	\begin{align}\label{eq:proof-adjustment-operator-bounded}
		\|\widehat{\mathsf J}_n\|_{\mathrm{op}}
		&\le
		\|\widehat{\mathsf J}_n-\mathsf J_{\theta_0,\law}\|_{\mathrm{op}}
		+\|\mathsf J_{\theta_0,\law}\|_{\mathrm{op}}
		=O_p(1).
	\end{align}
	We have therefore established both consistency and stochastic boundedness
	of the fitted adjustment operator.

	We next show that the fitted Stein features are consistent in sample-average
	mean square. On the event $\hat\theta_n\in U$, the line segment from
	$\theta_0$ to $\hat\theta_n$ lies in the ball $U$. Hence Taylor's theorem
	gives, for each $i\in[n]$,
	\begin{align*}
	&s_{\hat\theta_n}(X_i)-s_{\theta_0}(X_i)=\nabla_\theta s_{\theta_0}(X_i)
	(\hat\theta_n-\theta_0)\notag\\
	&\hspace{10em}+\int_0^1(1-t)
	\nabla_\theta^2s_{\theta_0+t(\hat\theta_n-\theta_0)}(X_i)
	[\hat\theta_n-\theta_0,\hat\theta_n-\theta_0] \mathrm dt.
	\end{align*}
	This identity, Eqn.~\eqref{eq:stein-feature}, the reproducing
	property, and Eqn.~\eqref{eq:kernel-regularity} give
	\begin{align}\label{eq:proof-stein-feature-pointwise}
	&\|\xi_{\hat\theta_n}(X_i)
	-\xi_{\theta_0}(X_i)\|_{\mH^d}\notag\\
	&\quad=
	\|\left(s_{\hat\theta_n}(X_i)-s_{\theta_0}(X_i)\right)
	\psi(X_i)\|_{\mH^d}\notag\\
	&\quad=K(X_i,X_i)^{1/2}
	\|s_{\hat\theta_n}(X_i)-s_{\theta_0}(X_i)\|\notag\\
	&\quad\le C_K^{1/2}\left(
	\|\nabla_\theta s_{\theta_0}(X_i)\|_{\mathrm{op}}
	\|\hat\theta_n-\theta_0\|
	+\frac12\|\hat\theta_n-\theta_0\|^2
	\sup_{\theta\in U}
	\|\nabla_\theta^2s_\theta(X_i)\|_{\mathrm{op}}
	\right).
	\end{align}
	To control the squared Hessian remainder after averaging over the sample,
	we first control the sample maximum. The integrability in the second line of
	Eqn.~\eqref{eq:model-density-moments} implies
	\begin{align}\label{eq:proof-hessian-maximum}
		&\P\left(
		\frac1n\max_{1\le i\le n}\sup_{\theta\in U}
		\|\nabla_\theta^2s_\theta(X_i)\|_{\mathrm{op}}>\varepsilon
		\right)\notag\\
		&\quad\le n\P\left(
		\sup_{\theta\in U}\|\nabla_\theta^2s_\theta(X_1)\|_{\mathrm{op}}
		>n\varepsilon\right)\notag\\
		&\quad\le\frac1\varepsilon
		\E\left[
		\sup_{\theta\in U}\|\nabla_\theta^2s_\theta(X_1)\|_{\mathrm{op}}
		\indicator\left(
		\sup_{\theta\in U}\|\nabla_\theta^2s_\theta(X_1)\|_{\mathrm{op}}
		>n\varepsilon\right)\right]
		\to0.
	\end{align}
	Combining Eqn.~\eqref{eq:proof-hessian-maximum} with the weak law of large
	numbers gives
	\begin{align}\label{eq:proof-hessian-square-negligibility}
		&\frac1{n^2}\sum_{i=1}^n
		\left(\sup_{\theta\in U}
		\|\nabla_\theta^2s_\theta(X_i)\|_{\mathrm{op}}\right)^2\notag\\
		&\le
		\left(\frac1n\max_{1\le i\le n}\sup_{\theta\in U}
		\|\nabla_\theta^2s_\theta(X_i)\|_{\mathrm{op}}\right)
		\left(\frac1n\sum_{i=1}^n\sup_{\theta\in U}
		\|\nabla_\theta^2s_\theta(X_i)\|_{\mathrm{op}}\right)
		\overset{p}{\to}0.
	\end{align}
	Applying the Cauchy--Schwarz inequality to Eqn.~\eqref{eq:proof-stein-feature-pointwise} yields
	\begin{align}\label{eq:proof-stein-feature-l2}
	&\frac1n\sum_{i=1}^n
	\|\xi_{\hat\theta_n}(X_i)
	-\xi_{\theta_0}(X_i)\|_{\mH^d}^2\notag\\
	&\quad\le
	2C_K\|\hat\theta_n-\theta_0\|^2
	\frac1n\sum_{i=1}^n
	\|\nabla_\theta s_{\theta_0}(X_i)\|_{\mathrm{op}}^2
	+\frac{C_K}{2}\|\hat\theta_n-\theta_0\|^4
	\frac1n\sum_{i=1}^n
	\left(\sup_{\theta\in U}
	\|\nabla_\theta^2s_\theta(X_i)\|_{\mathrm{op}}\right)^2\notag\\
	&\quad = O_p(n^{-1}) + O_p(n^{-1}) = o_p(1).
	\end{align}
	Indeed, Eqn.~\eqref{eq:proof-ale-rate} and the first line of
	Eqn.~\eqref{eq:model-density-moments} make the first term
	$O_p(n^{-1})O_p(1)=o_p(1)$, while
	Eqns.~\eqref{eq:proof-ale-rate}
	and~\eqref{eq:proof-hessian-square-negligibility} make the second term
	$nO_p(n^{-2})o_p(1)=o_p(1)$.
	Thus the fitted and population Stein features differ by $o_p(1)$ in
	sample-average mean square, as stated in
	Eqn.~\eqref{eq:proof-stein-feature-l2}.

	Having controlled both the fitted Stein feature and the fitted adjustment
	operator, we now prove the mean-square consistency stated at the beginning
	of the proof. By Eqns.~\eqref{eq:influence-adjusted-stein-feature} and
	\eqref{eq:estimated-influence-adjusted-feature}, we have the exact decomposition
	\begin{align*}
		\widehat\Psi_i-\Psi_{\theta_0,\law}(X_i)
		&=\xi_{\hat\theta_n}(X_i)-\xi_{\theta_0}(X_i)+(\widehat{\mathsf J}_n-\mathsf J_{\theta_0,\law})
		I_\law(X_i)
		+\widehat{\mathsf J}_n
		\left(\widehat I_i-I_\law(X_i)\right).
	\end{align*}
	We control the three terms in this decomposition separately. First,
	Eqn.~\eqref{eq:proof-stein-feature-l2} gives
	\begin{align*}
		\frac1n\sum_{i=1}^n
		\|\xi_{\hat\theta_n}(X_i)
		-\xi_{\theta_0}(X_i)\|_{\mH^d}^2
		&\overset{p}{\to}0.
	\end{align*}
	Second, Eqn.~\eqref{eq:proof-adjustment-operator-convergence} and the
	weak law of large numbers applied to $\|I_\law(X)\|^2$ give
	\begin{align*}
	&\frac1n\sum_{i=1}^n
	\|\left(\widehat{\mathsf J}_n-\mathsf J_{\theta_0,\law}\right)
	I_\law(X_i)\|_{\mH^d}^2\le
	\|\widehat{\mathsf J}_n-\mathsf J_{\theta_0,\law}\|_{\mathrm{op}}^2
	\frac1n\sum_{i=1}^n\|I_\law(X_i)\|^2=o_p(1)O_p(1)=o_p(1).
	\end{align*}
	Finally, Eqns.~\eqref{eq:proof-adjustment-operator-bounded}
	and~\eqref{eq:estimated-influence-function-l2} give
	\begin{align*}
	&\frac1n\sum_{i=1}^n
	\|\widehat{\mathsf J}_n
	\left(\widehat I_i-I_\law(X_i)\right)\|_{\mH^d}^2\le
	\|\widehat{\mathsf J}_n\|_{\mathrm{op}}^2
	\frac1n\sum_{i=1}^n
	\|\widehat I_i-I_\law(X_i)\|^2=O_p(1)o_p(1)=o_p(1).
	\end{align*}
	Therefore, applying
	$\|a+b+c\|_{\mH^d}^2\le
	3(\|a\|_{\mH^d}^2+\|b\|_{\mH^d}^2+\|c\|_{\mH^d}^2)$
	to the exact decomposition above and averaging over $i$ give
	\begin{equation}\label{eq:proof-estimated-adjusted-feature-l2}
	\begin{aligned}
	&\frac1n\sum_{i=1}^n
	\|\widehat\Psi_i-\Psi_{\theta_0,\law}(X_i)\|_{\mH^d}^2\\
	&\quad\le
		3\frac1n\sum_{i=1}^n
		\|\xi_{\hat\theta_n}(X_i)
		-\xi_{\theta_0}(X_i)\|_{\mH^d}^2
		+3\frac1n\sum_{i=1}^n
		\|\left(\widehat{\mathsf J}_n-\mathsf J_{\theta_0,\law}\right)
		I_\law(X_i)\|_{\mH^d}^2\\
		&\qquad+
		3\frac1n\sum_{i=1}^n
		\|\widehat{\mathsf J}_n
		\left(\widehat I_i-I_\law(X_i)\right)\|_{\mH^d}^2
	\overset{p}{\to}0.
	\end{aligned}
	\end{equation}
	Thus Eqn.~\eqref{eq:proof-estimated-adjusted-feature-l2} establishes the
	required sample-average $L_2$ approximation.

	We next use this feature approximation to verify Conditions
	\eqref{eq:stable-clt-1} and~\eqref{eq:stable-clt-2}. For this purpose, we
	first establish trace-norm convergence of the empirical covariance and
	maximal negligibility of the empirically centered features.
	Write $\overline{\widehat\Psi}=n^{-1}\sum_{i=1}^n\widehat\Psi_i$ and
	define $\widehat\Sigma_n:\mH^d\to\mH^d$ by
	\begin{align*}
		\widehat\Sigma_n f
		&=\frac1n\sum_{i=1}^n
		\langle f,\widehat\Psi_i-\overline{\widehat\Psi}\rangle_{\mH^d}
		(\widehat\Psi_i-\overline{\widehat\Psi}),\qquad f\in\mH^d.
	\end{align*}
	To justify the Hilbert-space and trace-class laws of large numbers used
	below, we first verify that the centered population feature has a finite
	second moment. Eqns.~\eqref{eq:model-density-moments}
	and~\eqref{eq:kernel-regularity} yield
	\begin{align*}
		\E_{X\sim\law}[\|\xi_{\theta_0}(X)\|_{\mH^d}^2]
		&\le2C_K\E_{X\sim\law}[\|s_{\theta_0}(X)\|^2+1]<\infty.
	\end{align*}
	By Eqns.~\eqref{eq:influence-adjusted-stein-feature}
	and~\eqref{eq:def-population-adjustment-operator},
	\begin{align}\label{eq:proof-adjusted-feature-second-moment}
		\E_{X\sim\law}[\|\Psi_{\theta_0,\law}^{\circ}(X)\|_{\mH^d}^2]
		&\le\E_{X\sim\law}[\|\Psi_{\theta_0,\law}(X)\|_{\mH^d}^2]
		\notag\\
		&=\E_{X\sim\law}\left[
		\|\xi_{\theta_0}(X)
		+\mathsf J_{\theta_0,\law}I_\law(X)\|_{\mH^d}^2
		\right]
		\notag\\
		&\le2\E_{X\sim\law}[\|\xi_{\theta_0}(X)\|_{\mH^d}^2]
		+2\E_{X\sim\law}\left[
		\|\mathsf J_{\theta_0,\law}I_\law(X)\|_{\mH^d}^2
		\right]
		\notag\\
		&\le2\E_{X\sim\law}[\|\xi_{\theta_0}(X)\|_{\mH^d}^2]
		+2\|\mathsf J_{\theta_0,\law}\|_{\mathrm{op}}^2
		\E_{X\sim\law}[\|I_\law(X)\|^2]<\infty.
	\end{align}
	The final expression is finite by
	Eqn.~\eqref{eq:proof-population-adjustment-operator-bounded},
	Assumption~\ref{assu:ale}, and the preceding Stein-feature moment bound.
	Thus $\Psi_{\theta_0,\law}^{\circ}(X)$ is square-integrable.

	We now compare $\widehat\Sigma_n$ with $\Sigma_{\theta_0,\law}$ by
	separating the centered population feature from the feature-estimation
	error. Set
	\begin{align*}
		e_i
		&=\widehat\Psi_i-\Psi_{\theta_0,\law}(X_i),
		&
		\bar e
		&=\frac1n\sum_{i=1}^ne_i,\\
		Y_i
		&=\Psi_{\theta_0,\law}^{\circ}(X_i),
		&
		\bar Y
		&=\frac1n\sum_{i=1}^nY_i.
	\end{align*}
	For comparison, define the empirical covariance of $Y_1,\ldots,Y_n$
	by its action on $f\in\mH^d$:
	\begin{align*}
		\Sigma_n^Y f
		&=\frac1n\sum_{i=1}^n\langle f,Y_i-\bar Y\rangle_{\mH^d}(Y_i-\bar Y)
		=\frac1n\sum_{i=1}^n\langle f,Y_i\rangle_{\mH^d}Y_i
		-\langle f,\bar Y\rangle_{\mH^d}\bar Y.
	\end{align*}
	Write $\|\cdot\|_1$ for the trace norm. For $u,v\in\mH^d$, the operator
	$f\mapsto\langle f,v\rangle_{\mH^d}u$ has at most one nonzero singular
	value, equal to $\|u\|_{\mH^d}\|v\|_{\mH^d}$, which is therefore
	its trace norm. In particular, the random operators
	$f\mapsto\langle f,Y_i\rangle_{\mH^d}Y_i$ have integrable trace norm
	$\|Y_i\|_{\mH^d}^2$ by Eqn.~\eqref{eq:proof-adjusted-feature-second-moment}.
	The law of large numbers in trace norm applies to their empirical
	average, whose expectation is $\Sigma_{\theta_0,\law}$.
	The Hilbert-space law of large numbers also
	gives $\|\bar Y\|_{\mH^d}\overset{p}{\to}0$. Therefore,
	\begin{align}\label{eq:proof-oracle-covariance-convergence}
	\|\Sigma_n^Y-\Sigma_{\theta_0,\law}\|_1
	&\le o_p(1)+\|\bar Y\|_{\mH^d}^2\overset{p}{\to}0.
	\end{align}
	Hence the empirical covariance formed from the centered population
	features converges to $\Sigma_{\theta_0,\law}$ in trace norm.
	Eqn.~\eqref{eq:proof-estimated-adjusted-feature-l2} further gives
	\begin{align}\label{eq:proof-centered-feature-error}
		\frac1n\sum_{i=1}^n\|e_i-\bar e\|_{\mH^d}^2
		&=\frac1n\sum_{i=1}^n\|e_i\|_{\mH^d}^2
		-\|\bar e\|_{\mH^d}^2
		\le\frac1n\sum_{i=1}^n\|e_i\|_{\mH^d}^2
		\overset{p}{\to}0.
	\end{align}
	Thus the sample mean of the squared, empirically centered feature errors
	converges to zero.

	It remains to transfer the preceding population-feature covariance
	convergence to $\widehat\Sigma_n$. Using
	$\widehat\Psi_i-\overline{\widehat\Psi}
	=Y_i-\bar Y+e_i-\bar e$, for $f\in\mH^d$ we have
	\begin{align*}
	&(\widehat\Sigma_n-\Sigma_n^Y)f\\
	&\quad=\frac1n\sum_{i=1}^n\left(
	\langle f,e_i-\bar e\rangle_{\mH^d}(Y_i-\bar Y)
	+\langle f,Y_i-\bar Y\rangle_{\mH^d}(e_i-\bar e)
	+\langle f,e_i-\bar e\rangle_{\mH^d}(e_i-\bar e)
	\right).
	\end{align*}
	The three rank-one operators have trace norms
	$\|Y_i-\bar Y\|_{\mH^d}\|e_i-\bar e\|_{\mH^d}$,
	$\|Y_i-\bar Y\|_{\mH^d}\|e_i-\bar e\|_{\mH^d}$, and
	$\|e_i-\bar e\|_{\mH^d}^2$, respectively. The triangle inequality
	and Cauchy--Schwarz give
	\begin{align*}
	&\|\widehat\Sigma_n-\Sigma_n^Y\|_1\\
	&\quad\le
	2\left(\frac1n\sum_{i=1}^n
	\|Y_i-\bar Y\|_{\mH^d}^2\right)^{1/2}
	\left(\frac1n\sum_{i=1}^n
	\|e_i-\bar e\|_{\mH^d}^2\right)^{1/2}
	+\frac1n\sum_{i=1}^n\|e_i-\bar e\|_{\mH^d}^2
	\overset{p}{\to}0,
	\end{align*}
	where the first empirical second moment is $O_p(1)$ by
	Eqn.~\eqref{eq:proof-adjusted-feature-second-moment}, and the second
	converges to zero by Eqn.~\eqref{eq:proof-centered-feature-error}.
	Combining this bound with
	Eqn.~\eqref{eq:proof-oracle-covariance-convergence} gives
	\begin{align}\label{eq:influence-adjusted-covariance-convergence}
		\|\widehat\Sigma_n-\Sigma_{\theta_0,\law}\|_1
		&\overset{p}{\to}0.
	\end{align}
	This establishes the required trace-norm convergence of the empirical
	covariance operator.

	The conditional variance of $D_{ni}(f)$ therefore satisfies
	\begin{align}\label{eq:iid-multiplier-conditional-variance}
		\sum_{i=1}^n
		\E\left[D_{ni}(f)^2\mid\mathcal F_{n,i-1}\right]
		&=\frac1n\sum_{i=1}^n
		\langle
		\widehat\Psi_i-\overline{\widehat\Psi},f
		\rangle_{\mH^d}^2\notag\\
		&=\langle\widehat\Sigma_nf,f\rangle_{\mH^d}
		\overset{p}{\to}
		\langle\Sigma_{\theta_0,\law}f,f\rangle_{\mH^d}.
	\end{align}
	Thus Eqn.~\eqref{eq:iid-multiplier-conditional-variance} verifies
	Condition~\eqref{eq:stable-clt-1} with
	$\eta^2=\langle\Sigma_{\theta_0,\law}f,f\rangle_{\mH^d}$.

	We next verify Condition~\eqref{eq:stable-clt-2}. We first establish the
	maximal negligibility needed for this verification by showing that the
	largest centered population feature is negligible. For every
	$\varepsilon>0$,
	Eqn.~\eqref{eq:proof-adjusted-feature-second-moment} gives
	\begin{align*}
		\P\left(
		\frac1n\max_{1\le i\le n}\|Y_i\|_{\mH^d}^2>\varepsilon
		\right)
		&\le n\P(\|Y_1\|_{\mH^d}^2>n\varepsilon)
		\le\frac1\varepsilon\E\left[
		\|Y_1\|_{\mH^d}^2
		\indicator(\|Y_1\|_{\mH^d}^2>n\varepsilon)
		\right]\to0.
	\end{align*}
	Since $\|\bar Y\|_{\mH^d}\overset{p}{\to}0$,
	\begin{align*}
		\frac1n\max_{1\le i\le n}\|Y_i-\bar Y\|_{\mH^d}^2
		&\le\frac2n\max_{1\le i\le n}\|Y_i\|_{\mH^d}^2
		+\frac2n\|\bar Y\|_{\mH^d}^2
		\overset{p}{\to}0.
	\end{align*}
	Together with
	Eqn.~\eqref{eq:proof-centered-feature-error}, this yields
	\begin{align*}
	&\frac1n\max_{1\le i\le n}
	\|\widehat\Psi_i-\overline{\widehat\Psi}\|_{\mH^d}^2\le
	\frac2n\max_{1\le i\le n}\|Y_i-\bar Y\|_{\mH^d}^2
	+\frac2n\sum_{i=1}^n\|e_i-\bar e\|_{\mH^d}^2
	\overset{p}{\to}0,
	\end{align*}
	so
	\begin{align}\label{eq:influence-adjusted-maximal-negligibility}
		\frac1n\max_{1\le i\le n}
		\|\widehat\Psi_i-\overline{\widehat\Psi}\|_{\mH^d}^2
		&\overset{p}{\to}0.
	\end{align}
	Eqn.~\eqref{eq:influence-adjusted-maximal-negligibility} gives the
	maximal negligibility of the empirically centered features. We now use it,
	together with the covariance convergence in
	Eqn.~\eqref{eq:influence-adjusted-covariance-convergence}, to verify
	Condition~\eqref{eq:stable-clt-2}.
	Moreover, Eqn.~\eqref{eq:influence-adjusted-maximal-negligibility}
	implies
	\begin{align}\label{eq:iid-multiplier-maximal-projection}
		\frac1{\sqrt n}\max_{1\le i\le n}
		\left|
		\langle
		\widehat\Psi_i-\overline{\widehat\Psi},f
		\rangle_{\mH^d}
		\right|
		&\le\|f\|_{\mH^d}
		\left(
		\frac1n\max_{1\le i\le n}
		\|\widehat\Psi_i-\overline{\widehat\Psi}\|_{\mH^d}^2
		\right)^{1/2}
		\overset{p}{\to}0.
	\end{align}
	For $\varepsilon>0$ and $M>0$, on the event that the left-hand side
	of Eqn.~\eqref{eq:iid-multiplier-maximal-projection} is at most
	$\varepsilon/M$, the conditional Lindeberg sum satisfies
	\begin{align*}
	&\sum_{i=1}^n
	\E\left[
	D_{ni}(f)^2\indicator(|D_{ni}(f)|>\varepsilon)
	\,\middle|\,\mathcal F_{n,i-1}
	\right]\le
	\langle\widehat\Sigma_nf,f\rangle_{\mH^d}
	\E_\nu\left[
	(w^*)^2\indicator(|w^*|>M)
	\right].
	\end{align*}
	Indeed, on this event $|D_{ni}(f)|>\varepsilon$ implies
	$|w_i^*|>M$. Consequently, for every $\eta>0$,
	\begin{align*}
	&\limsup_{n\to\infty}\P\left[
	\sum_{i=1}^n\E\left[
	D_{ni}(f)^2\indicator(|D_{ni}(f)|>\varepsilon)
	\,\middle|\,\mathcal F_{n,i-1}\right]>\eta\right]\\
	&\quad\le
	\limsup_{n\to\infty}\P\left[
	\frac1{\sqrt n}\max_{1\le i\le n}
	|\langle\widehat\Psi_i-\overline{\widehat\Psi},f\rangle_{\mH^d}|
	>\frac\varepsilon M\right]\\
	&\qquad+
	\limsup_{n\to\infty}\P\left[
	\langle\widehat\Sigma_nf,f\rangle_{\mH^d}
	\E_\nu[(w^*)^2\indicator(|w^*|>M)]>\eta\right].
	\end{align*}
	The first term is zero by
	Eqn.~\eqref{eq:iid-multiplier-maximal-projection}. The second can be
	made arbitrarily small by first using
	Eqn.~\eqref{eq:iid-multiplier-conditional-variance} and then letting
	$M\to\infty$, because
	\begin{align*}
		\E_\nu\left[
		(w^*)^2\indicator(|w^*|>M)
		\right]
		&\to0
		\qquad(M\to\infty)
	\end{align*}
	by Eqn.~\eqref{eq:iid-multiplier-weights}. This verifies
	Eqn.~\eqref{eq:stable-clt-2}. Together with
	Eqn.~\eqref{eq:iid-multiplier-conditional-variance},
	Lemma~\ref{lemma:stable-clt} gives, for every $f\in\mH^d$,
	\begin{align}\label{eq:influence-adjusted-scalar-stable}
		\langle\mathbb G_n^*,f\rangle_{\mH^d}
		&\overset{d}{\to}
		\N\left(0,
		\langle\Sigma_{\theta_0,\law}f,f\rangle_{\mH^d}
		\right)
		\qquad\mathcal F_\infty\text{-stably}.
	\end{align}
	Thus Eqn.~\eqref{eq:influence-adjusted-scalar-stable} establishes scalar
	stable convergence for every $f\in\mH^d$. Realize
	$\mathbb G_{\theta_0,\law}$ on an extension of the probability space,
	independently of $\mathcal F_\infty$.

	\paragraph{Step 2: Functional stable convergence.}
	We pass from the scalar limits in Step~1 to functional stable convergence
	by using the following lemma. Its proof is deferred to
	Appendix~\ref{sec:proof_functional_stable_convergence}.
	\begin{lemma}[Functional stable convergence]
	\label{lemma:functional-stable-convergence}
		Suppose $L_n:\mH^d\to\R$ are bounded linear random functionals. Let $S$
		be a centered Gaussian process with covariance $\sigma(f,g)$, realized
		independently of $\mathcal F_\infty$ on an extension of the probability
		space. Suppose the following conditions hold:
		\begin{enumerate}
			\item[(1)] For every $f\in\mH^d$,
			\begin{align}\label{eq:c1}
				L_n(f)\overset{d}{\to}S(f)\sim\N(0,\sigma(f,f))\qquad\mathcal F_\infty\text{-stably}.\tag{C1.1}
			\end{align}
			\item[(2)] For some orthonormal basis $(\phi_i)_{i\ge1}$ of $\mH^d$,
			\begin{align}\label{eq:c2}
				\sum_{i=1}^\infty\sigma(\phi_i,\phi_i)<\infty.\tag{C1.2}
			\end{align}
			\item[(3)] For the same basis and every $\varepsilon>0$,
			\begin{align}\label{eq:c3}
				\lim_{k\to\infty}\limsup_{n\to\infty}\P\left(\sum_{i=k}^\infty L_n(\phi_i)^2\ge\varepsilon\,\middle|\,F\right)=0,\tag{C1.3}
			\end{align}
			for every $F\in\mathcal F_\infty$ with $\P(F)>0$.
		\end{enumerate}
		Then
		\begin{align*}
			\left\{L_n(f)^2\right\}_{f\in\mH^d}\overset{d}{\to}\left\{S(f)^2\right\}_{f\in\mH^d}\qquad\mathcal F_\infty\text{-stably}.
		\end{align*}
	\end{lemma}

	Apply Lemma~\ref{lemma:functional-stable-convergence} with
	$L_n(f)=\langle\mathbb G_n^*,f\rangle_{\mH^d}$ and
	$\sigma(f,g)=\langle\Sigma_{\theta_0,\law}f,g\rangle_{\mH^d}$. For every
	$n$, the Cauchy--Schwarz inequality gives
	\begin{align*}
		|L_n(f)|\le\|\mathbb G_n^*\|_{\mH^d}\|f\|_{\mH^d},
	\end{align*}
	so $L_n$ is a bounded linear random functional.

	\paragraph{Check of Condition~\eqref{eq:c1}.}
	Eqn.~\eqref{eq:influence-adjusted-scalar-stable} gives
	$L_n(f)\overset{d}{\to}\N(0,\sigma(f,f))$
	$\mathcal F_\infty$-stably for every $f\in\mH^d$, which verifies
	Condition~\eqref{eq:c1}.

	\paragraph{Check of Condition~\eqref{eq:c2}.}
	For any orthonormal
	basis $(\phi_r)_{r\ge1}$ of $\mH^d$,
	Eqn.~\eqref{eq:proof-adjusted-feature-second-moment} gives
	\begin{align*}
		\sum_{r=1}^\infty\sigma(\phi_r,\phi_r)
		&=\sum_{r=1}^\infty
		\langle\Sigma_{\theta_0,\law}\phi_r,\phi_r\rangle_{\mH^d}
		=\tr(\Sigma_{\theta_0,\law})
		=\E_{X\sim\law}
		[\|\Psi_{\theta_0,\law}^{\circ}(X)\|_{\mH^d}^2]<\infty,
	\end{align*}
	which verifies Condition~\eqref{eq:c2}.

	\paragraph{Check of Condition~\eqref{eq:c3}.}
	We verify the required tail tightness in the chosen orthonormal basis.
	Conditional Markov's inequality and
	Eqn.~\eqref{eq:influence-adjusted-covariance-convergence} give
	\begin{align}\label{eq:influence-adjusted-tail-bound}
	&\P\left[
	\sum_{r>m}\langle\mathbb G_n^*,\phi_r\rangle_{\mH^d}^2
	>\varepsilon\,\middle|\,\mathcal F_\infty
	\right]\notag\\
	&\quad\le\frac1\varepsilon
	\sum_{r>m}\langle\widehat\Sigma_n\phi_r,\phi_r\rangle_{\mH^d}
	=\frac1\varepsilon\left(
	\sum_{r>m}\langle
	(\widehat\Sigma_n-\Sigma_{\theta_0,\law})\phi_r,
	\phi_r\rangle_{\mH^d}
	+\sum_{r>m}\langle\Sigma_{\theta_0,\law}\phi_r,
	\phi_r\rangle_{\mH^d}\right)\notag\\
	&\quad\le\frac1\varepsilon
	\left(
	\|\widehat\Sigma_n-\Sigma_{\theta_0,\law}\|_1
	+\sum_{r>m}
	\langle\Sigma_{\theta_0,\law}\phi_r,\phi_r\rangle_{\mH^d}
	\right).
	\end{align}
	For every $F\in\mathcal F_\infty$ with $\P(F)>0$, the probability in
	Condition~\eqref{eq:c3} is bounded by
	\begin{align*}
	&\frac1{\P(F)}\E\left[
	\indicator(F)\min\left\{1,\frac1\varepsilon
	\left(
	\|\widehat\Sigma_n-\Sigma_{\theta_0,\law}\|_1
	+\sum_{r>m}\langle\Sigma_{\theta_0,\law}\phi_r,
	\phi_r\rangle_{\mH^d}
	\right)\right\}\right].
	\end{align*}
	For fixed $m$, Eqn.~\eqref{eq:influence-adjusted-covariance-convergence}
	and boundedness of the integrand imply convergence in $L_1$ to the same
	bound with $\|\widehat\Sigma_n-\Sigma_{\theta_0,\law}\|_1$ replaced by
	zero. That limit tends to zero as $m\to\infty$ because
	$\Sigma_{\theta_0,\law}$ is trace class. Thus
	Condition~\eqref{eq:c3} holds.

	All three conditions of
	Lemma~\ref{lemma:functional-stable-convergence} have now been verified.
	Lemma~\ref{lemma:functional-stable-convergence} now gives
	\begin{align*}
		\left\{\langle\mathbb G_n^*,f\rangle_{\mH^d}^2:
		f\in\mH^d\right\}
		&\overset{d}{\to}
		\left\{\langle\mathbb G_{\theta_0,\law},f\rangle_{\mH^d}^2:
		f\in\mH^d\right\}
		\qquad\mathcal F_\infty\text{-stably}.
	\end{align*}
	Applying the continuous map
	\begin{align*}
		\{\langle z,f\rangle_{\mH^d}^2:f\in\mH^d\}
		&\longmapsto
		\sup_{\|f\|_{\mH^d}\le1}
		\langle z,f\rangle_{\mH^d}^2
		=\|z\|_{\mH^d}^2
	\end{align*}
	gives
	\begin{align*}
		n\widetilde T_{n,\mathrm{IA}}
		&\overset{d}{\to}W_\law
		\qquad \mathcal F_\infty\text{-stably}.
	\end{align*}
	This proves the stable convergence of the bootstrap statistic.

	\paragraph{Step 3: Conditional convergence.}
	We now use the stable convergence from Step~2 to prove
	Eqn.~\eqref{eq:general-law-bootstrap-convergence}. Stable convergence of
	a single bootstrap replication does not by itself imply convergence of its
	conditional law, so consider two multiplier sequences
	$(w_i^{*(1)})_{i=1}^n$ and $(w_i^{*(2)})_{i=1}^n$ that, conditional on
	$\mathcal F_n$, are independent copies of $(w_i^*)_{i=1}^n$. For
	$b=1,2$, define
	\begin{align*}
		\mathbb G_n^{*(b)}
		&=\frac1{\sqrt n}\sum_{i=1}^n
		w_i^{*(b)}(\widehat\Psi_i-\overline{\widehat\Psi}),
		&
		n\widetilde T_{n,\mathrm{IA}}^{(b)}
		&=\|\mathbb G_n^{*(b)}\|_{\mH^d}^2.
	\end{align*}
	Since
	\begin{align*}
		\E[w_i^{*(b)}w_i^{*(c)}\mid\mathcal F_n]
		&=\indicator(b=c),
		\qquad b,c=1,2,
	\end{align*}
	the conditional covariance operator of
	$(\mathbb G_n^{*(1)},\mathbb G_n^{*(2)})$ is
	\begin{align*}
		\begin{pmatrix}
		\widehat\Sigma_n&0\\
		0&\widehat\Sigma_n
		\end{pmatrix}
		&\overset{p}{\to}
		\begin{pmatrix}
		\Sigma_{\theta_0,\law}&0\\
		0&\Sigma_{\theta_0,\law}
		\end{pmatrix}
		\qquad\text{in trace norm}
	\end{align*}
	by Eqn.~\eqref{eq:influence-adjusted-covariance-convergence}. Moreover,
	Eqn.~\eqref{eq:influence-adjusted-maximal-negligibility} gives
	\begin{align*}
		\frac1n\max_{1\le i\le n}
		\left\|
		(\widehat\Psi_i-\overline{\widehat\Psi},
		\widehat\Psi_i-\overline{\widehat\Psi})
		\right\|_{(\mH^d)^2}^2
		&=\frac2n\max_{1\le i\le n}
		\|\widehat\Psi_i-\overline{\widehat\Psi}\|_{\mH^d}^2
		\overset{p}{\to}0.
	\end{align*}
	Therefore the covariance, Lindeberg, and tail calculations in Steps~1
	and~2 apply in the product Hilbert space $(\mH^d)^2$. If
	$\mathbb G_{\theta_0,\law}^{(1)}$ and
	$\mathbb G_{\theta_0,\law}^{(2)}$ are independent copies of
	$\mathbb G_{\theta_0,\law}$, those steps give
	\begin{align*}
		(\mathbb G_n^{*(1)},\mathbb G_n^{*(2)})
		&\overset{d}{\to}
		(\mathbb G_{\theta_0,\law}^{(1)},
		\mathbb G_{\theta_0,\law}^{(2)})
		\qquad\mathcal F_\infty\text{-stably}.
	\end{align*}
	The componentwise continuous mapping theorem then yields
	\begin{align}\label{eq:two-replication-stable-convergence}
		(n\widetilde T_{n,\mathrm{IA}}^{(1)},
		n\widetilde T_{n,\mathrm{IA}}^{(2)})
		&\overset{d}{\to}(W_\law^{(1)},W_\law^{(2)})
		\qquad\mathcal F_\infty\text{-stably},
	\end{align}
	where $W_\law^{(1)}$ and $W_\law^{(2)}$ are independent copies of
	$W_\law$, independent of $\mathcal F_\infty$.

	Fix a continuity point $t$ of the distribution of $W_\law$ and denote
	the conditional distribution function by
	\begin{align*}
		F_n^*(t)
		&=\P\left[
		n\widetilde T_{n,\mathrm{IA}}\le t\mid\mathcal F_n
		\right].
	\end{align*}
	The two bootstrap statistics are conditionally independent and have the
	same conditional distribution, so
	\begin{align*}
		\E[F_n^*(t)]
		&=\P[n\widetilde T_{n,\mathrm{IA}}^{(1)}\le t],\\
		\E[F_n^*(t)^2]
		&=\P[n\widetilde T_{n,\mathrm{IA}}^{(1)}\le t,
		n\widetilde T_{n,\mathrm{IA}}^{(2)}\le t].
	\end{align*}
	Eqn.~\eqref{eq:two-replication-stable-convergence} implies ordinary joint
	weak convergence. Hence
	\begin{align*}
		\E[F_n^*(t)]
		&\to\P[W_\law\le t],
		&
		\E[F_n^*(t)^2]
		&\to\P[W_\law\le t]^2.
	\end{align*}
	Consequently,
	\begin{align*}
		\E\left[
		\left|F_n^*(t)-\P[W_\law\le t]\right|^2
		\right]
		&\to0.
	\end{align*}
	Thus, for every $\delta>0$, Markov's inequality gives
	\begin{align*}
		\P\left[
		\left|F_n^*(t)-\P[W_\law\le t]\right|>\delta
		\right]
		&\le\frac1{\delta^2}
		\E\left[
		\left|F_n^*(t)-\P[W_\law\le t]\right|^2
		\right]
		\to0.
	\end{align*}
	This is Eqn.~\eqref{eq:general-law-bootstrap-convergence}, which proves the
	desired conditional bootstrap convergence.
\end{proof}

\subsection{Proof of Corollary~\ref{cor:influence-adjusted-bootstrap-validity}}
\label{sec:proof_influence_adjusted_bootstrap_validity}
\begin{proof}
	Under $\law=\law_{\theta_0}$, the shorthand following
	Eqn.~\eqref{eq:influence-adjusted-stein-feature} gives
	$\Psi_{\theta_0,\law_{\theta_0}}(x)=\Psi_{\theta_0}(x)$, while
	Eqn.~\eqref{eq:adjusted-feature-centering} gives
	\begin{align*}
		\E_{X\sim\law_{\theta_0}}
		\left[\Psi_{\theta_0,\law_{\theta_0}}(X)\right]
		&=\E_{X\sim\law_{\theta_0}}
		\left[\Psi_{\theta_0}(X)\right]=0.
	\end{align*}
	Therefore, for every $x\in\X$,
	\begin{align*}
		\Psi_{\theta_0,\law_{\theta_0}}^{\circ}(x)
		&=\Psi_{\theta_0,\law_{\theta_0}}(x)
		-\E_{X\sim\law_{\theta_0}}
		\left[\Psi_{\theta_0,\law_{\theta_0}}(X)\right]
		=\Psi_{\theta_0}(x),\\
		\Sigma_{\theta_0,\law_{\theta_0}}f
		&=\E_{X\sim\law_{\theta_0}}[
		\langle f,\Psi_{\theta_0}(X)\rangle_{\mH^d}\Psi_{\theta_0}(X)]
		=\Sigma_{\theta_0}f,\qquad f\in\mH^d.
	\end{align*}
	Hence $W_\law\overset{d}{=}W$, where $W$ is defined in
	Theorem~\ref{thm:asymptotic_distribution}.

	We next verify the regularity of this limiting distribution. Since
	$\Sigma_{\theta_0}\ne0$, there is an index $u_0$ such that
	$\lambda_{u_0}>0$. Write
	\begin{align*}
		W
		&=\lambda_{u_0}Z_{u_0}^2
		+\sum_{u\ne u_0}\lambda_uZ_u^2.
	\end{align*}
	The two terms are independent, and $\lambda_{u_0}Z_{u_0}^2$ has a
	continuous density that is positive on $(0,\infty)$. To see that
	$\sum_{u\ne u_0}\lambda_uZ_u^2$ has zero in its
	support, fix $\varepsilon>0$. Since $\sum_{u\ge1}\lambda_u<\infty$, choose
	$m\ge u_0$ such that $\sum_{u>m}\lambda_u<\varepsilon/4$. Markov's inequality
	then gives
	\begin{align*}
		\P\left(\sum_{u>m}\lambda_uZ_u^2<\frac\varepsilon2\right)
		&\ge1-\frac{2}{\varepsilon}\sum_{u>m}\lambda_u>\frac12,
	\end{align*}
	whereas independence and positivity of the Gaussian small-ball
	probabilities give
	\begin{align*}
		\P\left(\sum_{\substack{u\le m\\u\ne u_0}}
		\lambda_uZ_u^2<\frac\varepsilon2\right)&>0.
	\end{align*}
	The head and tail are independent, so
	$\P(\sum_{u\ne u_0}\lambda_uZ_u^2<\varepsilon)>0$. Thus
	convolution with
	$\lambda_{u_0}Z_{u_0}^2$ shows that $W$ has a continuous distribution
	function that is strictly increasing on $(0,\infty)$. Here $\Q_\tau(W)$ denotes the $\tau$-quantile of the unconditional distribution of $W$. In particular, we have
	\begin{align}\label{eq:null-limit-quantile-regularity}
		\Q_{1-\alpha}(W)&\in(0,\infty),
		&
		\P(W=\Q_{1-\alpha}(W))&=0.
	\end{align}

	To obtain convergence of the bootstrap critical value, we use the
	following lemma, proved in Appendix~\ref{sec:proof_conditional_quantile_convergence}.
	\begin{lemma}[Conditional quantile convergence]
	\label{lemma:quantile-converge-in-prob}
		Let $\alpha\in(0,1)$ and suppose $W$ has a continuous distribution
		function that is strictly increasing on an open interval containing
		$\Q_\alpha(W)$. If, for every $t\in\R$,
		\begin{align*}
			\P(W_n\le t\mid\F_n)
			&\overset{p}{\to}\P(W\le t),
		\end{align*}
		then
		\begin{align*}
			\Q_\alpha(W_n)
			&\overset{p}{\to}\Q_\alpha(W),
		\end{align*}
		where $\Q_\alpha(W_n)$ is the quantile conditional on $\F_n$.
	\end{lemma}
	Eqn.~\eqref{eq:general-law-bootstrap-convergence} states that the
	conditional distribution of $n\widetilde T_{n,\mathrm{IA}}$ given
	$\mathcal F_n$ converges weakly in probability to the distribution of
	$W$, because $W_\law\overset{d}{=}W$ under the null. Applying
	Lemma~\ref{lemma:quantile-converge-in-prob} therefore gives
	\begin{align*}
		\Q_{1-\alpha}(n\widetilde T_{n,\mathrm{IA}})
		&\overset{p}{\to}\Q_{1-\alpha}(W).
	\end{align*}
	Theorem~\ref{thm:asymptotic_distribution} and Slutsky's theorem now imply
	\begin{align*}
		nT_n-\Q_{1-\alpha}(n\widetilde T_{n,\mathrm{IA}})
		&\overset{d}{\to}W-\Q_{1-\alpha}(W).
	\end{align*}
	Finally, quantiles commute with multiplication by $n>0$:
	\begin{align*}
		\Q_{1-\alpha}(n\widetilde T_{n,\mathrm{IA}})
		&=n\Q_{1-\alpha}(\widetilde T_{n,\mathrm{IA}}).
	\end{align*}
	Thus Eqn.~\eqref{eq:influence-adjusted-bootstrap-test} and
	$\P(W=\Q_{1-\alpha}(W))=0$ give
	\begin{align*}
		\P_{\law_{\theta_0}}\left[\phi_{n,\alpha}^{\mathrm{IA}}=1\right]
		&=\P_{\law_{\theta_0}}\left[
		nT_n\ge\Q_{1-\alpha}(n\widetilde T_{n,\mathrm{IA}})
		\right] \to
		\P\left[W\ge\Q_{1-\alpha}(W)\right]=\alpha.
	\end{align*}
\end{proof}

\subsection{Proof of Theorem~\ref{thm:bootstrap_distribution}}\label{sec:proof_bootstrap_distribution}
\begin{proof}[Proof of Theorem~\ref{thm:bootstrap_distribution}]
    We use the same three-step structure as the proof of
    Theorem~\ref{thm:influence-adjusted-bootstrap}. First, we linearize
    the bootstrap feature, including the effect of refitting the parameter,
    and prove scalar stable convergence. Second, we check the conditions
    of Lemma~\ref{lemma:functional-stable-convergence} and obtain the
    stable limit of the bootstrap statistic. Third, we use two bootstrap
    replications to prove conditional distribution convergence. All outer
    probabilities are under the data-generating law $\law$, which enters
    through $\hat\theta_n\convp\theta_0$. All moment bounds below are
    under the resampling model laws, as specified in
    Assumption~\ref{ass:p-bootstrap}. All
    feature and operator bounds are used on a sufficiently small
    neighborhood of $\theta_0$, denoted as $U$. Set the summands in the
    linearized array to zero on $\hat\theta_n\notin U$. This does not
    affect any limit, since
    $\P(\hat\theta_n\notin U)\to0$.

    \paragraph{Step 1: Linearization and scalar stable convergence.}
    Conditional on $\mathcal F_n$, the variables
    $\widetilde X_1,\ldots,\widetilde X_n$ are i.i.d. from
    $\law_{\hat\theta_n}$, and
    $\widetilde\theta_n=\mathcal E(\widetilde X_1,\ldots,\widetilde X_n)$.
    We first justify the expansion of this refitted estimator.
    Eqn.~\eqref{eq:uale} means that, for every $\varepsilon>0$,
    \begin{equation}\label{eq:uale-uniform-probability}
        \sup_{\theta\in U}\P_{\law_\theta^{\otimes n}}\left(
        \left\|\sqrt n(\hat\theta_n-\theta)
        -\frac1{\sqrt n}\sum_{i=1}^n I_{\law_\theta}(X_i)\right\|
        >\varepsilon\right)\to0,
    \end{equation}
    where $X_1,\ldots,X_n$ are i.i.d. from $\law_\theta$ in each
    probability. Consequently, on $\hat\theta_n\in U$,
    \begin{align*}
        &\P\left[
        \left\|\sqrt n(\widetilde\theta_n-\hat\theta_n)
        -\frac1{\sqrt n}\sum_{i=1}^n
        I_{\law_{\hat\theta_n}}(\widetilde X_i)\right\|>\varepsilon
        \,\middle|\,\mathcal F_n\right]\\
        &\quad\le\sup_{\theta\in U}\P_{\law_\theta^{\otimes n}}\left(
        \left\|\sqrt n(\hat\theta_n-\theta)
        -\frac1{\sqrt n}\sum_{i=1}^n I_{\law_\theta}(X_i)\right\|
        >\varepsilon\right)\to0.
    \end{align*}
    Since $\P(\hat\theta_n\notin U)\to0$, this proves
    \begin{equation}\label{eq:parametric-bootstrap-uale-expansion}
        \sqrt n(\widetilde\theta_n-\hat\theta_n)
        =\frac1{\sqrt n}\sum_{i=1}^n
        I_{\law_{\hat\theta_n}}(\widetilde X_i)+o_p(1).
    \end{equation}
    \paragraph{Explanation of the asymptotic notation.} Here and below, $o_p(1)$ is understood under the joint law of the
    data and bootstrap sample. For any measurable remainder $R_n$,
    this is equivalent to convergence to zero in conditional probability,
    in probability:
    \[
        R_n=o_p(1)
        \quad\Longleftrightarrow\quad
        \P(\|R_n\|>\varepsilon\mid\mathcal F_n)\convp0
        \quad\text{for every }\varepsilon>0.
    \]
    Indeed, for every $\varepsilon,\delta>0$, the tower property gives
    \[
        \P(\|R_n\|>\varepsilon)
        =\E[\P(\|R_n\|>\varepsilon\mid\mathcal F_n)]
        \le\delta+\P\left(
        \P(\|R_n\|>\varepsilon\mid\mathcal F_n)>\delta\right).
    \]
    Conditional convergence therefore implies unconditional convergence
    by first letting $n\to\infty$ and then $\delta\downarrow0$.
    Conversely, Markov's inequality and the tower property give
    \[
        \P\left(\P(\|R_n\|>\varepsilon\mid\mathcal F_n)>\delta\right)
        \le\frac{\P(\|R_n\|>\varepsilon)}{\delta},
    \]
    which tends to zero whenever $R_n=o_p(1)$ under the joint law.
    Thus either formulation can be used for the remainder calculations
    below.

    The centering in Assumption~\ref{assu:uale} makes the cross terms
    vanish, so, on $\hat\theta_n\in U$,
    \begin{align}\label{eq:parametric-bootstrap-influence-second-moment}
        \E\Biggl[\biggl\|\frac1{\sqrt n}\sum_{i=1}^n
        I_{\law_{\hat\theta_n}}(\widetilde X_i)\biggr\|^2
        \,\Biggm|\,\mathcal F_n\Biggr]
        &=\frac1n\sum_{i=1}^n
        \E\left[\|I_{\law_{\hat\theta_n}}(\widetilde X_i)\|^2
        \,\middle|\,\mathcal F_n\right]\notag\\
        &=\E_{X\sim\law_{\hat\theta_n}}
        [\|I_{\law_{\hat\theta_n}}(X)\|^2]
        \le\sup_{\theta\in U}\E_{X\sim\law_\theta}
        [\|I_{\law_\theta}(X)\|^2]<\infty.
    \end{align}
    Conditional Markov's inequality and
    Eqns.~\eqref{eq:parametric-bootstrap-uale-expansion}--\eqref{eq:parametric-bootstrap-influence-second-moment}
    therefore imply
    \begin{equation}\label{eq:parametric-bootstrap-refit-rate}
        \sqrt n(\widetilde\theta_n-\hat\theta_n)=O_p(1),
        \qquad \widetilde\theta_n-\hat\theta_n=o_p(1).
    \end{equation}
    Here $O_p(1)$ has the corresponding conditional tightness meaning.
    In particular, both parameters lie in the ball $U$ with conditional
    probability tending to one, in probability.

    We next replace the empirical average of the parameter derivative of the score function,
    weighted by the kernel feature $\psi$, by its model expectation.
    We first calculate the conditional second moment in
    each coordinate direction, then control the operator norm.
    Let $u_1,\ldots,u_k$ be the standard basis of $\R^k$.
    Since $\|\psi(x)\|_{\mH}^2=K(x,x)$,
    Eqn.~\eqref{eq:kernel-regularity} gives
    \begin{equation}\label{eq:parametric-bootstrap-derivative-feature-bound}
        \|(\nabla_\theta s_\theta(x)u_j)\psi(x)\|_{\mH^d}^2
        =\|\nabla_\theta s_\theta(x)u_j\|^2K(x,x)
        \le C_K\|\nabla_\theta s_\theta(x)\|_{\mathrm{op}}^2.
    \end{equation}
    Conditional independence removes the cross terms between distinct
    centered observations. Thus, on $\hat\theta_n\in U$,
    \begin{align}
        &\E\Biggl[\biggl\|\frac1n\sum_{i=1}^n
        (\nabla_\theta s_{\hat\theta_n}(\widetilde X_i)u_j)\psi(\widetilde X_i)
        -\E_{X\sim\law_{\hat\theta_n}}\bigl[
        (\nabla_\theta s_{\hat\theta_n}(X)u_j)\psi(X)\bigr]
        \biggr\|_{\mH^d}^2\,\Biggm|\,\mathcal F_n\Biggr]\notag\\
        &\quad=\frac1n\E_{X\sim\law_{\hat\theta_n}}\left[
        \|(\nabla_\theta s_{\hat\theta_n}(X)u_j)\psi(X)\|_{\mH^d}^2\right]
        -\frac1n\left\|\E_{X\sim\law_{\hat\theta_n}}\left[
        (\nabla_\theta s_{\hat\theta_n}(X)u_j)\psi(X)\right]\right\|_{\mH^d}^2\notag\\
        &\quad\le\frac{C_K}{n}\E_{X\sim\law_{\hat\theta_n}}\left[
        \|\nabla_\theta s_{\hat\theta_n}(X)\|_{\mathrm{op}}^2\right]
        \le\frac{C_K}{n}\sup_{\theta\in U}\E_{X\sim\law_\theta}\left[
        \|\nabla_\theta s_\theta(X)\|_{\mathrm{op}}^2\right].
        \label{eq:parametric-bootstrap-derivative-variance}
    \end{align}
    The first inequality uses
    Eqn.~\eqref{eq:parametric-bootstrap-derivative-feature-bound}, and
    the supremum is finite by
    Eqn.~\eqref{eq:parametric-bootstrap-derivative-moments}.
    For a linear operator $A:\R^k\to\mH^d$, Cauchy--Schwarz gives
    \begin{equation}\label{eq:parametric-bootstrap-operator-coordinate-bound}
        \|A\|_{\mathrm{op}}^2
        =\sup_{\|a\|\le1}\left\|\sum_{j=1}^k a_jAu_j\right\|_{\mH^d}^2
        \le\sum_{j=1}^k\|Au_j\|_{\mH^d}^2.
    \end{equation}
    For the empirical-minus-population operator, its value at $u_j$
    is precisely the centered coordinate average in
    Eqn.~\eqref{eq:parametric-bootstrap-derivative-variance}.
    Consequently, on $\hat\theta_n\in U$,
    \begin{align}\label{eq:parametric-bootstrap-operator-variance}
        &\E\Biggl[\biggl\|\frac1n\sum_{i=1}^n
        (\nabla_\theta s_{\hat\theta_n}(\widetilde X_i)\,\cdot)\psi(\widetilde X_i)
        -\E_{X\sim\law_{\hat\theta_n}}\bigl[
        (\nabla_\theta s_{\hat\theta_n}(X)\,\cdot)\psi(X)\bigr]
        \biggr\|_{\mathrm{op}}^2\,\Biggm|\,\mathcal F_n\Biggr]\notag\\
        &\quad\overset{(a)}{\le}\sum_{j=1}^k
        \E\Biggl[\biggl\|\frac1n\sum_{i=1}^n
        (\nabla_\theta s_{\hat\theta_n}(\widetilde X_i)u_j)\psi(\widetilde X_i)
        -\E_{X\sim\law_{\hat\theta_n}}\bigl[
        (\nabla_\theta s_{\hat\theta_n}(X)u_j)\psi(X)\bigr]
        \biggr\|_{\mH^d}^2\,\Biggm|\,\mathcal F_n\Biggr]\notag\\
        &\quad\overset{(b)}{\le}\frac{kC_K}{n}\sup_{\theta\in U}
        \E_{X\sim\law_\theta}\left[
        \|\nabla_\theta s_\theta(X)\|_{\mathrm{op}}^2\right].
    \end{align}
    Here $(a)$ applies
    Eqn.~\eqref{eq:parametric-bootstrap-operator-coordinate-bound}
    inside the conditional expectation, and $(b)$ sums
    Eqn.~\eqref{eq:parametric-bootstrap-derivative-variance} over
    $j=1,\ldots,k$.
    For $\varepsilon>0$, separate the event $\{\hat\theta_n\notin U\}$
    and apply conditional Markov's inequality and the tower property:
    \begin{align}\label{eq:parametric-bootstrap-derivative-lln}
        &\P\Biggl(\biggl\|\frac1n\sum_{i=1}^n
        (\nabla_\theta s_{\hat\theta_n}(\widetilde X_i)\,\cdot)\psi(\widetilde X_i)
        -\E_{X\sim\law_{\hat\theta_n}}\bigl[
        (\nabla_\theta s_{\hat\theta_n}(X)\,\cdot)\psi(X)\bigr]
        \biggr\|_{\mathrm{op}}>\varepsilon\Biggr)\notag\\
        &\quad\le\P(\hat\theta_n\notin U)
        +\frac1{\varepsilon^2}\E\left[
        \vphantom{\E\Biggl[\biggl\|\sum_{i=1}^n\biggr\|_{\mathrm{op}}^2
        \,\Biggm|\,\mathcal F_n\Biggr]}
        \indicator(\hat\theta_n\in U)\right.\notag\\
        &\qquad\qquad\left.\times\E\Biggl[\biggl\|\frac1n\sum_{i=1}^n
        (\nabla_\theta s_{\hat\theta_n}(\widetilde X_i)\,\cdot)\psi(\widetilde X_i)
        -\E_{X\sim\law_{\hat\theta_n}}\bigl[
        (\nabla_\theta s_{\hat\theta_n}(X)\,\cdot)\psi(X)\bigr]
        \biggr\|_{\mathrm{op}}^2\,\Biggm|\,\mathcal F_n\Biggr]\right]\notag\\
        &\quad\le\P(\hat\theta_n\notin U)
        +\frac{kC_K}{n\varepsilon^2}\sup_{\theta\in U}
        \E_{X\sim\law_\theta}\left[
        \|\nabla_\theta s_\theta(X)\|_{\mathrm{op}}^2\right]
        \to0.
    \end{align}
    The last inequality substitutes
    Eqn.~\eqref{eq:parametric-bootstrap-operator-variance}. The limit uses
    $\hat\theta_n\convp\theta_0$ and the finite supremum in
    Eqn.~\eqref{eq:parametric-bootstrap-derivative-moments}.
    This proves convergence of the empirical derivative to its model
    expectation in operator norm, both unconditionally and conditionally
    in probability by the equivalence following
    Eqn.~\eqref{eq:parametric-bootstrap-uale-expansion}.

    We also need a uniform bound for the population operator when
    substituting the estimator expansion. The triangle inequality,
    Cauchy--Schwarz, and
    Eqns.~\eqref{eq:kernel-regularity} and
    \eqref{eq:parametric-bootstrap-derivative-moments} give
    \begin{align}\label{eq:parametric-bootstrap-derivative-bound}
        &\sup_{\theta\in U}\left\|\E_{X\sim\law_\theta}\left[
        (\nabla_\theta s_\theta(X)\,\cdot)\psi(X)\right]\right\|_{\mathrm{op}}\notag\\
        &\quad\le C_K^{1/2}\sup_{\theta\in U}\E_{X\sim\law_\theta}\left[
        \|\nabla_\theta s_\theta(X)\|_{\mathrm{op}}\right]
        \le\left(C_K\sup_{\theta\in U}\E_{X\sim\law_\theta}\left[
        \|\nabla_\theta s_\theta(X)\|_{\mathrm{op}}^2\right]\right)^{1/2}<\infty.
    \end{align}

    We now control the Taylor remainder in the bootstrap feature.
    On $\hat\theta_n,\widetilde\theta_n\in U$, the line segment joining
    them lies in $U$. Substituting Eqn.~\eqref{eq:stein-feature} and
    applying the fundamental theorem of calculus, justified by
    Eqn.~\eqref{eq:model-density-regularity}, gives
    \begin{align}\label{eq:parametric-bootstrap-pointwise-taylor}
        &\xi_{\widetilde\theta_n}(x)
        -\xi_{\hat\theta_n}(x)
        -(\nabla_\theta s_{\hat\theta_n}(x)
        (\widetilde\theta_n-\hat\theta_n))\psi(x)\notag\\
        &\quad=\left(s_{\widetilde\theta_n}(x)-s_{\hat\theta_n}(x)
        -\nabla_\theta s_{\hat\theta_n}(x)
        (\widetilde\theta_n-\hat\theta_n)\right)\psi(x)\notag\\
        &\quad=\int_0^1
        \left(\nabla_\theta s_{\hat\theta_n+
        t(\widetilde\theta_n-\hat\theta_n)}(x)
        -\nabla_\theta s_{\hat\theta_n}(x)\right)
        (\widetilde\theta_n-\hat\theta_n)\psi(x)\,\d t.
    \end{align}
    A second application of the fundamental theorem of calculus gives,
    for $0\le t\le1$,
    \begin{equation}\label{eq:parametric-bootstrap-taylor-derivative-bound}
        \left\|\nabla_\theta s_{\hat\theta_n+t(\widetilde\theta_n-\hat\theta_n)}(x)
        -\nabla_\theta s_{\hat\theta_n}(x)\right\|_{\mathrm{op}}
        \le t\|\widetilde\theta_n-\hat\theta_n\|
        \sup_{\eta\in U}\|\nabla_\eta^2s_\eta(x)\|_{\mathrm{op}}.
    \end{equation}
    We first apply the triangle inequality to the sum and to the integral
    in Eqn.~\eqref{eq:parametric-bootstrap-pointwise-taylor}. For
    $v\in\R^d$, the product feature satisfies
    $\|v\psi(x)\|_{\mH^d}=\|v\|\sqrt{K(x,x)}$.
    We then substitute
    Eqn.~\eqref{eq:parametric-bootstrap-taylor-derivative-bound} and
    integrate $t$ explicitly:
    \begin{align}
        &\left\|\frac1{\sqrt n}\sum_{i=1}^n\left(
        \xi_{\widetilde\theta_n}(\widetilde X_i)
        -\xi_{\hat\theta_n}(\widetilde X_i)
        -(\nabla_\theta s_{\hat\theta_n}(\widetilde X_i)
        (\widetilde\theta_n-\hat\theta_n))\psi(\widetilde X_i)
        \right)\right\|_{\mH^d}\notag\\
        &\quad\overset{(a)}{\le}\frac1{\sqrt n}\sum_{i=1}^n\int_0^1
        \Bigl\|\bigl(\nabla_\theta s_{\hat\theta_n+
        t(\widetilde\theta_n-\hat\theta_n)}(\widetilde X_i)
        -\nabla_\theta s_{\hat\theta_n}(\widetilde X_i)\bigr)
        (\widetilde\theta_n-\hat\theta_n)\psi(\widetilde X_i)
        \Bigr\|_{\mH^d}\,\d t\notag\\
        &\quad\overset{(b)}{\le}\frac{C_K^{1/2}}{\sqrt n}
        \|\widetilde\theta_n-\hat\theta_n\|
        \sum_{i=1}^n\int_0^1
        \left\|\nabla_\theta s_{\hat\theta_n+t(\widetilde\theta_n-\hat\theta_n)}(\widetilde X_i)
        -\nabla_\theta s_{\hat\theta_n}(\widetilde X_i)\right\|_{\mathrm{op}}\,\d t\notag\\
        &\quad\overset{(c)}{\le}\frac{C_K^{1/2}}2\sqrt n
        \|\widetilde\theta_n-\hat\theta_n\|^2
        \frac1n\sum_{i=1}^n\sup_{\eta\in U}
        \|\nabla_\eta^2s_\eta(\widetilde X_i)\|_{\mathrm{op}}.
        \label{eq:parametric-bootstrap-taylor-bound}
    \end{align}
    Here $(a)$ uses Eqn.~\eqref{eq:parametric-bootstrap-pointwise-taylor}
    and the two triangle inequalities. For $(b)$, we use the product-feature
    identity above, the operator-norm inequality, and
    $K(x,x)\le C_K$ in Eqn.~\eqref{eq:kernel-regularity}. For $(c)$, we
    substitute Eqn.~\eqref{eq:parametric-bootstrap-taylor-derivative-bound}.
    The last equality uses $\int_0^1t\,\d t=1/2$ and
    $n^{-1/2}\sum_{i=1}^n=\sqrt n\,n^{-1}\sum_{i=1}^n$.
    To bound the empirical Hessian average in
    Eqn.~\eqref{eq:parametric-bootstrap-taylor-bound}, use
    Eqn.~\eqref{eq:parametric-bootstrap-derivative-moments}. On
    $\hat\theta_n\in U$,
    \begin{align}\label{eq:parametric-bootstrap-hessian-average-moment}
        &\E\left[\frac1n\sum_{i=1}^n\sup_{\eta\in U}
        \|\nabla_\eta^2s_\eta(\widetilde X_i)\|_{\mathrm{op}}
        \,\middle|\,\mathcal F_n\right]\notag\\
        &\quad=\E_{X\sim\law_{\hat\theta_n}}\left[
        \sup_{\eta\in U}\|\nabla_\eta^2s_\eta(X)\|_{\mathrm{op}}\right]
        \le\sup_{\theta\in U}\E_{X\sim\law_\theta}\left[
        \sup_{\eta\in U}\|\nabla_\eta^2s_\eta(X)\|_{\mathrm{op}}\right]<\infty.
    \end{align}
    Conditional Markov's inequality in
    Eqn.~\eqref{eq:parametric-bootstrap-hessian-average-moment} gives,
    for every $M>0$,
    \begin{align*}
        &\P\left(\frac1n\sum_{i=1}^n\sup_{\eta\in U}
        \|\nabla_\eta^2s_\eta(\widetilde X_i)\|_{\mathrm{op}}>M\right)\le\P(\hat\theta_n\notin U)
        +\frac1M\sup_{\theta\in U}\E_{X\sim\law_\theta}\left[
        \sup_{\eta\in U}\|\nabla_\eta^2s_\eta(X)\|_{\mathrm{op}}\right].
    \end{align*}
    Letting $n\to\infty$ and then $M\to\infty$ therefore proves
    \begin{equation}\label{eq:parametric-bootstrap-hessian-average-tightness}
        \frac1n\sum_{i=1}^n\sup_{\eta\in U}
        \|\nabla_\eta^2s_\eta(\widetilde X_i)\|_{\mathrm{op}}=O_p(1).
    \end{equation}
    Together with Eqn.~\eqref{eq:parametric-bootstrap-refit-rate}, this
    controls the product in Eqn.~\eqref{eq:parametric-bootstrap-taylor-bound}:
    \begin{align*}
        &\sqrt n\|\widetilde\theta_n-\hat\theta_n\|^2
        \frac1n\sum_{i=1}^n\sup_{\eta\in U}
        \|\nabla_\eta^2s_\eta(\widetilde X_i)\|_{\mathrm{op}}\\
        &\quad=\frac1{\sqrt n}\|\sqrt n(\widetilde\theta_n-\hat\theta_n)\|^2
        \frac1n\sum_{i=1}^n\sup_{\eta\in U}
        \|\nabla_\eta^2s_\eta(\widetilde X_i)\|_{\mathrm{op}}
        =O_p(n^{-1/2})O_p(1)=o_p(1).
    \end{align*}
    Since $\hat\theta_n\convp\theta_0$ and
    Eqn.~\eqref{eq:parametric-bootstrap-refit-rate} imply
    $\P(\hat\theta_n\notin U)+\P(\widetilde\theta_n\notin U)\to0$,
    Eqn.~\eqref{eq:parametric-bootstrap-taylor-bound} proves
    \begin{align}\label{eq:parametric-bootstrap-taylor-remainder}
        \frac1{\sqrt n}\sum_{i=1}^n
        \xi_{\widetilde\theta_n}(\widetilde X_i)
        &=\frac1{\sqrt n}\sum_{i=1}^n
        \xi_{\hat\theta_n}(\widetilde X_i)\notag\\
        &\quad+\left(\frac1n\sum_{i=1}^n
        (\nabla_\theta s_{\hat\theta_n}(\widetilde X_i)\,\cdot)
        \psi(\widetilde X_i)\right)
        \sqrt n(\widetilde\theta_n-\hat\theta_n)+o_p(1).
    \end{align}

    It remains to identify the first-order term in
    Eqn.~\eqref{eq:parametric-bootstrap-taylor-remainder} with the
    influence adjustment. First, replacing its empirical derivative
    operator by the model expectation has error bounded by
    \begin{align}
        &\Biggl\|\biggl(\frac1n\sum_{i=1}^n
        (\nabla_\theta s_{\hat\theta_n}(\widetilde X_i)\,\cdot)\psi(\widetilde X_i)
        -\E_{X\sim\law_{\hat\theta_n}}\left[
        (\nabla_\theta s_{\hat\theta_n}(X)\,\cdot)\psi(X)\right]\biggr)
        \sqrt n(\widetilde\theta_n-\hat\theta_n)\Biggr\|_{\mH^d}\notag\\
        &\quad\le\left\|\frac1n\sum_{i=1}^n
        (\nabla_\theta s_{\hat\theta_n}(\widetilde X_i)\,\cdot)\psi(\widetilde X_i)
        -\E_{X\sim\law_{\hat\theta_n}}\left[
        (\nabla_\theta s_{\hat\theta_n}(X)\,\cdot)\psi(X)\right]
        \right\|_{\mathrm{op}}
        \|\sqrt n(\widetilde\theta_n-\hat\theta_n)\|\notag\\
        &\quad=o_p(1)O_p(1)=o_p(1),
        \label{eq:parametric-bootstrap-derivative-replacement-error}
    \end{align}
    where the last line uses
    Eqns.~\eqref{eq:parametric-bootstrap-derivative-lln} and
    \eqref{eq:parametric-bootstrap-refit-rate}.
    Next, substituting Eqn.~\eqref{eq:parametric-bootstrap-uale-expansion}
    into the population first-order term has error bounded by
    \begin{align}
        &\left\|\E_{X\sim\law_{\hat\theta_n}}\left[
        (\nabla_\theta s_{\hat\theta_n}(X)\,\cdot)\psi(X)\right]
        \left(\sqrt n(\widetilde\theta_n-\hat\theta_n)
        -\frac1{\sqrt n}\sum_{i=1}^n
        I_{\law_{\hat\theta_n}}(\widetilde X_i)\right)\right\|_{\mH^d}\notag\\
        &\quad\le\left\|\E_{X\sim\law_{\hat\theta_n}}\left[
        (\nabla_\theta s_{\hat\theta_n}(X)\,\cdot)\psi(X)\right]\right\|_{\mathrm{op}}
        \left\|\sqrt n(\widetilde\theta_n-\hat\theta_n)
        -\frac1{\sqrt n}\sum_{i=1}^n
        I_{\law_{\hat\theta_n}}(\widetilde X_i)\right\|\notag\\
        &\quad=O_p(1)o_p(1)=o_p(1),
        \label{eq:parametric-bootstrap-influence-replacement-error}
    \end{align}
    by Eqns.~\eqref{eq:parametric-bootstrap-derivative-bound} and
    \eqref{eq:parametric-bootstrap-uale-expansion}.
    Substituting
    Eqns.~\eqref{eq:parametric-bootstrap-derivative-replacement-error}
    and~\eqref{eq:parametric-bootstrap-influence-replacement-error}
    into Eqn.~\eqref{eq:parametric-bootstrap-taylor-remainder} gives
    \begin{align}
        \frac1{\sqrt n}\sum_{i=1}^n
        \xi_{\widetilde\theta_n}(\widetilde X_i)
        &=\frac1{\sqrt n}\sum_{i=1}^n
        \xi_{\hat\theta_n}(\widetilde X_i)\notag\\
        &\quad+\E_{X\sim\law_{\hat\theta_n}}\left[
        (\nabla_\theta s_{\hat\theta_n}(X)\,\cdot)\psi(X)\right]
        \frac1{\sqrt n}\sum_{i=1}^n
        I_{\law_{\hat\theta_n}}(\widetilde X_i)+o_p(1)\notag\\
        &=\frac1{\sqrt n}\sum_{i=1}^n\left(
        \xi_{\hat\theta_n}(\widetilde X_i)
        +\E_{X\sim\law_{\hat\theta_n}}\left[
        \nabla_\theta s_{\hat\theta_n}(X)
        I_{\law_{\hat\theta_n}}(\widetilde X_i)\psi(X)\right]\right)
        +o_p(1)\notag\\
        &=\frac1{\sqrt n}\sum_{i=1}^n
        \Psi_{\hat\theta_n}(\widetilde X_i)+o_p(1).
        \label{eq:parametric-bootstrap-feature-expansion}
    \end{align}
    The second equality uses linearity of the population operator,
    and the last equality uses the definition in
    Eqn.~\eqref{eq:influence-adjusted-stein-feature}. This establishes
    the required linearization, including the parameter-refitting term.

    We now prove scalar stable convergence of the linearized feature
    in Eqn.~\eqref{eq:parametric-bootstrap-feature-expansion}. Define
    \begin{equation}\label{eq:parametric-bootstrap-scalar-projection}
        S^*(f,x,\theta)\equiv\langle f,\Psi_\theta(x)\rangle_{\mH^d}
        =\mathcal A_\theta f(x)+
        \left\langle\E_{X'\sim\law_\theta}
        [\nabla_\theta s_\theta(X')^\top f(X')],
        I_{\law_\theta}(x)\right\rangle_{\R^k}.
    \end{equation}
    For its empirical sum, write
    \[
        \widetilde S_n(f)=\frac1{\sqrt n}\sum_{i=1}^n
        S^*(f,\widetilde X_i,\hat\theta_n)
        =\left\langle f,\frac1{\sqrt n}\sum_{i=1}^n
        \Psi_{\hat\theta_n}(\widetilde X_i)\right\rangle_{\mH^d}.
    \]
    We first use the following centering and integrability result. Its
    proof is deferred to Appendix~\ref{sec:proof_oracle_property}.
    \begin{lemma}[Centering and moment bounds for bootstrap features]
    \label{lemma:oracle-property}
        Under Assumption~\ref{assu:uale} and the regularity conditions in
        Appendix~\ref{sec:parametric-bootstrap-conditions}, after shrinking $U$
        if necessary, for every $\theta\in U$ and $f\in\mH^d$,
        \begin{equation}
            \E_{X\sim\law_\theta}[\Psi_\theta(X)]=0,
            \qquad \E_{X\sim\law_\theta}[S^*(f,X,\theta)]=0.
            \label{eq:parametric-bootstrap-centering}
        \end{equation}
        Moreover,
        \begin{equation}
            \sup_{\theta\in U}\E_{X\sim\law_\theta}
            [\|\Psi_\theta(X)\|_{\mH^d}^2]<\infty.
            \label{eq:parametric-bootstrap-feature-moment}
        \end{equation}
        In particular,
        \begin{equation}\label{eq:parametric-bootstrap-scalar-moment}
            \E_{X\sim\law_\theta}[S^*(f,X,\theta)^2]
            \le\|f\|_{\mH^d}^2
            \E_{X\sim\law_\theta}[\|\Psi_\theta(X)\|_{\mH^d}^2].
        \end{equation}
    \end{lemma}

    Realize the bootstrap samples using independent random seeds across
    rows and observations, independent of
    $\mathcal F_\infty=\sigma(X_1,X_2,\ldots)$. Thus their conditional law
    given $\mathcal F_\infty$ equals that given $\mathcal F_n$.
    As in the proof of Theorem~\ref{thm:influence-adjusted-bootstrap}, set
    \[
        \mathcal F_{n,0}=\mathcal F_\infty,\qquad
        \mathcal F_{n,i}=\sigma(\mathcal F_\infty,
        \widetilde X_1,\ldots,\widetilde X_i),\qquad 1\le i\le n,
    \]
    where the row index on $\widetilde X_i$ is suppressed. The rows
    are conditionally independent given $\mathcal F_\infty$, so the limiting
    sigma-field in Lemma~\ref{lemma:stable-clt} is $\mathcal F_\infty$
    (up to null sets). For fixed $f\in\mH^d$, we apply that lemma to
    $n^{-1/2}S^*(f,\widetilde X_i,\hat\theta_n)$, with the zero
    convention outside $U$ specified above. On $\hat\theta_n\in U$,
    Eqns.~\eqref{eq:parametric-bootstrap-centering} and
    \eqref{eq:parametric-bootstrap-scalar-moment} give
    \begin{align*}
        \E\left[\frac1{\sqrt n}S^*(f,\widetilde X_i,\hat\theta_n)
        \,\middle|\,\mathcal F_{n,i-1}\right]
        &=\frac1{\sqrt n}\E_{X\sim\law_{\hat\theta_n}}
        [S^*(f,X,\hat\theta_n)]=0,\\
        \E\left[\frac1nS^*(f,\widetilde X_i,\hat\theta_n)^2
        \,\middle|\,\mathcal F_{n,i-1}\right]
        &\le\frac{\|f\|_{\mH^d}^2}{n}
        \sup_{\theta\in U}\E_{X\sim\law_\theta}
        [\|\Psi_\theta(X)\|_{\mH^d}^2]<\infty.
    \end{align*}
    The zero convention and the uniform bound make this a
    square-integrable martingale difference array. We check the two
    remaining conditions of Lemma~\ref{lemma:stable-clt} next.

    \paragraph{Check of Condition~\eqref{eq:stable-clt-1}.}
    We first prove trace-norm convergence of the model covariance
    operators. We use the following continuity and moment result, whose
    proof is deferred to
    Appendix~\ref{sec:proof_parametric_bootstrap_feature_continuity}.
    \begin{lemma}[Continuity and second moment of the adjusted feature]
    \label{lemma:parametric-bootstrap-feature-continuity}
    Under Assumptions~\ref{assu:uale}, \ref{ass:p-bootstrap}, \ref{ass:k},
    and~\ref{ass:parametric-bootstrap}, as $\theta\to\theta_0$ within $U$,
    \begin{equation}\label{eq:parametric-bootstrap-feature-continuity}
        \int_\X\left\|
        \Psi_\theta(x)\sqrt{p_\theta(x)}
        -\Psi_{\theta_0}(x)\sqrt{p_{\theta_0}(x)}
        \right\|_{\mH^d}^2\,\mu(\d x)\to0,
    \end{equation}
    and
    \begin{equation}\label{eq:parametric-bootstrap-base-feature-moment}
        \E_{X\sim\law_{\theta_0}}
        [\|\Psi_{\theta_0}(X)\|_{\mH^d}^2]<\infty.
    \end{equation}
    \end{lemma}

    For $\theta\in U$, define $\Sigma_\theta:\mH^d\to\mH^d$ by
    \begin{equation}\label{eq:parametric-bootstrap-covariance-operator}
        \Sigma_\theta f=\E_{X\sim\law_\theta}
        [\langle f,\Psi_\theta(X)\rangle_{\mH^d}\Psi_\theta(X)],
        \qquad f\in\mH^d.
    \end{equation}
    Use the same zero convention for $\Sigma_{\hat\theta_n}$ on
    $\hat\theta_n\notin U$.
    This agrees with
    Eqn.~\eqref{eq:asymptotic_distribution} at $\theta_0$.
    For each $x\in\X$, set $a=\Psi_\theta(x)\sqrt{p_\theta(x)}$ and
    $b=\Psi_{\theta_0}(x)\sqrt{p_{\theta_0}(x)}$. The identity
    \[
        \langle f,a\rangle_{\mH^d}a-\langle f,b\rangle_{\mH^d}b
        =\langle f,a-b\rangle_{\mH^d}a+\langle f,b\rangle_{\mH^d}(a-b)
    \]
    decomposes the difference into the actions of two rank-one operators
    with trace norms
    $\|a\|_{\mH^d}\|a-b\|_{\mH^d}$ and
    $\|b\|_{\mH^d}\|a-b\|_{\mH^d}$.
    Applying the triangle inequality and Cauchy--Schwarz to
    Eqn.~\eqref{eq:parametric-bootstrap-covariance-operator} therefore gives
    \begin{align}
        &\|\Sigma_\theta-\Sigma_{\theta_0}\|_1\notag\\
        &\quad\le\int_\X
        \left(\|\Psi_\theta(x)\|_{\mH^d}\sqrt{p_\theta(x)}
        +\|\Psi_{\theta_0}(x)\|_{\mH^d}\sqrt{p_{\theta_0}(x)}\right)\notag\\
        &\qquad\qquad\times
        \left\|\Psi_\theta(x)\sqrt{p_\theta(x)}
        -\Psi_{\theta_0}(x)\sqrt{p_{\theta_0}(x)}\right\|_{\mH^d}
        \,\mu(\d x)\notag\\
        &\quad\le\left(2\E_{\law_\theta}[\|\Psi_\theta(X)\|_{\mH^d}^2]
        +2\E_{\law_{\theta_0}}[\|\Psi_{\theta_0}(X)\|_{\mH^d}^2]
        \right)^{1/2}\notag\\
        &\qquad\times\left(\int_\X
        \left\|\Psi_\theta(x)\sqrt{p_\theta(x)}
        -\Psi_{\theta_0}(x)\sqrt{p_{\theta_0}(x)}\right\|_{\mH^d}^2
        \,\mu(\d x)\right)^{1/2}\to0.
        \label{eq:parametric-bootstrap-covariance-continuity}
    \end{align}
    The limit follows from
    Eqns.~\eqref{eq:parametric-bootstrap-feature-continuity} and
    \eqref{eq:parametric-bootstrap-feature-moment}. In particular,
    $\|\Sigma_{\hat\theta_n}-\Sigma_{\theta_0}\|_1\convp0$.
    Conditional independence and
    Eqn.~\eqref{eq:parametric-bootstrap-scalar-projection} now give,
    on $\hat\theta_n\in U$,
    \begin{align}\label{eq:parametric-bootstrap-scalar-variance}
        \sum_{i=1}^n\E\left[
        \frac1n S^*(f,\widetilde X_i,\hat\theta_n)^2
        \,\middle|\,\mathcal F_{n,i-1}\right]
        &=\E_{X\sim\law_{\hat\theta_n}}
        [\langle f,\Psi_{\hat\theta_n}(X)\rangle_{\mH^d}^2]\notag\\
        &=\langle\Sigma_{\hat\theta_n}f,f\rangle_{\mH^d}
        \convp\langle\Sigma_{\theta_0}f,f\rangle_{\mH^d}.
    \end{align}
    Indeed, the absolute difference of the last two quadratic forms is
    at most $\|f\|_{\mH^d}^2
    \|\Sigma_{\hat\theta_n}-\Sigma_{\theta_0}\|_1$.
    Since $\P(\hat\theta_n\notin U)\to0$, this proves
    Condition~\eqref{eq:stable-clt-1} for the localized array as well.

    \paragraph{Check of Condition~\eqref{eq:stable-clt-2}.}
    We first control the feature tail, then bound each scalar tail by
    that feature tail. For $R>0$, use
    $\|a\|_{\mH^d}^2\indicator(\|a\|_{\mH^d}>R)
    \le4\|a-b\|_{\mH^d}^2+2\|b\|_{\mH^d}^2\indicator(\|b\|_{\mH^d}>R/2)$ with
    $a=\Psi_\theta(x)$ and
    $b=\Psi_{\theta_0}(x)\sqrt{p_{\theta_0}(x)/p_\theta(x)}$.
    Splitting the last indicator according to
    $p_\theta(x)\ge p_{\theta_0}(x)/2$ gives
    \begin{align}
        &\E_{X\sim\law_\theta}\left[
        \|\Psi_\theta(X)\|_{\mH^d}^2
        \indicator(\|\Psi_\theta(X)\|_{\mH^d}>R)\right]\notag\\
        &\quad\le4\int_\X
        \left\|\Psi_\theta(x)\sqrt{p_\theta(x)}
        -\Psi_{\theta_0}(x)\sqrt{p_{\theta_0}(x)}\right\|_{\mH^d}^2
        \,\mu(\d x)\notag\\
        &\qquad+2\E_{X\sim\law_{\theta_0}}\left[
        \|\Psi_{\theta_0}(X)\|_{\mH^d}^2
        \indicator(\|\Psi_{\theta_0}(X)\|_{\mH^d}>R/(2\sqrt2))\right]\notag\\
        &\qquad+2\int_\X\|\Psi_{\theta_0}(x)\|_{\mH^d}^2p_{\theta_0}(x)
        \indicator(p_\theta(x)<p_{\theta_0}(x)/2)\,\mu(\d x).
        \label{eq:parametric-bootstrap-tail-comparison}
    \end{align}
    To control the third term on the right-hand side of
    Eqn.~\eqref{eq:parametric-bootstrap-tail-comparison}, we use the
    following lemma. Its proof is deferred to
    Appendix~\ref{sec:proof_parametric_bootstrap_derivative_continuity}.
    \begin{lemma}[Continuity of the density and the parameter derivative of the score function]
    \label{lemma:parametric-bootstrap-derivative-continuity}
    Under Assumption~\ref{ass:p-bootstrap}, as $\theta\to\theta_0$ within $U$,
    \begin{equation}\label{eq:parametric-bootstrap-density-l1}
        p_\theta(x)\to p_{\theta_0}(x),\qquad
        \int_\X|p_\theta(x)-p_{\theta_0}(x)|\,\mu(\d x)
        \to 0,
    \end{equation}
    where the pointwise limit holds for $\mu$-almost every $x$, and
    \begin{equation}\label{eq:parametric-bootstrap-moments}
        \int_\X\|\nabla_\theta s_\theta(x)p_\theta(x)
        -\nabla_\theta s_{\theta_0}(x)p_{\theta_0}(x)\|_{\mathrm{op}}
        \,\mu(\d x)\to 0.
    \end{equation}
    Moreover, the parameter derivative of the score function satisfies the local mean-continuity property
    \begin{equation}\label{eq:parametric-bootstrap-derivative-modulus}
        \sup_{\theta\in U}\E_{X\sim\law_\theta}\left[
        \sup_{\substack{\eta\in U\\\|\eta-\theta\|\le\delta}}
        \|\nabla_\eta s_\eta(X)-\nabla_\theta s_\theta(X)\|_{\mathrm{op}}
        \right]\to 0\quad\text{as }\delta \to 0^+.
    \end{equation}
    \end{lemma}

    The first term in Eqn.~\eqref{eq:parametric-bootstrap-tail-comparison}
    tends to zero by
    Eqn.~\eqref{eq:parametric-bootstrap-feature-continuity}.
    The last tends to zero by
    Eqn.~\eqref{eq:parametric-bootstrap-density-l1}, positivity of
    $p_{\theta_0}$, and dominated convergence. The middle term tends
    to zero as $R\to\infty$ by
    Eqn.~\eqref{eq:parametric-bootstrap-base-feature-moment}. Hence, setting
    $\theta=\hat\theta_n$ and $R=\varepsilon\sqrt n$,
    \begin{equation}\label{eq:parametric-bootstrap-lindeberg}
        \E_{X\sim\law_{\hat\theta_n}}\left[
        \|\Psi_{\hat\theta_n}(X)\|_{\mH^d}^2
        \indicator(\|\Psi_{\hat\theta_n}(X)\|_{\mH^d}
        >\varepsilon\sqrt n)\right]\convp0
        \qquad(\varepsilon>0).
    \end{equation}
    For $f\ne0$, conditional independence and
    Eqn.~\eqref{eq:parametric-bootstrap-scalar-projection} give,
    on $\hat\theta_n\in U$,
    \begin{align}\label{eq:parametric-bootstrap-scalar-lindeberg}
        &\sum_{i=1}^n\E\Biggl[
        \frac1nS^*(f,\widetilde X_i,\hat\theta_n)^2
        \indicator\bigl(|S^*(f,\widetilde X_i,\hat\theta_n)|
        \ge\varepsilon\sqrt n\bigr)
        \,\Biggm|\,\mathcal F_{n,i-1}\Biggr]\notag\\
        &\quad=\E_{X\sim\law_{\hat\theta_n}}\Bigl[
        |\langle f,\Psi_{\hat\theta_n}(X)\rangle_{\mH^d}|^2
        \indicator\bigl(|\langle f,\Psi_{\hat\theta_n}(X)
        \rangle_{\mH^d}|\ge\varepsilon\sqrt n\bigr)\Bigr]\notag\\
        &\quad\le\|f\|_{\mH^d}^2
        \E_{X\sim\law_{\hat\theta_n}}\left[
        \|\Psi_{\hat\theta_n}(X)\|_{\mH^d}^2
        \indicator(\|\Psi_{\hat\theta_n}(X)\|_{\mH^d}>
        \varepsilon\sqrt n/(2\|f\|_{\mH^d}))\right]\convp0.
    \end{align}
    The inequality uses Cauchy--Schwarz, and the limit substitutes
    $\varepsilon/(2\|f\|_{\mH^d})$ for $\varepsilon$ in
    Eqn.~\eqref{eq:parametric-bootstrap-lindeberg}.
    For $f=0$, the sum is zero. Thus
    Condition~\eqref{eq:stable-clt-2} holds. Together with
    Eqn.~\eqref{eq:parametric-bootstrap-scalar-variance},
    Lemma~\ref{lemma:stable-clt} proves
    \begin{equation}\label{eq:parametric-bootstrap-scalar-stable}
        \widetilde S_n(f)\overset{d}{\to}
        \langle f,\mathbb G_{\theta_0}\rangle_{\mH^d}
        \sim\N(0,\langle\Sigma_{\theta_0}f,f\rangle_{\mH^d})
        \qquad\mathcal F_\infty\text{-stably}.
    \end{equation}
    This completes the linearization and scalar stable convergence in
    Step~1. Realize $\mathbb G_{\theta_0}$ on an extension of the probability
    space, independently of $\mathcal F_\infty$.

    \paragraph{Step 2: Functional stable convergence.}
    We apply Lemma~\ref{lemma:functional-stable-convergence} with
    $L_n(f)=\widetilde S_n(f)$ and
    $\sigma(f,g)=\langle\Sigma_{\theta_0}f,g\rangle_{\mH^d}$.
    The definition of $\widetilde S_n$ and Cauchy--Schwarz give
    \[
        |\widetilde S_n(f)|\le\|f\|_{\mH^d}
        \left\|\frac1{\sqrt n}\sum_{i=1}^n
        \Psi_{\hat\theta_n}(\widetilde X_i)\right\|_{\mH^d},
    \]
    so $L_n$ is a bounded linear random functional. We check each
    condition of the lemma before applying it to the squared norm.

    \paragraph{Check of Condition~\eqref{eq:c1}.}
    Eqn.~\eqref{eq:parametric-bootstrap-scalar-stable} gives
    $L_n(f)\overset{d}{\to}\N(0,\sigma(f,f))$
    $\mathcal F_\infty$-stably for every $f\in\mH^d$.
    This is Condition~\eqref{eq:c1}.

    \paragraph{Check of Condition~\eqref{eq:c2}.}
    Fix an orthonormal basis $(\phi_r)_{r\ge1}$ of $\mH^d$.
    Tonelli's theorem, Parseval's identity, and
    Eqn.~\eqref{eq:parametric-bootstrap-feature-moment} give
    \begin{align*}
        \sum_{r=1}^\infty\sigma(\phi_r,\phi_r)
        &=\E_{X\sim\law_{\theta_0}}\left[
        \sum_{r=1}^\infty
        \langle\Psi_{\theta_0}(X),\phi_r\rangle_{\mH^d}^2\right]\\
        &=\E_{X\sim\law_{\theta_0}}
        [\|\Psi_{\theta_0}(X)\|_{\mH^d}^2]
        =\tr(\Sigma_{\theta_0})<\infty.
    \end{align*}
    This verifies Condition~\eqref{eq:c2}.

    \paragraph{Check of Condition~\eqref{eq:c3}.}
    We control the coordinates beyond the first $m$ basis vectors.
    On $\hat\theta_n\in U$, conditional independence and
    Eqn.~\eqref{eq:parametric-bootstrap-centering} remove the cross
    terms between distinct observations. Consequently,
    \begin{align}
        \E\left[\sum_{r>m}\widetilde S_n(\phi_r)^2
        \,\middle|\,\mathcal F_\infty\right]
        &=\sum_{r>m}\frac1n\sum_{i=1}^n
        \E[S^*(\phi_r,\widetilde X_i,\hat\theta_n)^2
        \mid\mathcal F_\infty]\notag\\
        &=\sum_{r>m}\langle\Sigma_{\hat\theta_n}\phi_r,\phi_r
        \rangle_{\mH^d}\notag\\
        &\le\|\Sigma_{\hat\theta_n}-\Sigma_{\theta_0}\|_1
        +\sum_{r>m}\langle\Sigma_{\theta_0}\phi_r,\phi_r
        \rangle_{\mH^d}.
        \label{eq:parametric-bootstrap-basis-tail}
    \end{align}
    For $F\in\mathcal F_\infty$ with $\P(F)>0$ and $\varepsilon>0$,
    conditional Markov's inequality and the tower property yield
    \begin{align*}
        &\P\left(\sum_{r>m}\widetilde S_n(\phi_r)^2
        \ge\varepsilon\,\middle|\,F\right)\\
        &\quad\le\frac{\P(\hat\theta_n\notin U)}{\P(F)}
        +\frac1{\P(F)}\E\left[\indicator(F)
        \min\left(1,\frac1\varepsilon\left(
        \|\Sigma_{\hat\theta_n}-\Sigma_{\theta_0}\|_1
        +\sum_{r>m}\langle\Sigma_{\theta_0}\phi_r,\phi_r
        \rangle_{\mH^d}\right)\right)\right].
    \end{align*}
    For fixed $m$, Eqn.~\eqref{eq:parametric-bootstrap-covariance-continuity}
    and boundedness of the integrand give
    \[
        \limsup_{n\to\infty}\P\left(
        \sum_{r>m}\widetilde S_n(\phi_r)^2\ge\varepsilon
        \,\middle|\,F\right)
        \le\min\left(1,\frac1\varepsilon
        \sum_{r>m}\langle\Sigma_{\theta_0}\phi_r,\phi_r
        \rangle_{\mH^d}\right)\to0
    \]
    as $m\to\infty$ by Condition~\eqref{eq:c2}.
    This verifies Condition~\eqref{eq:c3}.

    All three conditions hold, so
    Lemma~\ref{lemma:functional-stable-convergence} gives
    \[
        \bigl\{\widetilde S_n(f)^2:f\in\mH^d\bigr\}
        \overset{d}{\to}
        \bigl\{\langle f,\mathbb G_{\theta_0}\rangle_{\mH^d}^2:
        f\in\mH^d\bigr\}
        \qquad\mathcal F_\infty\text{-stably}.
    \]
    As in Step~2 of the proof of
    Theorem~\ref{thm:influence-adjusted-bootstrap}, take the supremum
    over the unit ball to obtain
    \begin{equation}\label{eq:parametric-bootstrap-linearized-stable}
        \left\|\frac1{\sqrt n}\sum_{i=1}^n
        \Psi_{\hat\theta_n}(\widetilde X_i)\right\|_{\mH^d}^2
        =\sup_{\|f\|_{\mH^d}\le1}\widetilde S_n(f)^2
        \overset{d}{\to}\|\mathbb G_{\theta_0}\|_{\mH^d}^2
        =W\qquad\mathcal F_\infty\text{-stably}.
    \end{equation}

    We next transfer this stable limit to the refitted statistic.
    By the same reproducing-property calculation as
    Eqn.~\eqref{eq:theorem-two-statistic-representation},
    \[
        n\widetilde T_{n,\mathrm{PA}}
        =\left\|\frac1{\sqrt n}\sum_{i=1}^n
        \xi_{\widetilde\theta_n}(\widetilde X_i)\right\|_{\mH^d}^2.
    \]
    Eqn.~\eqref{eq:parametric-bootstrap-linearized-stable} makes the norm of
    the adjusted-feature sum $O_p(1)$. Applying
    $|\|a+b\|_{\mH^d}^2-\|a\|_{\mH^d}^2|\le2\|a\|_{\mH^d}\|b\|_{\mH^d}+\|b\|_{\mH^d}^2$
    to Eqn.~\eqref{eq:parametric-bootstrap-feature-expansion} gives
    \begin{equation}\label{eq:parametric-bootstrap-norm-remainder}
        n\widetilde T_{n,\mathrm{PA}} =
        \left\|\frac1{\sqrt n}\sum_{i=1}^n
        \Psi_{\hat\theta_n}(\widetilde X_i)\right\|_{\mH^d}^2 + o_p(1).
    \end{equation}
    Stable Slutsky applied to
    Eqns.~\eqref{eq:parametric-bootstrap-linearized-stable} and
    \eqref{eq:parametric-bootstrap-norm-remainder} now gives
    \begin{equation}\label{eq:parametric-bootstrap-statistic-stable}
        n\widetilde T_{n,\mathrm{PA}}\overset{d}{\to}W
        \qquad\mathcal F_\infty\text{-stably}.
    \end{equation}
    This proves the stable convergence of the bootstrap statistic.

    \paragraph{Step 3: Conditional convergence.}
    We now derive Eqn.~\eqref{eq:bootstrap_convergence} from joint
    stable convergence of two bootstrap replications, as in Step~3 of
    the proof of Theorem~\ref{thm:influence-adjusted-bootstrap}.
    Since $\Sigma_{\theta_0}\ne0$, some eigenvalue $\lambda_u$ is
    positive. In the representation of $W$ in
    Theorem~\ref{thm:asymptotic_distribution}, the summand
    $\lambda_uZ_u^2$ has a density and is independent of the remaining
    sum. Thus $W$ has a continuous distribution.

    Conditional on $\mathcal F_n$, generate two independent bootstrap
    samples $(\widetilde X_i^{(1)})_{i=1}^n$ and
    $(\widetilde X_i^{(2)})_{i=1}^n$ from $\law_{\hat\theta_n}$, and
    refit each sample to obtain $\widetilde T_{n,\mathrm{PA}}^{(1)}$ and
    $\widetilde T_{n,\mathrm{PA}}^{(2)}$. The scalar projections of the paired
    linearized feature are sums of
    \[
        \frac1{\sqrt n}\left(
        S^*(f_1,\widetilde X_i^{(1)},\hat\theta_n)
        +S^*(f_2,\widetilde X_i^{(2)},\hat\theta_n)\right),
        \qquad(f_1,f_2)\in(\mH^d)^2.
    \]
    We check that the calculations in Steps~1 and~2 also apply to this
    paired array. By conditional independence and
    Eqn.~\eqref{eq:parametric-bootstrap-centering}, its covariance
    operator is, on $\hat\theta_n\in U$,
    \[
        \begin{pmatrix}
            \Sigma_{\hat\theta_n}&0\\0&\Sigma_{\hat\theta_n}
        \end{pmatrix}
        \convp
        \begin{pmatrix}
            \Sigma_{\theta_0}&0\\0&\Sigma_{\theta_0}
        \end{pmatrix}
        \qquad\text{in trace norm},
    \]
    because the trace norm of the difference is
    $2\|\Sigma_{\hat\theta_n}-\Sigma_{\theta_0}\|_1\convp0$ by
    Eqn.~\eqref{eq:parametric-bootstrap-covariance-continuity}.
    For the Lindeberg condition, the scalar inequality
    \[
        (a+b)^2\indicator(|a+b|\ge r)
        \le4a^2\indicator(|a|\ge r/2)
        +4b^2\indicator(|b|\ge r/2)
    \]
    gives, on $\hat\theta_n\in U$,
    \begin{align*}
        &\frac1n\sum_{i=1}^n\E\Bigl[
        \bigl(S^*(f_1,\widetilde X_i^{(1)},\hat\theta_n)
        +S^*(f_2,\widetilde X_i^{(2)},\hat\theta_n)\bigr)^2\\
        &\qquad\times\indicator\bigl(
        |S^*(f_1,\widetilde X_i^{(1)},\hat\theta_n)
        +S^*(f_2,\widetilde X_i^{(2)},\hat\theta_n)|
        \ge\varepsilon\sqrt n\bigr)
        \,\Bigm|\,\mathcal F_n\Bigr]\\
        &\quad\le4\sum_{b=1}^2\E_{X\sim\law_{\hat\theta_n}}\Bigl[
        S^*(f_b,X,\hat\theta_n)^2
        \indicator\bigl(|S^*(f_b,X,\hat\theta_n)|
        \ge\varepsilon\sqrt n/2\bigr)\Bigr]\convp0
    \end{align*}
    by Eqn.~\eqref{eq:parametric-bootstrap-scalar-lindeberg}.
    Finally, in the product basis
    $(\phi_1,0),(0,\phi_1),(\phi_2,0),(0,\phi_2),\ldots$, the tail
    second moment is bounded by
    \[
        2\|\Sigma_{\hat\theta_n}-\Sigma_{\theta_0}\|_1
        +2\sum_{r>m}\langle\Sigma_{\theta_0}\phi_r,\phi_r
        \rangle_{\mH^d}
    \]
    by summing Eqn.~\eqref{eq:parametric-bootstrap-basis-tail} over the
    two replications. The same conditional Markov calculation verifies
    Condition~\eqref{eq:c3} in $(\mH^d)^2$.
    Thus Lemmas~\ref{lemma:stable-clt} and
    \ref{lemma:functional-stable-convergence}, followed by the two
    componentwise squared-norm maps, give joint stable convergence of
    the two linearized squared norms. Substituting
    Eqn.~\eqref{eq:parametric-bootstrap-norm-remainder} for each
    replication and applying stable Slutsky yields
    \begin{equation}\label{eq:parametric-bootstrap-two-replication-stable}
        (n\widetilde T_{n,\mathrm{PA}}^{(1)},n\widetilde T_{n,\mathrm{PA}}^{(2)})
        \overset{d}{\to}(W^{(1)},W^{(2)})
        \qquad\mathcal F_\infty\text{-stably},
    \end{equation}
    where $W^{(1)}$ and $W^{(2)}$ are independent copies of $W$,
    independent of $\mathcal F_\infty$.

    It remains to turn this joint limit into conditional convergence.
    For each $t\in\R$, conditional independence gives
    \begin{align*}
        \E[\P(n\widetilde T_{n,\mathrm{PA}}\le t\mid\mathcal F_n)]
        &=\P(n\widetilde T_{n,\mathrm{PA}}^{(1)}\le t)\to\P(W\le t),\\
        \E[\P(n\widetilde T_{n,\mathrm{PA}}\le t\mid\mathcal F_n)^2]
        &=\P(n\widetilde T_{n,\mathrm{PA}}^{(1)}\le t,
        n\widetilde T_{n,\mathrm{PA}}^{(2)}\le t)\to\P(W\le t)^2.
    \end{align*}
    Both limits follow from
    Eqn.~\eqref{eq:parametric-bootstrap-two-replication-stable} and
    continuity of $W$. Expanding the square therefore gives
    \begin{align*}
        &\E\left[\bigl(\P(n\widetilde T_{n,\mathrm{PA}}\le t\mid\mathcal F_n)
        -\P(W\le t)\bigr)^2\right]\\
        &\quad=\E[\P(n\widetilde T_{n,\mathrm{PA}}\le t\mid\mathcal F_n)^2]
        -2\P(W\le t)\E[\P(n\widetilde T_{n,\mathrm{PA}}\le t\mid\mathcal F_n)]
        +\P(W\le t)^2\to0.
    \end{align*}
    Markov's inequality now proves, for every $\delta>0$,
    \begin{align*}
        &\P\bigl(|\P(n\widetilde T_{n,\mathrm{PA}}\le t\mid\mathcal F_n)
        -\P(W\le t)|>\delta\bigr)\\
        &\quad\le\frac1{\delta^2}\E\left[
        \bigl(\P(n\widetilde T_{n,\mathrm{PA}}\le t\mid\mathcal F_n)
        -\P(W\le t)\bigr)^2\right]\to0.
    \end{align*}
    This establishes Eqn.~\eqref{eq:bootstrap_convergence} and
    completes Step~3.
\end{proof}

\subsection{Proof of Corollary~\ref{cor:bootstrap_convergence}}\label{sec:proof_bootstrap_as_convergence}
\begin{proof}[Proof of Corollary~\ref{cor:bootstrap_convergence}]
    We use the conditional limits already established in
    Theorems~\ref{thm:influence-adjusted-bootstrap} and
    \ref{thm:bootstrap_distribution}. For the almost-sure conclusions,
    we apply the estimates in their proofs along fixed data realizations,
    deriving the required IA feature consistency from the stronger
    influence-function conditions.

    \paragraph{Step 1: Uniform convergence in probability for both methods.}
    Under the respective nondegeneracy conditions, each limit has a
    positive eigenvalue. Its Gaussian series therefore contains a
    positive multiple of a $\chi_1^2$ variable independent of the
    remaining sum. The sums are finite by
    Eqns.~\eqref{eq:finite-general-law-limit}--\eqref{eq:finite-null-limit}.
    Conditioning on the remainder shows that both limit distributions
    have no atoms. We use the following lemma to strengthen pointwise
    conditional convergence to uniform convergence.
    \begin{lemma}[Lemma 5, \citet{niu2022reconciling}]
    \label{lemma:conditioanl-polya}
        Let $W$ have a continuous distribution function. If, for every
        $t\in\R$,
        \begin{align*}
            \P(W_n\le t\mid\F_n)
            &\overset{p}{\to}\P(W\le t),
        \end{align*}
        then, for every $\delta>0$,
        \begin{align*}
            \P\left[
            \sup_{t\in\R}
            \left|
            \P(W_n\le t\mid\F_n)-\P(W\le t)
            \right|>\delta
            \right]
            &\to0.
        \end{align*}
    \end{lemma}
    Applying Lemma~\ref{lemma:conditioanl-polya} to
    Eqns.~\eqref{eq:general-law-bootstrap-convergence} and
    \eqref{eq:bootstrap_convergence} gives
    Eqns.~\eqref{eq:main_ia_strong_convergence} and
    \eqref{eq:main_strong_convergence}, respectively.

    \paragraph{Step 2: Almost-sure convergence for the IA bootstrap.}
    We first derive an almost-sure rate for the estimator from its
    influence-function expansion. Since the partial sums of
    $I_\law(X_i)$ are mean-zero, square-integrable martingales,
    the maximal inequality gives, for every $\varepsilon>0$,
    \begin{align*}
        &\sum_{m=1}^\infty\P\left(
        \max_{1\le j\le2^m}\left\|\sum_{i=1}^j I_\law(X_i)\right\|
        >\varepsilon2^{3m/4}\right)\\
        &\quad\le\frac{\E[\|I_\law(X)\|^2]}{\varepsilon^2}
        \sum_{m=1}^\infty2^{-m/2}<\infty.
    \end{align*}
    Borel--Cantelli makes the maxima divided by $2^{3m/4}$ converge
    to zero almost surely. Comparing $n$ with the adjacent powers of
    two and using the almost-sure version of Eqn.~\eqref{eq:ale} gives
    \begin{equation}\label{eq:ia-as-estimator-rate}
        n^{1/4}\|\hat\theta_n-\theta_0\|
        \le n^{-3/4}\left\|\sum_{i=1}^nI_\law(X_i)\right\|
        +o(n^{-1/4})
        \to0\qquad\text{almost surely}.
    \end{equation}

    We now use this rate in the fitted-feature bounds from the proof
    of Theorem~\ref{thm:influence-adjusted-bootstrap}. In
    Eqn.~\eqref{eq:proof-stein-feature-l2}, bound the sum of squared
    Hessian suprema by the square of their sum to obtain
    \begin{align*}
        &\frac1n\sum_{i=1}^n
        \|\xi_{\hat\theta_n}(X_i)-\xi_{\theta_0}(X_i)\|_{\mH^d}^2\\
        &\quad\le2C_K\|\hat\theta_n-\theta_0\|^2
        \frac1n\sum_{i=1}^n\|\nabla_\theta s_{\theta_0}(X_i)\|_{\mathrm{op}}^2\\
        &\qquad+\frac{C_Kn}{2}\|\hat\theta_n-\theta_0\|^4
        \left(\frac1n\sum_{i=1}^n\sup_{\theta\in U}
        \|\nabla_\theta^2s_\theta(X_i)\|_{\mathrm{op}}\right)^2
        \to0\qquad\text{almost surely}.
    \end{align*}
    The limit uses Eqn.~\eqref{eq:ia-as-estimator-rate} and the strong
    law applied to the two integrable quantities in
    Eqn.~\eqref{eq:model-density-moments}. The same strong law in
    Eqn.~\eqref{eq:proof-adjustment-operator-convergence} gives
    $\|\widehat{\mathsf J}_n-\mathsf J_{\theta_0,\law}\|_{\mathrm{op}}
    \to0$ almost surely. Substituting these limits, the almost-sure
    version of Eqn.~\eqref{eq:estimated-influence-function-l2}, and
    $n^{-1}\sum_{i=1}^n\|I_\law(X_i)\|^2\to\E[\|I_\law(X)\|^2]$
    into the bound in Eqn.~\eqref{eq:proof-estimated-adjusted-feature-l2}
    proves
    \begin{equation}\label{eq:ia-as-adjusted-feature-l2}
        \frac1n\sum_{i=1}^n
        \|\widehat\Psi_i-\Psi_{\theta_0,\law}(X_i)\|_{\mH^d}^2
        \to0\qquad\text{almost surely}.
    \end{equation}
    Thus the required feature approximation is a consequence of the
    influence-function conditions, not an additional assumption.

    The strong laws and covariance comparison used in
    Eqns.~\eqref{eq:proof-oracle-covariance-convergence}--\eqref{eq:influence-adjusted-covariance-convergence},
    together with Eqn.~\eqref{eq:ia-as-adjusted-feature-l2}, now give
    \[
        \|\widehat\Sigma_n-\Sigma_{\theta_0,\law}\|_1
        \to0\qquad\text{almost surely}.
    \]
    For maximal negligibility, the finite second moment in
    Eqn.~\eqref{eq:proof-adjusted-feature-second-moment} gives
    \[
        \sum_{i=1}^\infty\P\left(
        \|\Psi_{\theta_0,\law}^{\circ}(X_i)\|_{\mH^d}^2
        >\varepsilon i\right)
        \le\frac{\E[\|\Psi_{\theta_0,\law}^{\circ}(X)\|_{\mH^d}^2]}
        {\varepsilon}<\infty.
    \]
    Borel--Cantelli implies
    $n^{-1}\max_{i\le n}\|\Psi_{\theta_0,\law}^{\circ}(X_i)\|_{\mH^d}^2
    \to0$ almost surely. The centering and approximation bounds leading
    to Eqn.~\eqref{eq:influence-adjusted-maximal-negligibility} therefore
    give
    \[
        \frac1n\max_{i\le n}
        \|\widehat\Psi_i-\overline{\widehat\Psi}\|_{\mH^d}^2
        \to0\qquad\text{almost surely}.
    \]

    Fix a data realization where these limits hold. The variance identity
    in Eqn.~\eqref{eq:iid-multiplier-conditional-variance}, the Lindeberg
    bound following Eqn.~\eqref{eq:iid-multiplier-maximal-projection},
    and the basis-tail bound in
    Eqn.~\eqref{eq:influence-adjusted-tail-bound} verify the conditions
    of Lemmas~\ref{lemma:stable-clt} and
    \ref{lemma:functional-stable-convergence} on this fixed conditional
    probability space. As in Steps~1--2 of the proof of
    Theorem~\ref{thm:influence-adjusted-bootstrap}, applying the lemmas
    with independent multiplier rows and a trivial conditioning
    sigma-field gives the squared-norm limit. Since $W_\law$ has a
    continuous distribution, this yields
    \begin{equation}\label{eq:ia-pathwise-bootstrap-cdf-limit}
        \P(n\widetilde T_{n,\mathrm{IA}}\le t\mid\mathcal F_n)
        \to\P(W_\law\le t)\quad \text{as } n\to\infty, \text{ for all } t\in\R
    \end{equation}
    along almost every data realization.

    \paragraph{Step 3: Almost-sure convergence for the parametric bootstrap.}
    Fix a data realization $\omega$ with
    $\hat\theta_n(\omega)\to\theta_0$ and suppress $\omega$ below.
    All conditional probabilities are evaluated at this realization. Here,
    $o_p(1)$, $O_p(1)$, and convergence in distribution refer to the bootstrap laws
    $\law_{\hat\theta_n}^{\otimes n}$. The fitted parameter is now
    deterministic and eventually belongs to $U$.

    The uniform probability bound in Eqn.~\eqref{eq:uale-uniform-probability}
    and the conditional second-moment bound in
    Eqn.~\eqref{eq:parametric-bootstrap-influence-second-moment} give the
    estimator expansion and tightness in
    Eqns.~\eqref{eq:parametric-bootstrap-uale-expansion}--\eqref{eq:parametric-bootstrap-refit-rate}
    under these fixed conditional laws. The conditional bounds in
    Eqns.~\eqref{eq:parametric-bootstrap-operator-variance},
    \eqref{eq:parametric-bootstrap-derivative-bound},
    \eqref{eq:parametric-bootstrap-taylor-bound}, and
    \eqref{eq:parametric-bootstrap-hessian-average-moment} then give
    the same linearization as
    Eqn.~\eqref{eq:parametric-bootstrap-feature-expansion}:
    \begin{equation}\label{eq:pathwise-bootstrap-feature-expansion}
        \frac1{\sqrt n}\sum_{i=1}^n
        \xi_{\widetilde\theta_n}(\widetilde X_i)
        =\frac1{\sqrt n}\sum_{i=1}^n
        \Psi_{\hat\theta_n}(\widetilde X_i)+o_p(1).
    \end{equation}
    These bounds are uniform in the fitted parameter on $U$, so their
    use here does not require almost-sure convergence of the joint-law
    remainders in the earlier proof.

    Next, Eqn.~\eqref{eq:parametric-bootstrap-covariance-continuity}
    and the tail comparison in
    Eqn.~\eqref{eq:parametric-bootstrap-tail-comparison} imply
    \begin{align*}
        \|\Sigma_{\hat\theta_n}-\Sigma_{\theta_0}\|_1
        &\to0,\\
        \E_{\law_{\hat\theta_n}}\left[
        \|\Psi_{\hat\theta_n}(X)\|_{\mH^d}^2
        \indicator(\|\Psi_{\hat\theta_n}(X)\|_{\mH^d}>
        \varepsilon\sqrt n)\right]
        &\to0\qquad(\varepsilon>0).
    \end{align*}
    For the second limit, substitute $\theta=\hat\theta_n$ and
    $R=\varepsilon\sqrt n$ in
    Eqn.~\eqref{eq:parametric-bootstrap-tail-comparison}: the first
    term vanishes by Eqn.~\eqref{eq:parametric-bootstrap-feature-continuity},
    and the other two by dominated convergence using
    Eqns.~\eqref{eq:parametric-bootstrap-density-l1} and
    \eqref{eq:parametric-bootstrap-base-feature-moment}.
    Thus Eqns.~\eqref{eq:parametric-bootstrap-scalar-variance} and
    \eqref{eq:parametric-bootstrap-scalar-lindeberg} verify Conditions
    \eqref{eq:stable-clt-1} and~\eqref{eq:stable-clt-2} along the fixed
    realization. The finite trace and
    Eqn.~\eqref{eq:parametric-bootstrap-basis-tail} also verify
    Conditions~\eqref{eq:c2} and~\eqref{eq:c3}. Applying
    Lemmas~\ref{lemma:stable-clt} and
    \ref{lemma:functional-stable-convergence} as in Steps~1--2 of the
    proof of Theorem~\ref{thm:bootstrap_distribution}, with independent
    bootstrap rows and a trivial conditioning sigma-field, gives
    \[
        \left\|\frac1{\sqrt n}\sum_{i=1}^n
        \Psi_{\hat\theta_n}(\widetilde X_i)\right\|_{\mH^d}^2
        \overset{d}{\to}W.
    \]
    In particular, the linearized norm is $O_p(1)$. Substituting
    Eqn.~\eqref{eq:pathwise-bootstrap-feature-expansion} into the
    squared-norm bound in Eqn.~\eqref{eq:parametric-bootstrap-norm-remainder}
    and applying Slutsky's theorem gives
    \begin{equation}\label{eq:pathwise-bootstrap-cdf-limit}
        \P(n\widetilde T_{n,\mathrm{PA}}\le t\mid\mathcal F_n)
        \to\P(W\le t)\quad \text{as } n\to\infty, \text{ for all } t\in\R.
    \end{equation}

    \paragraph{Step 4: Uniformity for both methods.}
    For an integer $m\ge2$, choose $-\infty=t_0<\cdots<t_m=\infty$ with
    $\P(W\le t_j)=j/m$ for $1\le j<m$. Continuity of the limiting
    distribution permits this choice, and monotonicity gives
    \begin{align*}
        &\sup_{t\in\R}\left|
        \P(n\widetilde T_{n,\mathrm{PA}}\le t\mid\mathcal F_n)-\P(W\le t)\right|\\
        &\quad\le\max_{1\le j<m}\left|
        \P(n\widetilde T_{n,\mathrm{PA}}\le t_j\mid\mathcal F_n)-\P(W\le t_j)
        \right|+\frac1m.
    \end{align*}
    For fixed $m$, Eqn.~\eqref{eq:pathwise-bootstrap-cdf-limit} makes the
    maximum tend to zero. Letting $m\to\infty$ proves uniform convergence
    along the fixed realization and hence Eqn.~\eqref{eq:main_as_convergence}.
    The same grid argument with $W_\law$ and
    $\widetilde T_{n,\mathrm{IA}}$, using
    Eqn.~\eqref{eq:ia-pathwise-bootstrap-cdf-limit}, proves
    Eqn.~\eqref{eq:main_ia_as_convergence}.
\end{proof}

\section{Proofs of results in Section~\ref{sec:asymptotic_power}}\label{sec:proof_power_results}
\subsection{Proof of Theorem~\ref{thm:asymptotic_power}}
\label{sec:proof_fixed_alternative}
\begin{proof}
	All probabilities in this proof are evaluated under the fixed alternative
	$\law^{\otimes n}$. We first establish the positive probability limit of
		the observed statistic. For part (1), the bootstrap limit makes the scaled
		critical value in Eqn.~\eqref{eq:influence-adjusted-bootstrap-test} bounded
		in probability. For part (2), the bootstrap linearization and the local
		second moment of the adjusted feature imply that the unscaled critical
		value in Eqn.~\eqref{eq:bootstrap_test} converges to zero.

	\paragraph{Observed statistic.}
	Under part (1), dividing Eqn.~\eqref{eq:ale} by $\sqrt n$ and applying the
	weak law of large numbers gives
	\begin{align*}
		\hat\theta_n-\theta_0
		&=\frac1n\sum_{i=1}^n I_\law(X_i)+o_p(n^{-1/2})
		\overset{p}{\to}0.
	\end{align*}
	Under part (2), this convergence and the data-law regularity conditions
	are assumed through Theorem~\ref{thm:sksd_consistency}. Hence, under
	either part, that theorem gives
	\begin{align}\label{eq:fixed-alternative-observed-limit}
		T_n
		&\overset{p}{\to}
		\mathrm{KSD}^2(\law_{\theta_0},\law)>0,
	\end{align}
	where positivity follows from $\theta_0\in\Theta$,
	$\law\notin\{\law_\theta:\theta\in\Theta\}$, and the characteristicity
	assumed in Section~\ref{sec:prelim_KSD}. For every fixed
	$M<\infty$, $M/n\to0$, so Eqn.~\eqref{eq:fixed-alternative-observed-limit}
	gives
	\begin{align}\label{eq:fixed-alternative-observed-divergence}
		\P_{\law^{\otimes n}}(nT_n\le M)
		&=\P_{\law^{\otimes n}}\left(T_n\le\frac Mn\right)
		\to0.
	\end{align}
	Thus $nT_n\overset{p}{\to}\infty$ under the fixed
	alternative.

	\paragraph{Part (1): Eqn.~\eqref{eq:influence-adjusted-bootstrap-test}.}
	The random variable $W_\law$ in
	Theorem~\ref{thm:influence-adjusted-bootstrap} is finite almost surely, so
	there is a finite continuity point $M>0$ of its distribution such that
	$\P(W_\law\le M)>1-\alpha$. By
	Eqn.~\eqref{eq:general-law-bootstrap-convergence},
	\begin{align*}
	&\P_{\law^{\otimes n}}\left[
	\Q_{1-\alpha}(n\widetilde T_{n,\mathrm{IA}})>M
	\right]\le
	\P_{\law^{\otimes n}}\left[
	\P\left(
	n\widetilde T_{n,\mathrm{IA}}\le M\mid\mathcal F_n
	\right)<1-\alpha
	\right]\to0.
	\end{align*}
	Hence $\Q_{1-\alpha}(n\widetilde T_{n,\mathrm{IA}})=O_p(1)$.
	Because quantiles commute with multiplication by $n>0$,
	Eqn.~\eqref{eq:influence-adjusted-bootstrap-test} implies
	\begin{align*}
	&\P_{\law^{\otimes n}}\left[
	\phi_{n,\alpha}^{\mathrm{IA}}=0
	\right]\le
	\P_{\law^{\otimes n}}(nT_n\le M)
	+\P_{\law^{\otimes n}}\left[
	\Q_{1-\alpha}(n\widetilde T_{n,\mathrm{IA}})>M
	\right]\to0
	\end{align*}
	by Eqn.~\eqref{eq:fixed-alternative-observed-divergence}. Therefore,
	\begin{align*}
		\P_{\law^{\otimes n}}\left[
		\phi_{n,\alpha}^{\mathrm{IA}}=1
		\right]&\to1.
	\end{align*}

\paragraph{Part (2): Eqn.~\eqref{eq:bootstrap_test}.}
    We show that the unscaled bootstrap statistic converges to zero.
    Its Gaussian limit is not needed here. Step~1 of
    Appendix~\ref{sec:proof_bootstrap_distribution} applies under the
    data law $\law$, because it uses
    $\hat\theta_n\convp\theta_0$ and the UALE expansion under the
    resampling model, not an ALE expansion under $\law$.
    Eqn.~\eqref{eq:parametric-bootstrap-feature-expansion} therefore gives
    \[
        \frac1{\sqrt n}\sum_{i=1}^n
        \xi_{\widetilde\theta_n}(\widetilde X_i)
        =\frac1{\sqrt n}\sum_{i=1}^n
        \Psi_{\hat\theta_n}(\widetilde X_i)+o_p(1).
    \]
    On $\hat\theta_n\in U$, conditional independence and
    Eqns.~\eqref{eq:parametric-bootstrap-centering}--\eqref{eq:parametric-bootstrap-feature-moment}
    give
    \begin{align*}
        \E\left[\left\|\frac1{\sqrt n}\sum_{i=1}^n
        \Psi_{\hat\theta_n}(\widetilde X_i)\right\|_{\mH^d}^2
        \,\middle|\,\mathcal F_n\right]
        &=\E_{X\sim\law_{\hat\theta_n}}
        [\|\Psi_{\hat\theta_n}(X)\|_{\mH^d}^2]\\
        &\le\sup_{\theta\in U}\E_{X\sim\law_\theta}
        [\|\Psi_\theta(X)\|_{\mH^d}^2]<\infty.
    \end{align*}
    Thus the adjusted-feature sum is conditionally tight. Its
    $o_p(1)$ approximation above is also conditionally tight, and
    the squared-norm representation of $\widetilde T_{n,\mathrm{PA}}$ gives
    $n\widetilde T_{n,\mathrm{PA}}=O_p(1)$ conditionally, in
    $\P_{\law^{\otimes n}}$-probability. Dividing by $n$ yields, for
    every $\varepsilon,\delta>0$,
    \begin{equation}\label{eq:fixed-alternative-bootstrap-statistic-vanishes}
        \P_{\law^{\otimes n}}\left[
        \P(\widetilde T_{n,\mathrm{PA}}>\varepsilon\mid\mathcal F_n)>\delta
        \right]\to0.
    \end{equation}
    By the definition of a conditional quantile,
    \[
        \P_{\law^{\otimes n}}[
        \Q_{1-\alpha}(\widetilde T_{n,\mathrm{PA}})>\varepsilon]
        \le\P_{\law^{\otimes n}}\left[
        \P(\widetilde T_{n,\mathrm{PA}}>\varepsilon\mid\mathcal F_n)>\alpha
        \right]\to0.
    \]
    Hence $\Q_{1-\alpha}(\widetilde T_{n,\mathrm{PA}})\convp0$.
    Choose $0<\varepsilon<\mathrm{KSD}^2(\law_{\theta_0},\law)$.
    Eqns.~\eqref{eq:bootstrap_test} and
    \eqref{eq:fixed-alternative-observed-limit} now give
    \[
        \P_{\law^{\otimes n}}[\phi_{n,\alpha}^{\mathrm{PA}}=0]
        \le\P_{\law^{\otimes n}}(T_n\le\varepsilon)
        +\P_{\law^{\otimes n}}[
        \Q_{1-\alpha}(\widetilde T_{n,\mathrm{PA}})>\varepsilon]\to0.
    \]
    This proves consistency for the parametric-bootstrap test.
\end{proof}

\subsection{Proof of Theorem~\ref{thm:power_local_alternative}}
\label{sec:proof_asymptotic_power_local}
\begin{proof}[Proof of Theorem~\ref{thm:power_local_alternative}]
	We first prove the result under the multiplicative local alternative in
	Definition~\ref{ass:local}(a). First, we establish contiguity of the local
	alternative to the null experiment. Second, we use this contiguity and
	Le Cam's third lemma to derive the shifted limit of the observed feature.
		Third, we transfer the null bootstrap convergence separately for
		Eqns.~\eqref{eq:influence-adjusted-bootstrap-test}
		and~\eqref{eq:bootstrap_test} and calculate their limiting rejection
		probabilities. Finally, we obtain the additive case by
	expressing it as a multiplicative local alternative.

	\paragraph{Step 1: Contiguity.}
	Write $\P_{0,n}=\law_{\theta_0}^{\otimes n}$ and $\P_n=\law_n^{\otimes n}$. The proportionality in
	Definition~\ref{ass:local}(a) and normalization of $\law_n$ give
	\begin{align}\label{eq:multiplicative-local-alternative}
		\frac{\mathrm d\law_n}{\mathrm d\law_{\theta_0}}(x)
		&=\frac{1+r(x)/\sqrt n}
		{1+\E_{X\sim\law_{\theta_0}}[r(X)]/\sqrt n}.
	\end{align}
	Define, within this proof,
	$
		r^{\circ}(x)
		=r(x)-\E_{X\sim\law_{\theta_0}}[r(X)]
	$.
	The finite second moment in Definition~\ref{ass:local}(a) gives, under
	$\P_{0,n}$,
	\begin{align}\label{eq:local-r-control}
		\frac1n\sum_{i=1}^nr(X_i)^2
		&\overset{p}{\to}
		\E_{X\sim\law_{\theta_0}}[r(X)^2],\qquad
		\frac1{\sqrt n}\max_{1\le i\le n}|r(X_i)|
		\overset{p}{\to}0.
	\end{align}
	For the second convergence, for every $\varepsilon>0$,
	\begin{align*}
	&\P_{0,n}\left(
	\max_{1\le i\le n}|r(X_i)|>\varepsilon\sqrt n
	\right)\\
	&\quad\le n\P_{\law_{\theta_0}}\left(
	|r(X)|>\varepsilon\sqrt n
	\right)
	\le\frac1{\varepsilon^2}
	\E_{X\sim\law_{\theta_0}}\left[
	r(X)^2\indicator(|r(X)|>\varepsilon\sqrt n)
	\right]\to0.
	\end{align*}
	Eqn.~\eqref{eq:local-r-control} permits the uniform Taylor expansion of
	$\log(1+u)$ over the observed values $u=r(X_i)/\sqrt n$. More explicitly,
	\begin{align*}
	&\left|\sum_{i=1}^n\left(
	\log\left(1+\frac{r(X_i)}{\sqrt n}\right)
	-\frac{r(X_i)}{\sqrt n}+\frac{r(X_i)^2}{2n}
	\right)\right|\\
	&\quad\le
	\sup_{0<|u|\le n^{-1/2}\max_i|r(X_i)|}
	\frac{|\log(1+u)-u+u^2/2|}{u^2}
	\frac1n\sum_{i=1}^nr(X_i)^2
	=o_p(1).
	\end{align*}
	Taylor expansion of the deterministic normalizing constant similarly
	gives
	\begin{align*}
		n\log\left(
	1+\frac{\E_{X\sim\law_{\theta_0}}[r(X)]}{\sqrt n}
	\right)
	&=\sqrt n\E_{X\sim\law_{\theta_0}}[r(X)]
	-\frac12\E_{X\sim\law_{\theta_0}}[r(X)]^2+o(1).
	\end{align*}
	Substituting these expansions into
	Eqn.~\eqref{eq:multiplicative-local-alternative} yields
	\begin{align}\label{eq:local-log-likelihood-ratio}
		\log\frac{\mathrm d\P_n}{\mathrm d\P_{0,n}}
		&=\frac1{\sqrt n}\sum_{i=1}^nr^{\circ}(X_i)
		-\frac12\E_{X\sim\law_{\theta_0}}[r^{\circ}(X)^2]+o_p(1)\notag\\
		&\overset{d}{\to}
		N\left(
		-\frac12\E_{X\sim\law_{\theta_0}}[r^{\circ}(X)^2],
		\E_{X\sim\law_{\theta_0}}[r^{\circ}(X)^2]
		\right).
	\end{align}
	The exponential of the limiting normal variable has expectation one.
	Le Cam's first lemma \citep[Lemma 6.4]{van2000asymptotic} therefore gives
	\begin{align}\label{eq:local-contiguity}
		\P_n&\mathrel{\lhd}\P_{0,n}.
	\end{align}
	Consequently, every $o_p(1)$ statement under $\P_{0,n}$ remains
	$o_p(1)$ under $\P_n$.

	\paragraph{Step 2: Limit of the observed statistic.}
	We now derive the local mean shift of the observed feature. Under
	$\P_{0,n}$, Eqn.~\eqref{eq:observed-feature-expansion} gives
	\begin{align}\label{eq:local-observed-feature-expansion}
		\frac1{\sqrt n}\sum_{i=1}^n
		\xi_{\hat\theta_n}(X_i)
		&=\frac1{\sqrt n}\sum_{i=1}^n
		\Psi_{\theta_0}(X_i)+o_p(1).
	\end{align}
	Eqn.~\eqref{eq:local-contiguity} transfers
	Eqn.~\eqref{eq:local-observed-feature-expansion} to $\P_n$.
	The Bochner expectation defining $\mu$ is finite because
	\begin{align*}
	&\E_{X\sim\law_{\theta_0}}\left[
	\|\Psi_{\theta_0}(X)r(X)\|_{\mH^d}
	\right]\\
	&\quad\le
	\E_{X\sim\law_{\theta_0}}\left[
	\|\Psi_{\theta_0}(X)\|_{\mH^d}^2
	\right]^{1/2}
	\E_{X\sim\law_{\theta_0}}[r(X)^2]^{1/2}<\infty
	\end{align*}
	by Eqn.~\eqref{eq:adjusted-feature-second-moment} and
	Definition~\ref{ass:local}(a).

	Eqns.~\eqref{eq:adjusted-feature-second-moment}
	and~\eqref{eq:adjusted-feature-centering}, together with
	Definition~\ref{ass:local}, permit the Hilbert-space central limit theorem
	in $\mH^d\times\R$:
	\begin{align}\label{eq:local-joint-clt}
		\left(
		\frac1{\sqrt n}\sum_{i=1}^n\Psi_{\theta_0}(X_i),
		\frac1{\sqrt n}\sum_{i=1}^nr^{\circ}(X_i)
		\right)
		&\overset{d}{\to}
		(\mathbb G_{\theta_0},H)
	\end{align}
	under $\P_{0,n}$, where the limit is jointly centered Gaussian. For
	every $f\in\mH^d$, Eqn.~\eqref{eq:adjusted-feature-centering} gives
	\begin{align}\label{eq:local-feature-score-covariance}
		\operatorname{Cov}\left(
		\langle\mathbb G_{\theta_0},f\rangle_{\mH^d},H
		\right)
		&=\E_{X\sim\law_{\theta_0}}\left[
		\langle\Psi_{\theta_0}(X),f\rangle_{\mH^d}
		r^{\circ}(X)
		\right]\notag\\
		&=\E_{X\sim\law_{\theta_0}}\left[
		\langle\Psi_{\theta_0}(X),f\rangle_{\mH^d}r(X)
		\right]
		=\langle\mu,f\rangle_{\mH^d}.
	\end{align}
	Applying Le Cam's third lemma \citep[Theorem 6.6]{van2000asymptotic}
	to Eqns.~\eqref{eq:local-log-likelihood-ratio},
	\eqref{eq:local-joint-clt}, and
	\eqref{eq:local-feature-score-covariance} gives, under $\P_n$,
	\begin{align}\label{eq:local-adjusted-feature-limit}
		\frac1{\sqrt n}\sum_{i=1}^n\Psi_{\theta_0}(X_i)
		&\overset{d}{\to}
		\mathbb G_{\theta_0}+\mu.
	\end{align}
	Plugging Eqn.~\eqref{eq:local-adjusted-feature-limit} into
	Eqn.~\eqref{eq:local-observed-feature-expansion}, and then using the exact
	representation in Eqn.~\eqref{eq:theorem-two-statistic-representation},
	proves by Slutsky's theorem and the continuous mapping theorem that
	\begin{align*}
		nT_n
		&\overset{d}{\to}
		W_\mu\equiv\|\mathbb G_{\theta_0}+\mu\|_{\mH^d}^2.
	\end{align*}
	We now derive the series representation of this squared norm. First,
	$\mu$ has no component in the null space of $\Sigma_{\theta_0}$.
	Indeed, for every $f\in\ker(\Sigma_{\theta_0})$,
	Eqns.~\eqref{eq:local-alternative-shift} and~\eqref{eq:asymptotic_distribution},
	followed by Cauchy--Schwarz, give
	\begin{align*}
		|\langle\mu,f\rangle_{\mH^d}|^2
		&=\left|\E_{X\sim\law_{\theta_0}}\left[
		\langle\Psi_{\theta_0}(X),f\rangle_{\mH^d}r(X)
		\right]\right|^2\\
		&\le\E_{X\sim\law_{\theta_0}}[r(X)^2]
		\E_{X\sim\law_{\theta_0}}\left[
		\langle\Psi_{\theta_0}(X),f\rangle_{\mH^d}^2\right]\\
		&=\E_{X\sim\law_{\theta_0}}[r(X)^2]
		\langle\Sigma_{\theta_0}f,f\rangle_{\mH^d}=0.
	\end{align*}
	The orthonormal eigenvectors with positive eigenvalues span
	$\big(\ker(\Sigma_{\theta_0})\big)^\perp$ after taking closure. By
	the Gaussian series representation in
	Eqn.~\eqref{eq:hilbert-gaussian-expansion}, this is also the topological
	support of $\mathbb G_{\theta_0}$. Hence
	\begin{align}\label{eq:local-shift-support}
		\mu&\in\big(\ker(\Sigma_{\theta_0})\big)^\perp
		=\operatorname{supp}(\mathbb G_{\theta_0}),
		\qquad
		\mu=\sum_{u:\lambda_u>0}\langle e_u,\mu\rangle_{\mH^d}e_u.
	\end{align}
	Combining this expansion with Eqn.~\eqref{eq:hilbert-gaussian-expansion}
	gives
	\begin{align*}
		\mathbb G_{\theta_0}+\mu
		&\overset{d}{=}\sum_{u:\lambda_u>0}
		\left(\sqrt{\lambda_u}Z_u+\langle e_u,\mu\rangle_{\mH^d}\right)e_u,
	\end{align*}
	where the $Z_u$ are independent standard normal variables. The Gaussian
	series converges almost surely in $\mH^d$, and the series for $\mu$
	converges in its Hilbert norm. Orthonormality and Parseval's identity
	therefore yield
	\begin{align*}
		W_\mu
		&\overset{d}{=}\left\|\sum_{u:\lambda_u>0}
		\left(\sqrt{\lambda_u}Z_u+\langle e_u,\mu\rangle_{\mH^d}\right)e_u
		\right\|_{\mH^d}^2=\sum_{u:\lambda_u>0}
		\left(\sqrt{\lambda_u}Z_u+\langle e_u,\mu\rangle_{\mH^d}\right)^2.
	\end{align*}
	These sums are finite if $\Sigma_{\theta_0}$ has finite rank. Terms
	associated with zero eigenvalues vanish by
	Eqn.~\eqref{eq:local-shift-support}. This proves the representation in
	Eqn.~\eqref{eq:local-observed-limit} and completes the local-limit argument.

	\paragraph{Step 3(a): Eqn.~\eqref{eq:influence-adjusted-bootstrap-test}.}
	We first identify the null version of the general-law limit in
	Theorem~\ref{thm:influence-adjusted-bootstrap}. The shorthand following
	Eqn.~\eqref{eq:influence-adjusted-stein-feature} and
	Eqn.~\eqref{eq:adjusted-feature-centering} give
	\begin{align*}
		\Psi_{\theta_0,\law_{\theta_0}}^{\circ}(x)
		&=\Psi_{\theta_0}(x),
		&
		\Sigma_{\theta_0,\law_{\theta_0}}
		&=\Sigma_{\theta_0},
		&
		W_{\law_{\theta_0}}
		&\overset{d}{=}W.
	\end{align*}
	Consequently, under $\P_{0,n}$,
	Eqn.~\eqref{eq:general-law-bootstrap-convergence} gives, for every
	continuity point $t$ of $W$ and every $\delta>0$,
	\begin{align*}
		\P_{0,n}\left[
		\left|
		\P\left(
		n\widetilde T_{n,\mathrm{IA}}\le t\mid\mathcal F_n
		\right)-\P(W\le t)
		\right|>\delta
		\right]&\to0.
	\end{align*}
	The event inside this probability is data-measurable, so
	Eqn.~\eqref{eq:local-contiguity} implies
	\begin{align}\label{eq:local-bootstrap-conditional}
		\P_n\left[
		\left|
		\P\left(
		n\widetilde T_{n,\mathrm{IA}}\le t\mid\mathcal F_n
		\right)-\P(W\le t)
		\right|>\delta
		\right]&\to0.
	\end{align}
	Because $\Sigma_{\theta_0}\ne0$,
	Eqn.~\eqref{eq:null-limit-quantile-regularity} verifies the conditions of
	Lemma~\ref{lemma:quantile-converge-in-prob}. That lemma, with
	the outer probabilities evaluated under $\P_n$, applies directly to
	Eqn.~\eqref{eq:local-bootstrap-conditional} and gives
	\begin{align}\label{eq:local-bootstrap-quantile}
		\Q_{1-\alpha}(n\widetilde T_{n,\mathrm{IA}})
		&\overset{p}{\to}\Q_{1-\alpha}(W).
	\end{align}

	We next verify that the limiting rejection boundary has probability zero.
	Since
	$\Sigma_{\theta_0}\ne0$, choose an eigenvector $e_{u_0}$ with eigenvalue
	$\lambda_{u_0}>0$. The Gaussian coordinate
	$\langle\mathbb G_{\theta_0},e_{u_0}\rangle_{\mH^d}$ is independent of
	the orthogonal Gaussian component and has distribution
	$N(0,\lambda_{u_0})$. Conditional on that orthogonal component,
	\begin{align*}
		W_\mu
		&=\left(
		\langle\mathbb G_{\theta_0},e_{u_0}\rangle_{\mH^d}
		+\langle\mu,e_{u_0}\rangle_{\mH^d}
		\right)^2
		+\left\|
		\mathbb G_{\theta_0}+\mu
		-\langle\mathbb G_{\theta_0}+\mu,e_{u_0}\rangle_{\mH^d}e_{u_0}
		\right\|_{\mH^d}^2
	\end{align*}
	is a constant plus the square of a nondegenerate normal variable. Hence
	$\P(W_\mu=\Q_{1-\alpha}(W))=0$. Eqns.~\eqref{eq:local-observed-limit}
	and~\eqref{eq:local-bootstrap-quantile}, followed by Slutsky's theorem,
	therefore give
	\begin{align*}
		nT_n-\Q_{1-\alpha}(n\widetilde T_{n,\mathrm{IA}})
		&\overset{d}{\to}
		W_\mu-\Q_{1-\alpha}(W)
	\end{align*}
	under $\P_n$. Positive scaling of quantiles and
	Eqn.~\eqref{eq:influence-adjusted-bootstrap-test} now imply
	\begin{align}\label{eq:local-power-limit}
		\P_n\left[\phi_{n,\alpha}^{\mathrm{IA}}=1\right]
		&\to
		\P\left[W_\mu\ge\Q_{1-\alpha}(W)\right].
	\end{align}

	\paragraph{Step 3(b): Eqn.~\eqref{eq:bootstrap_test}.}
	Under $\P_{0,n}$, Theorem~\ref{thm:bootstrap_distribution} gives, for every
	continuity point $t$ of $W$ and every $\delta>0$,
	\begin{align*}
		\P_{0,n}\left[
		\left|
		\P\left(
		n\widetilde T_{n,\mathrm{PA}}\le t\mid\mathcal F_n
		\right)-\P(W\le t)
		\right|>\delta
		\right]&\to0.
	\end{align*}
	The event inside this probability is again data-measurable. Hence
	Eqn.~\eqref{eq:local-contiguity} gives
	\begin{align}\label{eq:local-parametric-bootstrap-conditional}
		\P_n\left[
		\left|
		\P\left(
		n\widetilde T_{n,\mathrm{PA}}\le t\mid\mathcal F_n
		\right)-\P(W\le t)
		\right|>\delta
		\right]&\to0.
	\end{align}
	Applying Lemma~\ref{lemma:quantile-converge-in-prob} to
	Eqn.~\eqref{eq:local-parametric-bootstrap-conditional} gives
	\begin{align}\label{eq:local-parametric-bootstrap-quantile}
		\Q_{1-\alpha}(n\widetilde T_{n,\mathrm{PA}})
		&\overset{p}{\to}\Q_{1-\alpha}(W)
	\end{align}
	under $\P_n$. Eqns.~\eqref{eq:local-observed-limit}
	and~\eqref{eq:local-parametric-bootstrap-quantile}, followed by Slutsky's
	theorem, therefore give
	\begin{align*}
		nT_n-\Q_{1-\alpha}(n\widetilde T_{n,\mathrm{PA}})
		&\overset{d}{\to}
		W_\mu-\Q_{1-\alpha}(W)
	\end{align*}
	under $\P_n$. Positive scaling of quantiles,
	Eqn.~\eqref{eq:bootstrap_test}, and
	$\P(W_\mu=\Q_{1-\alpha}(W))=0$ now imply
	\begin{align}\label{eq:local-parametric-power-limit}
		\P_n\left[\phi_{n,\alpha}^{\mathrm{PA}}=1\right]
		&\to
		\P\left[W_\mu\ge\Q_{1-\alpha}(W)\right].
	\end{align}
	Thus Eqns.~\eqref{eq:local-power-limit}
	and~\eqref{eq:local-parametric-power-limit} give the common limit $\beta$
	in Theorem~\ref{thm:power_local_alternative}.

	Anderson's Gaussian inequality for the centered Hilbert ball
	\citep{anderson1955integral,lewandowski1995anderson} gives
	\begin{align*}
		\P\left(W_\mu\le\Q_{1-\alpha}(W)\right)
		&\le\P\left(W\le\Q_{1-\alpha}(W)\right)=1-\alpha.
	\end{align*}
	Together with the absence of an atom at the cutoff, this proves
	the claimed lower bound on power. If $\mu=0$, then $W_\mu=W$, so
	Eqns.~\eqref{eq:local-power-limit}--\eqref{eq:local-parametric-power-limit}
	and
	Eqn.~\eqref{eq:null-limit-quantile-regularity} give $\beta=\alpha$.

	For $\mu\ne0$, Eqn.~\eqref{eq:local-shift-support} verifies the support
	condition for the strict form of Anderson's inequality in
	\citet{lewandowski1995anderson}. Since $\mH^d$ is separable,
	$\mu$ is a nonzero element of this support,
	and the ball has finite positive radius
	$\sqrt{\Q_{1-\alpha}(W)}$ by
	Eqn.~\eqref{eq:null-limit-quantile-regularity}, the strict inequality gives
	\[
		\P\left(W_\mu\le\Q_{1-\alpha}(W)\right)
		<\P\left(W\le\Q_{1-\alpha}(W)\right)=1-\alpha.
	\]
	Together with the absence of an atom at the cutoff, this proves
	$\beta>\alpha$. Hence $\beta>\alpha$ if and only if $\mu\ne0$ under
	Definition~\ref{ass:local}(a).

	\paragraph{Step 4: Additive local alternative.}
	Consider the additive local alternative in
	Definition~\ref{ass:local}(b). Its density ratio is
	\begin{align*}
		\frac{\mathrm d\law_n}{\mathrm d\law_{\theta_0}}(x)
		&=1+\frac c{\sqrt n}\left(
		\frac{\mathrm d\law_g}{\mathrm d\law_{\theta_0}}(x)-1
		\right).
	\end{align*}
	Thus it has the multiplicative form in
	Definition~\ref{ass:local}(a). Specifically, throughout Steps 1--3, take
	\begin{align*}
		r(x)&=c\left(
		\frac{\mathrm d\law_g}{\mathrm d\law_{\theta_0}}(x)-1
		\right).
	\end{align*}
	The required second moment follows from
	\begin{align*}
	&\E_{X\sim\law_{\theta_0}}\left[
	c^2\left(
	\frac{\mathrm d\law_g}{\mathrm d\law_{\theta_0}}(X)-1
	\right)^2
	\right]\le2c^2\left(
	\E_{X\sim\law_{\theta_0}}\left[
	\left(\frac{\mathrm d\law_g}{\mathrm d\law_{\theta_0}}(X)\right)^2
	\right]+1
	\right)<\infty.
	\end{align*}
	Finally, Eqn.~\eqref{eq:adjusted-feature-centering} gives the corresponding
	shift as
	\begin{align*}
	&\E_{X\sim\law_{\theta_0}}\left[
	\Psi_{\theta_0}(X)c\left(
	\frac{\mathrm d\law_g}{\mathrm d\law_{\theta_0}}(X)-1
	\right)
	\right]\\
	&\quad=c\E_{X\sim\law_g}[\Psi_{\theta_0}(X)]
	-c\E_{X\sim\law_{\theta_0}}[\Psi_{\theta_0}(X)]
	=c\E_{X\sim\law_g}[\Psi_{\theta_0}(X)]=\mu.
	\end{align*}
	Thus the preceding three steps also prove the result under
	Definition~\ref{ass:local}(b), which completes the proof.
\end{proof}

\subsection{Local power and the null-model tangent space}\label{sec:local-power-geometry}

We justify Remark~\ref{rem:local-power-directions}. Throughout this
subsection, expectations are under $\law_{\theta_0}$. For the function $r$ given in Definition~\ref{ass:local}(a), consider the normalized
path, for small admissible $t\ge0$,
\begin{align*}
	\frac{\mathrm d\law_t}{\mathrm d\law_{\theta_0}}(x)
	&=\frac{1+tr(x)}{1+t\E_{X\sim\law_{\theta_0}}[r(X)]}.
\end{align*}
The parametric score of this path at the null is
\begin{align}\label{eq:local-path-score}
	\left.\frac{\partial}{\partial t}
	\log\frac{\mathrm d\law_t}{\mathrm d\law_{\theta_0}}(x)\right|_{t=0+}
	&=r(x)-\E_{X\sim\law_{\theta_0}}[r(X)].
\end{align}
Notice here that the centered perturbation is the parametric score of the path, not the score function
$s_{\theta_0}$. By Eqn.~\eqref{eq:adjusted-feature-centering},
\begin{align*}
	\E_{X\sim\law_{\theta_0}}\left[
	\Psi_{\theta_0}(X)\left(r(X)-\E_{X'\sim\law_{\theta_0}}[r(X')]\right)\right]
	&=\E_{X\sim\law_{\theta_0}}[\Psi_{\theta_0}(X)r(X)]=\mu.
\end{align*}
Under the multiplicative alternative, Eqns.~\eqref{eq:influence-adjusted-stein-feature}
and~\eqref{eq:local-alternative-shift} give
\begin{align}\label{eq:local-shift-decomposition}
	\mu
	&=\E_{X\sim\law_{\theta_0}}\left[\xi_{\theta_0}(X)r(X)\right]\notag\\
	&\quad+\E_{X'\sim\law_{\theta_0}}\left[
	\nabla_\theta s_{\theta_0}(X')
	\E_{X\sim\law_{\theta_0}}\left[I_{\law_{\theta_0}}(X)r(X)\right]
	\psi(X')\right].
\end{align}
The first term is the direct Stein shift. The second is the effect of
estimating the parameter. These terms may cancel, and
\begin{align*}
	\langle f,\mu\rangle_{\mH^d}
	&=\E_{X\sim\law_{\theta_0}}\left[
	\langle f,\Psi_{\theta_0}(X)\rangle_{\mH^d}r(X)\right].
\end{align*}
Thus $\mu=0$ exactly when $r$ is orthogonal to every scalar
adjusted Stein feature. The same argument applies to the additive
alternative by taking $r=c(\mathrm d\law_g/\mathrm d\law_{\theta_0}-1)$,
as in Step~4 of Appendix~\ref{sec:proof_asymptotic_power_local}.

To relate this criterion to the model geometry, suppose additionally that
the parametric model is differentiable in quadratic mean at $\theta_0$,
with square-integrable parametric score
$\ell_{\theta_0}(x)=\left.\nabla_\theta\log p_\theta(x)\right|_{\theta=\theta_0}$.
This differs from the score function $s_\theta=\nabla_x\log p_\theta$.
The null-model tangent space and its normal space are, respectively,
\begin{align*}
	\mathcal T_{\theta_0}
	&=\left\{x\mapsto a^\top\ell_{\theta_0}(x):a\in\R^k\right\},
	\qquad \mathcal T_{\theta_0}^\perp\subseteq L^2_0(\law_{\theta_0}),
\end{align*}
where $L^2_0(\law_{\theta_0})$ is the space of mean-zero, square-integrable
functions and orthogonality uses its usual $L^2$ inner product.
Assume the estimator is regular under local perturbations within the model, so that its
influence function satisfies the parametric score identity
\citep[Chapter~25]{van2000asymptotic}
\begin{align}\label{eq:influence-model-score-identity}
	\E_{X\sim\law_{\theta_0}}\left[
	I_{\law_{\theta_0}}(X)\ell_{\theta_0}(X)^\top\right]&=\operatorname{Id}_k.
\end{align}
Also suppose differentiation can be interchanged with the expectation
in the model Stein identity. For every $a\in\R^k$, differentiating
$\E_{X\sim\law_{\theta_0+ta}}[\xi_{\theta_0+ta}(X)]=0$ at $t=0$ gives
\begin{align}\label{eq:differentiated-model-stein-identity}
	0&=\E_{X\sim\law_{\theta_0}}\left[
	\xi_{\theta_0}(X)\ell_{\theta_0}(X)^\top a\right]
	+\E_{X\sim\law_{\theta_0}}\left[
	\nabla_\theta s_{\theta_0}(X)a\psi(X)\right].
\end{align}
Substituting Eqn.~\eqref{eq:influence-model-score-identity} into
Eqn.~\eqref{eq:influence-adjusted-stein-feature} and then using
Eqn.~\eqref{eq:differentiated-model-stein-identity} yields
\begin{align*}
	\E_{X\sim\law_{\theta_0}}\left[
	\Psi_{\theta_0}(X)\ell_{\theta_0}(X)^\top a\right]
	&=\E_{X\sim\law_{\theta_0}}\left[
	\xi_{\theta_0}(X)\ell_{\theta_0}(X)^\top a\right]\\
	&\quad+\E_{X'\sim\law_{\theta_0}}\left[
	\nabla_\theta s_{\theta_0}(X')
	\E_{X\sim\law_{\theta_0}}\left[
	I_{\law_{\theta_0}}(X)\ell_{\theta_0}(X)^\top a\right]\psi(X')\right]\\
	&=\E_{X\sim\law_{\theta_0}}\left[
	\xi_{\theta_0}(X)\ell_{\theta_0}(X)^\top a+
	\nabla_\theta s_{\theta_0}(X)a\psi(X)\right]=0.
\end{align*}
Consequently, every scalar adjusted feature
$x\mapsto\langle f,\Psi_{\theta_0}(x)\rangle_{\mH^d}$ belongs to
$\mathcal T_{\theta_0}^\perp$. Let $\Pi_{\mathcal T_{\theta_0}}$ denote
the orthogonal projection from $L^2(\law_{\theta_0})$ onto the tangent
space. Since every element of $\mathcal T_{\theta_0}$ has mean zero,
\begin{align*}
	\Pi_{\mathcal T_{\theta_0}}\left(r-\E_{X\sim\law_{\theta_0}}[r(X)]\right)
	&=\Pi_{\mathcal T_{\theta_0}}r.
\end{align*}
Using Eqn.~\eqref{eq:adjusted-feature-centering} and orthogonality to
$\mathcal T_{\theta_0}$, we therefore obtain
\begin{align*}
	\mu&=\E_{X\sim\law_{\theta_0}}\left[
	\Psi_{\theta_0}(X)
	\left(r(X)-\E_{X'\sim\law_{\theta_0}}[r(X')]
	-(\Pi_{\mathcal T_{\theta_0}}r)(X)\right)\right].
\end{align*}
The expression multiplying $\Psi_{\theta_0}(X)$ is the normal component
of the parametric score of the path in Eqn.~\eqref{eq:local-path-score}. In
particular, $r-\E_{X\sim\law_{\theta_0}}[r(X)]\in\mathcal T_{\theta_0}$
implies $\mu=0$. This does not imply that every
nonzero normal direction has $\mu\ne0$. A sufficient condition for that
stronger conclusion is that the scalar adjusted features have dense linear
span in $\mathcal T_{\theta_0}^\perp$. Under this additional condition,
$\mu=0$ if and only if
$r-\E_{X\sim\law_{\theta_0}}[r(X)]\in\mathcal T_{\theta_0}$.
Neither this density condition nor the parametric score identity in
Eqn.~\eqref{eq:influence-model-score-identity} follows from pointwise ALE
alone. Thus the feature-pairing criterion in
Remark~\ref{rem:local-power-directions} is the general statement. Its
tangent-space interpretation uses the additional conditions above.

\section{Proofs of main lemmas}
\subsection{Proof of Lemma~\ref{lemma:stable-clt}}\label{sec:proof_stable_clt}
\begin{proof}[Proof of Lemma~\ref{lemma:stable-clt}]
	The result follows from \citet[Theorem 6.1]{hausler2015stable}, specialized
	to scalar martingale difference arrays. In the notation of that theorem,
	Eqn.~\eqref{eq:stable-clt-1} is condition $(N)$,
	Eqn.~\eqref{eq:stable-clt-2} is the conditional Lindeberg condition
	$(CLB)$, and $\mathcal G$ is the limiting sigma-field of the filtration
	array. The cited theorem therefore gives
	\begin{align*}
		\sum_{k=1}^{k_n}Y_{nk}\overset{d}{\to}Q_\eta\qquad\mathcal G\text{-stably},
	\end{align*}
	where $Q_\eta$ is the Gaussian probability kernel defined in
	Lemma~\ref{lemma:stable-clt}. This is the claimed conclusion.
\end{proof}

\subsection{Proof of Lemma~\ref{lemma:functional-stable-convergence}}\label{sec:proof_functional_stable_convergence}
\begin{proof}[Proof of Lemma~\ref{lemma:functional-stable-convergence}]
    Fix $F\in\mathcal F_\infty$ with $\P(F)>0$ and write
    $\P_F(\cdot)=\P(\cdot\mid F)$. By
    Lemma~\ref{lemma:stable_equivalence}, Condition~\eqref{eq:c1} implies
    weak convergence of $L_n(f)$ to $S(f)$ under $\P_F$ for every
    $f\in\mH^d$. Since $S$ is independent of $\mathcal F_\infty$, its
    Gaussian law and covariance $\sigma$ are unchanged under $\P_F$.

    Condition~\eqref{eq:c2} is deterministic and is therefore unchanged.
    Condition~\eqref{eq:c3} supplies the required tail bound under $\P_F$.
    Applying \citet[Theorem~1]{fernandez2024general} under $\P_F$ now gives
    weak convergence of $\{L_n(f)^2\}_{f\in\mH^d}$ to
    $\{S(f)^2\}_{f\in\mH^d}$ under $\P_F$.
    Since $F\in\mathcal F_\infty$ with $\P(F)>0$ was arbitrary,
    Lemma~\ref{lemma:stable_equivalence} yields the claimed
    $\mathcal F_\infty$-stable convergence.
\end{proof}

\subsection{Proof of Lemma~\ref{lemma:oracle-property}}\label{sec:proof_oracle_property}
\begin{proof}[Proof of Lemma~\ref{lemma:oracle-property}]
    We first use the feature moment and continuity in
    Lemma~\ref{lemma:parametric-bootstrap-feature-continuity}.
    The triangle inequality in $L_2(\mu;\mH^d)$ gives
    \begin{align*}
        \left(\E_{\law_\theta}
        [\|\Psi_\theta(X)\|_{\mH^d}^2]\right)^{1/2}
        &\le\left(\E_{\law_{\theta_0}}
        [\|\Psi_{\theta_0}(X)\|_{\mH^d}^2]\right)^{1/2}\\
        &\quad+\left(\int_\X\left\|
        \Psi_\theta(x)\sqrt{p_\theta(x)}
        -\Psi_{\theta_0}(x)\sqrt{p_{\theta_0}(x)}\right\|_{\mH^d}^2
        \,\mu(\d x)\right)^{1/2}.
    \end{align*}
    Eqns.~\eqref{eq:parametric-bootstrap-base-feature-moment} and
    \eqref{eq:parametric-bootstrap-feature-continuity} therefore give
    the uniform bound in
    Eqn.~\eqref{eq:parametric-bootstrap-feature-moment} on a sufficiently
    small neighborhood $U$. Substituting
    Eqn.~\eqref{eq:parametric-bootstrap-scalar-projection} and applying
    Cauchy--Schwarz gives
    \[
        \E_{X\sim\law_\theta}[S^*(f,X,\theta)^2]
        =\E_{X\sim\law_\theta}
        [|\langle f,\Psi_\theta(X)\rangle_{\mH^d}|^2]
        \le\|f\|_{\mH^d}^2\E_{X\sim\law_\theta}
        [\|\Psi_\theta(X)\|_{\mH^d}^2],
    \]
    which proves Eqn.~\eqref{eq:parametric-bootstrap-scalar-moment}.

    To justify centering, first note that the adjustment is square
    integrable. Eqn.~\eqref{eq:kernel-regularity} gives
    \begin{align}
        &\E_{X\sim\law_\theta}\left[
        \left\|\E_{X'\sim\law_\theta}
        [\nabla_\theta s_\theta(X')
        I_{\law_\theta}(X)\psi(X')]\right\|_{\mH^d}^2\right]\notag\\
        &\quad\le C_K
        \left(\E_{X'\sim\law_\theta}
        [\|\nabla_\theta s_\theta(X')\|_{\mathrm{op}}]\right)^2
        \E_{X\sim\law_\theta}[\|I_{\law_\theta}(X)\|^2].
        \label{eq:parametric-bootstrap-adjustment-moment}
    \end{align}
    The right-hand side is bounded locally by
    Eqn.~\eqref{eq:parametric-bootstrap-derivative-bound} and
    Assumption~\ref{ass:parametric-bootstrap}. Since the adjusted feature is square
    integrable, Eqn.~\eqref{eq:influence-adjusted-stein-feature} shows
    that $\xi_{\theta}(X)$ is also square integrable.
    We may therefore take its Bochner expectation. For each $f\in\mH^d$,
    Eqn.~\eqref{eq:stein-identity-assumption} gives
    \[
        \left\langle f,\E_{\law_\theta}[\xi_{\theta}(X)]
        \right\rangle_{\mH^d}
        =\E_{\law_\theta}[\mathcal A_\theta f(X)]=0.
    \]
    Thus $\E_{\law_\theta}[\xi_{\theta}(X)]=0$.
    The bounded linear population derivative can also be moved outside
    the expectation of the influence function, so
    \begin{align*}
        \E_{\law_\theta}[\Psi_\theta(X)]
        &=\E_{\law_\theta}[\xi_{\theta}(X)]
        +\E_{X'\sim\law_\theta}\left[
        \nabla_\theta s_\theta(X')
        \E_{X\sim\law_\theta}[I_{\law_\theta}(X)]
        \psi(X')\right]=0
    \end{align*}
    by Assumption~\ref{assu:uale}. Taking its inner product with $f$
    gives $\E_{\law_\theta}[S^*(f,X,\theta)]=0$. This proves
    Eqn.~\eqref{eq:parametric-bootstrap-centering} and completes the proof.
\end{proof}

\subsection{Proof of Lemma~\ref{lemma:parametric-bootstrap-feature-continuity}}
\label{sec:proof_parametric_bootstrap_feature_continuity}
\begin{proof}[Proof of Lemma~\ref{lemma:parametric-bootstrap-feature-continuity}]
We first establish the density-weighted continuity of the adjusted
feature, then bound its second moment at $\theta_0$.
\paragraph{Continuity of the adjusted feature.}
We use the following continuity result for the score function and influence function.
\begin{lemma}[Continuity of the density-weighted score function and influence function]
\label{lemma:parametric-bootstrap-component-continuity}
Under Assumptions~\ref{assu:uale}, \ref{ass:p-bootstrap}, and
\ref{ass:parametric-bootstrap}, as $\theta\to\theta_0$ within $U$,
\begin{align}\label{eq:parametric-bootstrap-component-continuity}
    &\int_\X\left\|s_\theta(x)\sqrt{p_\theta(x)}
    -s_{\theta_0}(x)\sqrt{p_{\theta_0}(x)}\right\|^2\,\mu(\d x)\notag\\
    &\quad+\int_\X\left\|I_{\law_\theta}(x)\sqrt{p_\theta(x)}
    -I_{\law_{\theta_0}}(x)\sqrt{p_{\theta_0}(x)}\right\|^2
    \,\mu(\d x)\to0.
\end{align}
\end{lemma}

\begin{proof}[Proof of Lemma~\ref{lemma:parametric-bootstrap-component-continuity}]
Eqn.~\eqref{eq:model-density-regularity} and
Assumption~\ref{assu:uale} give, for $\mu$-almost every $x$,
\[
    s_\theta(x)\sqrt{p_\theta(x)}\to
    s_{\theta_0}(x)\sqrt{p_{\theta_0}(x)},\qquad
    I_{\law_\theta}(x)\sqrt{p_\theta(x)}\to
    I_{\law_{\theta_0}}(x)\sqrt{p_{\theta_0}(x)}.
\]
By continuity of the squared Euclidean norm, for $\mu$-almost every $x$,
\begin{align}\label{eq:parametric-bootstrap-component-pointwise}
    &\left\|s_\theta(x)\sqrt{p_\theta(x)}
    -s_{\theta_0}(x)\sqrt{p_{\theta_0}(x)}\right\|^2\notag\\
    &\quad+\left\|I_{\law_\theta}(x)\sqrt{p_\theta(x)}
    -I_{\law_{\theta_0}}(x)\sqrt{p_{\theta_0}(x)}\right\|^2
    \to0\qquad(\theta\to\theta_0).
\end{align}
Applying $\|a-b\|^2\le2\|a\|^2+2\|b\|^2$ to each difference gives
\begin{align}\label{eq:parametric-bootstrap-component-domination}
    &\left\|s_\theta(x)\sqrt{p_\theta(x)}
    -s_{\theta_0}(x)\sqrt{p_{\theta_0}(x)}\right\|^2
    +\left\|I_{\law_\theta}(x)\sqrt{p_\theta(x)}
    -I_{\law_{\theta_0}}(x)\sqrt{p_{\theta_0}(x)}\right\|^2\notag\\
    &\quad\le2\bigl(\|s_\theta(x)\|^2+\|I_{\law_\theta}(x)\|^2\bigr)p_\theta(x)
    +2\bigl(\|s_{\theta_0}(x)\|^2
    +\|I_{\law_{\theta_0}}(x)\|^2\bigr)p_{\theta_0}(x).
\end{align}
The upper bound in Eqn.~\eqref{eq:parametric-bootstrap-component-domination}
tends pointwise to
$4(\|s_{\theta_0}(x)\|^2+\|I_{\law_{\theta_0}}(x)\|^2)p_{\theta_0}(x)$,
and its integral converges to the integral of that limit by
Assumption~\ref{ass:parametric-bootstrap}. Together with
Eqn.~\eqref{eq:parametric-bootstrap-component-pointwise}, the generalized
dominated convergence theorem therefore proves
Eqn.~\eqref{eq:parametric-bootstrap-component-continuity}.
\end{proof}

By Lemma~\ref{lemma:parametric-bootstrap-component-continuity}, the
sum in Eqn.~\eqref{eq:parametric-bootstrap-component-continuity}
converges to zero. The two integrals are nonnegative,
so each converges to zero separately.
We first control the Stein feature using its definition in
Eqn.~\eqref{eq:stein-feature}:
\begin{align*}
    &\xi_{\theta}(x)\sqrt{p_\theta(x)}
    -\xi_{\theta_0}(x)\sqrt{p_{\theta_0}(x)}\\
    &\quad=\left(s_\theta(x)\sqrt{p_\theta(x)}
    -s_{\theta_0}(x)\sqrt{p_{\theta_0}(x)}\right)\psi(x)
    +\nabla_x\psi(x)
    \left(\sqrt{p_\theta(x)}-\sqrt{p_{\theta_0}(x)}\right).
\end{align*}
Eqn.~\eqref{eq:kernel-regularity} therefore implies
\begin{align}
    &\int_\X\left\|\xi_{\theta}(x)\sqrt{p_\theta(x)}
    -\xi_{\theta_0}(x)\sqrt{p_{\theta_0}(x)}\right\|_{\mH^d}^2
    \,\mu(\d x)\notag\\
    &\quad\le2C_K\int_\X\left\|
    s_\theta(x)\sqrt{p_\theta(x)}
    -s_{\theta_0}(x)\sqrt{p_{\theta_0}(x)}\right\|^2\,\mu(\d x)\notag\\
    &\qquad+2C_K\int_\X|p_\theta(x)-p_{\theta_0}(x)|\,\mu(\d x)
    \to0.
    \label{eq:parametric-bootstrap-stein-density-continuity}
\end{align}
Here $(\sqrt a-\sqrt b)^2\le|a-b|$ and
Eqns.~\eqref{eq:parametric-bootstrap-density-l1} and
\eqref{eq:parametric-bootstrap-component-continuity} give the limit.

Uniformly over $\|a\|\le1$, Eqn.~\eqref{eq:parametric-bootstrap-moments} in
Lemma~\ref{lemma:parametric-bootstrap-derivative-continuity}
gives
\begin{align}
    &\left\|\E_{X\sim\law_\theta}
    [(\nabla_\theta s_\theta(X)a)\psi(X)]
    -\E_{X\sim\law_{\theta_0}}
    [(\nabla_\theta s_{\theta_0}(X)a)\psi(X)]
    \right\|_{\mH^d}\notag\\
    &\quad\le C_K^{1/2}\|a\|
    \int_\X\|\nabla_\theta s_\theta(x)p_\theta(x)
    -\nabla_\theta s_{\theta_0}(x)p_{\theta_0}(x)\|_{\mathrm{op}}
    \,\mu(\d x)\to0.
    \label{eq:parametric-bootstrap-population-derivative-continuity}
\end{align}
By Eqn.~\eqref{eq:influence-adjusted-stein-feature},
\begin{align}\label{eq:parametric-bootstrap-adjustment-decomposition}
    &\E_{X'\sim\law_\theta}[
    \nabla_\theta s_\theta(X')I_{\law_\theta}(x)\psi(X')]
    \sqrt{p_\theta(x)}\notag\\*
    &\quad-\E_{X'\sim\law_{\theta_0}}[
    \nabla_\theta s_{\theta_0}(X')I_{\law_{\theta_0}}(x)\psi(X')]
    \sqrt{p_{\theta_0}(x)}\notag\\
    &\quad=\int_\X
    \bigl(\nabla_\theta s_\theta(x')p_\theta(x')
    -\nabla_\theta s_{\theta_0}(x')p_{\theta_0}(x')\bigr)
    I_{\law_\theta}(x)\psi(x')\,\mu(\d x')\sqrt{p_\theta(x)}\notag\\
    &\qquad+\E_{X'\sim\law_{\theta_0}}\left[
    \nabla_\theta s_{\theta_0}(X')
    \bigl(I_{\law_\theta}(x)\sqrt{p_\theta(x)}
    -I_{\law_{\theta_0}}(x)\sqrt{p_{\theta_0}(x)}\bigr)\psi(X')
    \right].
\end{align}
Using $\|a+b\|_{\mH^d}^2\le2\|a\|_{\mH^d}^2+2\|b\|_{\mH^d}^2$ in
Eqn.~\eqref{eq:parametric-bootstrap-adjustment-decomposition}, with the bounds in
Eqns.~\eqref{eq:kernel-regularity} and
\eqref{eq:parametric-bootstrap-population-derivative-continuity}, gives
\begin{align}\label{eq:parametric-bootstrap-adjustment-density-continuity}
    &\int_\X\Bigl\|\E_{X'\sim\law_\theta}[
    \nabla_\theta s_\theta(X')I_{\law_\theta}(x)\psi(X')]
    \sqrt{p_\theta(x)}\notag\\*
    &\qquad\quad{}-\E_{X'\sim\law_{\theta_0}}[
    \nabla_\theta s_{\theta_0}(X')I_{\law_{\theta_0}}(x)\psi(X')]
    \sqrt{p_{\theta_0}(x)}\Bigr\|_{\mH^d}^2\,\mu(\d x)\notag\\
    &\quad\le2C_K
    \left(\int_\X\|\nabla_\theta s_\theta(x)p_\theta(x)
    -\nabla_\theta s_{\theta_0}(x)p_{\theta_0}(x)\|_{\mathrm{op}}
    \,\mu(\d x)\right)^2
    \E_{\law_\theta}[\|I_{\law_\theta}(X)\|^2]\notag\\
    &\qquad+2C_K\left(\E_{X\sim\law_{\theta_0}}
    [\|\nabla_\theta s_{\theta_0}(X)\|_{\mathrm{op}}]\right)^2\notag\\
    &\qquad\qquad\times
    \int_\X\left\|I_{\law_\theta}(x)\sqrt{p_\theta(x)}
    -I_{\law_{\theta_0}}(x)\sqrt{p_{\theta_0}(x)}\right\|^2
    \,\mu(\d x)\to0.
\end{align}
The first term vanishes by
Eqn.~\eqref{eq:parametric-bootstrap-moments} and the local uniform
influence-function moment bound implied by
Assumption~\ref{ass:parametric-bootstrap}.
The second vanishes by
Eqn.~\eqref{eq:parametric-bootstrap-component-continuity}.
Combining Eqns.~\eqref{eq:parametric-bootstrap-stein-density-continuity}
and~\eqref{eq:parametric-bootstrap-adjustment-density-continuity}
with the definition in Eqn.~\eqref{eq:influence-adjusted-stein-feature}
proves Eqn.~\eqref{eq:parametric-bootstrap-feature-continuity}.

\paragraph{Second moment of the adjusted feature.}
Eqns.~\eqref{eq:stein-feature}--\eqref{eq:influence-adjusted-stein-feature}
and \eqref{eq:kernel-regularity} imply
\begin{align*}
    \E_{\law_{\theta_0}}[\|\Psi_{\theta_0}(X)\|_{\mH^d}^2]
    &\le4C_K\left(\E_{\law_{\theta_0}}[\|s_{\theta_0}(X)\|^2]+1\right)\\
    &\quad+2C_K\left(\E_{\law_{\theta_0}}
    [\|\nabla_\theta s_{\theta_0}(X)\|_{\mathrm{op}}]\right)^2
    \E_{\law_{\theta_0}}[\|I_{\law_{\theta_0}}(X)\|^2]<\infty.
\end{align*}
The right-hand side is finite by
Eqn.~\eqref{eq:parametric-bootstrap-derivative-moments},
Cauchy--Schwarz for the first moment of the parameter derivative of the score function, and the
influence-function second moment in Assumption~\ref{assu:uale}.
This proves Eqn.~\eqref{eq:parametric-bootstrap-base-feature-moment}
and completes the proof.
\end{proof}

The componentwise requirement in
Eqn.~\eqref{eq:parametric-bootstrap-component-continuity} is sufficient,
but is not claimed to be equivalent to
Eqn.~\eqref{eq:parametric-bootstrap-feature-continuity}. Cancellations
between the Stein feature and the influence adjustment may make the
combined feature continuous even when its components are not.

\subsection{Proof of Lemma~\ref{lemma:parametric-bootstrap-derivative-continuity}}
\label{sec:proof_parametric_bootstrap_derivative_continuity}
\begin{proof}[Proof of Lemma~\ref{lemma:parametric-bootstrap-derivative-continuity}]
We first prove the two density-weighted limits, then verify the local
mean-continuity bound.
\paragraph{Continuity of the density-weighted parameter derivative of the score function.}
We first show that the derivative moment condition implies convergence
of the density-weighted parameter derivative of the score function in $L_1(\mu)$.
Let $C<\infty$ bound both suprema in
Eqn.~\eqref{eq:parametric-bootstrap-derivative-moments}.
Eqn.~\eqref{eq:model-density-regularity} gives pointwise continuity of
$p_\theta$. Since $\int_\X p_\theta(x)\,\mu(\d x)=1$ for every
$\theta\in U$, Scheff\'e's lemma proves
Eqn.~\eqref{eq:parametric-bootstrap-density-l1}.
Convexity of $U$ and the mean-value formula give, uniformly over
the reference model law,
\begin{align}\label{eq:parametric-bootstrap-derivative-mean-value}
    &\sup_{\eta\in U}\E_{X\sim\law_\eta}\left[
    \|\nabla_\theta s_\theta(X)-\nabla_\theta s_{\theta_0}(X)\|_{\mathrm{op}}
    \right]\notag\\
    &\quad\le\|\theta-\theta_0\|
    \sup_{\eta\in U}\E_{X\sim\law_\eta}\left[
    \sup_{\theta'\in U}\|\nabla_{\theta'}^2s_{\theta'}(X)\|_{\mathrm{op}}\right]
    \le C\|\theta-\theta_0\|.
\end{align}
We separate changes in the derivative and the density by the identity
\begin{align}\label{eq:parametric-bootstrap-derivative-density-split}
    &\nabla_\theta s_\theta(x)p_\theta(x)
    -\nabla_\theta s_{\theta_0}(x)p_{\theta_0}(x)\notag\\
    &\quad=\left(\nabla_\theta s_\theta(x)
    -\nabla_\theta s_{\theta_0}(x)\right)
    \sqrt{p_\theta(x)p_{\theta_0}(x)}\notag\\
    &\qquad+\left(\nabla_\theta s_\theta(x)\sqrt{p_\theta(x)}
    +\nabla_\theta s_{\theta_0}(x)\sqrt{p_{\theta_0}(x)}\right)
    \left(\sqrt{p_\theta(x)}-\sqrt{p_{\theta_0}(x)}\right).
\end{align}
For the first term on the right, use
$\sqrt{p_\theta p_{\theta_0}}\le(p_\theta+p_{\theta_0})/2$
and Eqn.~\eqref{eq:parametric-bootstrap-derivative-mean-value}. Its
integrated norm is at most $C\|\theta-\theta_0\|$.
For the second term on the right of
Eqn.~\eqref{eq:parametric-bootstrap-derivative-density-split}, apply the
triangle inequality to the sum in parentheses. For the first summand,
Cauchy--Schwarz,
Eqn.~\eqref{eq:parametric-bootstrap-derivative-moments}, and
$(\sqrt a-\sqrt b)^2\le|a-b|$ give
\begin{align}\label{eq:parametric-bootstrap-derivative-density-bound}
    &\int_\X\|\nabla_\theta s_\theta(x)\|_{\mathrm{op}}
    \sqrt{p_\theta(x)}
    \left|\sqrt{p_\theta(x)}-\sqrt{p_{\theta_0}(x)}\right|\,\mu(\d x)\notag\\
    &\quad\le\left(\E_{X\sim\law_\theta}
    [\|\nabla_\theta s_\theta(X)\|_{\mathrm{op}}^2]\right)^{1/2}
    \left(\int_\X\left(\sqrt{p_\theta(x)}-\sqrt{p_{\theta_0}(x)}\right)^2
    \,\mu(\d x)\right)^{1/2}\notag\\
    &\quad\le C^{1/2}
    \left(\int_\X|p_\theta(x)-p_{\theta_0}(x)|\,\mu(\d x)\right)^{1/2}.
\end{align}
For the second summand, the same bound holds with
$\nabla_\theta s_\theta(x)\sqrt{p_\theta(x)}$ replaced by
$\nabla_\theta s_{\theta_0}(x)\sqrt{p_{\theta_0}(x)}$.
Combining these bounds in
Eqn.~\eqref{eq:parametric-bootstrap-derivative-density-split} and using
Eqn.~\eqref{eq:parametric-bootstrap-density-l1} proves
\begin{align*}
    &\int_\X\|\nabla_\theta s_\theta(x)p_\theta(x)
    -\nabla_\theta s_{\theta_0}(x)p_{\theta_0}(x)\|_{\mathrm{op}}
    \,\mu(\d x)\\
    &\quad\le C\|\theta-\theta_0\|+2C^{1/2}
    \left(\int_\X|p_\theta(x)-p_{\theta_0}(x)|\,\mu(\d x)\right)^{1/2}
    \to0.
\end{align*}
This is Eqn.~\eqref{eq:parametric-bootstrap-moments}.

\paragraph{Local mean continuity of the parameter derivative of the score function.}
We next control the derivative differences in the Taylor remainder.
For $\theta,\eta\in U$ with $\|\eta-\theta\|\le\delta$, the segment
joining them lies in $U$. The mean-value formula and
Eqn.~\eqref{eq:parametric-bootstrap-derivative-moments} therefore imply
\begin{align*}
    &\sup_{\theta\in U}\E_{X\sim\law_\theta}\left[
    \sup_{\substack{\eta\in U\\\|\eta-\theta\|\le\delta}}
    \|\nabla_\eta s_\eta(X)-\nabla_\theta s_\theta(X)\|_{\mathrm{op}}
    \right]\\
    &\quad\le\delta\sup_{\theta\in U}\E_{X\sim\law_\theta}\left[
    \sup_{\eta\in U}\|\nabla_\eta^2s_\eta(X)\|_{\mathrm{op}}\right]
    \le C\delta\to0\qquad(\delta\downarrow0).
\end{align*}
This proves Eqn.~\eqref{eq:parametric-bootstrap-derivative-modulus}
and completes the proof.
\end{proof}

\subsection{Proof of Lemma~\ref{lemma:quantile-converge-in-prob}}
\label{sec:proof_conditional_quantile_convergence}
\begin{proof}[Proof of Lemma~\ref{lemma:quantile-converge-in-prob}]
	Write $q=\Q_\alpha(W)$. For any $\varepsilon>0$, choose
	$\eta\in(0,\varepsilon)$ such that $q-\eta$ and $q+\eta$ belong to the
	interval on which the distribution function of $W$ is strictly
	increasing. Then
	\begin{align*}
		\P(W\le q-\eta)&<\alpha<\P(W\le q+\eta).
	\end{align*}
	The definition of the conditional quantile and
	$\P(W_n\le t\mid\F_n)\overset{p}{\to}\P(W\le t)$
	at $t=q-\eta$ and $t=q+\eta$ give
	\begin{align*}
		\P\left[\Q_\alpha(W_n)\le q-\eta\right]
		&\le\P\left[
		\P(W_n\le q-\eta\mid\F_n)\ge\alpha
		\right]\to0,\\
		\P\left[\Q_\alpha(W_n)>q+\eta\right]
		&\le\P\left[
		\P(W_n\le q+\eta\mid\F_n)<\alpha
		\right]\to0.
	\end{align*}
	Since $\eta<\varepsilon$, these two bounds imply
	$\P(|\Q_\alpha(W_n)-q|>\varepsilon)\to0$.
\end{proof}

\end{document}